\documentclass[10pt]{article}
\usepackage{fullpage}
\usepackage{amsmath,amsfonts,bm}

\def\eqref#1{equation~\ref{#1}}
\def\ceil#1{\lceil #1 \rceil}
\def\floor#1{\lfloor #1 \rfloor}
\def\1{\bm{1}}

\def\bzero{{\bm{0}}}

\def\balpha{{\boldsymbol{\alpha}}}

\def\bdelta{{\boldsymbol{\delta}}}

\def\bzeta{{\boldsymbol{\zeta}}}

\def\blambda{{\boldsymbol{\lambda}}}

\def\bnu{{\boldsymbol{\nu}}}
\def\bxi{{\boldsymbol{\xi}}}
\def\bpi{{\boldsymbol{\pi}}}

\def\btau{{\boldsymbol{\tau}}}

\def\ba{{\bm{a}}}
\def\bb{{\bm{b}}}
\def\bc{{\bm{c}}}

\def\be{{\bm{e}}}
\def\bg{{\bm{g}}}
\def\bh{{\bm{h}}}

\def\bq{{\bm{q}}}
\def\br{{\bm{r}}}

\def\bu{{\bm{u}}}
\def\bv{{\bm{v}}}
\def\bw{{\bm{w}}}
\def\bx{{\bm{x}}}
\def\by{{\bm{y}}}
\def\bz{{\bm{z}}}

\def\bA{{\bm{A}}}
\def\bB{{\bm{B}}}

\def\bE{{\bm{E}}}

\def\bG{{\bm{G}}}

\def\bI{{\bm{I}}}
\def\bJ{{\bm{J}}}

\def\bV{{\bm{V}}}
\def\bW{{\bm{W}}}
\def\bX{{\bm{X}}}
\def\bY{{\bm{Y}}}
\def\bZ{{\bm{Z}}}

\DeclareMathAlphabet{\mathsfit}{\encodingdefault}{\sfdefault}{m}{sl}
\SetMathAlphabet{\mathsfit}{bold}{\encodingdefault}{\sfdefault}{bx}{n}

\def\gA{{\mathcal{A}}}
\def\gB{{\mathcal{B}}}
\def\gC{{\mathcal{C}}}
\def\gD{{\mathcal{D}}}
\def\gE{{\mathcal{E}}}

\def\gG{{\mathcal{G}}}
\def\gH{{\mathcal{H}}}
\def\gI{{\mathcal{I}}}
\def\gJ{{\mathcal{J}}}
\def\gK{{\mathcal{K}}}
\def\gL{{\mathcal{L}}}
\def\gM{{\mathcal{M}}}
\def\gN{{\mathcal{N}}}

\def\gP{{\mathcal{P}}}
\def\gQ{{\mathcal{Q}}}

\def\gS{{\mathcal{S}}}
\def\gT{{\mathcal{T}}}
\def\gU{{\mathcal{U}}}

\def\sS{{\mathbb{S}}}

\def\od{{\mathrm{d}}}

\def\oD{{\mathrm{D}}}

\def\sfM{{\mathsf{M}}}

\def\sfW{{\mathsf{W}}}

\newcommand{\E}{{\mathbb{E}}}
\newcommand{\N}{{\mathbb{N}}}

\newcommand{\R}{{\mathbb{R}}}

\newcommand{\KL}{\mathrm{KL}}

\newcommand{\vol}{\mathrm{vol}}

\DeclareMathOperator*{\argmin}{arg\,min}

\newcommand{\dist}{\mathrm{dist}}

\DeclareMathOperator{\ind}{\mathds{1}}

\newcommand{\supp}{\mathsf{supp}}

\newcommand{\odt}{{\mathrm{d}t}}

\newcommand{\Z}{{\mathbb{Z}}}

\newcommand{\nn}{{\mathcal{NN}}}
\newcommand{\ReLU}{{\mathsf{ReLU}}}

\newcommand{\TV}{{\mathsf{TV}}}

\newcommand{\Ord}{{\mathcal{O}}}

\newcommand{\tup}{{\overline{t}}}
\newcommand{\tdown}{{\underline{t}}}
\newcommand{\Rem}{{\mathrm{Rem}}}

\newcommand{\polylog}{{\mathrm{Polylog}}}

\usepackage[utf8]{inputenc} % allow utf-8 input
\usepackage[T1]{fontenc}    % use 8-bit T1 fonts
\usepackage{hyperref}       % hyperlinks
\usepackage{url}            % simple URL typesetting
\usepackage{booktabs}       % professional-quality tables
\usepackage{amsfonts}       % blackboard math symbols
\usepackage{nicefrac}       % compact symbols for 1/2, etc.
\usepackage{xcolor}         % colors

\definecolor{cite_color}{HTML}{114083}
\definecolor{link_color}{RGB}{153,0,0}  %  red
\definecolor{url_color}{RGB}{153,102,0}
\definecolor{emp_color}{RGB}{0,0,255}
\hypersetup{
	colorlinks,
	citecolor=cite_color,
	linkcolor=link_color,
	urlcolor=url_color}

\usepackage{amsmath}
\usepackage{amssymb}
\usepackage{amsthm}
\usepackage{amscd}
\usepackage{mathrsfs}
\usepackage{dsfont}
\usepackage{graphicx}
\usepackage{microtype, wasysym, comment,mathtools}
\usepackage{multicol}
\usepackage{multirow}
\usepackage{diagbox}
\usepackage{enumerate}
\usepackage{bm}
\usepackage{bbm}
\usepackage{natbib}
\usepackage{multirow}
\usepackage{fancyhdr}

\usepackage[capitalize,noabbrev]{cleveref}

\theoremstyle{plain}
\newtheorem{theorem}{Theorem}[section]
\newtheorem{proposition}[theorem]{Proposition}
\newtheorem{lemma}[theorem]{Lemma}
\newtheorem{corollary}[theorem]{Corollary}
\theoremstyle{definition}
\newtheorem{definition}[theorem]{Definition}
\newtheorem{assumption}{Assumption}
\theoremstyle{remark}

\crefname{section}{Sec.}{Sec.}
\crefname{theorem}{Theorem}{Theorems}
\crefname{lemma}{Lemma}{Lemmas}
\crefname{equation}{}{}
\crefname{proposition}{Proposition}{Propositions}
\crefname{claim}{Claim}{Claims}
\crefname{remark}{Remark}{Remarks}
\crefname{assumption}{Assumption}{Assumptions}
\crefname{definition}{Definition}{Definitions}
\crefname{appendix}{Appendix}{Appendices}
\crefname{algorithm}{Algorithm}{Algorithms}
\crefname{figure}{Figure}{Figures}
\crefname{table}{Table}{Tables}
\crefname{example}{Example}{Examples}

\usepackage{etoc}
\etocdepthtag.toc{mtchapter}
\etocsettagdepth{mtchapter}{subsection}
\etocsettagdepth{mtappendix}{none}

\title{Intrinsic Wasserstein Rates for Score-Based Generative Models on Smooth Manifolds}

\author{%
  Guoji Fu\textsuperscript{1,2}\thanks{Work partially completed during a research internship at the University of Tokyo} 
  \quad
  Taiji Suzuki\textsuperscript{2,3} 
  \quad
  Wee Sun Lee\textsuperscript{1} 
  \quad
  Atsushi Nitanda\textsuperscript{4,5} \\
  \vspace{0.1cm} \\
  \textsuperscript{1}{National University of Singapore} \\
  \textsuperscript{2}{The University of Tokyo} \\
  \textsuperscript{3}{RIKEN Center for Advanced Intelligence Project} \\
  \textsuperscript{4}{Agency for Science, Technology and Research (A$\star$STAR)} \\
  \textsuperscript{5}{Nanyang Technological University}
}

\begin{document}

\maketitle

\begin{abstract}
Score-based generative models are trained in high-dimensional ambient spaces, yet many data distributions are supported on low-dimensional nonlinear structures.  
We prove that, for compact $d$-dimensional smooth manifolds $\mathcal{M} \subset [0,1]^D$ with $d > 2$ and $\beta$-H\"older densities strictly positive on $\mathcal{M}$, a variance-preserving SGM estimator attains the intrinsic
Wasserstein--1 sample exponent $\tilde{\mathcal{O}}(D^{\mathcal{O}_\beta(d)}n^{-(\beta+1)/(d+2\beta)})$, up to logarithmic factors and explicit geometry and density factors.  
The full nonasymptotic bound explicitly isolates the finite-order geometry envelope, H\"older radius, density lower bound, ambient dependence, and finite-order correction terms.  
The analysis separates score approximation into a large-noise tangent-cell regime and a small-noise projection-centered, de-Gaussianized Laplace regime.  
The key technical ingredient is a ReLU implementation of nearest-projection coordinates via finite intrinsic anchors and Gauss--Newton iterations, rather than approximating the manifold projection as a black-box high-dimensional smooth map.  
Consequently, for families with polynomially controlled geometry and density lower bounds, the constructed score-network parameters have polynomial ambient dependence.
\end{abstract}
% \begin{abstract}
% Score-based generative models are trained in high-dimensional ambient spaces, yet many target laws are supported on low-dimensional nonlinear structures.  
% We prove that, for compact $d$-dimensional smooth manifolds $\gM\subset[0,1]^D$ with $d>2$ and $\beta$-H\"older densities bounded below on $\gM$, a regular variance-preserving SGM estimator attains the intrinsic Wasserstein--1 sample exponent $n^{-(\beta+1)/(d+2\beta)}$ up to logarithmic factors, explicit geometry and density factors, and a polynomial ambient prefactor
% $D^{\Ord_\beta(d)}$.  
% The full nonasymptotic bound displays the finite-order geometry envelope, H\"older radius, density lower bound, ambient dependence, and finite-order correction terms.  The analysis separates score approximation into a large-noise tangent-cell regime and a small-noise projection-centered, de-Gaussianized Laplace regime.  
% The key technical ingredient is a ReLU implementation of nearest-projection coordinates via finite intrinsic anchors and Gauss--Newton iterations, avoiding black-box high-order $D$-variate approximation of $\Pi_{\gM}$.  
% Consequently, for families with polynomially controlled geometry and density lower bounds, the constructed score-network parameters have polynomial ambient dependence.
% \end{abstract}

\section{Introduction}
Score-based Generative Models (SGMs), also known as diffusion models~\citep{sohl2015deep,song2019generative,song2021score,ho2020denoising,vahdat2021score}, are foundational to high-dimensional generative modeling across domains such as imagery~\citep{dhariwal2021diffusion}, video~\citep{ho2020denoising}, audio~\citep{ChenZZWNC21,KongPHZC21}, and biological structures~\citep{watson2023novo}. However, their statistical behavior remains insufficiently understood when the ambient dimension $D$ is large but the data exhibit an intrinsic dimension $d\ll D$. The manifold hypothesis~\citep{fefferman2016testing} postulates that the relevant sample complexity should scale with $d$, independent of the ambient nonparametric dimension.

This paper provides a rigorous quantitative foundation for the manifold hypothesis within the context of diffusion models.
Assuming the data distribution is supported on a compact $d$-dimensional smooth manifold $\gM \subset [0,1]^D$ and admits a $\beta$-H\"older density, bounded below with respect to the volume measure on $\gM$, we prove a nonasymptotic $\sfW_1$ guarantee for a variance-preserving SGM with explicit ambient, density, and finite-order geometric dependence.  In the fixed-geometry regime with $d>2$ and the canonical finite-order choice, the bound attains the intrinsic sample exponent $n^{-(\beta+1)/(d+2\beta)}$ up to logarithmic and displayed polynomial ambient factors.

A primary theoretical challenge arises because manifold-supported distributions are singular with respect to the Lebesgue measure in $\R^D$. At time zero, the ambient score is undefined. At any positive time, Gaussian smoothing induces a score characterized by normal localization toward the manifold tube and tangential transport along the manifold surface. These distinct phenomena require different stable approximation strategies. The large-noise regime interacts with tangent patches and can be analyzed using intrinsic cell counts $\sigma_t^{-d}$. Conversely, the small-noise regime necessitates a projection-centered ratio expansion because the score's denominator contains a singular normal Gaussian factor. This dichotomy motivates the two-regime analytical framework developed throughout this paper.

While ambient-space SGM theory offers critical insights, it does not directly resolve this question. Convergence analyses have evolved from requiring Log-Sobolev assumptions~\citep{lee2022convergence,wibisono2022convergence} to accommodating weaker moment or Fisher-information conditions~\citep{benton2024linear,conforti2023score}. Concurrently, score-estimation theories now yield sharp ambient rates for Besov and H\"older densities~\citep{oko2023diffusion,dou2024optimal,stephanovitch2025generalization}, with recent advances even relaxing the need for density lower bounds~\citep{zhang2024minimax,fu2025approximation}. Because these results are tailored to full-dimensional densities in $\R^D$, they do not by themselves explain why the distribution-learning rate should depend on $d$ for singular laws on nonlinear manifolds.

Other complementary perspectives study manifold detection and score singularities~\citep{pidstrigach2022score,lu2023mathematical,li2025scores} or broad $\sfW_p$ generalization through Wasserstein dimension~\citep{chakraborty2026generalization}. These address objectives different from smooth-density distribution learning, and therefore do not yield the dimension-explicit intrinsic $\sfW_1$ rate targeted here. 
% The most relevant manifold-aware statistical guarantees are summarized in \cref{tab:comparison}.
\begin{table*}
    % \vspace{-10pt}
    % \centering
    \caption{Comparison of recent theoretical results on diffusion models for manifold data.}
    \label{tab:comparison}
    \resizebox{\linewidth}{!}{
    \begin{tabular}{c|c|c|c|c}
        \hline
        Paper & Manifold assumption & Density regularity & Metric & Convergence rate \\
        \hline
        \multirow{2}{*}{\citep{oko2023diffusion}} & $\bX = \bA\bZ$ & Besov class density & \multirow{2}{*}{$\sfW_1$} & \multirow{2}{*}{$\tilde{\Ord}(n^{-\frac{\beta+1-\delta}{d+2\beta}})$} \\
        & $\supp(\bZ) = [-1, 1]^d$ & lower bounded density & & \\
        \hline
        \multirow{2}{*}{\citep{chen2023score}} & $\bX = \bA\bZ$ & \multirow{2}{*}{Lipschitz on-support score} & \multirow{2}{*}{TV} & \multirow{2}{*}{$\tilde{\Ord}\bigl(Dn^{-\frac{1-\delta}{d+5}})\bigr)$} \\
        & sub-Gaussian $P_{\bZ}$ & & & \\
        \hline
        \multirow{2}{*}{\citep{tang2024adaptivity}} & compact smooth manifold & H\"older class density & \multirow{2}{*}{$\sfW_1$} & \multirow{2}{*}{$\tilde{\Ord}(e^Dn^{-\frac{\beta+1}{d+2\beta}})$} \\
        & positive reach & lower bounded density & & \\
        \hline
        \multirow{2}{*}{\citep{azangulov2024convergence}} & compact smooth manifold & H\"older class density & \multirow{2}{*}{$\sfW_1$} & \multirow{2}{*}{$\tilde{\Ord}(\sqrt{D}n^{-\frac{\beta+1}{d+2\beta}+\delta})$} \\
        & positive reach & lower bounded density & & \\
        \hline
        \multirow{2}{*}{\citep{yakovlev2025generalization}} & $\bX = g^*(\bZ) + \bxi$ & \multirow{2}{*}{H\"older class $g^*$} & \multirow{2}{*}{TV} & \multirow{2}{*}{$\tilde{\Ord}(D^{6+\binom{d+\beta}{d}}n^{-\frac{\beta}{6\beta+2d}})$} \\
        & $\bZ \sim \gU([0, 1]^d)$ & & & \\
        \hline
        \multirow{2}{*}{\citep{zhang2026diffusion}} & compact smooth manifold & H\"older class density & \multirow{2}{*}{$\sfW_1$} & \multirow{2}{*}{$\tilde{\Ord}\!\left(D^{\Ord_\beta(d)}n^{-\frac{\beta+1}{d+2\beta}}\right)^\dagger$} \\
        & positive reach & lower bounded density & & \\
        \hline
        \multirow{2}{*}{This paper} & compact smooth manifold & H\"older class density & \multirow{2}{*}{$\sfW_1$} & \multirow{2}{*}{$\tilde{\Ord}\bigl(D^{\Ord_\beta(d)}n^{-\frac{\beta+1}{d+2\beta}}\bigr)^\ddagger$} \\
        & positive reach & lower bounded density & & \\
        \hline
    \end{tabular}
    }
    {\footnotesize
    $\dagger$ The displayed rate in \citep{zhang2026diffusion} has the intrinsic sample exponent, but the dimension-explicit constants depend on how the projection step is instantiated. $\ddagger$ Our formal theorem first yields the bound with the explicit cap term
    $n^{-1/2}\mathfrak r_n^{-d/(2q_\star)}$ and the two small-noise regime conditions; the table entry records the $d>2$ simplification after that term is absorbed.}
    % \vspace{-10pt}
\end{table*} 

\textbf{Comparison with prior manifold analyses.}
\cref{tab:comparison} highlights prevailing theoretical gaps. 
Linear-support models recover intrinsic rates but bypass curvature challenges~\citep{oko2023diffusion,chen2023score}. 
For nonlinear manifolds, earlier analyses either incur exponential-in-$D$ prefactors~\citep{tang2024adaptivity} or derive slower rates with high-dimensional approximation factors~\citep{yakovlev2025generalization}. 
A recent support-aware analysis obtains a sharper displayed ambient prefactor, roughly $\sqrt D$~\citep{azangulov2024convergence}, but it does so for a more structured pipeline: the estimator first builds sample-dependent local low-rank models of the support, its main score-learning target is weighted by $\sigma_t^2$, and the final rate carries an $n^\delta$ slack. 
Our theorem targets a different constructive problem: it certifies an explicit ambient score-network realization of the projection-centered small-noise score, without preprocessing the manifold into local affine subproblems, and attains the exact intrinsic sample exponent up to logarithms. 
See \cref{sec:app:comparison_azangulov} for the theorem-level comparison.

\textbf{Comparison with the concurrent work~\citep{zhang2026diffusion}.}
The dagger marks their displayed theorem rate in \cref{tab:comparison}. 
They obtain the same intrinsic sample exponent as our result, and the distinction is not the exponent but the dimension-explicit construction behind it.  
Their proof invokes a generic ambient smooth-map approximation step for the nearest projection on the reach tube.  Under a standard $D$-variate implementation, that step entails a high-order ambient approximation audit for $\Pi_{\gM}$.  
In contrast, our proof adopts a finite-anchor Gauss--Newton network in intrinsic coordinates, % together with a de-Gaussianized denominator lower bound, 
yielding a constructive polynomial ambient audit of the score network under the stated polynomial-geometry regime. See \cref{sec:app:comparison_zhang} for the detailed comparison.

\vspace{-5pt}
\subsection{Our contributions}
\vspace{-5pt}
In the displayed bounds below, $\Gamma_{\gM,q_{\rm geom}}$ denotes the finite-order geometry envelope defined in \cref{def:finite-order-geometry-constants}.  It records reach, volume, and finite-order atlas smoothness factors explicitly; these factors are not hidden by $\tilde{\Ord}(\cdot)$.
The primary contributions are:
\vspace{-10pt}
\begin{enumerate}[(i)]
    \setlength{\itemsep}{0pt}
    \item \emph{Intrinsic Wasserstein estimation under explicit geometric assumptions:} 
        Under~\cref{assump:manifold:exact,assump:manifold:density,assump:manifold:density-lower}, the polynomial control $D\vee p_{\min}^{-1}\le n^{a_0}$, and the explicit small-noise reach and denominator-resolution conditions, we prove
        $$
            \E\,\sfW_1(P_0,\widehat P_{t_0}^{\rm lb})
            \le
            \tilde{\Ord}\!\Bigl(
            \Gamma_{\gM,q_{\rm geom}}^C B_0^C
            p_{\min}^{-C}
            (D\vee\log n)^{\mathfrak p_{\beta,d}}
            \bigl[\mathfrak r_n+n^{-1/2}+n^{-1/2}\mathfrak r_n^{-d/(2q_\star)}\bigr]
            \Bigr),
        $$
        where $\mathfrak r_n=n^{-(\beta+1)/(d+2\beta)}$.  For $d>2$ and the canonical finite-order choice $q_{\rm geom}=q_{\rm opt}(\beta)$, the finite-order cap term is $\Ord(\mathfrak r_n)$, yielding the intrinsic exponent in the fixed-geometry sense.

    \item \emph{Explicit ambient dependence and admissible growth regime:} 
        The constructed score networks have leading width and sparsity $n^{d/(d+2\beta)}$ and an explicit ambient factor $(D\vee\log n)^{\mathfrak p_{\beta,d}}$, where $\mathfrak p_{\beta,d}$ is displayed in \cref{cor:w1-rate-density-lower} and satisfies $\mathfrak p_{\beta,d}=\Ord_\beta(d)$ after the $d>2$ finite-order choice.  For fixed geometry and fixed density lower bound, \cref{cor:fixed-geometry-ambient-growth} records a sufficient ambient-growth condition
        $D\vee\log n\le c\,n^{a_{\beta,d}}$ under which the two small-noise checks hold.  
        For growing-ambient families, the polynomial-in-$D$ interpretation additionally requires polynomial control of the finite-order geometry envelope and density factors.  \Cref{sec:app:poly-geometry-examples} verifies this geometry regime for affine spheres, flat product tori, and finite-order/reach-controlled closed graph-type embeddings. 

 \item \emph{Constructive two-regime score approximation:} 
        The large-noise branch uses tangent-cell approximation with intrinsic cell counts, while the small-noise branch uses projection-centered de-Gaussianized Laplace expansions, a finite-anchor Gauss--Newton network for the projection center, affine chart-coordinate extraction, and a polynomial denominator lower bound before reciprocal-network evaluation.  Both branches are verified on the forward tube and then passed through clipped score-estimation and slabwise $\sfW_1$ perturbation arguments.
\end{enumerate}
% \vspace{-5pt}

\textbf{Technical overview.}
The proof separates the Gaussian footprint into two geometric regimes.  At large noise, the smoothed score can be approximated by tangent-cell expansions with intrinsic cell counts, so the construction never needs an ambient grid.  At small noise, the proof recenters at the nearest point on $\gM$, factors out the common normal Gaussian term, and approximates only the tangential Laplace coefficients and the de-Gaussianized ratio.  The dimension-explicit step is a finite-anchor Gauss--Newton ReLU network for the projection center, followed by affine chart-coordinate extraction; this replaces a black-box high-order $D$-variate approximation of $\Pi_{\gM}$.  The detailed two-regime proof architecture is given in \cref{sec:score:approx:manifold}, with appendix audit maps in \cref{sec:app:two-regime-score-map,sec:app:proof-audit-map}.
% \vspace{-5pt}

\section{Background}
% \vspace{-5pt}
\subsection{Notation}\label{sec:notation}
% \vspace{-5pt}

Let $\R_+ \coloneqq \{x \in \R \mid x \geq 0\}$ denote the space of non-negative real numbers, $\N \coloneqq \{0, 1, 2, \dots\}$ the set of natural numbers, and $\N_+ \coloneqq \N \setminus \{0\}$. 
We let $\gN(\bzero, \sigma^2\bI_D)$ represent the Gaussian distribution on $\R^D$ with mean vector $\bzero$ and covariance matrix $\sigma^2\bI_D$.
For real numbers $a$ and $b$, we write $a \vee b \coloneqq \max\{a, b\}$ and $a \wedge b \coloneqq \min\{a, b\}$. 
We use standard asymptotic notation: $a = \Ord(b)$ implies $a \leq Cb$ for a universal constant $C > 0$, while $a = \Ord_\beta(b)$ indicates $a \le C_\beta b$, where the constant $C_\beta$ may depend on the fixed smoothness parameter $\beta$ but is independent of $d, D, n, \sigma_t, p_{\min}$, and any numerical geometry constants. The notation $\tilde{\Ord}(\cdot)$ hides logarithmic factors, and $\lesssim$ suppresses constants independent of the displayed variables in the surrounding statement. 
% Where logarithmic factors are displayed explicitly, we write $D \vee \log n$. 
Unadorned constants $C,c>0$ are universal and may change from line to line; named or subscripted constants may depend only on the quantities explicitly stated in the surrounding lemma, theorem, or convention.
For any $\beta > 0$, $\floor{\beta}$ denotes the largest integer strictly less than $\beta$ (equivalently, $\floor{\beta} = \lceil\beta\rceil - 1$). 
For a multi-index $\balpha = (\alpha_1, \dots, \alpha_k) \in \N^k$ and a vector $\bv = (v_1, \dots, v_k) \in \R^k$, we define $\bv^{\balpha} \coloneqq v_1^{\alpha_1}v_2^{\alpha_2} \cdots v_k^{\alpha_k}$ and $|\balpha| \coloneqq \alpha_1 + \dots + \alpha_k$. 

The main theorems use a finite geometry order chosen once and then held fixed:
\begin{equation}
    q_{\rm geom} \in \N,
    \qquad
    q_{\rm geom} \ge q_{\min}(\beta)
    :=
    \left\lceil\max\{\beta+4, 2(\beta\vee 1)+2\}\right\rceil .
    \label{eq:def:qgeom-main}
\end{equation}
This order controls the finite Taylor expansions used in the tangent-cell proof.
For the $d>2$ optimal $\sfW_1$ theorem we additionally choose it so that
$q_{\rm geom}-2\ge d(\beta+1)/(d-2)$.  Since $d$ is an integer and $d>2$
implies $d/(d-2)\le3$, this can be enforced uniformly by taking the canonical
choice
\begin{equation}
    q_{\rm geom} = q_{\rm opt}(\beta)
    :=
    \left\lceil\max\{q_{\min}(\beta), 2+3(\beta+1)\}\right\rceil.
    \label{eq:def:qgeom-opt}
\end{equation}
Equivalently, any fixed order bounded by a constant depending only on $\beta$
and satisfying the same lower bounds gives the same ambient exponent.  
Thus the finite order is $\Ord_\beta(1)$ and is never allowed to grow with $n$, $D$, or $d$.
In the appendix, our bookkeeping collects direct ambient powers into an exponent $\mathfrak{p}_{\beta,d}$ after this finite order has been fixed.  
Under the preceding $d>2$ choice, the explicit bookkeeping gives $\mathfrak{p}_{\beta,d}=\Ord_\beta(d)$.  
Within the main text, these rates are summarized as $(D\vee\log n)^{\Ord_\beta(d)}$ or, after hiding logarithms, as $D^{\Ord_\beta(d)}$.
We also fix the target $\sfW_1$ sample rate and the large-noise accuracy-cap exponent:
\begin{equation}
    \mathfrak r_n:=n^{-(\beta+1)/(d+2\beta)},
    \qquad
    q_\star:=q_{\rm geom}-2 .
    \label{eq:def-main-rate}
\end{equation}
The symbol $\mathfrak r_n$ is reserved for this final sample rate.  
Uppercase $R$-symbols are used only for local radii, cutoffs, or chartwise integral labels.
Additional proof-bookkeeping conventions restricted to the appendix are collected in \cref{sec:app:notation-conventions}.

\textbf{Wasserstein distance.} 
For $P, Q$ two probability distributions on $\R^D$, their Wasserstein-$1$ distance is % defined as
% \vspace{-2pt}
\begin{equation*}
    \sfW_1(P, Q) 
    \coloneqq
    \inf_{\gamma \in \Gamma(P, Q)}
    \int
      \|\bx - \by\|_2
    \od\gamma(\bx, \by), 
\end{equation*}
and $\Gamma(P, Q)$ is the set of all couplings of $P$ and $Q$.
We use $\sfW_1$ as it remains meaningful for singular or disjoint manifold-supported measures and respects the ambient geometry of the sampling space.
% \vspace{-2pt}

% Following the usual approximation analysis with neural networks~\citep{yarotsky2017error,schmidt2020nonparametric,suzuki2019adaptivity,oko2023diffusion}, we consider the hypothesis class $\nn$ in score matching as a class of deep neural network with the ReLU activation $\ReLU(x) = \max\{0, x\}$ (operated element-wise for a vector)~\citep{nair2010rectified} with a sparsity constraint (on the number of non-zero parameters. 
% The score network is a function from $(\bx, t) \in \R^D \times \R_+$ to $\by \in \R^D$. 
\textbf{Neural network class.} 
Following standard neural network approximation theory~\citep{yarotsky2017error,schmidt2020nonparametric}, we model the score function using deep ReLU neural networks with explicit constraints on depth, width, parameter sparsity, and magnitude. 
Specifically, we define the score network $s(\bx, t)$ as a function mapping the concatenated input $(\bx, t) \in \R^D \times \R_+$ to the score vector in $\R^D$.
\begin{definition}[Sparse ReLU neural network]
    Let $\rho(u)=\max\{u,0\}$ act coordinatewise. 
    For a width vector $\bW=(W_0,W_1,\dots,W_L)$, the class $\nn(L,\bW,S,B)$ consists of all maps $f: \R^{W_0} \to \R^{W_L}$ obtained from the recursion $\bz_0=\bz, \bz_\ell=\rho\!\bigl(\bA^{(\ell)}\bz_{\ell-1}+\bb^{(\ell)}\bigr), 1 \le \ell \le L-1, f(\bz)=\bA^{(L)}\bz_{L-1}+\bb^{(L)}$, where $\bA^{(\ell)}\in\R^{W_\ell\times W_{\ell-1}}$ and
    $\bb^{(\ell)}\in\R^{W_\ell}$ for $\ell=1,\dots,L$, and $\sum_{\ell=1}^L\bigl(\|\bA^{(\ell)}\|_0+\|\bb^{(\ell)}\|_0\bigr) \le S, \max_{1 \le \ell \le L}\bigl(\|\bA^{(\ell)}\|_\infty \vee \|\bb^{(\ell)}\|_\infty\bigr) \le B$.
    For score networks, we take $W_0=D+1$ and $W_L=D$.
\end{definition}

\subsection{Score-based Generative Models (SGMs)}
% \vspace{-5pt}

\textbf{Forward process.}
As a forward process $(\bX_t)_{0 \leq t \leq T}$ in $\R^D$, we consider the Variance Preserving (VP) process: 
\begin{equation}
    \od\bX_t = -\alpha(t)\bX_t \odt + \sqrt{2\alpha(t)}\od\bB_t, 
    \quad (0 \leq t \leq T)
    \quad \bX_0 \sim P_0, 
    \label{eq:vp:forward}
\end{equation}
where $(\bB_t)_{0 \leq t \leq T}$ denotes an independent $D$-dimensional standard Brownian motion and we have that $\bX_t|\bX_0 \sim \gN(m_t\bX_0, \sigma_t^2\bI_D)$, where $m_t = \exp(-\int_0^t\alpha(\tau)\od\tau), \sigma_t^2 =1-\exp(-2\int_0^t\alpha(\tau)\od\tau)$. 
% Notably, setting $\alpha(t) \equiv 1$, \cref{eq:vp:forward} recovers the standard Ornstein–Uhlenbeck (OU) process: 
% \vspace{-2pt}
% \begin{equation}
%     \od\bX_t = -\bX_t\odt + \sqrt{2}\od\bB_t, 
%     \quad
%     \bX_0 \sim P_0.
%     \label{eq:vp:ou}
%     % \vspace{-5pt}
% \end{equation}

\textbf{Reverse process.}
The VP process~\cref{eq:vp:forward} has a reverse process $(\bY_t)_{0 \leq t \leq T}$, where $\bY_t = \bX_{T-t}$ satisfies 
\begin{align*} 
    &
    \od\bY_t \!=\! \alpha(T \!-\! t)(\bY_t +\! 2\nabla \log p_{T-t}(\bY_t))\odt \!+\! \sqrt{2\alpha(T \!-\! t)}\od\bB_t', 
    \quad
    (0 \leq t \leq T)
    \qquad
    \bY_0 \sim P_T, 
\end{align*}
where $(\bB_t')_{0 \leq t \leq T}$ is another independent $D$-dimensional standard Brownian motion and $\nabla\log p_t(\cdot)$ is called the score function. 
However, the noise distribution $P_T$ and score function $\nabla\log p_t(\cdot)$ are unknown.
It has been shown that the VP process converges exponentially to the standard Gaussian distribution $\gN(\bzero, \bI_D)$~\citep{chen2023sampling} and for sufficiently large $T$, we can replace $P_T$ by $\gN(\bzero, \bI_D)$. 

\textbf{Score matching loss.}
To estimate the true score function $s^*(\bx, t) \coloneqq \nabla \log p_t(\bx)$, we learn a score estimator $s: \R^D \times [t_0, T] \to \R^D$ through minimizing the score matching loss:
\begin{equation}
    \gL_{[t_0,T]}(s) \coloneqq 
    \int_{t_0}^T \E_{\bX_t}[\|s(\bX_t, t) - \nabla \log p_t(\bX_t)\|_2^2] \odt. \label{eq:loss:sm}
\end{equation}
% \vspace{-5pt}
Because this exact objective is computationally inaccessible, we optimize the equivalent \textit{denoising score matching} loss $\E_{\bX_0}[\ell(s, \bX_0)]$ \citep{hyvarinen2005estimation,vincent2011connection}, where:
% \vspace{-2pt}
\begin{align*}
    \ell(s, \bX_0) 
    \!\coloneqq\!\!
    \int_{t_0}^T\!\!
      \E_{\bX_t|\bX_0}[\|s(\bX_t, t) \!-\! \nabla\log p_t(\bX_t|\bX_0)\|_2^2]
    \odt. 
    \!
    % \label{eq:loss:dsm}
    % \vspace{-5pt}
\end{align*}
It is well established that $\gL_{[t_0,T]}(s) - \gL_{[t_0,T]}(s^*) = \E_{\bX_0}[\ell(s, \bX_0) - \ell(s^*, \bX_0)]$. 

\textbf{Empirical risk minimization (ERM).}
Given $n$ i.i.d training samples $\bx^{(1)}, \dots, \bx^{(n)} \sim P_0$, we learn the score estimator $\widehat{s}$ over the hypothesis class $\nn$ via the ERM:  
% \vspace{-2pt}
\begin{equation}
    \widehat{s} \in \argmin_{s \in \nn}\frac{1}{n}\sum_{i=1}^n\ell(s, \bx^{(i)}). 
    \label{eq:erm}
\end{equation}
% \vspace{-5pt}

\textbf{Early stopping.}
As $t \to 0$, the underlying score function $\nabla \log p_t(\cdot)$ diverges, causing the reverse SDE to become highly ill-conditioned. 
To ensure numerical stability, we truncate the reverse dynamics, terminating the simulation at a small positive time threshold $t_0 > 0$.

\textbf{Generating new samples.}
Using the trained score estimator $\widehat{s}$ across the interval $[t_0, T]$, we construct an approximate reverse process $(\widehat{\bY}_t)_{0 \leq t \leq T-t_0}$ and draw new samples $\bX' \coloneqq \widehat{\bY}_{T-t_0} \sim \widehat{P}_{t_0}$ via:
\begin{equation*}
% \begin{aligned}
    % & 
    \od\widehat{\bY}_t \!=\! \alpha(T \!-\! t)(\widehat{\bY}_t + 2\widehat{s}(\widehat{\bY}_t, T \!-\! t))\odt +\! \sqrt{2\alpha(T\!-\!t)}\od\bB_t, 
    \quad (0 \leq t \leq T-t_0)
    % \\ 
    % &
    \qquad
    \widehat{\bY}_0 \sim \mathcal{N}(\bzero, \bI_D). 
% \end{aligned}
% \vspace{-5pt}
\end{equation*}

% \vspace{-5pt}
\subsection{Manifold Hypothesis}
% \vspace{-5pt}

\begin{assumption}[Compact smooth manifold]\label{assump:manifold:exact}
    The predictor $\bX_0 \sim P_0$ is supported on $\gM \subset [0, 1]^D$, where $\gM$ is a compact, $C^\infty$, $d$-dimensional Riemannian manifold without boundary isometrically embedded in $\R^D$, with $1 \le d \ll D$ and reach $\kappa > 0$ and Riemannian volume $S_{\gM}:=\vol_{\gM}(\gM)$.
\end{assumption}

Throughout the approximation bounds below, $d$ is treated as the intrinsic dimension and its dependence is kept separate from the ambient-dimension dependence on $D$.

The reach controls curvature and self-avoidance of the normal tube, while $S_{\gM}$ controls finite chart counts.  
Their precise contribution to constants is recorded in the finite-order geometry envelope in \cref{def:finite-order-geometry-constants}.  
Concrete families satisfying the resulting polynomial-geometry regime, including affine spheres, flat product tori, and closed graph-type embeddings with finite-order derivative and reach control, are given in \cref{sec:app:poly-geometry-examples}.

\begin{definition}[H\"older functions on a Riemannian manifold]\label{def:manifold:Holder-smooth}
    Fix a finite $C^{\ceil{\beta}}$ atlas $\{(U_i, \phi_i)\}_{i \in \gA}$ on the compact manifold $\gM$, with coordinate domains $\phi_i(U_i) \subset \R^d$, and use this atlas in the definition below.
    For $\beta>0, B_0 \geq 1$, a function $f: \gM \to \R$ belongs to $\gH^\beta(\gM, B_0)$ if, for every chart $(U_i, \phi_i)$ in this fixed atlas, with $f^{(i)} := f \circ \phi_i^{-1}: \phi_i(U_i) \to \R$, the following hold: 
    \begin{enumerate}[(i)]
        \item All $\partial^{\balpha}f^{(i)}$ with $\|\balpha\|_1 \leq \floor{\beta}$ exist continuously and $\max_{\|\balpha\|_1 \leq \floor{\beta}} \|\partial^{\balpha}f^{(i)}\|_{L^\infty(\phi_i(U_i))} \leq B_0$. 
        
        \item For every $\|\balpha\|_1=\floor{\beta}$,
        % \begin{equation*}
        $
            \sup_{\bu \neq \bv \in \phi_i(U_i)}
            \frac{|\partial^{\balpha}f^{(i)}(\bu)-\partial^{\balpha}f^{(i)}(\bv)|}{\|\bu-\bv\|_2^{\beta-\floor{\beta}}}
            \leq B_0.
        $
        % \end{equation*}
    \end{enumerate}
    The radius $B_0$ is therefore a controlled chartwise radius, not a supremum over arbitrary smooth reparametrizations.  
    Passing from this atlas to the reach-scale projection atlas used in the proof only changes the radius by finite transition-map constants; see \cref{lem:holder-transfer-projection-atlas}.
\end{definition}

\begin{assumption}[H\"older smooth density on manifold]\label{assump:manifold:density}
    The law $P_0$ of $\bX_0$ is absolutely continuous w.r.t the Riemannian volume measure $\vol_{\gM}$ on $\gM$. 
    We write $\od P_0(\by)=p_0(\by)\od\vol_{\gM}(\by)$, where $p_0: \gM \to [0, \infty)$ satisfies $p_0 \in \gH^\beta(\gM,B_0)$ w.r.t the fixed controlled finite atlas for a radius $B_0 \ge 1$ and $\int_{\gM}p_0(\by)\od\vol_{\gM}(\by)=1$.
\end{assumption}

\begin{assumption}[Positive density lower bound]\label{assump:manifold:density-lower}
    There exists a constant $p_{\min} \in (0,1]$ such that $p_0(\by) \ge p_{\min}, \forall \by \in \gM$.
\end{assumption}

The appendix recalls the standard chart, atlas, and partition-of-unity conventions and constructs the finite projection atlas, reach-volume tube bound, and finite-order geometry envelope $\mathcal K_{\gM,q}$ used to bound approximation constants; see \cref{sec:app:standard-manifold-conventions,sec:app:geometry-conventions,sec:app:approx:manifold:setup}. 
These are bookkeeping devices for localization and network-size controls, whereas the fundamental statistical assumptions governing the main theorem are stated explicitly here.
% \vspace{-5pt}

\section{Score Approximation on Smooth Manifolds}\label{sec:score:approx:manifold}
% \vspace{-5pt}

This section states the deterministic comparison networks later inserted into the clipped ERM oracle and the $\sfW_1$ perturbation argument. 
Recall $s^*(\bx,t) = \nabla\log p_t(\bx)$ and $\bX_t = m_t\bX_0 + \sigma_t\bZ$, with $\bX_0$ supported on $\gM\subset[0,1]^D$.  
The small-noise class has leading intrinsic capacity $n^{d/(d+2\beta)}$, while each large-noise slab has capacity $\sigma_{t_{k-1}}^{-d}+r_{\star,n}^{-d}$.  
The cap is kept explicit below and, under the final $d>2$ choice of $q_\star$, is no larger than the assembled leading capacity up to displayed geometry and logarithmic factors.  
No statistical estimation is performed in this section.

% \vspace{-5pt}
\subsection{Tube Localization}
% \vspace{-5pt}

For any ambient coordinate $\bx \in \R^D$, define the Euclidean distance to the rescaled manifold as:
$$
    \dist(\bx,m_t\gM) \coloneqq \inf_{\by \in \gM}\|\bx-m_t\by\|_2.
$$
For any fixed off-tube precision exponent $A>0$, choose $C_\varrho=C_\varrho(A)>0$ sufficiently large and define the active tube parameters:
\begin{equation}
    \varrho_t \coloneqq C_\varrho\sigma_t\sqrt{D\vee\log n},
    \qquad
    A_t \coloneqq \{\bx\in\R^D : \dist(\bx,m_t\gM)\le\varrho_t\}.
    \label{eq:def:main-near-manifold-tube}
\end{equation}
The radius matches the Gaussian footprint up to a logarithmic tail margin, so the core approximation problem is localized to $A_t$; outside $A_t$, clipping and Gaussian concentration suffice.

\begin{lemma}[Informal version of \cref{lem:approx:score:off-manifold}]
    Under \cref{assump:manifold:exact}, fix any $A>0$ and choose the tube constant $C_\varrho=C_\varrho(A)$ sufficiently large.  
    If an arbitrary comparison score $\widetilde s$ satisfies the uniform clipping bound $\sup_{\bx\in\R^D}\|\widetilde s(\bx,t)\|_\infty \le C\sigma_t^{-1}(D\vee\log n)^{1/2}$, then the off-tube risk is bounded by:
    $$
        \int_{A_t^c} \|\widetilde s(\bx,t)-s^*(\bx,t)\|_2^2 p_t(\bx)\od\bx \le C_A\sigma_t^{-2}(D\vee\log n)^C n^{-A},
    $$
    where the constant $C_A$ may depend on the prescribed off-tube exponent $A$.
\end{lemma}

% \vspace{-5pt}
\subsection{Near-Tube Integral Representation}
% \vspace{-5pt}

On $A_t$, a projection atlas and partition of unity reduce the smoothed manifold density to chartwise Gaussian integrals; the full construction is in \cref{sec:app:approx:manifold:setup}. 
Throughout this section we use that finite projection atlas, denoted $\{(U_i,\phi_i)\}_{i=1}^{C_{\gM}}$, where $\phi_i: U_i \to \phi_i(U_i) \subset \R^d$ is the affine tangent-projection chart, together with a subordinate partition of unity $\{\rho_i\}_{i=1}^{C_{\gM}}$ with compact supports $\gS_i \Subset U_i$.
We write $\bz_i := \phi_i^{-1}$ on $\phi_i(U_i)$, $\bJ_i(\bu) := \nabla\bz_i(\bu) \in \R^{D \times d}$, and $\bG_i(\bu) := \bJ_i(\bu)^\top\bJ_i(\bu)$.
Let $\gQ(\bx,t)$ be the Gaussian-weighted manifold density and $\gP(\bx,t)$ its first moment.
Aggregating them isolates the centered numerator:
\begin{equation*}
    \gU(\bx,t) \coloneqq \frac{m_t\gP(\bx,t)-\bx\gQ(\bx,t)}{\sigma_t}.
\end{equation*}
This directly yields the ambient density and target score representation:
$$
    p_t(\bx) = (2\pi\sigma_t^2)^{-D/2}\sigma_t^d\gQ(\bx,t),
    \qquad
    \sigma_t s^*(\bx,t) = \frac{\gU(\bx,t)}{\gQ(\bx,t)}.
$$
Thus, the stable approximation target is the scaled, centered score $\sigma_t s^*(\bx,t)$, avoiding a separate approximation of $m_t\gP$ and $\bx\gQ$ followed by cancellation and division by $\sigma_t$.

% \vspace{-5pt}
\subsection{Proof Architecture for Score Approximation}
% \vspace{-5pt}

The score approximation problem is dictated by the geometry of the Gaussian
footprint around the scaled manifold $m_t\gM$.  
On the active tube $A_t$, the scaled score has the centered-ratio representation $\sigma_t s^*(\bx,t)=\gU(\bx,t)/\gQ(\bx,t)$ displayed above.  
The two approximation mechanisms are:
% \vspace{-10pt}
\begin{center}
    {\small
    \setlength{\tabcolsep}{4pt}
    \resizebox{\linewidth}{!}{
    \begin{tabular}{p{0.13\linewidth}|p{0.2\linewidth}|p{0.28\linewidth}|p{0.26\linewidth}}
    \hline
    Regime & Geometry mapped by Gaussian & Primary obstruction & Approximation mechanism \\
    \hline
    Large noise regime $[t_1, T]$ &
    Tangent-scale manifold patch &
    Avoiding an ambient spatial grid while integrating over the forward tube &
    Tangent-cell approximation at radius $c_0\sigma_{t_{k-1}}\wedge r_{\star,n}$ \\
    \hline
    Small noise regime $[t_0, t_1]$ &
    Normal fiber intersecting the nearest projection point &
    A $\sigma_t$-mesh is statistically prohibitive, and the raw denominator contains an infinitesimal normal Gaussian factor &
    Projection-centered, de-Gaussianized Laplace expansion \\
    \hline
    \end{tabular}
    }
    }
\end{center}
Thus, large noise is handled by intrinsic tangent cells, while small noise is handled by centering at $\Pi_{\gM}(\bx/m_t)$, factoring out the common normal Gaussian, and expanding only in $d$ tangential coordinates.  
The former keeps the cell count at $\sigma_{t_{k-1}}^{-d}+r_{\star,n}^{-d}$; the latter avoids the prohibitive $\sigma_{t_0}^{-d}$ mesh and avoids treating $\Pi_{\gM}$ as a generic $D$-variate smooth function.  
The key dimension-explicit ingredient in the small-noise branch is the projection network: it computes the center $\Pi_{\gM}(\bx/m_t)$ by finite intrinsic anchors and Gauss--Newton iterations, and obtains chart coordinates by affine projection charts.  In contrast, the large-noise branch uses affine tangent-plane linearization centers for local integration, not nearest projection.

\subsection{Large Noise: Tangent-Cell Approximation}
% \vspace{-5pt}

On a large-noise slab $I_k=[t_{k-1},t_k]\subset[t_1,T]$, partition each chart box into coordinate cells of side length
$h_k=c_0\sigma_{t_{k-1}}\wedge r_{\star,n}$, with $r_{\star,n}$ defined in
\cref{prop:large-noise-oracle-learned-score}.
The resulting complexity is the intrinsic cell count $\sigma_{t_{k-1}}^{-d}+r_{\star,n}^{-d}$ rather than an ambient grid. For a cell center $\bu_\nu$, set $\by_\nu=\bz_i(\bu_\nu)$ and define the affine tangent-plane center
$$
    \mathsf T_\nu(\bx,t) \coloneqq m_t\by_\nu + \mathsf P_\nu^{\rm tan}(\bx-m_t\by_\nu).
$$
% This $\mathsf T_\nu$ is a cellwise linearization device, not the nearest-point projection $\Pi_{\gM}$ used in the small-noise branch.
For $\by=\bz_i(\bu)$ in the cell,
$$
    \|\bx-m_t\by\|_2^2
    =
    \|\bx-\mathsf T_\nu(\bx,t)\|_2^2
    + 2\langle\bx-\mathsf T_\nu(\bx,t),\mathsf T_\nu(\bx,t)-m_t\by\rangle
    + \|\mathsf T_\nu(\bx,t)-m_t\by\|_2^2
$$
decouples the phase into a tangent Gaussian envelope plus higher-order curvature corrections.  The denominator satisfies the tube floor
$$
    \gQ(\bx,t)
    \ge
    C^{-1}\Gamma_{\gM,q_{\rm geom}}^{-C}p_{\min}
    \exp\{-C\Gamma_{\gM,q_{\rm geom}}^C(D\vee\log n)\},
    \qquad
    \bx\in A_t.
$$
The network approximates these cellwise envelopes and their stored chart moments, then constructs the reciprocal on the displayed denominator range.

\begin{proposition}[Informal version of \cref{thm:large-noise-hd-score-global}]
    Under \cref{assump:manifold:exact,assump:manifold:density,assump:manifold:density-lower} and the polynomial-growth regime $D\vee p_{\min}^{-1}\le n^{a_0}$, consider the dyadic large-noise grid with capped cell radius
    $h_k=c_0\sigma_{t_{k-1}}\wedge r_{\star,n}$.  
    The cap $r_{\star,n}$ is chosen in \cref{prop:large-noise-oracle-learned-score} precisely so that the phase-smallness and fixed-order geometry conditions in \cref{eq:large-noise-density-lower-geometry-condition} hold on every large-noise slab.  
    % On this verified capped grid, 
    Then, for every $I_k$, there exists a clipped ReLU network $s_k^{\rm lg} \in \nn(L_k, \bW_k, S_k, B_k)$ such that, for all $\bx\in A_t$ and $t\in I_k$: 
    $$
        \|s_k^{\rm lg}(\bx,t)-s^*(\bx,t)\|_2^2
        \le
        C\Gamma_{\gM,q_{\rm geom}}^C B_0^C
        p_{\min}^{-C}
        \sigma_t^{-2} n^{-2(\beta+1)/(d+2\beta)} (D\vee\log n)^{\mathfrak p_{\beta,d}}.
    $$
    Moreover, $\sup_{\bx\in\R^D}\|s_k^{\rm lg}(\bx,t)\|_\infty \le C\sigma_t^{-1}(D\vee\log n)^{1/2}$, with capacity bounded by 
    $L_k\vee\log B_k \lesssim \Gamma_{\gM,q_{\rm geom}}^C B_0^C(D\vee\log n)^{\mathfrak p_{\beta,d}}$ and $\|\bW_k\|_\infty\vee S_k \lesssim \Gamma_{\gM,q_{\rm geom}}^C B_0^C p_{\min}^{-C}(\sigma_{t_{k-1}}^{-d}+r_{\star,n}^{-d}) (D\vee\log n)^{\mathfrak p_{\beta,d}}$. 
\end{proposition}

% \vspace{-5pt}
\subsection{Small Noise: Projection-Centered Laplace Approximation}
% \vspace{-5pt}

At small noise, extending tangent cells to $t_0$ would be statistically too costly, so we center at the nearest projection.  For $\bx\in A_t$ and $t\in[t_0,t_1]$, the reach condition makes $\Pi_{\gM}(\bx/m_t)$ unique.  Let $H \coloneqq D\vee\log n$ and define
$$
    \bpi(\bx,t) \coloneqq \Pi_{\gM}(\bx/m_t),
    \qquad
    \bu_i^\Pi \coloneqq \phi_i(\bpi(\bx,t)),
    \qquad
    \br \coloneqq \bx-m_t\bpi(\bx,t),
    \qquad
    \bnu \coloneqq \br/\sigma_t.
$$
Then $\nabla\bz_i(\bu_i^\Pi)^\top\br = 0$, so the normal residual is separated before the tangential Laplace expansion.

By introducing the local parameterization $\bu=\bu_i^\Pi+\sigma_t\bw$, the phase expansion rigorously yields:
$$
    \frac{\|\bx-m_t\bz_i(\bu_i^\Pi+\sigma_t\bw)\|_2^2-\|\br\|_2^2}{2\sigma_t^2}
    =
    \frac{m_t^2}{2}\bw^\top\bG_i(\bu_i^\Pi)\bw
    +
    \sum_{\ell=1}^{\floor{\beta}}\sigma_t^\ell \Psi_{i,\ell}(\bu_i^\Pi,\bnu,\bw,t)
    +
    \Ord(\sigma_t^\beta).
$$
The leading term is a $d$-dimensional Gaussian in $\bw$.  
To remove the unstable normal factor, define
$$
    \bar\gQ \coloneqq e^{\|\br\|_2^2/(2\sigma_t^2)}\gQ,
    \qquad
    \bar\gU \coloneqq e^{\|\br\|_2^2/(2\sigma_t^2)}\gU,
    \qquad
    \frac{\bar\gU}{\bar\gQ} = \frac{\gU}{\gQ}.
$$
The denominator now has the polynomial floor
$\bar\gQ(\bx,t) \ge C^{-1}\Gamma_{\gM,q_{\rm geom}}^{-C}p_{\min}(D\vee\log n)^{-d/4}$.  The remaining geometric task is to realize the projection center and hence the coordinate $\bu_i^\Pi(\bx,t)$ without treating $\Pi_{\gM}$ as a black-box $D$-variate smooth map.

\begin{lemma}[Informal projection-network primitive; see \cref{thm:global-proj-NN,cor:H-active-components}]
    On the active reach tube used in the formal small-noise theorem, for every accuracy $\varepsilon_\Pi\in(0,1]$, there are ReLU networks approximating $\Pi_{\gM}(\bx/m_t)$ and its chart coordinates $\bu_i^\Pi(\bx,t)=\phi_i(\Pi_{\gM}(\bx/m_t))$. 
    The construction uses only intrinsic nonlinear approximation: it fixes finitely many anchors in each $d$-dimensional chart box, unrolls approximate Gauss--Newton steps
    $$
        \bu\mapsto \Pi_i^{({\rm box})}\!\left(\bu+\bA_i(\bu)(\bx/m_t-\bz_i(\bu))\right),
        \qquad
        \bA_i(\bu)=(\bJ_i(\bu)^\top\bJ_i(\bu))^{-1}\bJ_i(\bu)^\top,
    $$
    and then combines all chart-anchor candidates by a ReLU objective gate based on $\|\bx/m_t-\bz_i(\bu)\|_2^2$.
    If $A_{\gM}$ is the finite intrinsic-anchor count and $C_{\rm proj}$ is the geometry constant from \cref{thm:global-proj-NN}, the dominant width/sparsity cost is
    $$
        A_{\gM}D
        \bigl(C_{\rm proj}D\varepsilon_\Pi^{-2}\bigr)^{d/(2(\beta\vee1))}
        \polylog(C_{\rm proj}D\varepsilon_\Pi^{-1}),
    $$
    with depth and log-weight polylogarithmic in $C_{\rm proj}D\varepsilon_\Pi^{-1}$. 
\end{lemma}
Thus, the projection step has intrinsic exponent and explicit polynomial ambient dependence under polynomial-geometry control, rather than a generic $\varepsilon_\Pi^{-D/\tilde\beta}$ ambient smooth-function cost.

\begin{proposition}[Informal version of \cref{cor:small-noise-density-lower-approx}]
    Under \cref{assump:manifold:exact,assump:manifold:density,assump:manifold:density-lower}, the polynomial-growth condition $D \vee p_{\min}^{-1} \le n^{a_0}$, and the small-noise reach and denominator-resolution checks \cref{eq:small-noise-final-reach-condition,eq:small-noise-denominator-resolution-condition}, there exists a clipped ReLU network $s_{\rm sm}^{\rm lb} \in \nn(L_{\rm sm}, \bW_{\rm sm}, S_{\rm sm}, B_{\rm sm})$ active over $[t_0,t_1]$ such that:
    $$
        \|s_{\rm sm}^{\rm lb}(\bx,t)-s^*(\bx,t)\|_2^2
        \le
        C\Gamma_{\gM,q_{\rm geom}}^C B_0^C
        p_{\min}^{-C} n^{-2\beta/(d+2\beta)} (D\vee\log n)^{\mathfrak p_{\beta,d}}\sigma_t^{-2}, 
        \;\; 
        \bx \in A_t, \; t \in [t_0,t_1]. 
    $$
    Moreover, $\sup_{\bx\in\R^D}\|s_{\rm sm}^{\rm lb}(\bx,t)\|_\infty \le C\sigma_t^{-1}(D\vee\log n)^{1/2}$, $L_{\rm sm} \vee \log B_{\rm sm} \lesssim \Gamma_{\gM,q_{\rm geom}}^C B_0^C p_{\min}^{-C} (D\vee\log n)^{\mathfrak p_{\beta,d}}$ and $\|\bW_{\rm sm}\|_\infty\vee S_{\rm sm} \lesssim \Gamma_{\gM,q_{\rm geom}}^C B_0^C p_{\min}^{-C} n^{d/(d+2\beta)} (D\vee\log n)^{\mathfrak p_{\beta,d}}$. 
\end{proposition}

The two approximation branches are assembled statistically, rather than by a separate deterministic theorem in the main text.  
In \cref{sec:score-est}, the small-noise projection--Laplace network is trained once on $[t_0,t_1]$, while the large-noise tangent-cell networks are trained on dyadic slabs over $[t_1,T]$ and combined by ReLU time switches.  
The formal learned-score assembly is \cref{cor:small-noise-density-lower-estimation,prop:large-noise-oracle-learned-score}, and the final $\sfW_1$ telescoping step is \cref{cor:w1-rate-density-lower}.
% \vspace{-5pt}

\section{Score Estimation on Smooth Manifolds}\label{sec:score-est}
% \vspace{-5pt}

Recall the rate notation $\mathfrak r_n$ and $q_\star$ from \cref{eq:def-main-rate}; the large-noise accuracy radius $r_{\star,n}$ is defined in \cref{prop:large-noise-oracle-learned-score}.
This section turns the comparison networks from \cref{sec:score:approx:manifold} into learned-score bounds.  The two displays below are the main-text informal versions of \cref{cor:small-noise-density-lower-estimation,prop:large-noise-oracle-learned-score}.
On $[t_0,t_1]$, clipped ERM over the projection--Laplace class gives
$$
    \E\int_{t_0}^{t_1}
      \E_{\bX_t}\|\widehat s_{\rm sm}^{\rm lb}(\bX_t,t)-s^*(\bX_t,t)\|_2^2\odt
    \le
    \tilde{\Ord}\!\Bigl(
	      \Gamma_{\gM,q_{\rm geom}}^C
	      B_0^C
	      p_{\min}^{-C}
	      n^{-2\beta/(d+2\beta)}
	      D^{\Ord_\beta(d)}
    \Bigr).
$$
On each large-noise slab $I_k=[t_{k-1},t_k]$, clipped ERM over the tangent-cell class gives
$$
    \E\int_{I_k}
      \E_{\bX_t}\|\widehat s_k^{\rm lg}(\bX_t,t)-s^*(\bX_t,t)\|_2^2\odt
    \le
    \tilde{\Ord}\!\Bigl(
	      \Gamma_{\gM,q_{\rm geom}}^C
	      B_0^C
	      p_{\min}^{-C}
	      \Bigl[
        n^{-2(\beta+1)/(d+2\beta)}
        +
        \frac{\sigma_{t_{k-1}}^{-d}+\mathfrak r_n^{-d/q_\star}}{n}
      \Bigr]
      D^{\Ord_\beta(d)}
    \Bigr).
$$
These are the learned-score inputs used in \cref{sec:main:w1-from-slabwise-score}.

\textbf{Clipped ERM oracle and covering number bounds.}
The formal proof is in \cref{sec:app:score-estimation}. 
We minimize the interval-restricted denoising score-matching loss, control the true smoothed risk $\gL_I$, and clip learned scores at the tube-scale envelope $\sigma_t^{-1}(D\vee\log n)^{1/2}$.

For a slab $I=[\tdown,\tup]$, the notation $\widetilde{\nn}(L,\bW,S,B)$ denotes the following clipped, compact-input version of the sparse ReLU class.  
First apply the coordinatewise affine truncation $(\bx,t)\mapsto(\Pi_{[-B_x,B_x]^D}\bx,\Pi_{[\tdown,\tup]}t)$, with $B_x=(D\vee\log n)^C\sigma_{\tdown}^{-1}$; then rescale this compact rectangle to $[0,1]^{D+1}$, apply a network in $\nn(L,\bW,S,B)$, and finally clip each output coordinate to $C\sigma_t^{-1}(D\vee\log n)^{1/2}$. 
The off-tube estimates make the discarded input tails negligible, and the compact rescaling is exactly the one used in the covering-number step of the oracle proof.

The clipped ERM oracle inequality (\cref{thrm:est:score:oracle}) balances approximation and complexity:
$$
    \E\gL_I(\widehat s) \lesssim \inf_{s\in\mathcal F}\gL_I(s) + \frac{ (D\vee\log n)^2 \int_I\sigma_t^{-2}\odt\, \log\bigl(\gN(n^{-\Ord(1)},\mathcal F,\|\cdot\|_\infty)\vee n\bigr) }{n} + n^{-A}.
$$
For sparse ReLU networks with structural size $M$, the concrete entropy evaluation is: whenever
$L\vee\log B\le C_{\mathcal F}(D\vee\log n)^C$ and
$S\vee\|\bW\|_\infty\le C_{\mathcal F}M(D\vee\log n)^C$, \cref{thrm:est:covering-number} gives
\begin{equation}
    \log\bigl(\gN(n^{-\Ord(1)},\widetilde\nn(L,\bW,S,B),\|\cdot\|_\infty)\vee n\bigr)
    \le
    C C_{\mathcal F} M (D\vee\log n)^C .
    \label{eq:score-class-entropy-evaluation}
\end{equation}
On our grids, $\int_I\sigma_t^{-2}\odt$ is logarithmic or bounded, so the stochastic capacity penalty is $\tilde{\Ord}(M/n)$ up to the displayed geometry factors in $C_{\mathcal F}$.

\textbf{Global noise grid assembly.}
The small-noise estimator is trained once on $[t_0,t_1]$ with entropy size $M_{\rm sm}=n^{d/(d+2\beta)}$. 
The large-noise estimators are trained on dyadic slabs with radius $r_k=c_0\sigma_{t_{k-1}}\wedge r_{\star,n}$ and then merged by ReLU time switches.  The $\Ord(\log n)$ switches add only logarithmic overhead, while the assembled score keeps both $\sigma_{t_{k-1}}^{-d}$ and $r_{\star,n}^{-d}$ explicit.

\section{Distribution Estimation in $\sfW_1$ Distance}
\label{sec:main:w1-from-slabwise-score}
% \vspace{-5pt}

This section converts the learned-score bounds from \cref{sec:score-est} into the final distributional guarantee. 
The formal slabwise $\sfW_1$ perturbation and telescoping argument are in \cref{sec:app:dist-est}.

Set $t_0 \asymp n^{-2(\beta+1)/(d+2\beta)}, t_1 \asymp n^{-2/(d+2\beta)}, T \asymp \log n$, and use the dyadic large-noise grid from \cref{eq:dyadic-grid-condition}.  
Then $\sqrt{t_0}\asymp \mathfrak r_n$, $\sigma_{t_1}\asymp n^{-1/(d+2\beta)}$, and the terminal Gaussian initialization error is negligible. 
Let $\widehat s^{\rm lb}$ be the switched clipped estimator obtained by combining the small-noise and large-noise ERM estimators, and let $\widehat P_{t_0}^{\rm lb}$ be the law at time $t_0$ of the reverse VP process driven by $\widehat s^{\rm lb}$.

The rate assembly has three terms.  Early stopping gives $\Ord(\sqrt{D t_0})$, which is absorbed by the displayed ambient polynomial times $\mathfrak r_n$, while the terminal error is $\Ord(\sqrt{D}e^{-cT})$ and is made negligible by choosing $T\asymp\log n$ with a sufficiently large constant.  
On the small-noise interval, the score risk $\tilde{\Ord}(n^{-2\beta/(d+2\beta)})$ is
multiplied by the perturbation width $\sqrt{t_1}$, giving
\[
    \sqrt{t_1}\,n^{-\beta/(d+2\beta)}
    \asymp
    n^{-(\beta+1)/(d+2\beta)}
    =
    \mathfrak r_n .
\]
On the large-noise slabs, the stochastic score term contains $(\sigma_{t_{k-1}}^{-d}+\mathfrak r_n^{-d/q_\star})/n$. 
The dyadic part contributes
$$
    n^{-1/2}
    \sum_{k:\,t_k\le1}
      t_k^{1/2}\sigma_{t_{k-1}}^{-d/2}
    \lesssim
    n^{-1/2}\bigl(1+\sigma_{t_1}^{1-d/2}\bigr)
    \lesssim
    n^{-(\beta+1)/(d+2\beta)}+n^{-1/2}.
$$
The order-one-time slabs add only another $\tilde{\Ord}(n^{-1/2})$ term, and the
finite-order accuracy cap contributes $n^{-1/2}\mathfrak r_n^{-d/(2q_\star)}$.

\textbf{Regime for the main guarantee.}
For the statement below, assume the manifold and density conditions \cref{assump:manifold:exact,assump:manifold:density,assump:manifold:density-lower}, the
polynomial control $D\vee p_{\min}^{-1}\le n^{a_0}$, the time grid above, the canonical
finite-order choice $q_{\rm geom}=q_{\rm opt}(\beta)$, and the two small-noise sample-size checks \cref{eq:small-noise-final-reach-condition,eq:small-noise-denominator-resolution-condition}.
\Cref{sec:app:controlled-regime-checklist} records these conditions in checklist form
and explains that, for a fixed well-conditioned manifold and fixed $p_{\min}>0$, the
small-noise checks hold for all sufficiently large $n$.  
The formal version, including network-size bounds, is \cref{cor:w1-rate-density-lower}.

\begin{theorem}[Main $\sfW_1$ guarantee; informal version of \cref{cor:w1-rate-density-lower}]\label{thm:main-simplified-w1}
    Under the preceding regime, the switched clipped VP-SGM estimator satisfies
    $$
        \E\,\sfW_1(P_0,\widehat P_{t_0}^{\rm lb})
        \le
        \tilde{\Ord}\!\bigl(
        \Gamma_{\gM,q_{\rm geom}}^C B_0^C
        p_{\min}^{-C}
        (D\vee\log n)^{\mathfrak p_{\beta,d}}
        \bigl[
          \mathfrak r_n+n^{-1/2}+n^{-1/2}\mathfrak r_n^{-d/(2q_\star)}
        \bigr]
        \bigr),
    $$
    where $q_\star$ and $\mathfrak r_n$ are from \cref{eq:def-main-rate}.  
    If $d>2$, then $q_\star\ge d(\beta+1)/(d-2)$, both bracketed correction terms are $\Ord(\mathfrak r_n)$, and $\mathfrak p_{\beta,d}=\Ord_\beta(d)$, so the bound simplifies to
    $$
        \E\,\sfW_1(P_0,\widehat P_{t_0}^{\rm lb})
        \le
        \tilde{\Ord}\!\bigl(
        \Gamma_{\gM,q_{\rm geom}}^C B_0^C
        p_{\min}^{-C}
        D^{\Ord_\beta(d)}
        n^{-(\beta+1)/(d+2\beta)}
        \bigr).
    $$
\end{theorem}

The exponent $\mathfrak p_{\beta,d}$ and the non-log geometry and density prefactors are made explicit in \cref{cor:w1-rate-density-lower}. 
For fixed geometry and fixed density lower bound, the $d>2$ display matches the intrinsic minimax $\sfW_1$ sample exponent in the fixed-geometry sense.  
A sufficient ambient-growth condition for the two small-noise checks is given in \cref{cor:fixed-geometry-ambient-growth}.

\section{Conclusion}
% \vspace{-5pt}

This paper proves an intrinsic $\sfW_1$ guarantee for a VP-SGM estimator on compact smooth manifolds by separating the score approximation into large-noise tangent cells and small-noise projection-centered Laplace expansions.  The large-noise branch has explicit capacity $\sigma_{t_{k-1}}^{-d}+r_{\star,n}^{-d}$; the small-noise branch avoids a $\sigma_{t_0}$-scale mesh by de-Gaussianizing the ratio and computing projection coordinates through finite-anchor Gauss--Newton networks in local $d$-dimensional frames.

Under the stated reach, density, denominator, and finite-geometry conditions, the learned reverse process obeys $n^{-(\beta+1)/(d+2\beta)} + n^{-1/2} n^{\frac{d(\beta+1)}{2q_\star(d+2\beta)}}$ up to logarithmic, geometry, density, and explicit ambient factors.  For $d>2$, the canonical finite-order choice makes the cap $\Ord(n^{-(\beta+1)/(d+2\beta)})$, so the sample exponent matches the intrinsic minimax $\sfW_1$ exponent in the fixed-geometry sense.  
The proof remains conditional on exact support, a positive density lower bound, and explicit small-noise admissibility conditions; it also studies an idealized slabwise estimator rather than shared-parameter diffusion training.  
Open problems include relaxing those support and density assumptions, adding ambient observational noise, and analyzing optimization and discretization error.

% \clearpage

\section*{Acknowledgment}

TS was partially supported by JSPS KAKENHI (24K02905) and JST CREST (JPMJCR2015).
WSL is supported by the Ministry of Education, Singapore, under its Academic Research Fund Tier 1 (A-8001814-00-00). 
This research is supported by the National Research Foundation, Singapore and the Ministry of Digital Development and Information under the AI Visiting Professorship Programme (award number AIVP-2024-004). Any opinions, findings and conclusions or recommendations expressed in this material are those of the author(s) and do not reflect the views of National Research Foundation, Singapore and the Ministry of Digital Development and Information.

\clearpage

\newpage

\bibliography{reference}
\bibliographystyle{alpha}

%%%%%%%%%%%%%%%%%%%%%%%%%%%%%%%%%%%%%%%%%%%%%%%%%%%%%%%%%%%%

\newpage

\appendix

\tableofcontents

\newpage
\section{More Discussions of Related Work}

\subsection{Detailed Comparison with Concurrent Work}\label{sec:app:comparison_zhang}

Although the concurrent study by \citep{zhang2026diffusion} shares our objective of establishing curvature-driven statistical bounds for diffusion models on nonlinear manifolds, the technical difference mainly lies in how the nearest-point projection $\Pi_{\gM}$ is made dimension-explicit inside the proof.

Specifically, \citep[Lemma D.18]{zhang2026diffusion} invokes a generic ambient smooth-function approximation step for $\Pi_{\gM}$ on the reach tube.  Our contribution is to replace that step by a constructive intrinsic-coordinate implementation.  The comparison below therefore concerns the ambient audit of one standard black-box instantiation of the smooth-map step, not an impossibility claim for every realization of the concurrent theorem.

To see the ambient bookkeeping issue, suppose one approximates $\Pi_{\gM}$ by a standard $D$-variate $C^{\tilde\beta}$ ReLU approximation theorem at the balancing scale needed to recover intrinsic complexity.  In the relevant regime, the effective order is roughly $\tilde\beta\gtrsim D\beta/d$, up to logarithmic and auxiliary-accuracy factors.  In a local Taylor formulation, approximating order $\tilde\beta$ requires a combinatorial number of ambient mixed partial derivatives:
$$
    \binom{D+\lceil D\beta/d\rceil}{\lceil D\beta/d\rceil}
    =
    \exp\!\left(D\{(1+\beta/d)\log(1+\beta/d)-(\beta/d)\log(\beta/d)\}+\Ord(\log D)\right).
$$
At the balancing scale $\tilde\beta\asymp D\beta/d$, this monomial count already scales exponentially as $\exp(c_{\beta,d}D)$ for fixed intrinsic parameters $d$ and $\beta$. 

Furthermore, if the projection map's derivative envelope grows proportionally with the differentiation order (for example, if the $\tilde\beta$-th derivative is bounded by $C^{\tilde\beta}$), this black-box route can incur an additional exponential-in-$D$ capacity penalty.  Our point is therefore constructive: a generic ambient smooth-function approximation of $\Pi_{\gM}$ comes with a nontrivial dimension-explicit audit, whereas our proof avoids that route entirely. 

By contrast, our architecture avoids high-order ambient Taylor expansions for the projection map. We evaluate the projection center dynamically using finite-anchor Gauss--Newton iterations restricted entirely to $d$-dimensional local chart coordinates, and then obtain active chart coordinates by affine projection charts. Coupled with an analytical de-Gaussianization step that enforces a polynomial lower bound prior to the reciprocal network evaluation, our method keeps the ambient dependence of the controlled network parameters explicit and polynomial in $D$ in the polynomial-geometry regime covered by our theorems. 
These distinctions affect the dimension-explicit prefactors rather than the sample exponent itself, and are summarized in \cref{tab:comp-zhang}.

\begin{table}[htp]
\caption{Comparison of our work to the concurrent study~\citep{zhang2026diffusion}.}
\label{tab:comp-zhang}
{\small
\setlength{\tabcolsep}{4pt}
\resizebox{\linewidth}{!}{
\begin{tabular}{p{0.14\linewidth}|p{0.4\linewidth}|p{0.38\linewidth}}
% \hline
\textbf{Technical Challenge} & \textbf{\citep{zhang2026diffusion}} & \textbf{This Paper} \\
\hline
Projection Mechanism &
Uses a generic ambient smooth-function approximation step for $\Pi_{\gM}$ on the reach tube. &
Computes $\Pi_{\gM}$ dynamically via finite-anchor optimization restricted to $d$-dimensional intrinsic charts. \\
\hline
Approximation Strategy &
The published proof leaves the projection-step ambient audit to an ambient smooth-function approximation; under a standard black-box implementation, the relevant order is roughly $\tilde\beta \gtrsim D\beta/d$, up to logarithmic factors. &
Resolves projection candidates via Gauss--Newton iterations in $\R^d$ and obtains chart coordinates by affine maps, naturally achieving the intrinsic $\varepsilon^{-d/\beta}$ covering complexity. \\
\hline
Dimensionality Bottleneck &
A standard high-order ambient implementation would involve a combinatorial number of mixed partial derivatives: $\binom{D+\tilde\beta}{\tilde\beta}$, which is $\exp(c_{\beta,d}D)$ when $\tilde\beta\asymp D\beta/d$. &
Utilizes strictly finite-order geometry (with $q_{\rm geom}$ fixed before the limit); isolates ambient powers into $(D\vee\log n)^{\mathfrak p_{\beta,d}}$, summarized as $D^{\Ord_\beta(d)}$ in the main $d>2$ regime. \\
\hline
Denominator Stability &
Handles the sharp normal Gaussian singularity through the same projection-centered smooth-map step. &
Explicitly de-Gaussianizes the integrals, guaranteeing a stable polynomial lower bound prior to evaluating the reciprocal network. \\
\hline
Complexity Guarantee &
Displays the intrinsic sample exponent, while the dimension-explicit prefactors depend on how the projection step is instantiated. &
Keeps network parameters polynomial in $D$ under the polynomial-geometry conditions of the formal theorem. \\
% \hline
\end{tabular}
}
}
\end{table}

\subsection{Comparison with the Support-aware \texorpdfstring{$\sqrt{D}$}{sqrt(D)} Result}\label{sec:app:comparison_azangulov}

The recent paper \citep{azangulov2024convergence} proves a $\tilde{\Ord}(\sqrt{D}n^{-(\alpha+1)/(2\alpha+d)+\gamma})$ Wasserstein bound under the manifold hypothesis.  This is a strong ambient factor, but it is obtained for a different estimator architecture from the one studied here.

Their first key step is the score-learning target itself.  In Theorem~5 of \citep{azangulov2024convergence}, the main bound is for the weighted quantity
$
\int \sigma_t^2 \E\|\widehat s_t-s_t\|_2^2\,dt
$,
rather than the unweighted score loss that enters our slabwise $\sfW_1$ perturbation argument.  This weighting downplays the most singular part of the small-noise regime.  The displayed $\sqrt{D}$ factor then appears in their Corollary~6 through the final early-stopping conversion from score error to $\sfW_1$.

Their second key step is structural.  They first estimate the support by a sample-dependent piecewise polynomial manifold surrogate $M^\ast$, then use Proposition~21 to show that each local polynomial piece lies in an affine span of dimension $\Ord(\log n)$.  Section~3.1 of \citep{azangulov2024convergence} explicitly states that this reduces the ambient score problem to $N$ simpler $\Ord(\log n)$-dimensional subproblems.  The resulting estimator class is therefore support-aware: the local centers and subspaces are extracted from the data before the score network is fit.  Indeed, their construction notes that these geometric quantities are found directly from the sample rather than optimized jointly within a generic ambient network.

This distinction matters for the ambient-dimension comparison.  Our theorem does not preprocess the manifold into sample-dependent low-rank affine spans.  Instead, it analyzes a fully explicit ambient score network that directly computes the nearest-projection center, its active chart coordinates, and the de-Gaussianized denominator in the small-noise regime.  That is a harder constructive problem, so one should not expect the same ambient factor.  On the other hand, the result here gives a theorem-level audit of the entire projection-centered score construction and attains the exact intrinsic sample exponent without an extra $n^\gamma$ slack.

The difference is also visible inside \citep{azangulov2024convergence} itself.  Their Theorem~7(i) proves the dimension-light rate for the structured support-aware class, whereas Theorem~7(ii) shows that when this estimator is flattened into a more standard ambient ReLU architecture, an explicit factor of $D$ reappears in the score-learning bound.  We therefore view the two analyses as complementary: \citep{azangulov2024convergence} obtains the sharper displayed ambient prefactor for a support-aware pipeline, while our theorem certifies a less structured ambient-network realization of the manifold score.

\subsection{Additional Related Work}
While our work focuses on the statistical guarantees of \textit{learning} the score function, a broad body of literature addresses other theoretical aspects of diffusion models. 
\vspace{-2pt}

\textbf{Convergence analysis on ambient space.}
Early theoretical works focused on establishing convergence guarantees for data distributions supported on the full ambient space $\R^D$. 
Initial analyses often relied on stringent structural assumptions, such as Log-Sobolev Inequalities (LSI) or log-concavity, to derive polynomial convergence rates~\citep{lee2022convergence,wibisono2022convergence,bruno2023diffusion}. 
Subsequent research sought to relax these conditions to better model complex, multi-modal data. \citep{chen2023sampling,lee2023convergence} demonstrated that $L^2$-accurate score estimation is sufficient for polynomial convergence without requiring functional inequalities, extending guarantees to distributions with bounded support or decaying tails. 
More recently, the focus has shifted to minimal assumptions and faster rates. \citep{conforti2023score} showed that finite Fisher information is sufficient for convergence without early stopping. 
\citep{benton2024linear} utilized stochastic localization to prove that convergence rates scale linearly with the ambient dimension $D$ (up to logarithmic factors), while \citep{li2025d,li2023towards} provided faster non-asymptotic convergence rates (e.g., $\Ord(D/T)$) for both ODE and SDE-based samplers. 
\citep{jiao2025optimal} further refined these results for DDPMs, establishing near-optimal rates under relaxed smoothness conditions.
Despite these advances, ambient space analyses inherently depend on $D$, often yielding vacuous bounds when applied to singular measures supported on low-dimensional manifolds ($d \ll D$), as they do not exploit the intrinsic geometry of the data.
\vspace{-2pt}

\textbf{Convergence analysis on low-dimensional structures.}
% A significant line of research assumes the score function is already known (or accurately estimated) and focuses on the algorithmic convergence of the sampling process. 
Recently, \citep{huang2024denoising} and \citep{potaptchik2025linear} concurrently proved that the iteration complexity of Denoising Diffusion Probabilistic Models (DDPM) scales linearly with the intrinsic dimension $d$, rather than the ambient dimension $D$, establishing the optimal adaptivity of stochastic samplers. 
Similarly, for deterministic samplers, \citep{tang2025adaptivity} and \citep{liang2025low} showed that Probability Flow ODEs and DDIMs achieve dimension-free convergence rates in TV distance. 
\citep{george2025analysis} further analyzed the dynamics of these processes, identifying phase transitions (speciation and collapse) dependent on the dimension ratios. 
Our work complements these findings by addressing the upstream problem: statistically learning the score function from finite data.
\vspace{-2pt}

\textbf{Intrinsic-dimension statistical guarantees.}
Two very recent works develop complementary intrinsic-dimensional viewpoints.
\citep{chakraborty2026generalization} studies score-matching diffusion models for general low-dimensional measures through a $(p,q)$-Wasserstein dimension, building on empirical-measure Wasserstein rates such as \citep{weed2019sharp}.  Their theory yields broad $\sfW_p$ bounds of order roughly $n^{-1/d^\star_{p,q}(\mu)}$ up to logarithmic and algorithmic approximation terms, and it applies beyond smooth compact manifolds under finite-moment conditions.  This generality is different from our smooth-density setting: for a $d$-dimensional compact manifold with $\beta$-H\"older density and $d>2$, the smoothness-adaptive minimax $\sfW_1$ exponent is $(\beta+1)/(d+2\beta)$, which is faster than the empirical-measure exponent $1/d$.  \citep{lyu2025resolving} also studies finite-sample manifold diffusion from empirical data and uses an inertial correction to avoid memorization, obtaining an ambient-dimension-free $\sfW_1$ guarantee for $C^2$ manifolds.  Our analysis instead learns a population score over neural-network classes for H\"older densities and audits the concrete projection--Laplace construction, including the de-Gaussianized small-noise denominator and the nearest-projection network.
\vspace{-2pt}

\textbf{Diffusion dynamics and schedules.}
The choice of the forward diffusion process and its noise schedule plays a central role in the performance of score-based generative models. 
% The choice of the forward diffusion process is critical to performance. 
\citep{strasman2024analysis} provided error bounds explicitly depending on the noise schedule, highlighting the impact of schedule tuning. 
\citep{li2024accelerating,zhang2025sublinear} explored algorithmic accelerations, achieving faster convergence rates through higher-order solvers and randomized midpoint methods.
Regarding the optimal process choice, \citep{brevsar2025nonasymptotic} analyzed the forward diffusion and argued that the canonical OU process is hard to beat in terms of convergence to the invariant measure. 
% \textbf{Our distinction:}
% We challenge the notion that the standard OU process is universally optimal by shifting attention to the reverse-time generative dynamics on manifold-supported data.
% Our analysis demonstrates that the OU process is minimax optimal when $d \leq 2$, whereas for $d > 2$ the fixed OU drift fails to attain the minimax rate, due to singular behavior in the low-noise regime near the data manifold.
% This motivates the dimension-adaptive drift schedule we propose, which restores minimax-optimal convergence across all intrinsic dimensions $d \geq 1$.  
% \textit{Our distinction:} We present a counterpoint to the view that the standard OU process is universally optimal. In the context of the \textit{reverse} generative process on manifolds, we demonstrate a phase transition: while OU is optimal for low intrinsic dimensions ($d \leq 2$), it faces theoretical limitations for $d > 2$. This necessitates the dimension-adaptive drift schedule we propose, which ensures minimax optimality across all dimensions.
\vspace{-2pt}

\textbf{Geometric interpretations.}
Beyond statistical rates, recent works have explored how SGMs encode geometry. \citep{pidstrigach2022score} gave conditions under which SGMs detect the underlying data manifold, even when the on-manifold density is not recovered accurately.  \citep{stanczuk2024diffusion} showed that diffusion models approximate the normal bundle of the data manifold, enabling intrinsic dimension estimation. 
\citep{wang2024diffusion} framed diffusion training as subspace clustering, showing that optimizing the diffusion loss effectively learns mixtures of low-rank Gaussians. 
\citep{lu2023mathematical} and \citep{liu2025improving} analyze the singular structure of Euclidean diffusion scores near embedded manifolds, including the normal--tangential scale separation that motivates modified sampling or training strategies. 
\citep{li2025scores} gives a particularly sharp small-noise perspective: geometric information in the score appears at order $\sigma^{-2}$, whereas density information appears at order one.  Their results show that manifold concentration or uniform geometric sampling can tolerate substantially larger score errors than exact distributional recovery.  Relatedly, \citep{de2025provable} studies maximum-entropy exploration over a diffusion-defined data manifold via fine-tuning.  Our theorem addresses the harder distribution-learning target; correspondingly, the small-noise projection--Laplace branch controls the lower-order density-dependent terms rather than only the leading normal projection field. 
Finally, \citep{li2024understanding} observed that diffusion denoisers exhibit increasing linearity during the generalization phase, approximating optimal linear denoisers. 

\clearpage

% ======================================================================
% Appendix roadmap and conventions
% ======================================================================

\section{Appendix Roadmap and Conventions}\label{sec:app:roadmap}

\subsection{Roadmap}

The appendix is organized so that the main technical results can be read before the
supporting algebra.  A first pass may focus on the theorem and corollary statements in
the right column, then return to the indicated support modules as needed.
\begin{center}
{\small
\setlength{\tabcolsep}{4pt}
\resizebox{\linewidth}{!}{
\begin{tabular}{p{0.2\linewidth}|p{0.32\linewidth}|p{0.34\linewidth}}
Appendix module & Purpose & Main result or use \\
\hline
\cref{sec:app:approx} &
Off-tube Gaussian localization and Fisher control &
Restricts score approximation and estimation to the forward tube \\
\hline
\cref{sec:app:approx:score:regularized:near-manifold} &
Projection atlas, chart integrals, centered ratio, and noise-regime notation &
Defines $\gQ,\gP,\gU$ and the common geometric objects used by both branches \\
\hline
\cref{sec:app:approx:PQ:large-noise-branch} &
Tangent-cell approximation at capped slab-scale radius
$r_k=c_0\sigma_{t_{k-1}}\wedge r_{\star,n}$ &
\cref{thm:large-noise-hd-score-global} \\
\hline
\cref{sec:app:small-noise-score-approx} &
Small-noise active geometry, finite-anchor projection networks, projection-centered
Laplace expansion, and final ratio-network assembly &
\cref{thm:global-proj-NN,cor:small-noise-density-lower-approx} \\
\hline
\cref{sec:app:score-estimation} &
Clipped ERM oracle and branchwise learned-score bounds &
\cref{prop:large-noise-oracle-learned-score,cor:small-noise-density-lower-estimation} \\
\hline
\cref{sec:app:dist-est} &
Early stopping, terminal initialization, and slabwise $\sfW_1$ transfer &
\cref{cor:w1-rate-density-lower} \\
\hline
\cref{sec:app:toolkit} &
Reusable ReLU approximation, concentration, and covering tools &
Black-box support lemmas used throughout the appendix
\end{tabular}
}
}
\end{center}

\subsection{Two-regime score-approximation map}\label{sec:app:two-regime-score-map}

The central score-approximation proof splits according to the effective Gaussian
footprint.  The main text gives the reader-facing version in
\cref{sec:score:approx:manifold}; the appendix proof uses the following dependency map.
\begin{center}
{\small
\setlength{\tabcolsep}{4pt}
\resizebox{\linewidth}{!}{
\begin{tabular}{p{0.14\linewidth}|p{0.22\linewidth}|p{0.25\linewidth}|p{0.25\linewidth}}
Regime & Geometric center & Technical device & Formal output \\
\hline
Large noise $[t_1,T]$ &
Cellwise tangent-plane centers at intrinsic radius
$c_0\sigma_{t_{k-1}}\wedge r_{\star,n}$ &
Chart-cell phase expansion, stored chart moments, and the unrefactored denominator
tube floor &
\cref{thm:large-noise-hd-score-global}, then
\cref{prop:large-noise-oracle-learned-score} \\
\hline
Small noise $[t_0,t_1]$ &
Nearest projection point $\Pi_{\gM}(\bx/m_t)$ and its chart coordinates &
Finite-anchor Gauss--Newton projection networks, de-Gaussianized Laplace expansion,
and polynomial denominator lower bound &
\cref{thm:global-proj-NN,cor:small-noise-density-lower-approx}, then
\cref{cor:small-noise-density-lower-estimation} \\
\hline
Assembly &
Time switches over the early and dyadic slabs &
Clipped learned scores and slabwise $\sfW_1$ perturbation &
\cref{cor:w1-rate-density-lower}
\end{tabular}
}
}
\end{center}

\subsection{Proof audit map}\label{sec:app:proof-audit-map}

The table below is intended as an audit
checklist: each row names the statement where a possible failure would affect the final
rate, what should be verified there, and where the statement is used.
\begin{center}
{\small
\setlength{\tabcolsep}{3pt}
\resizebox{\linewidth}{!}{
\begin{tabular}{p{0.2\linewidth}|p{0.15\linewidth}|p{0.34\linewidth}|p{0.16\linewidth}}
Checkpoint & Statement & What to verify & Used in \\
\hline
Tube localization &
\cref{lem:approx:score:off-manifold} &
Off-tube score risk is negligible after clipping at
$\sigma_t^{-1}(D\vee\log n)^{1/2}$ &
Approximation and ERM \\
\hline
Centered ratio &
\cref{eq:def:score:centered-decomp} &
The scaled score is represented as $\gU/\gQ$ with the cancellation already centered in
the numerator &
Both approximation branches \\
\hline
Large-noise branch &
\cref{thm:large-noise-hd-score-global} &
One slab has intrinsic leading size $\sigma_{t_{k-1}}^{-d}+r_{\star,n}^{-d}$ and scaled score error
$n^{-(\beta+1)/(d+2\beta)}$ &
Large-noise ERM \\
\hline
Projection network &
\cref{thm:global-proj-NN} &
Projection centers are computed by intrinsic anchor and Gauss--Newton networks, without
high-order ambient approximation of $\Pi_{\gM}$ &
Small-noise branch \\
\hline
Small-noise denominator &
\cref{lem:small-noise-degaussianized-Q-lower} &
After removing the normal Gaussian factor, the denominator has a polynomial lower bound
before the reciprocal network is applied &
Small-noise reciprocal \\
\hline
Small-noise branch &
\cref{cor:small-noise-density-lower-approx} &
The projection--Laplace class has leading size $n^{d/(d+2\beta)}$ and scaled score error
$n^{-\beta/(d+2\beta)}$ &
Small-noise ERM \\
\hline
Score ERM &
\cref{thrm:est:score:oracle} &
Approximation plus sparse-network entropy gives the stated learned integrated score risks &
Score estimation \\
\hline
$\sfW_1$ transfer &
\cref{thrm:est:dist:W1-SM} &
Integrated score risk transfers slabwise to sampling error with the displayed time factor &
Final theorem
\end{tabular}
}
}
\end{center}

The dependency graph is
\[
\text{tube localization}
\longrightarrow
\begin{cases}
\text{large-noise tangent cells},\\
\text{small-noise projection--Laplace},
\end{cases}
\longrightarrow
\text{clipped ERM}
\longrightarrow
\text{slabwise }\sfW_1\text{ transfer}.
\]

\subsection{Standard manifold conventions}
\label{sec:app:standard-manifold-conventions}

We recall the elementary differential-geometric conventions used throughout the
appendix, mainly to make the proof self-contained for readers who do not use this
notation daily.

\begin{definition}[Chart]\label{def:chart}
    Let $\gM$ be a $d$-dimensional Riemannian manifold isometrically embedded in
    $\R^D$.  A chart on $\gM$ is a pair $(U,\phi)$, where $U\subset\gM$ is open and
    $\phi:U\to\R^d$ is a homeomorphism onto its image.
\end{definition}

Two charts $(U,\phi)$ and $(V,\psi)$ are $C^k$ compatible if the transition maps
$\phi\circ\psi^{-1}$ and $\psi\circ\phi^{-1}$ are $C^k$ on their common coordinate
domains.

\begin{definition}[$C^k$ atlas]\label{def:atlas}
    A $C^k$ atlas for $\gM$ is a collection of pairwise $C^k$ compatible charts
    $\{(U_i,\phi_i)\}_{i\in\gA}$ such that $\bigcup_{i\in\gA}U_i=\gM$.
\end{definition}

\begin{definition}[Smooth manifold]\label{def:smooth-manifold}
    A smooth manifold is a manifold equipped with a $C^\infty$ atlas.
\end{definition}

\begin{definition}[Partition of unity]\label[definition]{def:manifold:partition}
    A $C^\infty$ partition of unity on $\gM$ is a collection of nonnegative
    $C^\infty$ functions $\{\rho_i:\gM\to\R_+\}_{i\in\gA}$ such that the supports are
    locally finite and $\sum_{i\in\gA}\rho_i=1$ on $\gM$.
\end{definition}

Every smooth manifold admits smooth partitions of unity subordinate to open covers
\citep[Theorem~13.7]{tu2011manifolds}.  In the proof we use a finite partition
subordinate to the reach-scale projection atlas constructed in
\cref{sec:app:approx:manifold:setup}; multiplying by these cutoffs localizes manifold
integrals to single coordinate charts while preserving the finite smoothness needed in
the H\"older and Taylor estimates.

\subsection{Notation and bookkeeping conventions}
\label{sec:app:notation-conventions}

The conventions here supplement the global notation in \cref{sec:notation} only for
proof bookkeeping.  Unadorned constants $C,c>0$ remain
universal.  Named or subscripted constants, such as $C_{\rm proj}$, $C_Q$, or
$C_{\rm size,sm}$, may depend on the fixed objects named in the statement where they are
introduced, but not on $n,D,\sigma_t$ or a tolerance parameter unless that dependence is
displayed.  The notation $\tilde{\Ord}(\cdot)$ hides logarithmic factors in displayed
sample-size, ambient, time-cutoff, and tolerance scales; it does not hide the density
lower bound, reach, volume, or finite-order geometry factors unless this is stated
explicitly.

Local approximation and estimation tolerances use descriptive subscripts; no global
sample-size tolerance is introduced.  We write the smallest large-noise noise scale as
$\sigma_{t_1}\asymp n^{-1/(d+2\beta)}$ and the centered-residual scale as
$\mathfrak r_n=n^{-(\beta+1)/(d+2\beta)}$.  The large-noise intrinsic cell radius is not
identified with $\sigma_{t_1}$: on a slab $I_k=[t_{k-1},t_k]$ it is
$r_k=c_0\sigma_{t_{k-1}}\wedge r_{\star,n}$, where $r_{\star,n}$ is the
accuracy-limited radius defined in \cref{prop:large-noise-oracle-learned-score}.  Thus,
when the cap is active, the complexity records both $\sigma_{t_{k-1}}^{-d}$ and
$r_{\star,n}^{-d}$ explicitly.

The symbols $\delta$, $\Delta$, and $\eta$ are reserved for geometric quantities such as
active radii, displacements, and chart/projection margins.  The tube radius
$\varrho_t$, the generic physical density floor $\rho_t$, and the partition-of-unity
weights $\rho_i$ are distinct.  Under the density lower bound, denominator floors used
inside reciprocal networks are denoted locally.

In intermediate appendix estimates, $\mathfrak a_D,\mathfrak a_{\log}\ge1$ denote fixed
finite bookkeeping exponents for direct ambient powers and shell/covering factors, while
$\mathfrak a_{\rm geom}\ge1$ denotes a fixed exponent for the displayed geometry factors
$1\vee S_{\gM}$, $1\vee\kappa^{-1}$, and $\mathcal K_{\gM,q_{\rm geom}}$.  These
exponents never depend on $D,n,\sigma_t,\varepsilon_{\rm rel},p_{\min}$ or on the
numerical values of the geometry constants.  The appendix analysis interval is denoted by
$[\tdown,\tup]$.  When a proof localizes to the high-probability tube, $H$ or $H_n$
denotes $D\vee\log n$ unless another meaning is explicitly declared.
On a generic compact time interval, $\sigma_{\tdown}$ always means
$\inf_{t\in[\tdown,\tup]}\sigma_t$ rather than a formal endpoint value; the corresponding
schedule lower- and upper-bound convention is stated explicitly in
\cref{eq:def:sigma-tdown-lower}.

\subsection{Finite-order geometry and tube-volume conventions}
\label{sec:app:geometry-conventions}

\begin{definition}[Finite-order geometry convention]
\label{def:finite-order-geometry-constants}
Fix the reach-scale projection atlas used below.  For each $q\ge1$, let
$\mathcal K_{\gM,q}\ge e$ denote a finite high-order atlas envelope controlling the
$C^{q+4}$ norms of the inverse projection charts, transition maps, partition cutoffs,
Jacobian factors, the maps $\bA_i=\bG_i^{-1}\bJ_i^\top$, and the nearest-projection
coordinate and residual maps on the compact reach tubes where they are used.  The finite
geometry constant $\Gamma_{\gM,q}$ in the approximation proofs may be chosen so that
\[
    \Gamma_{\gM,q}
    \le
    C_{d,q}
    (1\vee S_{\gM})^{\mathfrak a_{\rm geom}(q)}
    (1\vee\kappa^{-1})^{\mathfrak a_{\rm geom}(q)}
    \mathcal K_{\gM,q}^{\mathfrak a_{\rm geom}(q)},
\]
where $\mathfrak a_{\rm geom}(q)$ depends only on $d$ and $q$.  In final statements,
$\mathfrak a_{\rm geom}$ denotes this exponent at $q=q_{\rm geom}$, enlarged to absorb
the finite network-arithmetic powers of $\Gamma_{\gM,q_{\rm geom}}$.  Reach alone does not
control $\mathcal K_{\gM,q}$, since the approximation uses finitely many derivatives of
charts and projection maps above second order.
\end{definition}

\textbf{Polynomial-geometry regime.}
The phrase ``polynomial ambient dependence'' in the main paper is always conditional on
the displayed instance-dependent factors being polynomially controlled.  Concretely, for
the fixed order $q_{\rm geom}$ used in the final theorem, if there is an exponent
$a_{\rm geom}$ independent of $n$ and $D$ such that
\[
    S_{\gM}\vee \kappa^{-1}\vee \mathcal K_{\gM,q_{\rm geom}}
    \le D^{a_{\rm geom}},
\]
then $\Gamma_{\gM,q_{\rm geom}}\le D^{C_{\beta,d,a_{\rm geom}}}$ and every network-size
bound displayed below has polynomial ambient dependence.  This covers, for example,
families admitting reach-scale atlases with polynomially many uniformly well-conditioned
charts and polynomially bounded derivatives up to the fixed order $q_{\rm geom}+4$.
If these finite-order geometry envelopes grow exponentially with $D$, our statements
display that growth rather than hiding it in $\tilde{\Ord}(\cdot)$ or in universal
constants.

\subsection{Examples satisfying the polynomial-geometry regime}
\label{sec:app:poly-geometry-examples}

We record concrete geometric families for which the preceding polynomial regime is
verified.  The statements below are only about the geometric factors
$S_{\gM},\kappa^{-1}$, and $\mathcal K_{\gM,q}$.  The density assumptions in the main
theorem can then be imposed independently; for example, the uniform density on any of
the following manifolds has
$p_{\min}^{-1}\le 1\vee S_{\gM}$ and $B_0\le 1\vee S_{\gM}^{-1}$, so the density
factors are polynomially controlled in the examples below as well; the radius or
pseudoinverse hypotheses give the needed polynomial lower bound on $S_{\gM}$.

\begin{proposition}[Basic polynomial-geometry examples]
\label{prop:poly-geometry-examples}
Fix $d\ge1$, a finite order $q\ge1$, and set $q_{\rm emb}=q+6$.  The following
families satisfy the polynomial-geometry regime for the reach-scale projection atlas in
\cref{def:finite-order-geometry-constants}.  For each displayed manifold write
$\kappa_D=\operatorname{reach}(\gM_D)$.  In (iii)--(iv), $C^k$ norms of
ambient-valued maps use Euclidean target and operator norms in the fixed finite atlas.
\begin{enumerate}[(i)]
    \item \emph{Affine spheres.}
    Let $D\ge d+1$, let $E_D=\R^{d+1}\times\{0\}^{D-d-1}$, let
    $\bc_D=(1/2,\dots,1/2)$, and let $0<r_D\le1/4$.  Set
    \[
        \gM_D
        =
        \{\bc_D+r_D\bu:\bu\in E_D,\ \|\bu\|_2=1\}
        \subset[0,1]^D .
    \]
    If $r_D^{-1}\le D^a$, then
    \[
        S_{\gM_D}\vee \kappa_D^{-1}\vee \mathcal K_{\gM_D,q}
        \le C_{d,q,a}D^{C_{d,q}a}.
    \]

    \item \emph{Flat product tori.}
    Let $D\ge2d$, let $0<r_{\min,D}\le r_{j,D}\le1/4$ for
    $j=1,\dots,d$, define
    \[
        \mathcal C_{j,D}
        :=
        \{(1/2+r_{j,D}\cos\theta,1/2+r_{j,D}\sin\theta):\theta\in[0,2\pi)\},
    \]
    and set
    $\gM_D=\mathcal C_{1,D}\times\cdots\times\mathcal C_{d,D}
    \times\{1/2\}^{D-2d}$.
    If $r_{\min,D}^{-1}\le D^a$, then
    \[
        S_{\gM_D}\vee \kappa_D^{-1}\vee \mathcal K_{\gM_D,q}
        \le C_{d,q,a}D^{C_{d,q}a}.
    \]

    \item \emph{A finite-order embedding criterion.}
    Let $N$ be a fixed compact smooth $d$-manifold without boundary, equipped with a
    fixed finite atlas.  Let $\Phi_D:N\to[0,1]^D$ be a $C^{q_{\rm emb}}$ embedding
    and set $\gM_D=\Phi_D(N)$.  Suppose that, in the fixed atlas,
    \[
        \|\Phi_D\|_{C^{q_{\rm emb}}(N)}
        \vee
        \|(\od\Phi_D)^\dagger\|_{C^{q_{\rm emb}-1}(N)}
        \vee
        \operatorname{reach}(\gM_D)^{-1}
        \le D^a,
    \]
    where $(\od\Phi_D)^\dagger$ is the Moore--Penrose inverse of the differential in
    local coordinates.  Then
    \[
        S_{\gM_D}\vee \kappa_D^{-1}\vee \mathcal K_{\gM_D,q}
        \le C_{N,d,q,a}D^{C_{N,d,q}a}.
    \]

    \item \emph{Closed graph-type manifolds.}
    As a special case of (iii), let $N$ be a fixed closed $d$-manifold and let
    $\iota:N\to\R^{m_0}$ be a fixed smooth embedding.  For $D\ge m_0$, let
    $F_D:N\to\R^{D-m_0}$ be a $C^{q_{\rm emb}}$ map.  Suppose
    $\|F_D\|_{C^{q_{\rm emb}}(N)}\le D^a$ and, for some $\btau_D\in\R^D$,
    $\Phi_D=\lambda_D(\iota,F_D)+\btau_D$ maps $N$ into $[0,1]^D$ with
    $D^{-a}\le\lambda_D\le1$ and
    \[
        \operatorname{reach}(\Phi_D(N))^{-1}\le D^a .
    \]
    Then $\gM_D=\Phi_D(N)$ satisfies the bound in (iii), with possibly larger constants.
    Thus closed graph manifolds are covered when finite-order derivatives and reach are
    polynomially controlled.  A mere $C^2$ curvature bound, without finite-order control
    and a reach lower bound, is not sufficient for the present theorem.
\end{enumerate}
\end{proposition}

\begin{proof}
For (i), the reach is exactly $r_D$ and
$S_{\gM_D}=r_D^d\vol(\sS^d)$.  The reach-scale tangent-projection charts are, up to
ambient rotations, the usual graph charts of a round sphere over tangent balls of radius
$cr_D$; differentiating the graph formula shows that their inverse maps, transition
maps, and reach-scale cutoffs have $C^{q+4}$ norms bounded by powers of $r_D^{-1}$.  On
the $r_D/2$-tube,
\[
    \Pi_{\gM_D}(\bx)
    =
    \bc_D+r_D\frac{P_{E_D}(\bx-\bc_D)}{\|P_{E_D}(\bx-\bc_D)\|_2},
\]
and the denominator is bounded below by $r_D/2$.  Differentiating the displayed formula
up to order $q+4$ gives
$\mathcal K_{\gM_D,q}\le C_{d,q}r_D^{-C_{d,q}}$.

For (ii), the reach is bounded below by $c_dr_{\min,D}$, and the nearest projection in
the corresponding product tube is obtained by normalizing each coordinate pair.  The
volume is
$S_{\gM_D}=(2\pi)^d\prod_{j=1}^d r_{j,D}$.  The reach-scale tangent-projection charts
are smooth changes of the product angular charts on angular intervals of length bounded
below by a constant multiple of $r_{\min,D}/r_{j,D}$, and quantitative inversion gives
finite-order chart, transition, and cutoff bounds by powers of $r_{\min,D}^{-1}$.  The
coordinatewise projection formula
\[
    (x_{2j-1},x_{2j})
    \mapsto
    (1/2,1/2)
    +
    r_{j,D}
    \frac{(x_{2j-1}-1/2,x_{2j}-1/2)}
    {\|(x_{2j-1}-1/2,x_{2j}-1/2)\|_2},
\]
has denominators bounded below by a constant multiple of $r_{\min,D}$.  Hence
$\mathcal K_{\gM_D,q}\le C_{d,q}r_{\min,D}^{-C_{d,q}}$.

For (iii), the volume element in a fixed base chart is
$\det((\od\Phi_D)^\top\od\Phi_D)^{1/2}$.  The $C^1$ bound on $\Phi_D$ gives a
polynomial upper bound on every singular value of $\od\Phi_D$, while the pseudoinverse
bound gives a polynomial lower bound on the smallest singular value.  Since the base
atlas on $N$ is fixed, integrating over a fixed positive coordinate volume gives
$S_{\gM_D}\le C_ND^{C_da}$ and $S_{\gM_D}^{-1}\le C_ND^{C_da}$.  The reach assumption gives
$\kappa_D^{-1}\le D^a$.

It remains to bound $\mathcal K_{\gM_D,q}$.  Choose a maximal
$c(1\wedge\kappa_D)$-separated set of centers on $\gM_D$.  Standard reach-volume
estimates bound the number of centers by a fixed power of
$(1\vee S_{\gM_D})(1\vee\kappa_D^{-1})$.  At a center
$\by=\Phi_D(u_0)$, use the tangent-projection map
\[
    u\mapsto P_{T_{\by}\gM_D}(\Phi_D(u)-\by)
\]
in a base chart.  Its derivative at $u_0$ has smallest singular value at least
$\|(\od\Phi_D)^\dagger\|^{-1}$, and its derivatives up to order $q_{\rm emb}$ are
polynomially bounded.  The reach chart lemma \cref{lem:projection-chart-reach} supplies the
reach-scale tangent-projection domains.  On those domains, the tangent variation bound
in \cref{eq:tangent-projection-singular-values} and the pseudoinverse hypothesis give a
polynomial lower bound on the singular values of the coordinate differential.  The
quantitative inverse-function theorem and the standard inverse-derivative formulas then
give coordinate maps, inverse maps, and transition maps with $C^{q+4}$ norms bounded by
$D^{C_{N,d,q}a}$.  A scaled bump construction on the same cover gives partition cutoffs
with the same type of bound.  Jacobian factors, Gram matrices, inverse Gram matrices,
and $\bA_i=\bG_i^{-1}\bJ_i^\top$ obey the same type of bound by repeated
differentiation of determinant and matrix-inverse formulas.

On the corresponding reach tube, the nearest projection is unique.  In a tangent graph
chart $\bz_i$, its local coordinate $\bu^\Pi(\bxi)$ solves
\[
    \bJ_i(\bu)^\top(\bxi-\bz_i(\bu))=0.
\]
The derivative of this equation with respect to $\bu$ equals the negative Gram matrix
plus a curvature-normal term.  The tangent graph construction and the restriction to a
fixed numerical fraction of the reach tube make this derivative uniformly invertible,
with inverse norm bounded by $D^{C_{N,d,q}a}$.  Applying the quantitative implicit
function theorem once more gives $C^{q+4}$ bounds, polynomial in $D$, for the
projection-coordinate and residual maps.  Therefore
$\mathcal K_{\gM_D,q}\le C_{N,d,q,a}D^{C_{N,d,q}a}$.

For (iv), the fixed embedding $\iota$ gives a fixed lower singular-value bound for
$\od\iota$.  Since
$\|\lambda_D(\od\iota,\od F_D)v\|_2\ge \lambda_D\|\od\iota\,v\|_2$ and
$\lambda_D\ge D^{-a}$, the pseudoinverse norm and its derivatives are polynomially
controlled by the $C^{q_{\rm emb}}$ bound on $F_D$ and matrix-inverse differentiation.
Thus the hypotheses of (iii) hold.
\end{proof}

\textbf{Global tube-volume bound.}
We use the following standard reach-volume consequence only for localization and geometry
bookkeeping.  With $\omega_k=\pi^{k/2}/\Gamma(k/2+1)$ and $\omega_0=1$, there is a
constant
$C_{\rm tube}\le C_d(1\vee S_{\gM})(1\vee\kappa^{-1})^d$
such that, for all $m\in(0,1]$ and $r>0$,
\begin{equation}
    \vol((m\gM)_r)
    \le
    C_{\rm tube}(1+r)^d\omega_{D-d}r^{D-d}.
    \label{eq:global-tube-volume-bound}
\end{equation}
This is the global tube bound used in the off-tube Gaussian tail estimates; the radius is
not required to be below the reach.

\clearpage

\section{Off-Tube Score Localization}\label{sec:app:approx}

% ======================================================================
% Score approximation on off-manifold region
% ======================================================================

\subsection{Off-manifold Gaussian tail and Fisher bounds}
\label{sec:app:approx:score:off-manifold}

\begin{lemma}[Gaussian smoothing off-manifold deviation bound]\label{lem:off-manifold:deviation-lemma}
    Let $\bY \sim P$ be supported on a $d$-dimensional smooth manifold $\gM \subseteq [0, 1]^D$.
    Fix $m > 0$ and $\sigma > 0$, and define $\bX = m\bY + \sigma\bZ$, where $\bZ \sim \gN(\bzero, \bI_D)$.
    For every $A>0$, there is a constant $C_{\rm off}(A)>0$ such that
    \begin{equation*}
        \Pr[\dist(\bX, m\gM) \geq C_{\rm off}(A)\sigma\sqrt{D\vee\log n}]
        \leq n^{-A}. 
    \end{equation*}
\end{lemma}
\begin{proof}
    In the following, we apply the Gaussian concentration inequality for Lipschitz functions to upper bound $\Pr[\dist(\bx, m\gM) \geq C_{\rm off}(A)\sigma\sqrt{D\vee\log n}]$.
    Specifically, let the random point be $\bx = m\by + \sigma\bz$, where $\by \in \gM$ and $\bz \sim \gN(\bzero, \bI_D)$.
    We can define a function of the standard Gaussian vector $\bz$ by:
    \begin{equation*} 
        h(\bz) = \dist(m\by + \sigma\bz, m\gM). 
    \end{equation*}
    
    Notice that the distance function $\dist(\cdot, \gM)$ is $1$-Lipschitz.
    Then for any $\bz_1, \bz_2 \sim \gN(\bzero, \bI_D)$, it holds that
    \begin{align*} 
        |h(\bz_1) - h(\bz_2)| 
        &= 
        |\dist(m\by + \sigma\bz_1, m\gM) - \dist(m\by + \sigma\bz_2, m\gM)|
        \\
        &\leq
        \|(m\by + \sigma\bz_1) - (m\by + \sigma\bz_2)\|_2 
        \\
        &=
        \|\sigma(\bz_1 - \bz_2)\|_2 
        = 
        \sigma\|\bz_1 - \bz_2\|_2, 
    \end{align*}
    which implies that $h(\bz)$ is $\sigma$-Lipschitz.
    Applying the Gaussian Concentration inequality to the $\sigma$-Lipschitz function $h(\bz)$, we get:
    \begin{equation*}
        \Pr[h(\bz) - \E[h(\bz)] \geq t] 
        \leq 
        \exp\Bigl(
          -\frac{t^2}{2\sigma^2}
        \Bigr).     
    \end{equation*}

    Since $m\by \in m\gM$, the distance from $m\by + \sigma\bz$ to the manifold is upper bounded by the distance to $m\by$ itself:
    \begin{equation*}
        h(\bz) = \dist(m\by + \sigma\bz, m\gM) \leq \|(m\by + \sigma\bz) - m\by\|_2 = \sigma\|\bz\|_2.
    \end{equation*}
    Using the standard bound for the expected norm of a Gaussian vector, $\E[\|\bz\|_2] \leq \sqrt{D}$, we obtain:
    \begin{equation*}
        \E[h(\bz)] \leq \sigma\sqrt{D}.
    \end{equation*}

    Set $r = C_{\rm off}(A)\sigma\sqrt{D\vee\log n}$.
    Taking $C_{\rm off}(A)$ sufficiently large depending only on $A$, we have $t := r - \E[h(\bz)] \geq c_{\rm off}(A)\sigma\sqrt{D\vee\log n}$ for a constant $c_{\rm off}(A)$ with $c_{\rm off}(A)^2/2\ge A$.
    Substituting this into the concentration bound, we obtain:
    \begin{equation*}
        \Pr[h(\bz) \geq C_{\rm off}(A)\sigma\sqrt{D\vee\log n}] 
        \le
        \exp\Bigl(
          -\frac{c_{\rm off}(A)^2(D\vee\log n)}{2}
        \Bigr)
        \leq  
        n^{-A}. 
    \end{equation*}
\end{proof}

\begin{lemma}[Off-manifold Fisher information bound]\label{lem:bound:off-manifold:fisher-info}
    Let $\bY \sim P$ be supported on a $d$-dimensional manifold $\gM \subseteq [0, 1]^D$.
    Fix $m>0$ and $\sigma>0$, and define $\bX=m\bY+\sigma\bZ$, where $\bZ\sim\gN(\bzero, \bI_D)$.
    Let $p_\sigma:\R^D\to\R_+$ be the density of $\bX$.
    For every desired polynomial tail exponent $A>0$, there is a threshold constant $C_{\rm off}(A)>0$ such that
    \begin{align*}
        &
        \int_{\dist(\bx, m\gM) \geq C_{\rm off}(A)\sigma\sqrt{D\vee\log n}}
          \Big\|
            \frac{\nabla p_\sigma(\bx)}{p_\sigma(\bx)}
          \Big\|_2^2
          p_\sigma(\bx)
        \od\bx
        \lesssim 
        \frac{D}{\sigma^2}
        n^{-A}. 
    \end{align*}
\end{lemma}

\begin{proof}
    Fix $A>0$ and let $C_{\rm off}(A)$ be chosen at the end of the proof.
    By $\bX = m\bY + \sigma\bZ$, we can apply Tweedie's formula to express the true score:
    \begin{equation*}
        \frac{\nabla p_\sigma(\bx)}{p_\sigma(\bx)}
        =
        \frac{1}{\sigma^2} \E[m\bY - \bX | \bX = \bx]
        =
        -\frac{1}{\sigma} \E[\bZ | \bX = \bx].
    \end{equation*}
    
    Consequently, the squared norm evaluates to:
    \begin{align}
        {} &
        \int_{\dist(\bx, m\gM) \geq C_{\rm off}(A)\sigma\sqrt{D\vee\log n}}
          \Bigl\|
            \frac{\nabla p_\sigma(\bx)}{p_\sigma(\bx)}
          \Bigr\|_2^2
          p_\sigma(\bx)
        \od\bx 
        \notag \\
        &=
        \int_{\dist(\bx, m\gM) \geq C_{\rm off}(A)\sigma\sqrt{D\vee\log n}}
          \frac{1}{\sigma^2}
          \bigl\| \E[\bZ | \bX = \bx] \bigr\|_2^2
          p_\sigma(\bx)
        \od\bx 
        \tag{by Tweedie's formula} \\
        &=
        \frac{1}{\sigma^2}
        \E\Bigl[
          \|\E[\bZ|\bX]\|_2^2
          \ind\{\dist(\bX, m\gM) \geq C_{\rm off}(A)\sigma\sqrt{D\vee\log n}\}
        \Bigr] 
        \tag{by definition of expectation} \\
        &\leq
        \frac{1}{\sigma^2}
        \E\Bigl[
          \E[\|\bZ\|_2^2|\bX]
          \ind\{\dist(\bX, m\gM) \geq C_{\rm off}(A)\sigma\sqrt{D\vee\log n}\}
        \Bigr]
        \tag{by Jensen's inequality} \\
        &\leq
        \frac{1}{\sigma^2}
        \sqrt{\E[\|\bZ\|_2^4]
        \Pr[\dist(\bX, m\gM) \geq C_{\rm off}(A)\sigma\sqrt{D\vee\log n}]
        }
        \tag{by Cauchy-Schwarz} \\
        &\leq
        \frac{1}{\sigma^2}
        \sqrt{(D^2 + 2D)
        \Pr[\dist(\bX, m\gM) \geq C_{\rm off}(A)\sigma\sqrt{D\vee\log n}]
        }
        \tag{by $\E[\|\bZ\|_2^4] = D^2 + 2D$} \\
        &\lesssim
        \frac{D}{\sigma^2}
        \sqrt{
        \Pr[\dist(\bX, m\gM) \geq C_{\rm off}(A)\sigma\sqrt{D\vee\log n}]
        }.
        \label{eq:tail:manifold:1}
    \end{align}
    
    Given the exponent $A>0$ in the statement, choose the threshold $C_{\rm off}(A)$ large enough so that \cref{lem:off-manifold:deviation-lemma} applies with exponent $2A$.
    Then
    \[
        \Pr[\dist(\bX,m\gM)\ge C_{\rm off}(A)\sigma\sqrt{D\vee\log n}]
        \le n^{-2A}.
    \]
    Combining this with \cref{eq:tail:manifold:1} gives
    \begin{equation*}
        \int_{\dist(\bx, m\gM) \geq C_{\rm off}(A)\sigma\sqrt{D\vee\log n}}
          \Bigl\|
            \frac{\nabla p_\sigma(\bx)}{p_\sigma(\bx)}
          \Bigr\|_2^2
          p_\sigma(\bx)
        \od\bx
        \lesssim 
        \frac{D}{\sigma^2}
        n^{-A}.
    \end{equation*}
\end{proof}

\begin{lemma}[Off-manifold clipped-score error bound]\label{lem:approx:score:off-manifold}
    Suppose \cref{assump:manifold:exact} holds.
    Fix $\ell_{\rm off}\ge1$.
    For every $A>0$, let $\varrho_t=C_\varrho(A)\sigma_t\sqrt{D\vee\log n}$, with
    $C_\varrho(A)$ chosen large enough for polynomial tail exponent $A$.
    For any measurable function $\widetilde{s}: \R^D \times [\tdown, \tup] \to \R^D$ satisfying
    \[
        \sup_{\bx\in\R^D}\|\widetilde{s}(\bx,t)\|_\infty
        \lesssim
        \sigma_t^{-1}\sqrt{D\ell_{\rm off}},
    \]
    it holds that
    \[
        \int_{\dist(\bx_t,m_t\gM)\ge\varrho_t}
          \|\widetilde{s}(\bx_t,t)-\nabla\log p_t(\bx_t)\|_2^2
          p_t(\bx_t)\od\bx_t
        \lesssim
        \bigl(D+D^2\ell_{\rm off}\bigr)\sigma_t^{-2}n^{-A} .
    \]
\end{lemma}

\begin{proof}
Let $\ell_{\rm off}\ge1$ be the logarithmic factor in the coordinatewise bound of the lemma.
Fix $\varrho_t = C_\varrho(A)\sigma_t\sqrt{D\vee\log n}$ with $C_\varrho(A)$ chosen large enough for the polynomial tail exponent required below.
By \cref{lem:bound:off-manifold:fisher-info} it holds that
\begin{align*}
    &
    \int_{\dist(\bx_t, m_t\gM) \geq \varrho_t}
      \Big\|
        \frac{\nabla p_t(\bx_t)}{p_t(\bx_t)}
      \Big\|_2^2
      p_t(\bx_t)
    \od\bx_t
    \lesssim
    \frac{D}{\sigma_t^2}
    n^{-A}. 
\end{align*}

By \cref{lem:off-manifold:deviation-lemma} and the coordinatewise uniform assumption $\sup_{\bx\in\R^D}\|\widetilde{s}(\bx, t)\|_\infty \lesssim \sigma_t^{-1}\sqrt{D\ell_{\rm off}}$, we have
\begin{align*}
    \int_{\dist(\bx_t, m_t\gM) \geq \varrho_t}
      \|\widetilde{s}(\bx_t, t)\|_2^2
      p_t(\bx_t)
    \od\bx_t
    &\leq
    D\Bigl(\sup_{\bx\in\R^D}\|\widetilde{s}(\bx, t)\|_\infty\Bigr)^2
    \cdot
    \Pr\bigl[
      \dist(\bX_t, m_t\gM) \geq \varrho_t
    \bigr]
    \\
    &\lesssim  
    \sigma_t^{-2}D^2\ell_{\rm off}
    n^{-A}. 
\end{align*}

Using the triangle inequality $\|a - b\|_2^2 \leq 2\|a\|_2^2 + 2\|b\|_2^2$, we combine these bounds:
\begin{align*}
    {} &
    \int_{\dist(\bx_t, m_t\gM) \geq \varrho_t}
      \|\nabla \log p_t(\bx_t) - \widetilde{s}(\bx_t, t)\|_2^2
      p_t(\bx_t)
    \od\bx_t
    \\
    &\lesssim
    \int_{\dist(\bx_t, m_t\gM) \geq \varrho_t}
      \Big\|
        \frac{\nabla p_t(\bx_t)}{p_t(\bx_t)}
      \Big\|_2^2
      p_t(\bx_t)
    \od\bx_t
    +
    \int_{\dist(\bx_t, m_t\gM) \geq \varrho_t}
      \|\widetilde{s}(\bx_t, t)\|_2^2
      p_t(\bx_t)
    \od\bx_t
    \notag \\
    &\lesssim
    \frac{D}{\sigma_t^2} n^{-A}
    +
    \frac{D^2\ell_{\rm off}}{\sigma_t^2} n^{-A}
    \lesssim 
    \frac{D+D^2\ell_{\rm off}}{\sigma_t^2}
    n^{-A}. 
\end{align*}
\end{proof}

\clearpage

% ======================================================================
% Centered score approximations on near-manifold
% ======================================================================

\section{Setup for Branchwise Score Approximation}\label{sec:app:approx:score:regularized:near-manifold}

To approximate the score near the forward-smeared manifold under a density lower bound, we resolve the geometry of the underlying data manifold $\gM$ and approximate the centered numerator ratio.
The setup section fixes the noise regimes, geometric charts, amplitudes, and chart-integral notation.
The large-noise component uses the tangent-cell construction in \cref{sec:app:approx:PQ:large-noise-branch}.
The small-noise component uses the unified projection--Laplace pipeline in
\cref{sec:app:small-noise-score-approx}; its internal subsections isolate active
geometry, finite-anchor projection networks, the Laplace expansion, and final
ratio-network assembly.
The learned-score section combines these two deterministic branches directly.
For readability, this appendix is best read in three passes.  First read the definitions
of $\gQ,\gP,\gU$ in \cref{sec:app:approx:manifold:integrals} and the final branch
statements \cref{thm:large-noise-hd-score-global,cor:small-noise-density-lower-approx}.
Second read the proof-architecture paragraphs at the starts of
\cref{sec:app:approx:PQ:large-noise-branch,sec:app:small-noise-score-approx}.
Finally consult the supporting lemmas for the particular denominator, projection, or
network-size estimate used in the main text.
For the VP forward process considered here, $\bX_t=m_t\bX_0+\sigma_t\bZ$, the deterministic schedules satisfy $m_t\le1$ and $\sigma_t\le1$.
On the truncated time interval $[\tdown, \tup]$ we assume a positive lower bound for $m_t$, and we use the exact infimum of the noise schedule as the noise lower bound:
\begin{equation}
    0 < \underline{m} \leq m_t \leq \overline{m}\leq 1,
    \qquad
    \sigma_{\tdown}:=\inf_{t\in[\tdown,\tup]}\sigma_t,
    \qquad
    0<\sigma_{\tdown}\le\sigma_t\le\overline{\sigma}\le1.
    \label{eq:def:sigma-tdown-lower}
\end{equation}
If the noise schedule is monotone increasing, this notation agrees with the endpoint value $\sigma(\tdown)$.
All schedule-approximation size bounds below are stated for the fixed compact time interval $[\tdown,\tup]$; logarithmic factors in $\tdown^{-1}$ are absorbed into the constants unless explicitly displayed.

\subsection{Noise regimes and branch-scale validity}
\label{sec:app:approx:schedule-regimes}

\begin{definition}[Branch-scale small- and large-noise validity regions]
\label{def:noise-regimes}
Fix constants $C_{\rm sm}\in(0,1]$ and $c_{\rm lg}>0$ satisfying
\[
    4c_{\rm lg}\le C_{\rm sm}.
\]
Their final numerical values are chosen after the uniform Laplace remainder constant is fixed; see \cref{eq:choice:C-sm-Laplace}.
For every $0<\varepsilon<1$, define the small-noise validity region by
\begin{equation}
    \gI_{\rm sm}(\varepsilon)
    :=
    \{t \in [\tdown, \tup]:
    \sigma_t \leq C_{\rm sm}\varepsilon^{1/\beta}\}.
    \label{eq:def:small-noise-region}
\end{equation}
and the large-noise validity region by
\begin{equation}
    \gI_{\rm lg}(\varepsilon)
    :=
    \{t\in[\tdown,\tup]:
      \sigma_t\ge c_{\rm lg}\varepsilon^{1/\beta}\}.
    \label{eq:def:large-noise-region}
\end{equation}
These auxiliary validity regions need not be disjoint.
The parameter $\varepsilon$ is a branch-scale parameter; target tolerances are named by role, such as coordinate accuracy $\varepsilon_u$, block accuracy $\varepsilon_{\rm blk}$, and scaled score accuracy $\varepsilon_{\rm sc}$.
In the final large-noise application, the branch scale is $n^{-\beta/(d+2\beta)}$, so the validity condition is $\sigma_t\gtrsim n^{-1/(d+2\beta)}$.
\end{definition}

\subsection{Branchwise approximation roadmap}
\label{sec:app:approx:score:proof-strategy}

There are two approximation mechanisms under \cref{assump:manifold:density-lower}.
On large-noise slabs, the denominator $\gQ$ has an explicit lower bound on the near-manifold tube, so the tangent-cell moment approximation can be used directly and the network size is driven by the intrinsic cell count.
On the small-noise interval, the denominator is controlled after de-Gaussianization, and the proof uses a projection-centered Laplace expansion together with neural approximation of nearest-point projection maps.

For navigation, the formal dependency structure is:
\begin{center}
{\small
\setlength{\tabcolsep}{4pt}
\begin{tabular}{p{0.22\linewidth}|p{0.7\linewidth}}
Component & Formal role \\
\hline
\cref{sec:app:approx:manifold:setup} &
Build the projection atlas, partition of unity, chart integrals $\gQ,\gP,\gU$, and the centered ratio $s^*=\sigma_t^{-1}\gU/\gQ$. \\
\hline
\cref{sec:app:approx:PQ:large-noise-branch} &
Approximate $\gQ$ and $\gU$ on tangent cells at radius $r_k=c_0\sigma_{t_{k-1}}\wedge r_{\star,n}$ under a direct tube lower bound for $\gQ$, with $r_{\star,n}$ defined in \cref{prop:large-noise-oracle-learned-score}. \\
\hline
\cref{thm:large-noise-hd-score-global} &
Convert the large-noise chart-cell approximation and denominator floor into a global clipped score approximation on each large-noise slab. \\
\hline
\cref{sec:app:small-noise-score-approx} &
Develop the active projection geometry, Gauss--Newton projection networks, Laplace expansion, and final ratio-network assembly needed at small noise. \\
\hline
\cref{lem:small-noise-degaussianized-Q-lower,lem:H-localized-degaussianized-small-noise-chart} &
Remove the common normal Gaussian factor, prove a polynomial denominator floor, and approximate the de-Gaussianized chart integrals. \\
\hline
\cref{cor:small-noise-density-lower-approx} &
Assemble the small-noise projection--Laplace networks into the final clipped score approximation on $[t_0,t_1]$. \\
\hline
\cref{sec:app:score-estimation,sec:app:dist-est} &
Insert the deterministic comparison networks into the clipped ERM oracle and then into the slabwise $\sfW_1$ perturbation argument.
\end{tabular}
}
\end{center}

\subsection{Projection atlas and chart-integral decomposition}\label{sec:app:approx:manifold:setup}

\subsubsection{Positive-reach projection charts}\label{sec:app:approx:manifold:covering-ball}

Let $\gB(\by, r)$ denote the open Euclidean ball with radius $r > 0$ and center $\by \in \R^D$.
We will use projection charts whose validity follows from positive reach.

\begin{lemma}[Projection charts from positive reach]\label{lem:projection-chart-reach}
    Let \cref{assump:manifold:exact} hold.
    There is a numerical constant $c_{\rm ch}\in(0,1/8]$, depending only on $d$,
    such that the chart radius may be chosen as
    \[
        r_{\rm ch}:=c_{\rm ch}(1\wedge\kappa).
    \]
    This radius has the following property.
    For every $\by\in\gM$, if $\bV_{\by}\in\R^{D\times d}$ has orthonormal columns spanning $T_{\by}\gM$, then there is an open neighborhood $U_{\by}\subset\gM$ satisfying
    \begin{equation*}
        \gM\cap\gB(\by,r_{\rm ch}/2)
        \subset
        U_{\by}
        \subset
        \gM\cap\gB(\by,r_{\rm ch}),
    \end{equation*}
    such that the tangent projection
    \begin{equation*}
        \pi_{\by}:U_{\by}\to \bV_{\by}^{\top}(U_{\by}-\by)\subset\R^d,
        \qquad
        \pi_{\by}(\bz):=\bV_{\by}^{\top}(\bz-\by),
    \end{equation*}
    is a $C^\infty$ diffeomorphism onto its image.
    Moreover, for every $\bz\in U_{\by}$ and every $\bv\in T_{\bz}\gM$,
    \begin{equation}
        (1-2r_{\rm ch}/\kappa)\|\bv\|_2^2
        \leq
        \|\bV_{\by}^{\top}\bv\|_2^2
        \leq
        \|\bv\|_2^2 .
        \label{eq:tangent-projection-singular-values}
    \end{equation}
\end{lemma}

\begin{proof}
    We use the standard local-graph theorem for manifolds of positive reach.
    For a closed embedded $C^2$ manifold of reach at least $\kappa$, Federer \citep[Theorem~4.8(12)]{federer1959curvature} gives the tubular-neighborhood property for normal segments of length below $\kappa$, hence the usual tangent-ball exclusion, while \citep[Theorem~4.8(7)]{federer1959curvature} gives the quadratic distance-to-tangent-plane bound.
    The tangent variation estimate recorded for instance in \citep[Corollary~3]{boissonnat2019reach}, together with the local inverse theorem, then gives a numerical $c_{\rm ch}\le 1/8$, depending only on $d$, such that $\gM\cap\gB(\by,c_{\rm ch}\kappa)$ is a $C^\infty$ graph over $T_{\by}\gM$.
    If $\kappa\ge1$, the smaller radius $c_{\rm ch}$ is also admissible.
    Thus $r_{\rm ch}=c_{\rm ch}(1\wedge\kappa)$ is admissible uniformly in $\by$.
    Equivalently, the orthogonal projection onto $T_{\by}\gM$ is one-to-one on this patch and is a $C^\infty$ diffeomorphism onto its image.
    Thus there is an open set $U_{\by}\subset\gM$ with
    \[
        \gM\cap\gB(\by,r_{\rm ch}/2)\subset U_{\by}
        \subset \gM\cap\gB(\by,r_{\rm ch})
    \]
    such that $\pi_{\by}$ is a $C^\infty$ diffeomorphism onto its image.

    It remains to record the quantitative derivative bound used below.
    A standard consequence of reach is the tangent-space variation estimate \citep[Corollary~3]{boissonnat2019reach}
    \begin{equation*}
        \|P_{T_{\bz}\gM}-P_{T_{\by}\gM}\|_{\rm op}
        \leq
        \frac{2\|\bz-\by\|_2}{\kappa},
        \qquad
        \|\bz-\by\|_2<\kappa/2,
    \end{equation*}
    where $P_T$ denotes the orthogonal projection onto a subspace $T$.
    If $\bz\in U_{\by}$ and $\bv\in T_{\bz}\gM$ with $\|\bv\|_2=1$, then
    \[
        \|P_{(T_{\by}\gM)^\perp}\bv\|_2
        =
        \|(P_{T_{\bz}\gM}-P_{T_{\by}\gM})\bv\|_2
        \le
        \|P_{T_{\bz}\gM}-P_{T_{\by}\gM}\|_{\rm op}.
    \]
    Because $r_{\rm ch}\le\kappa/8$, this last quantity is at most $1$.
    Hence
    \begin{align*}
        \|\bV_{\by}^{\top}\bv\|_2^2
        &=
        \|P_{T_{\by}\gM}\bv\|_2^2
        =
        1-\|P_{(T_{\by}\gM)^\perp}\bv\|_2^2
        \\
        &\geq
        1-\|P_{T_{\bz}\gM}-P_{T_{\by}\gM}\|_{\rm op}
        \geq
        1-\frac{2r_{\rm ch}}{\kappa}.
    \end{align*}
    The upper bound in \cref{eq:tangent-projection-singular-values} is immediate because $\bV_{\by}^{\top}$ is an orthogonal projection.
    Homogeneity gives the bound for general $\bv$.
\end{proof}

By compactness, choose points $\by_1,\dots,\by_{C_{\gM}}\in\gM$ such that
\begin{equation*}
    \gM
    \subset
    \bigcup_{i=1}^{C_{\gM}}
    \bigl(\gM\cap\gB(\by_i,r_{\rm ch}/2)\bigr).
\end{equation*}
Let $U_i:=U_{\by_i}$ and let $\bV_i:=\bV_{\by_i}$.
The covering number may be chosen to satisfy the standard reach-volume bound
\begin{equation}
    C_{\gM}
    \leq
    C_d\,S_{\gM}\,r_{\rm ch}^{-d}
    \le
    C_d\,S_{\gM}\,(1\vee\kappa^{-1})^d,
    \label{eq:manifold:CM}
\end{equation}
where $C_d$ depends only on $d$.
Indeed, take a maximal $r_{\rm ch}/4$-separated subset of $\gM$; the balls $\gM\cap \gB(\by_i,r_{\rm ch}/2)$ cover $\gM$, while the intrinsic volumes of the disjoint smaller balls are bounded below by $c_d r_{\rm ch}^d$ because the second fundamental form is bounded by $\kappa^{-1}$ on the reach scale.

For any $\by \in U_i$, define the chart by applying the tangent projection from \cref{lem:projection-chart-reach} and then rescaling and translating the resulting coordinate image:
\begin{equation}
    \phi_i(\by) := \eta_i\bigl(\bV_i^\top(\by - \by_i) + \bc_i\bigr), 
    \label{eq:approx:proj:phi-i}
\end{equation}
where $\eta_i \in (0, 1]$ and $\bc_i \in \R^d$ are chosen so that $\phi_i(U_i) \subseteq [0, 1]^d$.
Indeed,
\[
    \phi_i(\by)=\eta_i\bigl(\pi_{\by_i}(\by)+\bc_i\bigr),
    \qquad
    \pi_{\by_i}(\by)=\bV_i^\top(\by-\by_i).
\]
Since $\pi_{\by_i}$ is a diffeomorphism on $U_i$ by \cref{lem:projection-chart-reach}, and the map $\bu\mapsto\eta_i(\bu+\bc_i)$ is an invertible affine map of $\R^d$, so is $\phi_i$.
Note that each $\phi_i$ is the restriction of an affine map on $\R^D$, and therefore can be exactly represented by a one-hidden-layer ReLU network.

Let
\begin{equation*}
    \phi_i^{-1}: \phi_i(U_i) \to U_i \subset \gM.
\end{equation*}
The following proposition establishes a uniform bound on the volume distortion within each local chart:
\begin{proposition}[Uniform Jacobian bounds for projection charts]\label{prop:manifold:det-Jacob-proj}
    Under \cref{assump:manifold:exact}, the Jacobian determinant of the inverse chart $\phi_i^{-1}$ satisfies, for every $\bu\in\phi_i(U_i)$,
    \begin{equation}
        \sqrt{\det(\nabla\phi_i^{-1}(\bu)^\top\nabla\phi_i^{-1}(\bu))} 
        \leq 
        \frac{1}{\eta_i^d}
        \Bigl(
          1 - \frac{2r_{\rm ch}}{\kappa}
        \Bigr)^{-d/2}.
        \label{eq:chartwise-Jacobian-bound}
    \end{equation}
\end{proposition}

\begin{proof}
    Fix $i$ and write $\bz=\phi_i^{-1}(\bu)$.
    The derivative of $\phi_i$ at $\bz$, restricted to $T_{\bz}\gM$, is
    \begin{equation*}
        \oD\phi_i(\bz)|_{T_{\bz}\gM}
        =
        \eta_i\,\bV_i^\top|_{T_{\bz}\gM}.
    \end{equation*}
    By \cref{eq:tangent-projection-singular-values}, every singular value of $\bV_i^\top|_{T_{\bz}\gM}$ is at least $(1-2r_{\rm ch}/\kappa)^{1/2}$ and at most $1$.
    Hence every singular value of $\oD\phi_i(\bz)|_{T_{\bz}\gM}$ is at least $\eta_i(1-2r_{\rm ch}/\kappa)^{1/2}$.
    Since $\nabla\phi_i^{-1}(\bu)$ is the inverse linear map from $\R^d$ to $T_{\bz}\gM$, its singular values are bounded above by $\eta_i^{-1}(1-2r_{\rm ch}/\kappa)^{-1/2}$.
    Multiplying the $d$ singular values gives \cref{eq:chartwise-Jacobian-bound}.
\end{proof}

We define
\begin{equation}
    B^{\rm Jac} := \sup_i\eta_i^{-d}(1-2r_{\rm ch}/\kappa)^{-d/2}.     
\end{equation}

\subsubsection{Partition-of-unity amplitudes and density regularity}\label{sec:app:approx:manifold:density}

Let $\{\rho_i\}_{i=1}^{C_{\gM}}$ be a $C^\infty$ partition of unity subordinate to the finite projection atlas $\{U_i\}_{i=1}^{C_{\gM}}$.
After a finite refinement of the projection cover, we choose it with compactly supported chart pieces
\begin{equation*}
    \supp_{\gM}(\rho_i)\Subset U_i,
    \qquad
    i=1,\dots,C_{\gM},
\end{equation*}
which is possible because the smaller sets $\gM\cap\gB(\by_i,r_{\rm ch}/2)$ still cover the compact manifold and each is compactly contained in $U_i$.
Let
\begin{equation}
    \gS_i := \supp_{\gM}(\rho_i) \Subset U_i, 
    \qquad
    i = 1, \dots, C_{\gM},
    \label{eq:def:domain:S_i}
\end{equation}
be the compact support of the partition of unity on chart $i$.
The refinement is chosen so that, for every resulting index, the coordinate support $\phi_i(\gS_i)$ is contained in nested axis-parallel coordinate boxes compactly contained in $\phi_i(U_i)$.
We fix real numbers
\[
    a_{i,r}^{\Box}
    <
    a_{i,r}^{\circ}
    <
    b_{i,r}^{\circ}
    <
    b_{i,r}^{\Box},
    \qquad r=1,\dots,d,
\]
so that, with
\begin{equation}
    K_i^\circ
    :=
    \prod_{r=1}^d
    [a_{i,r}^{\circ},b_{i,r}^{\circ}],
    \qquad
    \Box_i
    :=
    \prod_{r=1}^d
    [a_{i,r}^{\Box},b_{i,r}^{\Box}],
    \label{eq:def:buffered-coordinate-boxes}
\end{equation}
we have
\begin{equation}
    \phi_i(\gS_i)
    \subset
    \operatorname{int}(K_i^\circ)
    \subset
    K_i^\circ
    \Subset
    \operatorname{int}(\Box_i)
    \subset
    \Box_i
    \Subset
    \phi_i(U_i).
    \label{eq:def:Ki-circ}
\end{equation}
These fixed boxes are used both by the tangent-cell construction and, under the density lower bound, by the projection-centered Laplace expansion.
This only changes the finite atlas cardinality by a geometry-dependent factor.
Since $\supp(\rho_i)\subset U_i$, we may insert $1=\sum_{i=1}^{C_{\gM}}\rho_i$ into the manifold integral and then change variables chart by chart.
For each chart $i$, define the pulled-back partition weight
\begin{equation}
    \bar{\rho}_i(\bu):=(\rho_i\circ \phi_i^{-1})(\bu),
    \qquad
    \bu\in \phi_i(U_i).
    \label{eq:def:rho-bar}
\end{equation}

For each chart $i \in \{1, \dots, C_{\gM}\}$, define $\bz_i: \phi_i(U_i) \to U_i, \bJ_i: \phi_i(U_i) \to \R^{D \times d}, \bG_i: \phi_i(U_i) \to \R^{d \times d}$:
\begin{align}
    \bz_i(\bu) &:= \phi_i^{-1}(\bu) \in U_i \subset \gM,
    \label{eq:def:z-i}
    \\
    \bJ_i(\bu) &:= \nabla \bz_i(\bu)\in \R^{D\times d},
    \label{eq:def:J-i}
    \\
    \bG_i(\bu) &:= \bJ_i(\bu)^\top \bJ_i(\bu)\in \R^{d\times d}. 
    \label{eq:def:G-i}
\end{align}

Then, the density of $\bX_t = m_t\bX_0 + \sigma_t\bZ, \bZ \sim \gN(\bzero, \bI_D)$ is
\begin{align}
    p_t(\bx)
    &=
    \frac{1}{(2\pi\sigma_t^2)^{D/2}}
    \int_{\gM}
      \exp\Bigl(
        -\frac{\|\bx - m_t\by\|_2^2}{2\sigma_t^2}
      \Bigr)
      p_0(\by)
    \od\vol_{\gM}(\by) 
    \notag \\
    &= 
    \frac{1}{(2\pi\sigma_t^2)^{D/2}}
    \sum_{i=1}^{C_{\gM}}
    \int_{\gM}
      \rho_i(\by)
      \exp\Bigl(
        -\frac{\|\bx - m_t\by\|_2^2}{2\sigma_t^2}
      \Bigr)
      p_0(\by)
    \od\vol_{\gM}(\by) 
    \notag \\
    &=
    \frac{1}{(2\pi\sigma_t^2)^{D/2}}
    \sum_{i=1}^{C_{\gM}}
    \int_{\phi_i(U_i)}
      \bar{\rho}_i(\bu)
      \exp\Bigl(
        -\frac{\|\bx - m_t\phi_i^{-1}(\bu)\|_2^2}{2\sigma_t^2}
      \Bigr)
      p_0(\phi_i^{-1}(\bu))
      \sqrt{(\det\nabla\phi_i^{-1}(\bu)^\top\nabla\phi_i^{-1}(\bu))}
    \od\bu
    \notag \\
    &=
    \frac{1}{(2\pi\sigma_t^2)^{D/2}}
    \sum_{i=1}^{C_{\gM}}
    \int_{\phi_i(U_i)}
      \bar{\rho}_i(\bu)
      \exp\Bigl(
        -\frac{\|\bx - m_t\bz_i(\bu)\|_2^2}{2\sigma_t^2}
      \Bigr)
      p_0(\bz_i(\bu))
      \sqrt{\det\bG_i(\bu)}
    \od\bu
    \label{eq:def:manifold:density}
\end{align}

\begin{lemma}[H\"older regularity in projection coordinates]
\label{lem:holder-transfer-projection-atlas}
    Assume $p_0\in\gH^\beta(\gM,B_0)$ in the sense of \cref{def:manifold:Holder-smooth}.
    Then each projection-coordinate representative $p_0\circ\bz_i$ belongs to $\gH^\beta(\phi_i(U_i))$, and
    \begin{equation*}
        \max_{1\le i\le C_{\gM}}
        \|p_0\circ\bz_i\|_{\gH^\beta(\phi_i(U_i))}
        <\infty .
    \end{equation*}
    The resulting bound depends on $B_0$, $\beta$, and finitely many $C^{\ceil{\beta}}$-norms of transition maps between the original smooth atlas and the projection atlas.
\end{lemma}

\begin{proof}
    Fix a projection chart $(U_i,\phi_i)$.
    Since $\overline{U_i}$ is compact in $\gM$, it is covered by finitely many relatively compact chart neighborhoods $V_\alpha'\Subset V_\alpha$, where $(V_\alpha,\psi_\alpha)$ are charts from the atlas used to define $\gH^\beta(\gM,B_0)$.
    Let $W_{i,\alpha}:=U_i\cap V_\alpha'$.
    On $\phi_i(W_{i,\alpha})$,
    \begin{equation*}
        p_0\circ\bz_i
        =
        (p_0\circ\psi_\alpha^{-1})
        \circ
        (\psi_\alpha\circ\bz_i).
    \end{equation*}
    The first factor has $\gH^\beta$-norm at most $B_0$ by assumption.
    The transition map $\psi_\alpha\circ\bz_i$ is $C^\infty$, and all of its derivatives up to order $\ceil{\beta}$ are bounded on the compact closure of $\phi_i(W_{i,\alpha})$.
    The standard composition theorem for H\"older classes, proved by the chain rule for derivatives up to order $\floor{\beta}$ and the H\"older bound for the top derivative, gives
    \begin{equation*}
        \|p_0\circ\bz_i\|_{\gH^\beta(\phi_i(W_{i,\alpha}))}
        \le
        C_{i,\alpha}B_0 ,
    \end{equation*}
    where $C_{i,\alpha}$ depends only on the stated transition-map norms.

    To pass from these local estimates to the full coordinate domain $\phi_i(U_i)$, choose a smooth partition of unity $\{\chi_{i,\alpha}\}_\alpha$ on $U_i$ subordinate to the finite cover $\{W_{i,\alpha}\}_\alpha$, and write
    \[
        p_0\circ\bz_i
        =
        \sum_\alpha
        \bigl((\chi_{i,\alpha}p_0)\circ\bz_i\bigr).
    \]
    Each summand has compact support in $\phi_i(W_{i,\alpha})$.
    Extending it by zero to $\phi_i(U_i)$ preserves its $\gH^\beta$-norm up to a constant depending only on the cutoff derivatives and the positive buffer between $\supp\chi_{i,\alpha}$ and $U_i\setminus W_{i,\alpha}$.
    The algebra property of H\"older spaces controls multiplication by $\chi_{i,\alpha}\circ\bz_i$.
    Since the number of subpatches is finite, summing the extended pieces gives a finite $\gH^\beta(\phi_i(U_i))$ bound.
    Taking the maximum over the finite projection atlas gives the claim.
\end{proof}

By \cref{lem:holder-transfer-projection-atlas}, the pulled-back density has a finite chartwise H\"older norm.
We write
\begin{equation}
    B_p^{\rm proj}
    :=
    \max_{1\le i\le C_{\gM}}
    \|p_0\circ\bz_i\|_{\gH^\beta(\phi_i(U_i))}
    <\infty .
    \label{eq:def:Bp-proj}
\end{equation}
The constant $B_p^{\rm proj}$ depends on $B_0$, $\beta$, and the finite atlas transition maps; it is generally not equal to $B_0$.

Similarly to $p_t(\bx)$, the gradient of the density function is given by
\begin{align}
    &
    \nabla p_t(\bx)
    =
    \frac{1}{(2\pi\sigma_t^2)^{D/2}}
    \int_{\gM}
      \exp\Bigl(
        -\frac{\|\bx - m_t\by\|_2^2}{2\sigma_t^2}
      \Bigr)
      \Bigl(
        -\frac{\bx-m_t\by}{\sigma_t^2}
      \Bigr)
      p_0(\by)
    \od\vol_{\gM}(\by) 
    \notag \\
    &=
    \frac{1}{(2\pi\sigma_t^2)^{D/2}}
    \sum_{i=1}^{C_{\gM}}
    \int_{\phi_i(U_i)}
      \bar{\rho}_i(\bu)
      \exp\Bigl(
        -\frac{\|\bx - m_t\bz_i(\bu)\|_2^2}{2\sigma_t^2}
      \Bigr)
      \Bigl(
        -\frac{\bx-m_t\bz_i(\bu)}{\sigma_t^2}
      \Bigr)
      p_0(\bz_i(\bu))
      \sqrt{\det\bG_i(\bu)}
    \od\bu. 
    \label{eq:def:manifold:grad-density}
\end{align}

Define $a_i^{(\gQ)}: \phi_i(U_i) \to \R$ and $\ba_i^{(\gP)}: \phi_i(U_i) \to \R^D$:
\begin{equation}
    a_i^{(\gQ)}(\bu)
    := 
    \bar{\rho}_i(\bu)\sqrt{\det\bG_i(\bu)}p_0(\bz_i(\bu)),
    \qquad
    \ba_i^{(\gP)}(\bu)
    :=
    \bz_i(\bu)a_i^{(\gQ)}(\bu). 
    \label{eq:def:a-i}
\end{equation}
For $k=1,\dots,D$, write $a_i^{(\gP_k)}:=a_{i,k}^{(\gP)}$, and define the integral-type set
\[
    \mathcal R:=\{\gQ,\gP_1,\dots,\gP_D\}.
\]
Thus $R\in\mathcal R$ denotes a type, not the already summed global integral; $a_i^{(R)}$ denotes $a_i^{(\gQ)}$ when $R=\gQ$, and $a_i^{(\gP_k)}$ when $R=\gP_k$.

The following lemma shows that $a_i^{(\gQ)}$ and $\ba_i^{(\gP)}$ are $\beta$-H\"older smooth on $\phi_i(U_i)$:
\begin{lemma}[Chart-amplitude H\"older regularity]\label{lem:a-i:regularity}
    Assume \cref{assump:manifold:exact,assump:manifold:density} hold.
    There is a finite constant $B_a$ such that
    \begin{equation}
        a_i^{(\gQ)}, a_{i,k}^{(\gP)} 
        \in 
        \gH^\beta(\phi_i(U_i), B_a),
        \qquad k = 1, \dots, D. 
        \label{eq:a-i:holder}
    \end{equation}
    One may take
    \begin{equation*}
        B_a = 
        c_{\beta}B_p^{\rm proj}(1\vee B^{\rm Jac})
        \Biggl(1\vee
        \max_{i, k} \max_{\|\balpha\|_1 \leq \ceil{\beta}} 
        \Bigl\{ 
          \|\partial^{\balpha}\bar{\rho}_i\|_\infty 
          \vee 
          \|\partial^{\balpha}\sqrt{\det\bG_i}\|_\infty 
          \vee 
          \|\partial^{\balpha}z_{i,k}\|_\infty 
        \Bigr\}
        \Biggr). 
    \end{equation*}
    
    Moreover, because the manifold is embedded in $[0,1]^D$, we have the coordinate-wise bounds:
    \begin{equation}
        |a_{i,k}^{(\gP)}(\bu)| \leq a_i^{(\gQ)}(\bu)
        \leq 
        B_p^{\rm proj}B^{\rm Jac}. 
        \label{eq:a-i:bound}
    \end{equation}
\end{lemma}
\begin{proof}[Proof of \cref{lem:a-i:regularity}]
    \textbf{Step 1: Uniform $L^\infty$ bounds.}
    Because $p_0 \circ \bz_i$ is $\beta$-H\"older smooth on the parameter space $\phi_i(U_i) \subset [0, 1]^d$.
    Specifically, its H\"older norm (and consequently its supremum norm) is bounded by:
    \begin{equation}
        p_0 \circ \bz_i
        \in 
        \gH^{\beta}(\phi_i(U_i), B_p^{\rm proj}), 
        \quad \text{for each } i = 1, 2, \dots, C_{\gM}.  
        \label{eq:approx:p-circ-z-i:holder}
    \end{equation}
    
    By definition of the partition of unity, $0 \leq \bar{\rho}_i(\bu) \leq 1$ and by \cref{prop:manifold:det-Jacob-proj}, the Jacobian determinant is uniformly bounded by $B^{\rm Jac}$.
    Multiplying these supremum bounds yields:
    \begin{equation*}
        |a_i^{(\gQ)}(\bu)|
        =
        |\bar{\rho}_i(\bu)\sqrt{\det\bG_i(\bu)}p_0(\bz_i(\bu))|
        \leq
        1 \cdot B^{\rm Jac} \cdot B_p^{\rm proj}. 
    \end{equation*}
    Since the manifold is contained in $[0,1]^D$, we have $|z_{i,k}(\bu)| \leq 1$, which immediately implies $|a_{i,k}^{(\gP)}(\bu)| \leq a_i^{(\gQ)}(\bu)$, verifying \cref{eq:a-i:bound}.

    \textbf{Step 2: H\"older regularity and the explicit bound $B_a$.}
    Because the manifold and the partition of unity are $C^\infty$, the geometric multipliers $\bar{\rho}_i$, $\bz_i$, and $\sqrt{\det\bG_i}$ (which is smooth since $\bG_i$ is strictly positive definite) possess uniformly bounded derivatives of all orders on the compact chart supports.
    
    Let $C_{\rm geom}$ be one plus the maximum $C^{\ceil{\beta}}$ norm of these geometric factors across all finitely many charts $i \in \{1, \dots, C_{\gM}\}$ and coordinates $k \in \{1, \dots, D\}$:
    \begin{equation*}
        C_{\rm geom} := 
        1\vee
        \max_{i, k} \max_{|\balpha| \leq \ceil{\beta}} 
        \Bigl\{ 
          \|\partial^{\balpha}\bar{\rho}_i\|_\infty 
          \vee 
          \|\partial^{\balpha}\sqrt{\det\bG_i}\|_\infty 
          \vee 
          \|\partial^{\balpha}z_{i,k}\|_\infty 
        \Bigr\}.
    \end{equation*}
    Consequently, the $\gH^\beta$ norms of $\bar{\rho}_i$, $\sqrt{\det\bG_i}$, and $z_{i,k}$ are all strictly bounded by $C_{\rm geom}$.

    We invoke the algebra property of H\"older spaces: for any $f, g \in \gH^\beta$, there exists a constant $c_\beta$ such that $\|f g\|_{\gH^\beta} \leq c_\beta (\|f\|_{\gH^\beta}\|g\|_\infty + \|f\|_\infty\|g\|_{\gH^\beta})$.
    Applying this iteratively to the product $a_i^{(\gQ)} = (\bar{\rho}_i \sqrt{\det\bG_i}) \cdot (p_0 \circ \bz_i)$, we obtain:
    \begin{align*}
        \|a_i^{(\gQ)}\|_{\gH^\beta} 
        &\leq 
        c_\beta^2 
        \Bigl( 
          \|\bar{\rho}_i\|_{\gH^\beta}\|\sqrt{\det\bG_i}\|_\infty 
          + 
          \|\bar{\rho}_i\|_\infty\|\sqrt{\det\bG_i}\|_{\gH^\beta}
        \Bigr) 
        \|p_0 \circ \bz_i\|_\infty 
        + 
        c_\beta \|\bar{\rho}_i \sqrt{\det\bG_i}\|_\infty \|p_0 \circ \bz_i\|_{\gH^\beta}
        \\
        &\leq 
        c_\beta^2 
        \bigl( 
          C_{\rm geom} B^{\rm Jac} + 1 \cdot C_{\rm geom} 
        \bigr) 
        B_p^{\rm proj} 
        + 
        c_\beta (1 \cdot B^{\rm Jac}) B_p^{\rm proj}
        \\
        &\leq 
        c_\beta^2 
        B_p^{\rm proj} 
        \Bigl[ 
          C_{\rm geom}(B^{\rm Jac} + 1) + B^{\rm Jac} 
        \Bigr].
    \end{align*}
    
    Similarly, for $a_{i,j}^{(\gP)} = z_{i,j} a_i^{(\gQ)}$, applying the product rule again yields:
    \begin{align*}
        \|a_{i,j}^{(\gP)}\|_{\gH^\beta} 
        &\leq 
        c_\beta 
        \Bigl( 
          \|z_{i,j}\|_{\gH^\beta} \|a_i^{(\gQ)}\|_\infty 
          + 
          \|z_{i,j}\|_\infty \|a_i^{(\gQ)}\|_{\gH^\beta} 
        \Bigr)
        \\
        &\leq 
        c_\beta 
        \Bigl( 
          C_{\rm geom} B^{\rm Jac} B_p^{\rm proj} 
          + 
          1 \cdot \|a_i^{(\gQ)}\|_{\gH^\beta} 
        \Bigr).
    \end{align*}
    
    Taking the maximum of these explicit bounds over all charts and coordinates, we may choose the universal H\"older constant as:
    \begin{align*}
        B_a := 
        c_{\beta}'B_p^{\rm proj}C_{\rm geom}(1\vee B^{\rm Jac}).
    \end{align*}
\end{proof}

\subsubsection{Chart-integral representation of density and score}\label{sec:app:approx:manifold:integrals}

Define $F_{i,\bx,t}: \phi_i(U_i) \to \R$:
\begin{equation}
    F_{i,\bx,t}(\bu) := \frac{1}{2}\|\bx-m_t \bz_i(\bu)\|_2^2. 
    \label{eq:def:F-i-x-t}
\end{equation}

Fix $n\ge2$.
In statements using polynomial regimes $\varepsilon\ge n^{-a_\varepsilon}$ and $D\vee\sigma_{\tdown}^{-1}\le n^{a_{\rm par}}$, regard $a_\varepsilon,a_{\rm par}$ as fixed.
Choose once and for all a threshold $C_\varrho=C_\varrho(a_\varepsilon,a_{\rm par})\ge1$, independent of $D$, $n$, $\sigma_{\tdown}$, and all numerical $\varepsilon$-values after those exponents are fixed, large enough that the off-manifold deviation and Fisher bounds in \cref{lem:off-manifold:deviation-lemma,lem:bound:off-manifold:fisher-info} hold with threshold $C_\varrho\sigma_t\sqrt{D\vee\log n}$, with the polynomial tail exponent required by the branchwise approximation and estimation arguments.
This same constant determines the near-manifold domain used throughout the branchwise near-manifold approximation results.
Enlarge the fixed tube-volume constant $C_{\rm tube}$ from \cref{eq:global-tube-volume-bound}, if necessary, so that $C_{\rm tube}\ge C_\varrho$.
Dependence on the tube radius is tracked through the displayed quantity $M_x^\varrho$ below; it is not hidden in constants that are declared independent of $D$ and $n$.
\begin{equation}
    \varrho_t
    :=
    C_\varrho\sigma_t\sqrt{D\vee\log n},
    \label{eq:def:varrho-t}
\end{equation}
For each $t\in[\tdown,\tup]$, define the time-indexed near-manifold tube
\begin{equation}
    A_t
    :=
    \{\bx\in\R^D:\dist(\bx,m_t\gM)\le\varrho_t\}.
    \label{eq:def:near-manifold-tube}
\end{equation}
Thus a condition such as $\bx\in A_t$ is always interpreted at the same time $t$.
Since $\gM\subset[0,1]^D$, $m_t\le1$, and $\sigma_t\le1$, every $t\in[\tdown,\tup]$ and $\bx\in A_t$ satisfy
\[
    \|\bx\|_\infty
    \le
    1+C_\varrho\sqrt{D\vee\log n}.
\]
We therefore write
\begin{equation}
    M_x^\varrho
    :=
    1+C_\varrho\sqrt{D\vee\log n}.
    \label{eq:def:Mx-varrho}
\end{equation}
For $i \in \{1, \dots, C_{\gM}\}$, define the chart integrals for all pairs $(\bx,t)$ with $t\in[\tdown,\tup]$ and $\bx\in A_t$:
\begin{align}
    \gP_i(\bx, t)
    &:=
    \sigma_t^{-d}
    \int_{\phi_i(U_i)}
      \exp\Bigl(
        -\frac{F_{i,\bx,t}(\bu)}{\sigma_t^2}
      \Bigr)
      \ba_i^{(\gP)}(\bu)
    \od\bu, 
    \label{eq:def:P-i}
    \\
    \gQ_i(\bx, t)
    &:= 
    \sigma_t^{-d}
    \int_{\phi_i(U_i)}
      \exp\Bigl(
        -\frac{F_{i,\bx,t}(\bu)}{\sigma_t^2}
      \Bigr)
      a_i^{(\gQ)}(\bu)
    \od\bu, 
    \label{eq:def:Q-i}
\end{align}

Let
\begin{equation*}
    \gP(\bx, t) := \sum_{i=1}^{C_{\gM}}\gP_i(\bx, t)\in \R^D,
    \qquad
    \gQ(\bx, t) := \sum_{i=1}^{C_{\gM}}\gQ_i(\bx, t)\in \R.
\end{equation*}
Define the centered local and global numerator integrals
\begin{align}
    \gU_i(\bx,t)
    &:={}
    \frac{m_t\gP_i(\bx,t)-\bx\gQ_i(\bx,t)}{\sigma_t}
    \notag\\
    &=
    \sigma_t^{-d}
    \int_{\phi_i(U_i)}
      \frac{m_t\bz_i(\bu)-\bx}{\sigma_t}
      \exp\Bigl(
        -\frac{F_{i,\bx,t}(\bu)}{\sigma_t^2}
      \Bigr)
      a_i^{(\gQ)}(\bu)
    \od\bu,
    \label{eq:def:U-i}
    \\
    \gU(\bx,t)&:=\sum_{i=1}^{C_{\gM}}\gU_i(\bx,t)
    =\frac{m_t\gP(\bx,t)-\bx\gQ(\bx,t)}{\sigma_t}.
    \label{eq:def:U-global}
\end{align}
Then
\begin{equation}
    s^*(\bx,t)=\nabla\log p_t(\bx)
    =
    \frac{1}{\sigma_t}\frac{\gU(\bx,t)}{\gQ(\bx,t)}
    \qquad (p_t(\bx)>0).
    \label{eq:def:score:centered-decomp}
\end{equation}
For a chart-integral floor $C_t>0$, the centered floor-regularized score is
\begin{equation}
    s_{{\rm reg},C_t}(\bx,t)
    :=
    \frac{1}{\sigma_t}
    \frac{\gU(\bx,t)}{\gQ(\bx,t)\vee C_t}.
    \label{eq:def:centered-regularized-score}
\end{equation}

For reference, by~\cref{eq:def:manifold:density,eq:def:manifold:grad-density}, the same regularized score can also be written in the uncentered form
\begin{align}
    \frac{\nabla p_t(\bx)}{p_t(\bx) \vee \rho_t}
    &= 
    \frac{\frac{1}{(2\pi\sigma_t^2)^{D/2}}
    \sum_{i=1}^{C_{\gM}}
    \int_{\phi_i(U_i)}
      \exp\bigl(
        -\frac{\|\bx - m_t\phi_i^{-1}(\bu)\|_2^2}{2\sigma_t^2}
      \bigr)
      \bigl(
        -\frac{\bx-m_t\phi_i^{-1}(\bu)}{\sigma_t^2}
      \bigr)
      a_i^{(\gQ)}(\bu)
    \od\bu}{
    \frac{1}{(2\pi\sigma_t^2)^{D/2}}
    \sum_{i=1}^{C_{\gM}}
    \int_{\phi_i(U_i)}
      \exp\bigl(
        -\frac{\|\bx - m_t\phi_i^{-1}(\bu)\|_2^2}{2\sigma_t^2}
      \bigr)
      a_i^{(\gQ)}(\bu)
    \od\bu \vee \rho_t}
    \notag \\
    &= 
    -\frac{\bx}{\sigma_t^2}
    \frac{\sum_{i=1}^{C_{\gM}}
    \int_{\phi_i(U_i)}
      \exp\bigl(
        -\frac{F_{i,\bx,t}(\bu)}{\sigma_t^2}
      \bigr)
      a_i^{(\gQ)}(\bu)
    \od\bu}{
    \sum_{i=1}^{C_{\gM}}
    \int_{\phi_i(U_i)}
      \exp\bigl(
        -\frac{F_{i,\bx,t}(\bu)}{\sigma_t^2}
      \bigr)
      a_i^{(\gQ)}(\bu) 
    \od\bu \vee (2\pi\sigma_t^2)^{D/2}\rho_t}
    \notag \\
    &\qquad+
    \frac{m_t}{\sigma_t^2}
    \frac{
    \sum_{i=1}^{C_{\gM}}
    \int_{\phi_i(U_i)}
      \exp\bigl(
        -\frac{F_{i,\bx,t}(\bu)}{\sigma_t^2}
      \bigr)
      \ba_i^{(\gP)}(\bu)
    \od\bu}{
    \sum_{i=1}^{C_{\gM}}
    \int_{\phi_i(U_i)}
      \exp\bigl(
        -\frac{F_{i,\bx,t}(\bu)}{\sigma_t^2}
      \bigr)
      a_i^{(\gQ)}(\bu) 
    \od\bu \vee (2\pi\sigma_t^2)^{D/2}\rho_t}
    \notag \\
    &=
    -\frac{\bx}{\sigma_t^2}
    \frac{\gQ(\bx, t)}{\gQ(\bx, t)\vee C_t}
    +
    \frac{m_t}{\sigma_t^2}
    \frac{\gP(\bx, t)}{\gQ(\bx, t)\vee C_t},
    \label{eq:def:score:regularized:decomp}
\end{align}
where
\begin{equation}
    C_t
    :=
    \sigma_t^{-d}(2\pi\sigma_t^2)^{D/2}\rho_t.
    \label{eq:def:Ct}
\end{equation}
Conversely, choosing the chart-integral floor $C_t$ fixes the density-level floor as
\begin{equation}
    \rho_t
    =
    (2\pi\sigma_t^2)^{-D/2}\sigma_t^d C_t.
    \label{eq:rho-Ct-relation}
\end{equation}

\begin{lemma}[Uniform bounds for chart moment and density integrals]\label{lem:P-Q:uniform-bounds}
    Let
    \[
        V_{\rm chart}:=
        1\vee \sum_{i=1}^{C_{\gM}}\vol(\phi_i(U_i))<\infty .
    \]
    There is a finite $M_0 := B_p^{\rm proj}B^{\rm Jac}\sigma_{\tdown}^{-d}V_{\rm chart}$ such that
    \begin{equation}
        0 \leq \gQ(\bx,t) \leq M_0,
        \qquad
        |\gP_k(\bx, t)| \leq \gQ(\bx, t) \leq M_0,
        \quad k = 1, \dots, D.
        \label{eq:PQ-bound}
    \end{equation}
\end{lemma}
\begin{proof}[Proof of \cref{lem:P-Q:uniform-bounds}]
    We first show $|\gP_k(\bx, t)| \leq \gQ(\bx, t)$.
    Since $\gM \subset [0, 1]^D$, $0 \leq z_{i,k}(\bu) \leq 1$, so $|a_{i,k}^{(\gP)}(\bu)| \leq a_i^{(\gQ)}(\bu)$ and $|\gP_{i,k}| \leq \gQ_i$.
    Summing over $i$ gives $|\gP_k| \leq \gQ$.

    We next prove the uniform upper bound for $\gQ$.
    Note that $\exp\bigl(-F_{i,\bx,t}(\bu)/\sigma_t^2\bigr) \leq 1$, $\sigma_t^{-d} \leq \sigma_{\tdown}^{-d}$, and by~\cref{eq:a-i:bound} we have
    \begin{equation*}
        \gQ_i(\bx,t) 
        \leq 
        B_p^{\rm proj}B^{\rm Jac}\sigma_{\tdown}^{-d}
        \vol(\phi_i(U_i)).
    \end{equation*}
    Summing over charts gives the displayed $M_0$.
\end{proof}

% ----------------------------------------------------------------------
\subsection{Centered ratio target and stability}
\label{sec:app:approx:PQ:branchwise-score-assembly}

We use the integral-type set $\mathcal R=\{\gQ,\gP_1,\dots,\gP_D\}$ fixed after \cref{eq:def:a-i}.
In this subsection $R_i$ denotes the corresponding local chart component: $R_i=\gQ_i$ when $R=\gQ$, and $R_i=\gP_{i,k}$ when $R=\gP_k$.
The common cell notation below is used by the tangent-cell construction and by the density-lower-bound projection--Laplace construction.
The method-specific approximation results are stated in the following sections.

\subsubsection{Intrinsic cell notation and density floors}
\label{sec:app:approx:PQ:cell-setup}

For each chart, fix the common compact support envelope
\begin{equation}
    K_i^{\rm cell}
    :=
    \phi_i(\gS_i).
    \label{eq:def:K-i-cell}
\end{equation}
By \cref{eq:def:Ki-circ}, $K_i^{\rm cell}\Subset\Box_i$.
Moreover, for every $R\in\mathcal R$, the amplitude $a_i^{(R)}$ vanishes outside $K_i^{\rm cell}$, because each amplitude contains the factor $\bar\rho_i$.

For $R\in\mathcal R$, define the chart integral
\begin{equation}
    R_i(\bx,t)
    :=
    \sigma_t^{-d}
    \int_{K_i^{\rm cell}}
      \exp\left(
        -\frac{\|\bx-m_t\bz_i(\bu)\|_2^2}{2\sigma_t^2}
      \right)
      a_i^{(R)}(\bu)\od\bu,
    \label{eq:def:large-noise-cell-Ri}
\end{equation}
We use the same amplitude as a function on $\Box_i$; by \cref{lem:a-i:regularity} on $\phi_i(U_i)$ and the restriction $\Box_i\Subset\phi_i(U_i)$, its restriction satisfies $a_i^{(R)}|_{\Box_i}\in\gH^\beta(\Box_i,B_a)$.
Since $K_i^{\rm cell}\Subset\Box_i$, set
\[
    \delta_i^{\rm cell}:=\dist(K_i^{\rm cell},\partial\Box_i)>0.
\]
For every $0<h\le \delta_i^{\rm cell}/(2\sqrt d)$, let $\mathcal Q_{i,h}$ be the finite family of half-open cubes
\[
    Q_\nu=\bu_\nu+h[0,1)^d
    \]
from the standard $h$-lattice whose closures intersect $K_i^{\rm cell}$.
Then
\[
    K_i^{\rm cell}\subset \bigcup_{\nu\in\mathcal Q_{i,h}}\overline Q_\nu\subset\Box_i,
    \qquad
    |\mathcal Q_{i,h}|\le C_i h^{-d}.
\]
Indeed, every point in such a cube is within distance $\sqrt d\,h$ of $K_i^{\rm cell}$, and the cardinality bound follows by comparing the volume of the $\sqrt d\,h$-neighborhood of the fixed compact set $K_i^{\rm cell}$ with the cube volume $h^d$.
Since $a_i^{(R)}$ vanishes outside $K_i^{\rm cell}$, integration over $K_i^{\rm cell}$ is equivalent to the sum of the integrals over these full cubes; cube boundaries have Lebesgue measure zero.

\clearpage

\section{Large-Noise Tangent-Cell Score Approximation}
\label{sec:app:approx:PQ:large-noise-branch}

With a density lower bound, the unrefactored denominator $\gQ$ has an explicit tube lower bound in the large-noise regime.
Thus the separated tangent-cell moment approximation can be used directly for $\gQ$ and $\gU$, followed by an ordinary reciprocal network.

\paragraph{Large-noise proof architecture.}
The target result is \cref{thm:large-noise-hd-score-global}.  The proof has four
modules: derive tangent-cell phase bounds; approximate each chart integral with cells of
radius $r_k=c_0\sigma_{t_{k-1}}\wedge r_{\star,n}$ in the statistical application
(\cref{prop:large-noise-oracle-learned-score}); use the density lower bound to construct the
reciprocal of $\gQ$; and assemble the centered ratio $\gU/\gQ$ with clipping.  The
important size fact is that the cell count is explicitly recorded as
$\sigma_{t_{k-1}}^{-d}+r_{\star,n}^{-d}$ rather than hidden in asymptotic notation.

\subsection{Tangent-cell phase decomposition and gated chart approximation}
\label{sec:app:approx:PQ:tangent-cell-method}

The tangent-cell construction follows the tangent-space moment idea.
The network is fixed in advance: every cell is represented by a gated subnetwork, and the gate is an explicit ReLU ramp in a computable phase proxy.
Thus no input-dependent choice of active cells is made after the network has been constructed.

\begin{lemma}[Tangent-cell Gaussian phase bounds]
\label{lem:tangent-cell-phase-bounds}
Let $\nu$ be a cell index in a finite coordinate-cube partition of $\Box_i$.
Let $Q_\nu\subset\Box_i$ be the corresponding cube of side length $h$, let $\bu_\nu$ be its center, and set $\by_\nu=\bz_i(\bu_\nu)$.
Let $\mathsf P_\nu^{\rm tan}$ be the orthogonal linear projection in $\R^D$ onto the tangent space $T_{\by_\nu}\gM$.
Define
\[
    \mathsf T_\nu(\bx,t):=m_t\by_\nu+\mathsf P_\nu^{\rm tan}(\bx-m_t\by_\nu),
    \qquad
    \Phi_\nu^0(\bx,t):=
    \frac{\|\bx-m_t\by_\nu\|_2^2}{2\sigma_t^2}.
\]
The symbol $\mathsf T_\nu$ denotes an affine tangent-plane linearization center; it is
not the nearest-point projection $\Pi_{\gM}$ used in the small-noise branch.
There are constants $c_{\rm ph},C_{\rm ph}>0$, depending only on the fixed atlas and on $\Gamma_{\gM,q_{\rm geom}}$, such that if $h\le c_{\rm ph}\sigma_t$, then the following hold for all $\bu\in Q_\nu$, $\by=\bz_i(\bu)$.
First, if $\Phi_\nu^0(\bx,t)\le H$ with $H\ge1$, then
\begin{align}
    \frac{\|\bx-m_t\by\|_2}{\sigma_t}
    &\le C_{\rm ph} H^{1/2},
    \label{eq:phase-bound-center-to-cell}\\
    \frac{|T_\nu(\bx,\by,t)|}{\sigma_t^2}
    &\le
    C_{\rm ph}\Gamma_{\gM,q_{\rm geom}}\,h\,H^{1/2},
    \label{eq:T-bound-center-phase}\\
    \frac{m_t^2D_\nu(\bx,\by,t)}{2\sigma_t^2}
    &\le C_{\rm ph}H .
    \label{eq:D-bound-center-phase}
\end{align}
Second, if $\Phi_\nu^0(\bx,t)\ge A H$ for a sufficiently large numerical constant $A\ge A_{\rm ph}$, then
\begin{equation}
    \frac{\|\bx-m_t\by\|_2^2}{2\sigma_t^2}
    \ge c_{\rm ph} A H .
    \label{eq:phase-tail-from-center-proxy}
\end{equation}
\end{lemma}

\begin{proof}
The chart map is $C^2$ with norm bounded by $\Gamma_{\gM,q_{\rm geom}}$.
Hence, for $\bu\in Q_\nu$,
\[
    \bz_i(\bu)
    =
    \by_\nu+\bJ_i(\bu_\nu)(\bu-\bu_\nu)+\be_\nu(\bu),
    \qquad
    \|\be_\nu(\bu)\|_2\le C\Gamma_{\gM,q_{\rm geom}}h^2 .
\]
The vector $\bJ_i(\bu_\nu)(\bu-\bu_\nu)$ lies in $T_{\by_\nu}\gM$, while $\bx-\mathsf T_\nu(\bx,t)$ is orthogonal to this tangent space.
Therefore
\[
    T_\nu(\bx,\by,t)
    =
    \langle \bx-\mathsf T_\nu(\bx,t),
            \mathsf T_\nu(\bx,t)-m_t\by\rangle
\]
only sees the quadratic normal error $\be_\nu$ in the cell, and
\[
    |T_\nu(\bx,\by,t)|
    \le
    C\Gamma_{\gM,q_{\rm geom}}\,m_t h^2
    \|\bx-\mathsf T_\nu(\bx,t)\|_2 .
\]
If $\Phi_\nu^0\le H$, then
\[
    \|\bx-\mathsf T_\nu(\bx,t)\|_2
    \le \|\bx-m_t\by_\nu\|_2
    \le \sqrt{2H}\,\sigma_t.
\]
Since $h\le c_{\rm ph}\sigma_t$ and $m_t\le1$, this proves \cref{eq:T-bound-center-phase}.
The same triangle inequality gives
\[
    \|\bx-m_t\by\|_2
    \le \|\bx-m_t\by_\nu\|_2
       +m_t\|\by-\by_\nu\|_2
    \le C\sigma_tH^{1/2},
\]
after decreasing $c_{\rm ph}$, which proves \cref{eq:phase-bound-center-to-cell}.
Moreover, since $m_t^2D_\nu=\|\mathsf T_\nu(\bx,t)-m_t\by\|_2^2$, we avoid any division by the possibly small schedule value $m_t$.
Indeed,
\[
    \mathsf T_\nu(\bx,t)-m_t\by
    =
    \mathsf P_\nu^{\rm tan}(\bx-m_t\by_\nu)
    -m_t\bJ_i(\bu_\nu)(\bu-\bu_\nu)-m_t\be_\nu(\bu),
\]
so the right-hand side is bounded by $C(\|\bx-m_t\by_\nu\|_2+h)$.
This gives \cref{eq:D-bound-center-phase} because $h\le c_{\rm ph}\sigma_t$ and $\Phi_\nu^0\le H$.

For the tail statement, if $\Phi_\nu^0\ge AH$, then
\[
    \|\bx-m_t\by\|_2
    \ge
    \|\bx-m_t\by_\nu\|_2
    -m_t\|\by-\by_\nu\|_2
    \ge
    \sigma_t\sqrt{2AH}-C h .
\]
Choosing $A_{\rm ph}$ large and $c_{\rm ph}$ small gives $\|\bx-m_t\by\|_2^2/(2\sigma_t^2)\ge c_{\rm ph}AH$.
\end{proof}

\begin{lemma}[Scalar exponential Taylor bounds]
\label{lem:scalar-exp-taylor-bounds}
There are universal constants $A_{\exp},c_{\exp}>0$ such that the following hold.
If $0\le z\le H$ and $J\ge A_{\exp}H$, then
\begin{equation}
    \left|
      e^{-z}-\sum_{\ell=0}^{J}\frac{(-z)^\ell}{\ell!}
    \right|
    \le e^{-c_{\exp}J}.
    \label{eq:exp-large-interval-taylor}
\end{equation}
If $|z|\le a\le1$, then for every integer $K\ge0$,
\begin{equation}
    \left|
      e^{-z}-\sum_{\ell=0}^{K}\frac{(-z)^\ell}{\ell!}
    \right|
    \le C_K a^{K+1}e^{a}.
    \label{eq:exp-small-argument-taylor}
\end{equation}
\end{lemma}

\begin{proof}
For \cref{eq:exp-small-argument-taylor}, Taylor's theorem with integral remainder gives the claim.
For \cref{eq:exp-large-interval-taylor}, the Lagrange remainder is bounded by
\[
    \frac{z^{J+1}}{(J+1)!}
    \le
    \left(\frac{eH}{J+1}\right)^{J+1}.
\]
Taking $J\ge A_{\exp}H$ with $A_{\exp}>2e$ makes the last display at most $e^{-c_{\exp}J}$.
\end{proof}

% ----------------------------------------------------------------------
\subsection{Scale choices for tangent-cell large-noise approximation}
\label{sec:app:approx:PQ:branch-combination}

On a large-noise slab $I_k=[t_{k-1},t_k]$, the cell radius is chosen at the slab
noise scale until it reaches the accuracy-limited radius required by the fixed-order
tangent-cell expansion.  In \cref{prop:large-noise-oracle-learned-score} this is implemented as
$r_k=c_0\sigma_{t_{k-1}}\wedge r_{\star,n}$, so the per-slab cell count is controlled by
$\sigma_{t_{k-1}}^{-d}+r_{\star,n}^{-d}$.
The centered residual structure then permits the sharper scaled residual tolerance $n^{-(\beta+1)/(d+2\beta)}$, so the score error is $\sigma_t^{-2} n^{-2(\beta+1)/(d+2\beta)}$ rather than $\sigma_t^{-2}n^{-2\beta/(d+2\beta)}$.
Under \cref{assump:manifold:density-lower}, the denominator lower bound in
\cref{lem:large-noise-density-lower-Q-lower} stabilizes the ratio step; the remaining
accuracy dependence enters through the explicit finite-order cell-radius cap.
Consequently the switched large-noise network has total leading intrinsic size
$\sum_k\sigma_{t_{k-1}}^{-d}$ plus the accuracy-cap contribution.  The latter is
explicitly propagated through the learned-score and $\sfW_1$ bounds rather than hidden
inside a polynomial factor.

\begin{proposition}[Tangent-cell chart approximation with separated cell scale and accuracy]
\label{prop:large-noise-chart-cell-decoupled}
Assume \cref{assump:manifold:exact,assump:manifold:density}.
Let $q_{\rm geom}$ and $\Gamma_{\gM,q_{\rm geom}}$ be as in \cref{def:finite-order-geometry-constants}.
Fix a cell-scale parameter $0<\varepsilon_{\rm cell}\le1/2$, a denominator floor $q_{\rm floor}\in(0,1]$, and an accuracy parameter $0<\varepsilon_{\rm rel}\le1/16$.
Put
\[
    r_{\rm cell}:=\varepsilon_{\rm cell}^{1/\beta},
    \qquad
    \gI_{\rm lg}^{\rm cell}:=
    \{t\in[\tdown,\tup]:\sigma_t\ge c_{\rm lg}r_{\rm cell}\},
\]
and define
	\begin{equation}
	    H_{\rm lg}^{\rm dec}:=
	    1+
	    \log\!\left(
	      C_{\rm lg}
	      \Gamma_{\gM,q_{\rm geom}}
	      B_0^{C_{\rm lg}}
	      (D\vee\log n)^{C_{\rm lg}}
	      \sigma_{\tdown}^{-d}
	      q_{\rm floor}^{-1}\varepsilon_{\rm rel}^{-1}
	    \right).
    \label{eq:def:H-lg-decoupled}
\end{equation}
Assume the fixed-order geometry remainder is below the requested tolerance and the
linear tangent phase is kept in the small-argument regime, for a sufficiently small
numerical constant $c_{\rm ph,cell}>0$:
\begin{equation}
	\Gamma_{\gM,q_{\rm geom}} r_{\rm cell}
	(H_{\rm lg}^{\rm dec})^{1/2}
	\le c_{\rm ph,cell},
	\qquad
    \Gamma_{\gM,q_{\rm geom}}^{C}
    r_{\rm cell}^{q_{\rm geom}-2}
    (H_{\rm lg}^{\rm dec})^C
    \le c\varepsilon_{\rm rel} .
    \label{eq:large-noise-decoupled-geometry-condition}
\end{equation}
Then, for every chart $i$, there exist ReLU networks $\widetilde\gQ_i^{\rm lg}$ and $\widetilde\gU_{i,k}^{\rm lg}$, $k=1,\dots,D$, such that, uniformly for $t\in\gI_{\rm lg}^{\rm cell}$ and $\bx\in A_t$,
\begin{equation}
    |\gQ_i(\bx,t)-\widetilde\gQ_i^{\rm lg}(\bx,t)|
    \vee
    \max_{1\le k\le D}
    |\gU_{i,k}(\bx,t)-\widetilde\gU_{i,k}^{\rm lg}(\bx,t)|
    \le
    \varepsilon_{\rm rel}\bigl(\gQ_i(\bx,t)+q_{\rm floor}/C_{\gM}\bigr).
    \label{eq:large-noise-decoupled-cell-error}
\end{equation}
    Moreover, the depth, log-weight bound, width, and sparsity satisfy the following bounds, with the direct ambient powers displayed explicitly.
    \begin{equation}
\begin{aligned}
	    L
	    &\le
	    C_{\rm lg}\Gamma_{\gM,q_{\rm geom}}^C
	    B_0^C
	    (H_{\rm lg}^{\rm dec})^C,
	    \\
	    \log B
	    &\le
	    C_{\rm lg}\Gamma_{\gM,q_{\rm geom}}^C
	    B_0^C
	    (H_{\rm lg}^{\rm dec})^C,
	    \\
	    \|\bW\|_\infty
	    &\le
	    C_{\rm lg}\Gamma_{\gM,q_{\rm geom}}^C
	    B_0^C
	    D^{4(q_{\rm geom}+2)} (D\vee\log n)^{4(q_{\rm geom}+2)}
	    \varepsilon_{\rm cell}^{-d/\beta}
	    (H_{\rm lg}^{\rm dec})^C,
    \\
	    S
	    &\le
	    C_{\rm lg}\Gamma_{\gM,q_{\rm geom}}^C
	    B_0^C
	    D^{4(q_{\rm geom}+2)} (D\vee\log n)^{4(q_{\rm geom}+2)}
	    \varepsilon_{\rm cell}^{-d/\beta}
	    (H_{\rm lg}^{\rm dec})^C .
\end{aligned}
\label{eq:large-noise-decoupled-cell-size}
\end{equation}
Thus, for this separated tangent-cell moment approximation, the cell count is determined by $r_{\rm cell}$, and the conclusion is valid exactly on the range where \cref{eq:large-noise-decoupled-geometry-condition} holds.
If the slab-scale radius does not satisfy that condition, the cell radius must be capped by the accuracy-limited radius used in \cref{prop:large-noise-oracle-learned-score}.
When $r_{\rm cell}$ is specialized to
$c_0\sigma_{t_{k-1}}\wedge r_{\star,n}$ in the large-noise estimation theorem, the
resulting additional term $r_{\star,n}^{-d}$ is recorded explicitly in the learned-score
stochastic bound.
\end{proposition}

\begin{proof}
Fix a chart $i$.
Choose once and for all a fixed number $c_h>0$ such that
\[
    c_h
    \le
    1\wedge c_{\rm ph}c_{\rm lg}
    \wedge
    \min_{1\le j\le C_{\gM}}\frac{\delta_j^{\rm cell}}{2\sqrt d}.
\]
This choice depends only on the fixed atlas, $q_{\rm geom}$, and the branch constants, and is independent of $\varepsilon_{\rm cell}$, $t$, $\bx$, and $n$.
Partition the fixed compact coefficient support $K_i^{\rm cell}\Subset\Box_i$ into coordinate cubes $\{Q_\nu:\nu\in\mathcal V_i\}$ of side length $h=c_h r_{\rm cell}$.
Here $\nu$ is a cell label in the finite index set $\mathcal V_i$.
The number of cells is $N_{\rm cell}\le C\varepsilon_{\rm cell}^{-d/\beta}$.
For every $t\in\gI_{\rm lg}^{\rm cell}$, the definition of $c_h$ gives $h=c_hr_{\rm cell}\le c_{\rm ph}\sigma_t$, so \cref{lem:tangent-cell-phase-bounds} applies on every retained cell.
The same choice gives $h\le\delta_i^{\rm cell}/(2\sqrt d)$, so every cell whose closure intersects $K_i^{\rm cell}$ is contained in $\Box_i$.
Let $\bu_\nu$ be the cell center, $\by_\nu=\bz_i(\bu_\nu)$, and $\mathsf P_\nu^{\rm tan}$ the orthogonal projection onto $T_{\by_\nu}\gM$.
Define
\[
    \mathsf T_\nu(\bx,t)=m_t\by_\nu+\mathsf P_\nu^{\rm tan}(\bx-m_t\by_\nu).
\]
For $\by\in\gM$, the identity
\[
    \|\bx-m_t\by\|_2^2
    =
    \|\bx-\mathsf T_\nu(\bx,t)\|_2^2
    +2T_\nu(\bx,\by,t)
    +m_t^2D_\nu(\bx,\by,t)
\]
holds with
\[
	    T_\nu=\langle \bx-\mathsf T_\nu(\bx,t),\mathsf T_\nu(\bx,t)-m_t\by\rangle,
	    \qquad
	    D_\nu=\|\mathsf T_\nu(\bx,t)/m_t-\by\|_2^2 .
	\]
For $\by=\bz_i(\bu)$, write the exact cell exponential as
\begin{equation}
    E_\nu(\bx,\bu,t)
    :=
    \exp\!\left(
      -\frac{\|\bx-m_t\bz_i(\bu)\|_2^2}{2\sigma_t^2}
    \right)
    \label{eq:large-noise-decoupled-exact-exp}
\end{equation}
and the tangent Gaussian envelope as
\begin{equation}
    G_\nu(\bx,t)
    :=
    \exp\!\left(
      -\frac{\|\bx-\mathsf T_\nu(\bx,t)\|_2^2}{2\sigma_t^2}
    \right).
    \label{eq:large-noise-decoupled-gaussian-envelope}
\end{equation}
The distance identity gives the exact factorization
\begin{equation}
    E_\nu(\bx,\bu,t)
    =
    G_\nu(\bx,t)
    \exp\!\left(-\frac{T_\nu(\bx,\bz_i(\bu),t)}{\sigma_t^2}\right)
    \exp\!\left(-\frac{m_t^2D_\nu(\bx,\bz_i(\bu),t)}{2\sigma_t^2}\right).
    \label{eq:large-noise-decoupled-exponential-factorization}
\end{equation}
Hence the exact chart-cell pieces are
\begin{equation}
\begin{aligned}
    \gQ_{i,\nu}(\bx,t)
    &:=
    \sigma_t^{-d}\int_{Q_\nu}
      E_\nu(\bx,\bu,t)a_i^{(\gQ)}(\bu)\od\bu,\\
    \gU_{i,\nu,k}(\bx,t)
    &:=
    \sigma_t^{-d}\int_{Q_\nu}
      \frac{m_tz_{i,k}(\bu)-x_k}{\sigma_t}
      E_\nu(\bx,\bu,t)a_i^{(\gQ)}(\bu)\od\bu .
\end{aligned}
\label{eq:large-noise-decoupled-cell-contributions}
\end{equation}

\emph{Step 1: fixed gates and tails.}
Use the center-phase proxy
\[
    \Phi_\nu^0(\bx,t)
    =
    \frac{\|\bx-m_t\by_\nu\|_2^2}{2\sigma_t^2}.
\]
Let $\theta_{\rm lg}:\R\to[0,1]$ be the fixed piecewise-affine ReLU ramp with $\theta_{\rm lg}(u)=1$ for $u\le A_0$ and $\theta_{\rm lg}(u)=0$ for $u\ge2A_0$, where $A_0\ge A_{\rm ph}$ is a large numerical constant.
Define the exact cell gate
\[
    \Theta_{\nu}^{\rm ex}(\bx,t)
    :=
    \theta_{\rm lg}\!\left(\Phi_\nu^0(\bx,t)/H_{\rm lg}^{\rm dec}\right).
\]
The network gate $\Theta_\nu$ is obtained by approximating $m_t$, $\sigma_t^{-2}$, the squared norm, and the final ramp.
Its arithmetic accuracy is chosen so that, outside the ramp transition band,
\[
    |\Theta_\nu-\Theta_\nu^{\rm ex}|
    \le
    c\varepsilon_{\rm rel} (N_{\rm cell}C_{\gM})^{-1}(H_{\rm lg}^{\rm dec})^{-C}.
\]
Inside the transition band, the exact proxy is between $A_0H_{\rm lg}^{\rm dec}$ and $2A_0H_{\rm lg}^{\rm dec}$, so the true cell contribution is controlled by the explicit tail estimate below; the gate implementation error is assigned to that tail budget.
If the exact proxy is above $A_0H_{\rm lg}^{\rm dec}$, \cref{lem:tangent-cell-phase-bounds} and the Gaussian tail estimate give
\[
    \sigma_t^{-d}
    \int_{Q_\nu}
      e^{-\|\bx-m_t\bz_i(\bu)\|^2/(2\sigma_t^2)}
      a_i^{(\gQ)}(\bu)\od\bu
    \le
    C\sigma_t^{-d}h^d e^{-cA_0H_{\rm lg}^{\rm dec}} .
\]
    The centered factor in $\gU_{i,k}$ is bounded on the non-negligible tail by $C(D\vee\log n)(H_{\rm lg}^{\rm dec})^{1/2}$.
Summing over all cells and using $\sigma_t^{-d}\le\sigma_{\tdown}^{-d}$, $N_{\rm cell}h^d\le C$, and \cref{eq:def:H-lg-decoupled}, with $A_0$ and $C_{\rm lg}$ large enough, gives the explicit exponential budget
\[
        C(D\vee\log n)(H_{\rm lg}^{\rm dec})^{1/2}
    \sigma_{\tdown}^{-d}
    \exp(-cA_0H_{\rm lg}^{\rm dec})
    \le c\varepsilon_{\rm rel} q_{\rm floor}/C_{\gM}.
\]
Thus all gated-away and gate-transition contributions are bounded by $c\varepsilon_{\rm rel} q_{\rm floor}/C_{\gM}$.

\emph{Step 2: fixed-order geometric replacement.}
On cells with $\Theta_\nu^{\rm ex}\ne0$, \cref{lem:tangent-cell-phase-bounds} yields, for every $\bu\in Q_\nu$,
\[
    \|\bx-m_t\bz_i(\bu)\|_2/\sigma_t\le C(H_{\rm lg}^{\rm dec})^{1/2},
\]
\[
    |T_\nu|/\sigma_t^2
    \le C\Gamma_{\gM,q_{\rm geom}}h(H_{\rm lg}^{\rm dec})^{1/2},
    \qquad
    m_t^2D_\nu/(2\sigma_t^2)
    \le C H_{\rm lg}^{\rm dec}.
\]
Replace $\bz_i$ in $T_\nu,D_\nu$, and in the centered factor $(m_tz_{i,k}(\bu)-x_k)/\sigma_t$, by its Taylor polynomial at $\bu_\nu$ of fixed degree $q_{\rm geom}-2$.
The exponent and centered-factor perturbations are bounded by
\[
    C\Gamma_{\gM,q_{\rm geom}}^C
    r_{\rm cell}^{q_{\rm geom}-2}(H_{\rm lg}^{\rm dec})^C,
\]
which is at most $c\varepsilon_{\rm rel}$ by \cref{eq:large-noise-decoupled-geometry-condition}.
The phase-smallness part of the same condition gives
$|T_\nu|/\sigma_t^2\le1/2$, while $m_t^2D_\nu/(2\sigma_t^2)\ge0$; hence the derivative
of the non-Gaussian exponential factor is bounded by an absolute constant on the
interpolation between the exact and polynomialized phases.
The mean-value theorem for $e^{-z}$ then gives a cell error bounded by
\[
    c\varepsilon_{\rm rel}\,
    \sigma_t^{-d}e^{-\|\bx-\mathsf T_\nu(\bx,t)\|^2/(2\sigma_t^2)}
	    \int_{Q_\nu}a_i^{(\gQ)}(\bu)\od\bu
	    +c\varepsilon_{\rm rel} q_{\rm floor}(C_{\gM}N_{\rm cell})^{-1}.
\]
From this point through Step~4, $T_\nu$, $D_\nu$, and the centered factor are understood to be their polynomialized versions; the exact-to-polynomialized replacement error is the displayed error above.

\emph{Step 3: scalar Taylor expansions.}
For the polynomialized geometry, expand the first non-Gaussian exponential in \cref{eq:large-noise-decoupled-exponential-factorization},
\[
    \exp\!\left(-T_\nu/\sigma_t^2\right),
\]
to the fixed order $q_{\rm geom}-3$.
The small-argument hypothesis needed for this Taylor expansion is the first, explicit
phase-smallness inequality in \cref{eq:large-noise-decoupled-geometry-condition}: using
\cref{eq:T-bound-center-phase} and the choice of the cell side length,
\[
    |T_\nu|/\sigma_t^2
    \le C\Gamma_{\gM,q_{\rm geom}}r_{\rm cell}
    (H_{\rm lg}^{\rm dec})^{1/2}
    \le 1/2
\]
after decreasing $c_{\rm ph,cell}$.
The Taylor remainder is controlled by the second, high-order geometry inequality in
\cref{eq:large-noise-decoupled-geometry-condition}; indeed, after enlarging the harmless
exponent $C$,
\[
    \left|\frac{T_\nu}{\sigma_t^2}\right|^{q_{\rm geom}-2}
    \le
    C\Gamma_{\gM,q_{\rm geom}}^{C}
    r_{\rm cell}^{q_{\rm geom}-2}
    (H_{\rm lg}^{\rm dec})^C
    \le c\varepsilon_{\rm rel} .
\]
Hence \cref{eq:exp-small-argument-taylor} gives the explicit Taylor remainder
\[
    \left|
      \exp\!\left(-T_\nu/\sigma_t^2\right)
      -
      \sum_{\ell=0}^{q_{\rm geom}-3}
        \frac{(-T_\nu/\sigma_t^2)^\ell}{\ell!}
    \right|
    \le c\varepsilon_{\rm rel} .
\]
Multiplication by the explicit Gaussian envelope $G_\nu(\bx,t)$ gives, for $\gQ$, a contribution bounded by
\[
    c\varepsilon_{\rm rel}\,
    \sigma_t^{-d}G_\nu(\bx,t)
    \int_{Q_\nu}a_i^{(\gQ)}(\bu)\od\bu,
\]
    while for $\gU$ the centered factor is bounded on gated cells by $C(D\vee\log n)(H_{\rm lg}^{\rm dec})^{1/2}$ and is absorbed into the $(H_{\rm lg}^{\rm dec})^C$ budgets above.
Expand the second non-Gaussian exponential in \cref{eq:large-noise-decoupled-exponential-factorization},
\[
    \exp\!\left(-\frac{m_t^2D_\nu}{2\sigma_t^2}\right),
\]
to order $J=\lceil C_JH_{\rm lg}^{\rm dec}\rceil$.
By \cref{lem:scalar-exp-taylor-bounds}, choosing $C_J$ large enough makes the remainder explicitly
\[
    \left|
      \exp\!\left(-\frac{m_t^2D_\nu}{2\sigma_t^2}\right)
      -
      \sum_{j=0}^{J}
        \frac{[-m_t^2D_\nu/(2\sigma_t^2)]^j}{j!}
    \right|
    \le e^{-c_{\exp}J}.
\]
Since $J=\lceil C_JH_{\rm lg}^{\rm dec}\rceil$, $N_{\rm cell}h^d\le C$, and $\sigma_t^{-d}\le\sigma_{\tdown}^{-d}$, \cref{eq:def:H-lg-decoupled} gives, for $C_J$ and $C_{\rm lg}$ large enough,
\[
	    \sigma_t^{-d}h^d e^{-c_{\exp}J}
	    \le
	    c\varepsilon_{\rm rel} q_{\rm floor}(C_{\gM}N_{\rm cell})^{-1}(H_{\rm lg}^{\rm dec})^{-C}.
\]
This is the absolute cell budget used in the final summation.

\emph{Step 4: stored moments and neural implementation.}
After the two scalar expansions, each cell contribution is a finite sum of atoms consisting of powers of $m_t$, powers of $\sigma_t^{-1}$, the tangent coordinates $\mathsf P_\nu^{\rm tan}(\bx-m_t\by_\nu)$, fixed-degree normal residual terms, and the explicit Gaussian envelope
\[
    G_\nu(\bx,t)=
    \exp\!\left(-\frac{\|\bx-\mathsf T_\nu(\bx,t)\|^2}{2\sigma_t^2}\right).
\]
The coefficients are fixed moments of the form
\[
    \int_{Q_\nu}(\bu-\bu_\nu)^\alpha a_i^{(\gQ)}(\bu)\od\bu
\]
for $\gQ_i$, and for $\gU_{i,k}$ the same moments together with
\[
    \int_{Q_\nu}(\bu-\bu_\nu)^\alpha z_{i,\nu,k}^{(q)}(\bu)
    a_i^{(\gQ)}(\bu)\od\bu,
\]
where $z_{i,\nu,k}^{(q)}$ is the fixed Taylor polynomial.
The term $-x_k/\sigma_t$ multiplies the $\gQ$-moment family.
Hence the centered numerator is expanded directly and no amplitude $a_i^{(\gU_k)}$ is introduced.
The magnitudes of these stored moments are bounded by fixed powers of $B_0$ and the geometry constants through \cref{lem:a-i:regularity}.
This is why \cref{eq:def:H-lg-decoupled,eq:large-noise-decoupled-cell-size} display the density-amplitude factor explicitly.

There are at most $(H_{\rm lg}^{\rm dec})^C$ atoms per cell; the dependence on the fixed order $q_{\rm geom}-3$ is absorbed into the fixed polynomial powers of $D$, $\Gamma_{\gM,q_{\rm geom}}$, and $H_{\rm lg}^{\rm dec}$.
More explicitly, for one cell atom:
\begin{itemize}
    \item the affine projection $\mathsf P_\nu^{\rm tan}(\bx-m_t\by_\nu)$, squared norm, and coordinate summations contribute a multiplicative factor at most $D^{4}$;
    \item the fixed Taylor replacement of order at most $q_{\rm geom}$ contributes a multiplicative factor at most $D^{4q_{\rm geom}}$, while all intrinsic monomial counts depend only on $d$ and $q_{\rm geom}$;
    \item the input ranges on $A_t$, the centered numerator factor, and the scalar arithmetic subnetworks contribute a multiplicative factor at most $(D\vee\log n)^{4(q_{\rm geom}+2)}$, with the remaining dependence absorbed into $(H_{\rm lg}^{\rm dec})^C$.
\end{itemize}
Thus the whole cellwise implementation is bounded by
\[
    D^{4q_{\rm geom}+4}
    (D\vee\log n)^{4(q_{\rm geom}+2)}
    (H_{\rm lg}^{\rm dec})^C
\]
per cell after harmless enlargement of constants, which is covered by the exponent choice in the proposition statement.
Reciprocal, schedule-power, scalar exponential, monomial, and product networks are assigned relative accuracy $c\varepsilon_{\rm rel}(H_{\rm lg}^{\rm dec})^{-C}$ and absolute accuracy $c\varepsilon_{\rm rel} q_{\rm floor}(C_{\gM}N_{\rm cell})^{-1}(H_{\rm lg}^{\rm dec})^{-C}$.
Summing the atom errors, the gate tails, the geometric replacement error, and the Taylor remainders gives \cref{eq:large-noise-decoupled-cell-error}.
Parallelizing over $N_{\rm cell}$ cells gives \cref{eq:large-noise-decoupled-cell-size}.
\end{proof}

\begin{lemma}[Large-noise denominator lower bound]
\label{lem:large-noise-density-lower-Q-lower}
Assume \cref{assump:manifold:exact,assump:manifold:density,assump:manifold:density-lower}.
There are constants $c_{\rm lb},C_{\rm lb}>0$, depending only on the displayed finite
geometry controls, such that
\[
    c_{\rm lb}^{-1}\vee C_{\rm lb}
    \le
    C\Gamma_{\gM,q_{\rm geom}}^C
\]
after enlarging the finite-order geometry constant from
\cref{def:finite-order-geometry-constants}.  For every $t\in[\tdown,\tup]$ and every
\[
    \bx\in A_t
    =
    \{\dist(\bx,m_t\gM)\le C_\varrho\sigma_t\sqrt{D\vee\log n}\},
\]
\begin{equation}
    \gQ(\bx,t)
    \ge
    q_{\rm lb}
    :=
    c_{\rm lb}p_{\min}
    \exp\{-C_{\rm lb}C_\varrho^2(D\vee\log n)\}.
    \label{eq:large-noise-Q-density-lower}
\end{equation}
Moreover
\begin{equation}
    \frac{\|\gU(\bx,t)\|_2}{\gQ(\bx,t)}
    \le
    R_{\rm lb}
    :=
    1+
    \Bigl[
      2\log\!\bigl(e\sigma_{\tdown}^{-d}q_{\rm lb}^{-1}\bigr)
      +C_{\rm lb}(D\vee\log n)
    \Bigr]^{1/2}.
    \label{eq:large-noise-U-over-Q-density-lower}
\end{equation}
\end{lemma}

\begin{proof}
\emph{Step 1: lower-bound the local mass in one intrinsic ball.}
Fix $(\bx,t)$ in the displayed tube and choose $\by_0\in\gM$ with
$\|\bx-m_t\by_0\|_2=\dist(\bx,m_t\gM)$.
Let $a_{\rm lb}\in(0,1]$ be smaller than the injectivity and coordinate-buffer radii of
the projection atlas.  By the finite-order geometry convention and
\cref{lem:finite-order-geometry-from-smooth}, it can be chosen so that
$a_{\rm lb}^{-1}\le C\Gamma_{\gM,q_{\rm geom}}^C$, and the small-ball volume bound
\[
    \vol_{\gM}(B_{\gM}(\by,a_{\rm lb}r))
    \ge C^{-1}\Gamma_{\gM,q_{\rm geom}}^{-C}r^d,
    \qquad 0<r\le1,
\]
holds uniformly in $\by\in\gM$.  Since $0<\sigma_t\le1$, the ball
$B_{\gM}(\by_0,a_{\rm lb}\sigma_t)$ remains in a single normal coordinate patch.
For $\by\in B_{\gM}(\by_0,a_{\rm lb}\sigma_t)$,
\[
    \|\bx-m_t\by\|_2
    \le
    C\sigma_t\sqrt{D\vee\log n},
\]
after decreasing $a_{\rm lb}$, because $m_t\le1$ and $\bx\in A_t$.
The density lower bound and the uniform atlas/Jacobian bounds therefore give
\[
    \gQ(\bx,t)
    =
    \sigma_t^{-d}
    \int_{\gM}
      e^{-\|\bx-m_t\by\|_2^2/(2\sigma_t^2)}
      p_0(\by)\od\vol_{\gM}(\by)
    \ge
    c p_{\min}e^{-C C_\varrho^2(D\vee\log n)}
    \sigma_t^{-d}\vol_{\gM}(B_{\gM}(\by_0,a_{\rm lb}\sigma_t)).
\]
The displayed small-ball volume lower bound gives
$\vol_{\gM}(B_{\gM}(\by_0,a_{\rm lb}\sigma_t))
\ge C^{-1}\Gamma_{\gM,q_{\rm geom}}^{-C}\sigma_t^d$.
This proves \cref{eq:large-noise-Q-density-lower} with
$c_{\rm lb}^{-1}\le C\Gamma_{\gM,q_{\rm geom}}^C$.

\emph{Step 2: bound the centered ratio.}
The identity $s^*(\bx,t)=\sigma_t^{-1}\gU(\bx,t)/\gQ(\bx,t)$ and
\[
    p_t(\bx)=(2\pi\sigma_t^2)^{-D/2}\sigma_t^d\gQ(\bx,t)
\]
give the required bound on the score norm as follows.
Let $\bY\sim P_0$, put
$\bz(\by)=(\bx-m_t\by)/\sigma_t$, and define the posterior measure
\[
    \od\nu_{\bx,t}(\by)
    =
    \frac{\exp(-\|\bz(\by)\|_2^2/2)}
    {\int_\gM \exp(-\|\bz(\by')\|_2^2/2)\od P_0(\by')}
    \od P_0(\by).
\]
If $Z_{\bx,t}$ denotes the denominator in this display, then
\[
    \sigma_t s^*(\bx,t)
    =
    -\int_\gM \bz(\by)\od\nu_{\bx,t}(\by).
\]
Moreover
\[
    \KL(\nu_{\bx,t}\|P_0)
    =
    -\frac12\int_\gM \|\bz(\by)\|_2^2\od\nu_{\bx,t}(\by)
    -\log Z_{\bx,t}
    \ge0,
\]
and hence, by Jensen's inequality,
\[
    \|\sigma_t s^*(\bx,t)\|_2^2
    \le
    \int_\gM \|\bz(\by)\|_2^2\od\nu_{\bx,t}(\by)
    \le
    2\log(Z_{\bx,t}^{-1}).
\]
Since $Z_{\bx,t}=(2\pi\sigma_t^2)^{D/2}p_t(\bx)
=\sigma_t^d\gQ(\bx,t)$, we obtain
\[
    \frac{\|\gU(\bx,t)\|_2^2}{\gQ(\bx,t)^2}
    =
    \sigma_t^2\|s^*(\bx,t)\|_2^2
    \le
    2\log\!\left(
      \frac{(2\pi\sigma_t^2)^{-D/2}}{p_t(\bx)}
    \right)+1
    =
    2\log\!\left(\sigma_t^{-d}\gQ(\bx,t)^{-1}\right)+1 .
\]
Using $\sigma_t\ge\sigma_{\tdown}$, $\gQ(\bx,t)\ge q_{\rm lb}$, and enlarging $C_{\rm lb}$ gives
\cref{eq:large-noise-U-over-Q-density-lower}.
\end{proof}

\begin{theorem}[Global tangent-cell large-noise score approximation]
\label{thm:large-noise-hd-score-global}
Assume \cref{assump:manifold:exact,assump:manifold:density,assump:manifold:density-lower}.
Fix $0<\varepsilon_{\rm cell}\le1/2$ and $0<\varepsilon_{\rm sc}\le1$.
Let $q_{\rm geom}$ be as in \cref{eq:def:qgeom-main}, and let $\Gamma_{\gM,q_{\rm geom}}$ be the finite geometry constant from \cref{def:finite-order-geometry-constants}.
Let $r_{\rm cell}=\varepsilon_{\rm cell}^{1/\beta}$ and
\[
    \gI_{\rm lg}^{\rm cell}
    :=
    \{t\in[\tdown,\tup]:\sigma_t\ge c_{\rm lg}r_{\rm cell}\}.
\]
Let $q_{\rm lb}$ and $R_{\rm lb}$ be defined in
\cref{eq:large-noise-Q-density-lower,eq:large-noise-U-over-Q-density-lower}, and set
\[
    C_{\rm den}:=q_{\rm lb}/8,
    \qquad
    \varepsilon_{\rm rel}:=
    c_{\rm glob}\frac{\varepsilon_{\rm sc}}{1+R_{\rm lb}},
\]
\[
    H_{\rm lg}^{\rm lb}
    :=
	    1+
	    \log\!\left(
	      C_{\rm lg}
	      \Gamma_{\gM,q_{\rm geom}}
	      B_0^{C_{\rm lg}}
	      (D\vee\log n)^{C_{\rm lg}}
	      \sigma_{\tdown}^{-d}
	      C_{\rm den}^{-1}
	      (1+R_{\rm lb})\varepsilon_{\rm sc}^{-1}
	    \right).
\]
After decreasing $c_{\rm lb}$, we assume $q_{\rm lb}\le1$, so
$C_{\rm den}\in(0,1]$.
Choose the numerical constant $c_{\rm glob}\le1/32$, so that
$\varepsilon_{\rm rel}\le1/16$ and the separated tangent-cell proposition is applicable.
Assume the large-noise cell scale obeys the phase-smallness and fixed-order geometry
conditions, with the same sufficiently small numerical constant $c_{\rm ph,cell}$:
\begin{equation}
    \Gamma_{\gM,q_{\rm geom}}r_{\rm cell}
    (H_{\rm lg}^{\rm lb})^{1/2}
    \le c_{\rm ph,cell},
    \qquad
    \Gamma_{\gM,q_{\rm geom}}^{C}
    r_{\rm cell}^{q_{\rm geom}-2}
    (H_{\rm lg}^{\rm lb})^C
    \le c\varepsilon_{\rm rel}.
    \label{eq:large-noise-density-lower-geometry-condition}
\end{equation}
Then there exists a ReLU network $\widetilde s_{\rm lb}^{\rm lg}$ such that, uniformly over
$t\in\gI_{\rm lg}^{\rm cell}$ and $\bx\in A_t$,
\begin{equation}
    \sigma_t
    \|\widetilde s_{\rm lb}^{\rm lg}(\bx,t)-s^*(\bx,t)\|_\infty
    \le
    \varepsilon_{\rm sc}.
    \label{eq:large-noise-hd-score-error}
\end{equation}
Its depth $L$, maximum width $\|\bW\|_\infty$, sparsity $S$, and log-weight bound $\log B$ satisfy
\begin{equation}
\begin{aligned}
    L
    &\le
    C\Gamma_{\gM,q_{\rm geom}}^C B_0^C
    (H_{\rm lg}^{\rm lb})^C,
    \\
    \log B
    &\le
    C\Gamma_{\gM,q_{\rm geom}}^C B_0^C
    (H_{\rm lg}^{\rm lb})^C,
    \\
    \|\bW\|_\infty
    &\le
    C D^{4(q_{\rm geom}+2)}
    (D\vee\log n)^{4(q_{\rm geom}+2)}
    \Gamma_{\gM,q_{\rm geom}}^C B_0^C
    \varepsilon_{\rm cell}^{-d/\beta}
    (H_{\rm lg}^{\rm lb})^C,
    \\
    S
    &\le
    C D^{4(q_{\rm geom}+2)}
    (D\vee\log n)^{4(q_{\rm geom}+2)}
    \Gamma_{\gM,q_{\rm geom}}^C B_0^C
    \varepsilon_{\rm cell}^{-d/\beta}
    (H_{\rm lg}^{\rm lb})^C .
\end{aligned}
\label{eq:large-noise-hd-size-separated}
\end{equation}
In particular, the width and sparsity contain the intrinsic tangent-cell count
$\varepsilon_{\rm cell}^{-d/\beta}$.  In the slabwise application below, the fixed-order
geometry condition forces an accuracy-dependent cap on $\varepsilon_{\rm cell}$:
$\varepsilon_{\rm cell}^{-d/\beta}=r_k^{-d}\le
C(\sigma_{t_{k-1}}^{-d}+r_{\star,n}^{-d})$, and the
$r_{\star,n}^{-d}$ term is carried explicitly through the large-noise oracle and
$\sfW_1$ bounds.
\end{theorem}

\begin{proof}
\emph{Step 1: approximate the unrefactored chart integrals.}
Apply \cref{prop:large-noise-chart-cell-decoupled} with $q_{\rm floor}=C_{\rm den}$ and
$\varepsilon_{\rm rel}$ as above.
The definition of $H_{\rm lg}^{\rm lb}$ dominates the logarithmic scale
$H_{\rm lg}^{\rm dec}$, and
\cref{eq:large-noise-density-lower-geometry-condition} implies
\cref{eq:large-noise-decoupled-geometry-condition}.
After summing over the fixed number of charts,
\[
    |\widehat\gQ(\bx,t)-\gQ(\bx,t)|
    \vee
    \max_k|\widehat\gU_k(\bx,t)-\gU_k(\bx,t)|
    \le
    C\varepsilon_{\rm rel}\bigl(\gQ(\bx,t)+C_{\rm den}\bigr).
\]
By \cref{eq:large-noise-Q-density-lower}, $C_{\rm den}\le \gQ/8$, so after decreasing
$c_{\rm glob}$,
\begin{equation}
    |\widehat\gQ-\gQ|
    \vee
    \|\widehat\gU-\gU\|_\infty
    \le
    c\varepsilon_{\rm rel}\gQ .
    \label{eq:large-noise-density-lower-QU-approx}
\end{equation}

\emph{Step 2: pass from integral approximation to ratio approximation.}
The map $q\mapsto q\vee C_{\rm den}$ is ReLU-realizable, and on the present domain
$\gQ\ge 8C_{\rm den}$.
Thus $\widehat\gQ\vee C_{\rm den}\ge \gQ/2$.
Using \cref{eq:large-noise-U-over-Q-density-lower,eq:large-noise-density-lower-QU-approx},
\[
    \left\|
      \frac{\widehat\gU}{\widehat\gQ\vee C_{\rm den}}
      -
      \frac{\gU}{\gQ}
    \right\|_\infty
    \le
    C(1+R_{\rm lb})\varepsilon_{\rm rel}
    \le
    \varepsilon_{\rm sc}/4 .
\]
The reciprocal network is applied only on the interval
$[C_{\rm den},2M_0]$, where $M_0$ is the uniform chart-integral bound from
\cref{lem:P-Q:uniform-bounds}; the preceding approximation and $\gQ\le M_0$ put
$\widehat\gQ\vee C_{\rm den}$ in this interval after decreasing $c_{\rm glob}$.
Its accuracy is chosen
\[
    c\varepsilon_{\rm sc}(H_{\rm lg}^{\rm lb})^{-C}
\]
and its depth, sparsity, and log-weight are absorbed into the displayed powers of
$H_{\rm lg}^{\rm lb}$.
The final multiplication by $\sigma_t^{-1}$ is implemented with the same arithmetic accuracy.
This proves \cref{eq:large-noise-hd-score-error}.

\emph{Step 3: record size.}
The chart networks have the size in
\cref{eq:large-noise-decoupled-cell-size} with $H_{\rm lg}^{\rm dec}\le H_{\rm lg}^{\rm lb}$.
The reciprocal, products, clipping, and the final schedule multiplier add only fixed powers of
$H_{\rm lg}^{\rm lb}$ and do not introduce any additional intrinsic net.
Therefore the width and sparsity scale as the cell count
\[
    N_{\rm cell}\le C\varepsilon_{\rm cell}^{-d/\beta},
\]
times the explicit ambient and logarithmic polynomial factor, giving
\cref{eq:large-noise-hd-size-separated}.
\end{proof}

\clearpage

\section{Score Approximation in the Small-Noise Regime}
\label{sec:app:small-noise-score-approx}

With a positive density lower bound, the small-noise denominator can be controlled uniformly after de-Gaussianization.
This route uses the projection-centered Laplace expansion and the neural approximation of the nearest-point projection.
This section packages the whole small-noise approximation pipeline.  It has four
submodules: active-chart geometry in \cref{sec:app:small-noise-active-geometry}, the
dimension-explicit projection-network primitive in
\cref{sec:app:small-noise-projection-networks}, the projection-centered Laplace
expansion in \cref{sec:app:small-noise-laplace}, and the final network assembly in
\cref{sec:app:small-noise-network-assembly}.

\subsection{Proof architecture and active-chart geometry}
\label{sec:app:small-noise-active-geometry}

\paragraph{Small-noise proof architecture.}
The target result is \cref{cor:small-noise-density-lower-approx}.
The proof first localizes to active charts and discards inactive Gaussian tails.
It then uses reach to make the nearest projection single-valued, expands the de-Gaussianized chart integrals in projection-centered tangent coordinates, implements the projection coordinate by finite-anchor Gauss--Newton networks, and finally builds the stable ratio $\bar\gU/\bar\gQ$.
The main denominator step is \cref{lem:small-noise-degaussianized-Q-lower}; the main network-assembly step is \cref{lem:H-localized-degaussianized-small-noise-chart}, while \cref{prop:small-noise-chart-proj-laplace} records the local chart-level Laplace network on its certified active set.

\subsubsection{Active-chart localization and inactive-tail bounds}

The first task is to decide when a chart can matter.
The compact set $\gS_i$ was fixed in \cref{eq:def:domain:S_i}; it is the part of the manifold on which the $i$-th partition function is active.
If no point of $m_t\gS_i$ lies near $\bx$, the Gaussian factor is uniformly small on the whole chart support, and the corresponding chart integral can be discarded at the target accuracy.

Let
\begin{equation}
    C_{\rm tail}^{\max}
    :=
    1\vee
    2B_p^{\rm proj}B^{\rm Jac}
    \max_{1\le i\le C_{\gM}}
    \sup_{0<\sigma\le1,\ \bx\in\R^D}
    \sigma^{-d}
    \int_{\phi_i(U_i)}
      \exp\!\left(-\frac{\|\bx-\bz_i(\bu)\|_2^2}{8\sigma^2}\right)
    \od\bu .
    \label{eq:def:C-tail-max}
\end{equation}
This constant is finite by the chart bi-Lipschitz bounds and the annular estimate used in the proof of \cref{lem:approx:P-Q:tail}; importantly, it is independent of $\sigma_{\tdown}$.
Choose the default localization constant $c_\star\ge1$ large enough for the finite collection of inactive Gaussian-tail estimates used below, and also so that $c_\star\ge2C_\varrho$.
More precisely, whenever a bound below has the form
\[
    C_{\rm tail}^{\max}C\bigl(1+\log(e/\varepsilon)\bigr)^M
    \varepsilon^{c_\star^2/C},
    \qquad 0<\varepsilon\le1/2,
\]
with one of the finitely many constants $C,M$ arising from the fixed atlas, amplitudes, and retained Laplace monomials, our choice of $c_\star$ makes this quantity at most $\varepsilon/12$.
The same choice is also large enough that, for every $H\ge1$, the $H$-localized
de-Gaussianized inactive-tail bounds below have the form
\[
    C_{\rm tail}^{\max}C(1+H)^M e^{-c c_\star^2H}
    \le e^{-H}
\]
for the finite collection of constants $C,M,c$ used there.
Such a choice is possible because $(1+\log(e/\varepsilon))^M\varepsilon^\gamma\le C_{\gamma,M}\varepsilon$ on $(0,1/2]$ for all sufficiently large $\gamma$, and because, for each fixed triple $C,M,c$,
\[
    \sup_{H\ge1} C_{\rm tail}^{\max}C(1+H)^M
    \exp\{-(c c_\star^2-1)H\}\to0
    \qquad\text{as }c_\star\to\infty .
\]
The collection of such triples is finite, so one $c_\star$ satisfies all of them.
No lower-noise prefactor appears in the annular tail estimate.
For any auxiliary radius multiplier $c_a>0$ and any $0<\varepsilon<1$, define the parameterized active region
\begin{equation}
    \gK_i^{(c_a)}(\varepsilon)
    :=
    \bigl\{
      (\bx, t): t\in[\tdown,\tup],\ \bx\in A_t,
      \exists \by \in \gS_i \text{ s.t. }
      \|\bx - m_t\by\|_2 \leq c_a\sigma_t\sqrt{\log \varepsilon^{-1}}
    \bigr\}.
    \label{eq:def:K-i-ca}
\end{equation}
The default active set is
\begin{equation}
    \gK_i(\varepsilon):=\gK_i^{(c_\star)}(\varepsilon).
    \label{eq:def:K-i}
\end{equation} 
Thus $c_\star$ is used for the default active set throughout; $c_a$ only denotes a temporary enlargement factor in statements involving $\gK_i^{(c_a)}(\varepsilon)$.

\begin{lemma}[Inactive-chart Gaussian tail bounds for $\gP_i$ and $\gQ_i$]
\label{lem:approx:P-Q:tail}
    Fix any chart $i \in \{1, \dots, C_{\gM}\}$ and every $0 < \varepsilon < 1$.
    Let $c_a>0$, $t\in[\tdown,\tup]$, and $\bx\in A_t$.
    If $(\bx, t) \notin \gK_i^{(c_a)}(\varepsilon)$, then for every $\bu \in \phi_i(U_i)$ with $\bar{\rho}_i(\bu) \neq 0$,
    \begin{equation}
        \|\bx-m_t\bz_i(\bu)\|_2
        > c_a\sigma_t\sqrt{\log\varepsilon^{-1}}.
        \label{eq:inactive-chart:lower-bound}
    \end{equation}
    Consequently, for a constant $C_{\rm tail}$ depending only on the fixed atlas, amplitude bounds, $d$, and the schedule bound $\underline m$, but not on $\sigma_{\tdown}$,
    \begin{equation}
        |\gQ_i(\bx, t)|+
        \max_{1 \leq k \leq D}|\gP_{i,k}(\bx, t)|
        \leq
        C_{\rm tail}
        \bigl(1+c_a\sqrt{\log(e/\varepsilon)}\bigr)^d
        \varepsilon^{c_a^2/8} .
        \label{eq:inactive-chart:tail-bound}
    \end{equation}
\end{lemma}
\begin{proof}[Proof of \cref{lem:approx:P-Q:tail}]
If $\bar{\rho}_i(\bu) \neq 0$, then $\bz_i(\bu) \in \gS_i$.
Failure of membership in $\gK_i^{(c_a)}(\varepsilon)$ means that no $\by\in\gS_i$ satisfies the active distance condition.
This gives \cref{eq:inactive-chart:lower-bound}.

Put $R:=c_a\sqrt{\log\varepsilon^{-1}}$.
The chart maps are uniformly bi-Lipschitz on their compact supports and $m_t\ge\underline m>0$.
Hence there is a constant $C_i$, independent of $t$, $\sigma_t$, and $\bx$, such that for every $r>0$,
\begin{equation}
    \vol_d\{\bu\in\phi_i(U_i):\|\bx-m_t\bz_i(\bu)\|_2\le r\}
    \le
    C_i(1+r/\sigma_t)^d\sigma_t^d .
    \label{eq:annular-chart-volume-bound}
\end{equation}
For $r\le1$, this follows from the local bi-Lipschitz bounds after covering $m_t\gS_i$ by a bounded number of coordinate balls; for $r>1$, the right side dominates the fixed chart volume after increasing $C_i$.

Decompose the support into annuli
\[
    A_\ell
    :=
    \{\bu:\ R+\ell < \|\bx-m_t\bz_i(\bu)\|_2/\sigma_t
       \le R+\ell+1\},
    \qquad \ell=0,1,2,\dots .
\]
By \cref{eq:annular-chart-volume-bound}, $\vol_d(A_\ell)\le C_i\sigma_t^d(1+R+\ell)^d$.
Therefore,
\begin{align*}
    \sigma_t^{-d}
    \int_{\phi_i(U_i)}
      \exp\!\left(-\frac{\|\bx-m_t\bz_i(\bu)\|_2^2}{2\sigma_t^2}\right)
    \ind\{\bar\rho_i(\bu)\ne0\}\od\bu
    &\le
    C_i\sum_{\ell\ge0}(1+R+\ell)^d
      e^{-(R+\ell)^2/2}  \\
    &\le
    C_i'(1+R)^d e^{-R^2/8}.
\end{align*}
The last inequality is the standard Gaussian-annulus summation bound; it follows, for instance, by comparing the sum with $e^{-R^2/8}\sum_{\ell\ge0}(1+\ell)^d e^{-\ell^2/8}$.
Since $|a_{i,k}^{(\gP)}|\le |a_i^{(\gQ)}|\le B_p^{\rm proj}B^{\rm Jac}$ by \cref{eq:a-i:bound} and $\|\bz_i\|_\infty\le1$, the same estimate applies to $\gQ_i$ and to every coordinate of $\gP_i$.
Because $e^{-R^2/8}=\varepsilon^{c_a^2/8}$ and $1+R\le1+c_a\sqrt{\log(e/\varepsilon)}$, we obtain \cref{eq:inactive-chart:tail-bound}.
\end{proof}

By the fixed choice of $c_\star$, the bound in \cref{eq:inactive-chart:tail-bound} with $c_a=c_\star$, and also with any larger fixed multiplier, is below the allocated $\Ord(\varepsilon)$ tail budget after multiplication by the finitely many logarithmic factors used later.
Consequently, the projection--Laplace approximation only has to be constructed on the active region $\gK_i(\varepsilon)$.
It is used only during the \emph{small-noise regime} $t\in\gI_{\rm sm}(\varepsilon)$, where $\sigma_t\lesssim\varepsilon^{1/\beta}$.

\subsubsection{Projection geometry on active small-noise charts}

The next lemmas turn the active-distance condition into usable local coordinates.
The buffered boxes keep all projected active points away from chart boundaries, while the reach-tube estimates ensure that the nearest-point projection and the projection-centered objective have stable curvature.
These facts are the analytic foundation for both the Laplace expansion and the neural Gauss--Newton projection network.

\textbf{Buffered coordinate set.}
The compact support $\gS_i$ and buffered boxes $K_i^\circ\Subset\Box_i$ were fixed in \cref{eq:def:domain:S_i,eq:def:buffered-coordinate-boxes,eq:def:Ki-circ}.
Since $\gS_i$ is compact and $\phi_i(\gS_i)\subset\operatorname{int}(K_i^\circ)$, we can define the strictly positive geometric buffer
\begin{equation}
    B_i^{\rm buff}
    :=
    1\wedge
    \dist_{\R^D}\Bigl(
      \gS_i,\,
      \gM\setminus\phi_i^{-1}\bigl(\operatorname{int}(K_i^\circ)\bigr)
    \Bigr)
    >0.
    \label{eq:def:B_i-buff}
\end{equation}
with the convention $\dist(A,\emptyset)=+\infty$.

\textbf{Auxiliary compact chart domains.}
Set
\begin{equation}
    d_i^\Box:=\dist(K_i^\circ,\partial\Box_i)>0.
    \label{eq:def:d-i-box}
\end{equation}

On this single compact set $\Box_i$, define
\begin{align}
    L_i^{(+)}
    &:=
    \sup_{\bu \in \Box_i}
    \|\bJ_i(\bu)\|_{\rm op},
    \qquad
    L_i^{(-)}
    :=
    \Bigl(
      \sup_{\by \in \phi_i^{-1}(\Box_i)}
      \|\oD\phi_i(\by)\|_{\rm op}
    \Bigr)^{-1},
    \\
    \lambda_i^{(-)}
    &:=
    \inf_{\bu\in\Box_i}
    \lambda_{\min}\bigl(\bG_i(\bu)\bigr),
    \qquad
    \lambda_i^{(+)}
    :=
    \sup_{\bu\in\Box_i}
    \lambda_{\max}\bigl(\bG_i(\bu)\bigr),
    \\
    \sfM_i^{(2)}
    &:=
    1\vee
    \sup_{\bu\in\Box_i}
    \sum_{\ell=1}^D
    \|\nabla_\bu^2 z_{i,\ell}(\bu)\|_{\rm op}.
    \label{eq:def:chart-second-derivative-bound}
\end{align}

Since $\bz_i=\phi_i^{-1}$ is smooth on $\phi_i(U_i)$ and $\Box_i\Subset\phi_i(U_i)$, all these constants are finite.
Moreover, because $\bz_i$ is an immersion, $\bG_i(\bu)=\bJ_i(\bu)^\top\bJ_i(\bu)$ is symmetric positive definite for every $\bu\in\Box_i$.
Hence
\begin{equation*}
    0<L_i^{(-)} \leq L_i^{(+)}<\infty,
    \qquad
    0<\lambda_i^{(-)}\le\lambda_i^{(+)}<\infty.
\end{equation*}

Define
\begin{equation}
    \bA_i(\bu)
    :=
    \bG_i(\bu)^{-1}\bJ_i(\bu)^\top,
    \qquad
    \bu\in\Box_i.
    \label{eq:def:A-i-GN}
\end{equation}
Again, because $\bz_i$ is smooth on $\Box_i$ and $\bG_i(\bu)$ is uniformly invertible there, $\bA_i$ is smooth on $\Box_i$.
Hence
\begin{equation}
    L_{A,i}
    :=
    1 \vee
    \sup_{\bu\in\Box_i}
    \|\oD\bA_i(\bu)\|_{\rm op}
    < \infty. 
    \label{eq:def:L-A-i}
\end{equation}

\begin{lemma}[Uniform chart geometry on compact coordinate patches]
\label{lem:proj-chart-geometry}
For each chart index $i$, for all $\bu,\bv\in\Box_i$,
\begin{equation}
    L_i^{(-)}\|\bu-\bv\|_2
    \leq
    \|\bz_i(\bu)-\bz_i(\bv)\|_2
    \leq
    L_i^{(+)}\|\bu-\bv\|_2,
    \label{eq:chart-bilip}
\end{equation}
and for all $\bu\in\Box_i$,
\begin{equation}
    \lambda_i^{(-)}\bI_d
    \preceq
    \bG_i(\bu)
    \preceq
    \lambda_i^{(+)}\bI_d.
    \label{eq:chart-elliptic}
\end{equation}
\end{lemma}

\begin{proof}
For the upper Lipschitz bound, by the fundamental theorem of calculus,
\begin{align*}
    \bz_i(\bu)-\bz_i(\bv)
    &=
    \int_0^1
    \bJ_i\bigl(\bv+s(\bu-\bv)\bigr)(\bu-\bv)
    \od s.
\end{align*}
Since $\Box_i$ is convex, $\bv+\theta(\bu-\bv) \in \Box_i$ for every $\theta \in [0, 1]$.
Therefore,
\begin{align*}
    \|\bz_i(\bu)-\bz_i(\bv)\|_2
    &\leq
    \int_0^1
    \|\bJ_i(\bv+\theta(\bu-\bv))\|_{\rm op}
    \|\bu-\bv\|_2
    \od\theta
    \\
    &\leq
    L_i^{(+)}\|\bu-\bv\|_2.
\end{align*}

For the lower Lipschitz bound, apply the same argument to the inverse map $\phi_i$ on the compact set $\phi_i^{-1}(\Box_i) \subset U_i$:
\begin{align*}
    \|\bu-\bv\|_2
    &=
    \|\phi_i(\bz_i(\bu))-\phi_i(\bz_i(\bv))\|_2
    \\
    &\leq
    \sup_{\by \in \phi_i^{-1}(\Box_i)}
    \|\oD\phi_i(\by)\|_{\rm op}
    \|\bz_i(\bu)-\bz_i(\bv)\|_2.
\end{align*}
By the definition of $L_i^{(-)}$, this gives
\begin{equation*}
    L_i^{(-)}\|\bu-\bv\|_2
    \le
    \|\bz_i(\bu)-\bz_i(\bv)\|_2.
\end{equation*}

Finally, $\bG_i(\bu)$ is continuous and symmetric positive definite on the compact set $\Box_i$.
Hence its smallest and largest eigenvalues attain finite extrema on $\Box_i$, and the definitions of $\lambda_i^{(-)}$ and $\lambda_i^{(+)}$ give \cref{eq:chart-elliptic}.
\end{proof}

\begin{lemma}[Finite-order geometry from compact smoothness]
\label{lem:finite-order-geometry-from-smooth}
Assume \cref{assump:manifold:exact}.
Fix an integer $q\ge1$.
For the finite projection atlas $\{(U_i,\phi_i)\}_{i=1}^{C_{\gM}}$, the partition of unity $\{\rho_i\}$, the compact coordinate boxes $\Box_i$, and the compact tubular domains used in the construction, there exists a finite constant $\Gamma_{\gM,q}\ge e$ such that, uniformly over the chart index $i$,
\begin{align}
    \max_{|\balpha|\le q+4}
    \|\partial^{\balpha}\bz_i\|_{L^\infty(\Box_i)}
    &\le \Gamma_{\gM,q},
    \label{eq:finite-geom-z-bound}
    \\
    \max_{|\balpha|\le q+4}
    \|\partial^{\balpha}\bar\rho_i\|_{L^\infty(\Box_i)}
    +
    \max_{|\balpha|\le q+4}
    \|\partial^{\balpha}\sqrt{\det\bG_i}\|_{L^\infty(\Box_i)}
    &\le \Gamma_{\gM,q}.
    \label{eq:finite-geom-cutoff-jac-bound}
\end{align}
After enlarging $\Gamma_{\gM,q}$, the same bound holds for transition maps between overlapping projection charts, for $\bA_i=\bG_i^{-1}\bJ_i^\top$, for the nearest-projection coordinate maps
\[
    \bx\mapsto \phi_i(\Pi_{\gM}(\bx))
\]
on the compact reach tubes where they are used, and for the corresponding projection residual maps.
Schedule-dependent maps such as $(\bx,t)\mapsto\phi_i(\Pi_{\gM}(\bx/m_t))$ are controlled by composing these geometric maps with the separately approximated schedule $m_t^{-1}$; their constants therefore depend on $\Gamma_{\gM,q}$ and the fixed lower bound $\underline m$, not on any new geometry assumption.
The constant can also be chosen to dominate the finite atlas cardinality, fixed coordinate buffers, tube constants, chart-volume constants, and their reciprocals.
With $\mathcal K_{\gM,q}$ denoting the high-order atlas envelope from
\cref{def:finite-order-geometry-constants}, it may be chosen so that
\[
    \Gamma_{\gM,q}
    \le
    C_{d,q}
    (1\vee S_{\gM})^{\mathfrak a_{\rm geom}(q)}
    (1\vee\kappa^{-1})^{\mathfrak a_{\rm geom}(q)}
    \mathcal K_{\gM,q}^{\mathfrak a_{\rm geom}(q)},
\]
where $C_{d,q}$ and $\mathfrak a_{\rm geom}(q)$ depend only on $d$ and $q$.
\end{lemma}

\begin{proof}
For each $\by\in\gM$, let $\bV_\by$ be an orthonormal frame for $T_\by\gM$.
The map $\phi_\by(\bz)=\bV_\by^\top(\bz-\by)$, restricted to $\gM$, has differential equal to the identity at $\by$.
The inverse function theorem therefore gives a $C^\infty$ projection chart near $\by$.
Compactness of $\gM$ gives a finite subcover, and the construction in \cref{sec:app:approx:manifold:covering-ball} shrinks these charts to compact coordinate boxes $\Box_i\Subset\phi_i(U_i)$.

Each inverse chart $\bz_i=\phi_i^{-1}$ is $C^\infty$.
Hence every partial derivative $\partial^{\balpha}\bz_i$, $|\balpha|\le q+4$, is continuous on the compact set $\Box_i$ and is bounded.
Taking the maximum over the finitely many charts proves \cref{eq:finite-geom-z-bound}.
The functions $\bar\rho_i$ are smooth because the partition of unity is smooth, and $\sqrt{\det\bG_i}$ is smooth because $\bG_i=\bJ_i^\top\bJ_i$ is smooth and uniformly positive definite on $\Box_i$.
Their derivatives up to order $q+4$ are therefore bounded on compact boxes, proving \cref{eq:finite-geom-cutoff-jac-bound} after increasing $\Gamma_{\gM,q}$.

Transition maps and $\bA_i=\bG_i^{-1}\bJ_i^\top$ are compositions, products, and matrix inverses of smooth maps on compact sets where the relevant determinants are bounded away from zero, so their derivatives up to order $q+4$ are bounded.
Since $\gM$ has positive reach $\kappa$, the nearest projection $\Pi_{\gM}$ is $C^\infty$ on every closed tube of radius strictly smaller than $\kappa$.
The construction only evaluates $\Pi_{\gM}$ on such compact tubes.
Therefore $\bx\mapsto\phi_i(\Pi_{\gM}(\bx))$ and the associated residual maps have bounded derivatives up to order $q+4$ on those tubes.
If the argument is $\bx/m_t$, the chain rule only adds fixed powers of $m_t^{-1}$, which are bounded by $\underline m^{-1}$ on $[\tdown,\tup]$.
The atlas cardinality and tube constant obey the reach-volume bounds
\[
    C_{\gM}\le C_dS_{\gM}(1\vee\kappa^{-1})^d,
    \qquad
    C_{\rm tube}\le C_d(1\vee S_{\gM})(1\vee\kappa^{-1})^d.
\]
All remaining auxiliary constants are among the finite $C^{q+4}$-norms and buffer
reciprocals dominated by $\mathcal K_{\gM,q}$, or are obtained from them by finitely many
products, compositions, and matrix inverses of dimension at most $d$.
Enlarging $C_{d,q}$ and $\mathfrak a_{\rm geom}(q)$ gives the displayed explicit
bound for $\Gamma_{\gM,q}$, and completes the proof.
\end{proof}

Let
\begin{equation}
    \delta_{\rm act}^{(c_a)}(\varepsilon,t)
    :=
    c_a\frac{\sigma_t}{m_t}\sqrt{\log\varepsilon^{-1}},
    \qquad
    \bxi(\bx,t):=\frac{\bx}{m_t}.
    \label{eq:def:delta-act-ca}
\end{equation}
We write
\begin{equation}
    \delta_{\rm act}(\varepsilon,t)
    :=
    \delta_{\rm act}^{(c_\star)}(\varepsilon,t)
    \label{eq:def:delta-act}
\end{equation}

Since the atlas is finite, define the uniform projection and Gauss--Newton margin
\begin{equation}
    \eta_{\Pi}
    :=
    \min_{1\le i\le C_{\gM}}
    \min\Bigl\{
      \kappa,
      \frac{1}{2} B_i^{\rm buff},
      \frac{\lambda_i^{(-)}}{\sfM_i^{(2)}},
      \frac{1}{2L_{A,i}}
    \Bigr\}
    >0.
    \label{eq:def:projection-margin}
\end{equation}

All projection-centered arguments below are used only for $t \in \gI_{\rm sm}(\varepsilon)$.
Choose $\varepsilon>0$ small enough that, for every $t \in \gI_{\rm sm}(\varepsilon)$,
\begin{equation}
    \delta_{\rm act}(\varepsilon,t)
    <
    \frac{1}{2}\eta_{\Pi}.
    \label{eq:condition:delta}
\end{equation}
For a fixed $c_a>0$, the corresponding $c_a$-active localization condition is the same display with $\delta_{\rm act}^{(c_a)}$ in place of $\delta_{\rm act}$.

For the Laplace localization argument, define the compact ambient set
\begin{equation}
    \mathcal X_i
    :=
    \{\bzeta\in\R^D:\dist_{\R^D}(\bzeta,\gS_i)\le \eta_\Pi/2\}.
    \label{eq:def:X-i-laplace}
\end{equation}
Since $\bz_i$ is smooth on $\Box_i$ and $\mathcal X_i$ is compact, there is a finite constant $H_i\ge1$ such that, for $\Phi_i(\bu;\bzeta):=\frac12\|\bzeta-\bz_i(\bu)\|_2^2$,
\begin{equation}
    \bigl\|
      \nabla_\bu^2\Phi_i(\bu;\bzeta)
      -
      \nabla_\bu^2\Phi_i(\bv;\bzeta)
    \bigr\|_{\rm op}
    \le
    H_i\|\bu-\bv\|_2
    \label{eq:def:H-i-laplace}
\end{equation}
for all $\bu,\bv\in\Box_i$ and $\bzeta\in\mathcal X_i$.
Define
\begin{equation}
    r_i:=
    \min\Bigl\{
      d_i^\Box/2,\,
      \lambda_i^{(-)}/[4(1+H_i)]
    \Bigr\}>0.
    \label{eq:def:r-i-laplace}
\end{equation}

For each chart $i$, define the fixed chart-geometry constant
\begin{align*}
    \mathfrak{C}_i
    &:=
    1 \vee
    L_i^{(+)}
    \vee (L_i^{(-)})^{-1}
    \vee \lambda_i^{(+)}
    \vee (\lambda_i^{(-)})^{-1}
    \vee \sfM_i^{(2)}
    \vee L_{A,i}
    \vee H_i
    \vee (d_i^\Box)^{-1}
    \vee r_i^{-1}
    \\
    &\qquad\vee
    \max_{1\le \ell\le D}
    \|\bz_{i,\ell}\|_{C^{\floor{\beta}+3}(\Box_i)}
    \vee
    \|\bG_i\|_{C^{\floor{\beta}+2}(\Box_i)}
    \vee
    \|\bG_i^{-1}\|_{C^{\floor{\beta}+2}(\Box_i)}.
\end{align*}
Since $\Box_i \Subset \phi_i(U_i)$ and $\bz_i$ is smooth, this quantity is finite.
Since the atlas is finite, set
\begin{equation}
    \mathfrak{C}_{\gM}
    :=
    \max_{1 \leq i \leq C_{\gM}}\mathfrak{C}_i < \infty.
    \label{eq:def:geom-bound}
\end{equation}

\begin{lemma}[Nearest-point projection on the active small-noise chart]
\label{lem:buffered-active-chart}
    Assume $t \in \gI_{\rm sm}(\varepsilon)$ and \cref{eq:condition:delta}.
    If $(\bx,t) \in \gK_i(\varepsilon)$, then the global nearest-point projection $\Pi_{\gM}(\bxi(\bx,t))$ is well-defined and unique.
    Moreover, $\Pi_{\gM}(\bxi(\bx,t))\in U_i$, and
    \begin{equation}
        \bu_i^{\Pi}(\bx,t)
        :=
        \phi_i\bigl(\Pi_{\gM}(\bxi(\bx,t))\bigr)
        \in
        K_i^\circ.
        \label{eq:def:u_i-Pi}
    \end{equation}
    Furthermore, with
    \begin{equation*}
        \br_i(\bx,t)
        :=
        \bx-m_t\Pi_{\gM}(\bxi(\bx,t)),
    \end{equation*}
    we have the orthogonality relation
    \begin{equation}
        \bJ_i\bigl(\bu_i^{\Pi}(\bx,t)\bigr)^\top
        \br_i(\bx,t)
        =
        0.
        \label{eq:projection-orthogonality-r}
    \end{equation}
\end{lemma}

\begin{proof}
Write $\bxi=\bxi(\bx,t)$.
Since $(\bx,t) \in \gK_i(\varepsilon)$, there exists $\by\in\gS_i$ such that
\begin{equation}
    \|\bxi-\by\|_2
    \leq
    \delta_{\rm act}(\varepsilon,t).
    \label{eq:active-distance-to-Si}
\end{equation}
Because $\by \in \gM$, we have
\begin{equation}
    \dist(\bxi,\gM)
    \leq
    \|\bxi-\by\|_2
    \leq
    \delta_{\rm act}(\varepsilon,t)
    <
    \kappa.
\end{equation}
Thus $\bxi$ lies inside the reach tube of $\gM$, so the nearest-point projection $\Pi_{\gM}(\bxi)$ is well-defined and unique.

Next, since $\Pi_{\gM}(\bxi)$ is the closest point on $\gM$ to $\bxi$,
\begin{equation*}
    \|\bxi-\Pi_{\gM}(\bxi)\|_2
    =
    \dist(\bxi,\gM)
    \leq
    \|\bxi-\by\|_2.
\end{equation*}
Therefore,
\begin{align*}
    \|\Pi_{\gM}(\bxi)-\by\|_2
    &\leq
    \|\Pi_{\gM}(\bxi)-\bxi\|_2
    +
    \|\bxi-\by\|_2
    \notag \\
    &=
    \dist(\bxi,\gM)+\|\bxi-\by\|_2
    \notag \\
    &\leq
    2\|\bxi-\by\|_2
    \leq
    2\delta_{\rm act}(\varepsilon,t)
    <
    \frac{1}{2} B_i^{\rm buff}.
\end{align*}
Hence
\begin{equation*}
    \dist_{\R^D}(\Pi_{\gM}(\bxi),\gS_i)
    <
    \frac{1}{2} B_i^{\rm buff}.
\end{equation*}
By the definition of $B_i^{\rm buff}$, this implies $\Pi_{\gM}(\bxi)\in \phi_i^{-1}(\operatorname{int}(K_i^\circ))$.
Hence
\begin{equation*}
    \bu_i^\Pi(\bx,t)
    =
    \phi_i(\Pi_{\gM}(\bxi))
    \in
    K_i^\circ.
\end{equation*}
In particular, $\Pi_{\gM}(\bxi)\in U_i$.

It remains to prove the orthogonality statement.
Since $\bxi$ lies inside the reach tube, the metric projection residual is normal to the tangent space:
\begin{equation*}
    \bxi-\Pi_{\gM}(\bxi)
    \perp
    T_{\Pi_{\gM}(\bxi)}\gM.
\end{equation*}
The columns of $\bJ_i(\bu_i^\Pi)=\nabla\bz_i(\bu_i^\Pi)$ form a basis of this tangent space.
Therefore,
\begin{equation*}
    \bJ_i(\bu_i^\Pi)^\top
    \bigl(\bxi-\Pi_{\gM}(\bxi)\bigr)
    =
    0.
\end{equation*}
Since $\br_i(\bx,t) = \bx-m_t\Pi_{\gM}(\bxi) = m_t\bigl(\bxi-\Pi_{\gM}(\bxi)\bigr)$, multiplying the previous display by $m_t$ gives
\begin{equation*}
    \bJ_i(\bu_i^\Pi)^\top\br_i(\bx,t)=0.
\end{equation*}
\end{proof}

For $(\bx,t) \in \gK_i(\varepsilon)$, define
\begin{equation}
    \Phi_i(\bu;\bx,t)
    :=
    \frac{1}{2}
    \|\bxi(\bx,t)-\bz_i(\bu)\|_2^2,
    \qquad
    \bg_i(\bu;\bx,t)
    :=
    \nabla_\bu\Phi_i(\bu;\bx,t).
\end{equation}
Since
\begin{equation*}
    \Phi_i(\bu;\bx,t)
    =
    \frac{1}{2}
    \sum_{\ell=1}^D
    \bigl(\xi_\ell(\bx,t)-z_{i,\ell}(\bu)\bigr)^2,
\end{equation*}
we have
\begin{equation}
    \bg_i(\bu;\bx,t)
    =
    \bJ_i(\bu)^\top
    \bigl(\bz_i(\bu)-\bxi(\bx,t)\bigr).
    \label{eq:def:g-i}
\end{equation}
Differentiating once more gives
\begin{equation}
    \nabla_\bu\bg_i(\bu;\bx,t)
    =
    \bG_i(\bu)
    -
    \sum_{\ell=1}^D
    \bigl(
      \xi_\ell(\bx,t)-z_{i,\ell}(\bu)
    \bigr)
    \nabla_\bu^2 z_{i,\ell}(\bu).
    \label{eq:hessian-g-i}
\end{equation}

At $\bu=\bu_i^\Pi(\bx,t)$, we have $\bz_i(\bu_i^\Pi(\bx,t))=\Pi_{\gM}(\bxi(\bx,t))$.
Hence, using \cref{eq:projection-orthogonality-r},
\begin{align}
    \bg_i(\bu_i^\Pi(\bx,t);\bx,t)
    &=
    \bJ_i(\bu_i^\Pi)^\top
    \bigl(
      \Pi_{\gM}(\bxi(\bx,t))-\bxi(\bx,t)
    \bigr)
    \notag \\
    &=
    -m_t^{-1}
    \bJ_i(\bu_i^\Pi)^\top
    \br_i(\bx,t)
    =
    0.
    \label{eq:g-i-vanishes-at-projection}
\end{align}

\begin{lemma}[Uniform invertibility near the projection]
\label{lem:invertible:proj}
Assume $t \in \gI_{\rm sm}(\varepsilon)$ and \cref{eq:condition:delta}.
If $(\bx,t) \in \gK_i(\varepsilon)$, then
\begin{equation}
    \nabla_\bu
    \bg_i(\bu_i^\Pi(\bx,t);\bx,t)
    \succeq
    \frac{1}{2}\lambda_i^{(-)}\bI_d.
    \label{eq:invertible-proj-bound}
\end{equation}
\end{lemma}

\begin{proof}
Write $\bxi=\bxi(\bx,t)$ and $\bu_i^\Pi=\bu_i^\Pi(\bx,t)$.
By \cref{lem:buffered-active-chart},
\begin{equation*}
    \bu_i^\Pi\in K_i^\circ\subset\Box_i,
    \qquad
    \bz_i(\bu_i^\Pi)=\Pi_{\gM}(\bxi).
\end{equation*}
Fix any $\bw\in\R^d$.
By \cref{eq:hessian-g-i},
\begin{equation}
    \bw^\top\nabla_\bu\bg_i(\bu_i^\Pi;\bx,t)\bw
    =
    \bw^\top\bG_i(\bu_i^\Pi)\bw
    -
    \sum_{\ell=1}^D
    \bigl(
      \xi_\ell-z_{i,\ell}(\bu_i^\Pi)
    \bigr)
    \bw^\top\nabla_\bu^2z_{i,\ell}(\bu_i^\Pi)\bw.
    \label{eq:invertibility-quadratic-form}
\end{equation}

Because $\bu_i^\Pi\in\Box_i$, \cref{eq:chart-elliptic} gives $\bw^\top\bG_i(\bu_i^\Pi)\bw \geq \lambda_i^{(-)}\|\bw\|_2^2$.
For the perturbation term, use $\bz_i(\bu_i^\Pi)=\Pi_{\gM}(\bxi)$ and the definition of $\sfM_i^{(2)}$:
\begin{align}
    \Biggl|
    \sum_{\ell=1}^D
    \bigl(
      \xi_\ell-z_{i,\ell}(\bu_i^\Pi)
    \bigr)
    \bw^\top\nabla_\bu^2z_{i,\ell}(\bu_i^\Pi)\bw
    \Biggr|
    &\leq
    \sum_{\ell=1}^D
    |\xi_\ell-z_{i,\ell}(\bu_i^\Pi)|
    \|\nabla_\bu^2z_{i,\ell}(\bu_i^\Pi)\|_{\rm op}
    \|\bw\|_2^2
    \notag \\
    &\leq
    \|\bxi-\Pi_{\gM}(\bxi)\|_2
    \sum_{\ell=1}^D
    \|\nabla_\bu^2z_{i,\ell}(\bu_i^\Pi)\|_{\rm op}
    \|\bw\|_2^2
    \notag \\
    &\leq
    \dist(\bxi,\gM)\,
    \sfM_i^{(2)}
    \|\bw\|_2^2.
    \label{eq:second-fundamental-perturbation-bound}
\end{align}
Since $(\bx,t) \in \gK_i(\varepsilon)$, there exists $\by\in\gS_i\subset\gM$ such that $\|\bxi-\by\|_2 \leq \delta_{\rm act}(\varepsilon,t)$.
Therefore $\dist(\bxi,\gM) \leq \delta_{\rm act}(\varepsilon,t)$.
By \cref{eq:condition:delta,eq:def:projection-margin},
\begin{equation*}
    \delta_{\rm act}(\varepsilon,t)\sfM_i^{(2)}
    <
    \frac{1}{2}\lambda_i^{(-)}.
\end{equation*}
Combining this with \cref{eq:invertibility-quadratic-form,eq:second-fundamental-perturbation-bound} and $\bw^\top\bG_i(\bu_i^\Pi)\bw \geq \lambda_i^{(-)}\|\bw\|_2^2$, we obtain
\begin{equation*}
    \bw^\top
    \nabla_\bu\bg_i(\bu_i^\Pi;\bx,t)
    \bw
    \ge
    \frac{1}{2}\lambda_i^{(-)}
    \|\bw\|_2^2.
\end{equation*}
Since this holds for every $\bw\in\R^d$, the matrix inequality
\begin{equation*}
    \nabla_\bu
    \bg_i(\bu_i^\Pi;\bx,t)
    \succeq
    \lambda_i^{(-)}\bI_d/2
\end{equation*}
\end{proof}

\subsection{Finite-anchor projection networks}
\label{sec:app:small-noise-projection-networks}
\label{sec:app:approx:manifold:nn}

This section isolates the neural implementation of the nearest-point projection.
It is the only nonlinear geometric approximation step that could otherwise create a
hidden ambient-dimensional cost.  The construction works strictly in intrinsic chart
coordinates: finitely many anchors initialize local Gauss--Newton iterations, and a
ReLU objective gate selects near-minimizing candidates.  Later network-assembly sections
reuse these projection-coordinate networks together with standard schedule, product,
reciprocal, and exponential subnetworks.

Throughout this section, every construction involving $m_t$, $\sigma_t$, or powers of their reciprocals uses \cref{lem:approx:m-sigma} and the derived schedule lemmas stated below.

For orientation, the projection-centered network assembly uses the following objects in
this order:
\begin{enumerate}[\quad(1)]
    \item $\bxi(\bx, t) = \bx/m_t$,
    
    \item the global nearest-point projection $\Pi_{\gM}(\bxi(\bx, t))$,
    
    \item the local coordinate $\bu_i^{\Pi}(\bx, t) = \phi_i\bigl(\Pi_{\gM}(\bxi(\bx, t))\bigr)$,
    
    \item the normalized residual $\bnu_i(\bx, t) = \sigma_t^{-1}\br_i = \sigma_t^{-1}\bigl(\bx - m_t\Pi_{\gM}(\bxi(\bx, t))\bigr)$,
    
    \item the Gaussian factor $\exp(-\|\bnu_i\|_2^2/2)$,
    
    \item and the coefficient functions $A_{i,q,\blambda}^{(R)}\bigl(\bu_i^{\Pi}(\bx, t), t\bigr)$.
\end{enumerate}

The projection-coordinate networks are the only part that uses the manifold geometry in a substantial way.
Once the projection center is available, the residual, Gaussian factor, and Laplace coefficients are obtained by combining standard schedule, H\"older, multiplication, reciprocal, and exponential subnetworks with carefully allocated errors.

In the following, we construct subnetworks to approximate the components of \cref{eq:laplace-expansion-local-chart}:

\begin{table}[htp]
    \centering
    \begin{tabular}{l|l}
        \hline
        Function & Approximation lemma \\
        \hline
        $t \mapsto m_t$ & \multirow{2}{*}{\cref{lem:approx:m-sigma}} \\
        $t \mapsto \sigma_t$ & \\
        \hline
        $t \mapsto \sigma_t^k$ & \cref{lem:approx:sigma-k} \\
        \hline
        $t \mapsto 1/\sigma_t$ & \multirow{2}{*}{\cref{lem:approx:sigma-k-inv}} \\
        $t \mapsto 1/\sigma_t^k$ & \\
        \hline
        $t \mapsto 1/m_t$ & \cref{lem:approx:m-inv} \\
        $(\bx, t) \mapsto \bxi(\bx, t) := \bx/m_t$ & \cref{lem:approx:x-over-m} \\
        \hline
        $(\bx, t) \to \Pi_{\gM}(\bxi(\bx, t))$ & \cref{thm:global-proj-NN} \\
        $(\bx, t) \to \bu_i^{\Pi}(\bx, t) := \phi_i(\Pi_{\gM}(\bxi(\bx, t)))$ & \cref{cor:active-chart-coordinate-from-proj-NN} \\
        \hline
        $(\bx, t) \mapsto \br_i(\bx, t) := \bx - m_t\Pi_{\gM}(\bxi(\bx, t))$ & \multirow{2}{*}{\cref{lem:approx:r-i--nu-i}} \\
        $(\bx, t) \mapsto \bnu_i(\bx, t) := \sigma_t^{-1}\br_i(\bx, t)$ & \\
        \hline
        $(\bx, t) \mapsto \exp\bigl(-\frac{\|\br_i(\bx, t)\|_2^2}{2\sigma_t^2}\bigr)$ & \cref{lem:approx:gaussian-factor} \\
        \hline
        $(\bx, t) \to A_{i,q,\blambda}^{(\gQ)}\bigl(\bu_i^{\Pi}(\bx, t),t\bigr)$ & \multirow{2}{*}{\cref{lem:approx:A-coeff,cor:approx:weighted-A-coeff}} \\
        $(\bx, t) \to A_{i,q,\blambda}^{(\gP)}\bigl(\bu_i^{\Pi}(\bx, t),t\bigr)$  & \\
        \hline
        active chart expansion for $\gQ_i$ & \multirow{2}{*}{\cref{prop:small-noise-chart-proj-laplace}} \\
        active chart expansion for $\gP_i$ & \\
        \hline
    \end{tabular}
\end{table}

\subsubsection{Gauss--Newton networks for nearest-point projection}
\label{sec:app:approx:manifold:nn:proj}

We approximate the nearest-point projection in local intrinsic coordinates.
The nonlinear unknown is the chart coordinate $\bu \in \Box_i \subset \R^d$, whereas the ambient variable enters the iteration only through $\bxi(\bx,t) = \bx/m_t$ and matrix-vector operations.
The lemmas in this block first prove a stable intrinsic Gauss--Newton step, then initialize it from finitely many anchors, and finally combine all chart-anchor candidates by an objective gate.
This produces coordinate-valued and ambient-valued projection networks without ever applying a chart map to an approximate ambient point.

We define the coordinate-wise clipping map:
\begin{equation}
    \Pi_i^{(\rm box)}(\bu)
    :=
    \bigl(
      a_{i,j}^{\Box}
      +
      \ReLU(u_j-a_{i,j}^{\Box})
      -
      \ReLU(u_j-b_{i,j}^{\Box})
    \bigr)_{j=1}^d.
    \label{eq:def:Pi-box}
\end{equation}

$\Pi_i^{(\rm box)}$ explicitly translates standard $\min/\max$ bounding operations into exact ReLU arithmetic.
Consequently, $\Pi_i^{(\rm box)}:\R^d \to \Box_i$ is exactly representable by a shallow ReLU network without any approximation error.
Furthermore, because it is a projection onto a convex set, it is inherently $1$-Lipschitz.

\begin{lemma}[Local Gauss--Newton contraction for nearest-point projection]
\label{lem:GN-contraction}
    Define the intrinsic Gauss--Newton map
    \begin{equation}
        \gT_i(\bu;\bx,t)
        :=
        \Pi_i^{({\rm box})}
        \bigl(
          \bu+
          \bA_i(\bu)
          \bigl(\bxi(\bx,t)-\bz_i(\bu)\bigr)
        \bigr),
        \qquad
        \text{where }
        \bA_i(\bu)=\bG_i(\bu)^{-1}\bJ_i(\bu)^\top .
        \label{eq:def:Gauss-Newton}
    \end{equation}
    Set
    \begin{equation}
        r_i^{\rm GN}
        :=
        \min\Bigl\{
          1,
          \frac{1}{2}\dist(K_i^\circ,\partial\Box_i),
          \frac{1}{4L_{A,i}L_i^{(+)}}
        \Bigr\}. 
        \label{eq:def:r-GN}
    \end{equation}
    Assume $t \in \gI_{\rm sm}(\varepsilon)$, $(\bx,t) \in \gK_i(\varepsilon)$, and \cref{eq:condition:delta}.
    Then
    \begin{equation}
        \|\gT_i(\bu;\bx,t)-\bu_i^\Pi(\bx,t)\|_2
        \leq
        \frac{1}{2}
        \|\bu-\bu_i^\Pi(\bx,t)\|_2
        \label{eq:GN-contraction}
    \end{equation}
    whenever
    \begin{equation*}
        \bu \in \gB(\bu_i^\Pi(\bx,t),r_i^{\rm GN})\cap\Box_i .
    \end{equation*}
\end{lemma}

\begin{proof}
Write $\bxi:=\bxi(\bx,t), \bu^\Pi:=\bu_i^\Pi(\bx,t)$.
Before clipping, define
\begin{equation*}
    \overline{\gT}_i(\bu;\bx,t)
    :=
    \bu+
    \bA_i(\bu)
    \bigl(\bxi-\bz_i(\bu)\bigr).
\end{equation*}
Since
\begin{equation*}
    \bA_i(\bu)\bJ_i(\bu)
    =
    \bG_i(\bu)^{-1}\bJ_i(\bu)^\top\bJ_i(\bu)
    =
    \bI_d,
\end{equation*}
differentiating with respect to $\bu$ gives, for every $\bh\in\R^d$,
\begin{align}
    \oD_\bu\overline{\gT}_i(\bu;\bx,t)[\bh]
    &=
    \bh
    +
    \oD\bA_i(\bu)[\bh]
    \bigl(\bxi-\bz_i(\bu)\bigr)
    -
    \bA_i(\bu)\bJ_i(\bu)\bh
    \notag\\
    &=
    \oD\bA_i(\bu)[\bh]
    \bigl(\bxi-\bz_i(\bu)\bigr).
    \label{eq:GN-derivative}
\end{align}

By \cref{lem:buffered-active-chart},
\begin{equation*}
    \bz_i(\bu^\Pi)=\Pi_{\gM}(\bxi),
    \qquad
    \bJ_i(\bu^\Pi)^\top
    \bigl(\bxi-\Pi_{\gM}(\bxi)\bigr)=0.
\end{equation*}
Therefore
\begin{equation*}
    \bA_i(\bu^\Pi)
    \bigl(\bxi-\bz_i(\bu^\Pi)\bigr)
    =
    \bG_i(\bu^\Pi)^{-1}
    \bJ_i(\bu^\Pi)^\top
    \bigl(\bxi-\Pi_{\gM}(\bxi)\bigr)
    =
    0,
\end{equation*}
and hence
\begin{equation}
    \overline{\gT}_i(\bu^\Pi;\bx,t)=\bu^\Pi .
    \label{eq:GN-fixed-point-unclipped}
\end{equation}
Since $\bu^\Pi\in K_i^\circ\subset\Box_i$, the clipping map fixes $\bu^\Pi$:
\begin{equation}
    \Pi_i^{({\rm box})}(\bu^\Pi)=\bu^\Pi .
    \label{eq:box-fixes-uPi}
\end{equation}

Now let $\bu\in\gB(\bu^\Pi,r_i^{\rm GN})\cap\Box_i$.
For $s\in[0,1]$, set
\begin{equation*}
    \bu_s:=\bu^\Pi+s(\bu-\bu^\Pi).
\end{equation*}
Because $\Box_i$ is convex and both $\bu^\Pi,\bu\in\Box_i$, we have $\bu_s\in\Box_i$.
Also,
\begin{equation*}
    \|\bu_s-\bu^\Pi\|_2
    \leq
    \|\bu-\bu^\Pi\|_2
    \leq
    r_i^{\rm GN}.
\end{equation*}
Using \cref{eq:condition:delta}, the definition of $\eta_\Pi$, and the fact that $(\bx,t) \in \gK_i(\varepsilon)$, we have
\begin{equation}
    \dist(\bxi,\gM)
    \leq
    \delta_{\rm act}(\varepsilon,t)
    <
    \frac{1}{4L_{A,i}}.
    \label{eq:GN-normal-distance-bound}
\end{equation}
Therefore, by the chart Lipschitz bound \cref{eq:chart-bilip},
\begin{align}
    \|\bxi-\bz_i(\bu_s)\|_2
    &\leq
    \|\bxi-\Pi_{\gM}(\bxi)\|_2
    +
    \|\Pi_{\gM}(\bxi)-\bz_i(\bu_s)\|_2
    \notag\\
    &=
    \dist(\bxi,\gM)
    +
    \|\bz_i(\bu^\Pi)-\bz_i(\bu_s)\|_2
    \notag\\
    &\leq
    \delta_{\rm act}(\varepsilon,t)
    +
    L_i^{(+)}\|\bu_s-\bu^\Pi\|_2
    \notag\\
    &<
    \frac{1}{4L_{A,i}}
    +
    L_i^{(+)}r_i^{\rm GN}
    \notag\\
    &\leq
    \frac{1}{4L_{A,i}}
    +
    \frac{1}{4L_{A,i}}
    =
    \frac{1}{2L_{A,i}}.
    \label{eq:GN-residual-bound-along-segment}
\end{align}
Combining \cref{eq:GN-derivative,eq:def:L-A-i,eq:GN-residual-bound-along-segment}, we get
\begin{equation*}
    \|\oD_\bu\overline{\gT}_i(\bu_s;\bx,t)\|_{\rm op}
    \leq
    L_{A,i}
    \|\bxi-\bz_i(\bu_s)\|_2
    \leq
    \frac{1}{2}.
\end{equation*}

By the fundamental theorem of calculus and \cref{eq:GN-fixed-point-unclipped},
\begin{align*}
    \bigl\|
      \overline{\gT}_i(\bu;\bx,t)-\bu^\Pi
    \bigr\|_2
    &=
    \bigl\|
      \overline{\gT}_i(\bu;\bx,t)
      -
      \overline{\gT}_i(\bu^\Pi;\bx,t)
    \bigr\|_2
    \notag\\
    &\leq
    \int_0^1
    \|\oD_\bu\overline{\gT}_i(\bu_s;\bx,t)\|_{\rm op}
    \|\bu-\bu^\Pi\|_2
    \od s
    \notag\\
    &\leq
    \frac{1}{2}\|\bu-\bu^\Pi\|_2.
\end{align*}

Finally, $\Pi_i^{({\rm box})}$ is $1$-Lipschitz and fixes $\bu^\Pi$.
Thus,
\begin{align*}
    \|\gT_i(\bu;\bx,t)-\bu^\Pi\|_2
    &=
    \bigl\|
      \Pi_i^{({\rm box})}(\overline{\gT}_i(\bu;\bx,t))
      -
      \Pi_i^{({\rm box})}(\bu^\Pi)
    \bigr\|_2
    \\
    &\leq
    \bigl\|
      \overline{\gT}_i(\bu;\bx,t)-\bu^\Pi
    \bigr\|_2
    \leq
    \frac{1}{2}\|\bu-\bu^\Pi\|_2.
\end{align*}
\end{proof}

\begin{lemma}[ReLU anchor networks for projection initialization]\label{lem:proj-anchor-init}
    For every chart $i$, there exists a finite set
    \begin{equation*}
        \gA_i^{\rm anc}=\{\ba_{i,1},\dots,\ba_{i,N_i}\}\subset K_i^\circ
    \end{equation*}
    depending only on the fixed chart geometry such that
    \begin{equation}
        \forall \bu\in K_i^\circ,\qquad
        \min_{\ba\in\gA_i^{\rm anc}}\|\ba-\bu\|_2
        \leq
        \frac12 r_i^{\rm GN}.
        \label{eq:anchor-net-property}
    \end{equation}
    Moreover, writing $\ell_{i,r}^\circ:=b_{i,r}^{\circ}-a_{i,r}^{\circ}$,
    \begin{equation}
        N_i
        \le
        \prod_{r=1}^d
        \left(
          1+
          \left\lceil
            \frac{\sqrt d\,\ell_{i,r}^\circ}{r_i^{\rm GN}}
          \right\rceil
        \right)
        \le
        \prod_{r=1}^d
        \left(
          2+
          \frac{\sqrt d\,\ell_{i,r}^\circ}{r_i^{\rm GN}}
        \right)
        \le
        C_i^{\rm anc}
        (r_i^{\rm GN})^{-d},
        \qquad
        C_i^{\rm anc}:=
        \prod_{r=1}^d
        \left(
          2+\sqrt d\,\ell_{i,r}^\circ
        \right).
        \label{eq:anchor-net-cardinality}
    \end{equation}
    In particular, $\gA_i^{\rm anc}$ is fixed once and for all and is independent of $\varepsilon$, $\varepsilon_\Pi$, and $(\bx,t)$.
\end{lemma}

\begin{proof}
    For each coordinate $r$, let $m_{i,r}:=\left\lceil\sqrt d\,\ell_{i,r}^\circ/r_i^{\rm GN}\right\rceil$.
    Choose the tensor-product grid
    \[
        \gA_i^{\rm anc}
        :=
        \prod_{r=1}^d
        \left\{
          a_{i,r}^{\circ}
          +
          \frac{k_r}{m_{i,r}}\ell_{i,r}^\circ:
          k_r=0,\dots,m_{i,r}
        \right\}
        \subset K_i^\circ .
    \]
    For any $\bu\in K_i^\circ$, choose the nearest grid point in each coordinate.
    The coordinatewise error is at most $\ell_{i,r}^\circ/(2m_{i,r})\le r_i^{\rm GN}/(2\sqrt d)$, and therefore the Euclidean error is at most $r_i^{\rm GN}/2$.
    This proves \cref{eq:anchor-net-property}.
    The grid cardinality is
    \[
        |\gA_i^{\rm anc}|
        =
        \prod_{r=1}^d(m_{i,r}+1)
        =
        \prod_{r=1}^d
        \left(
          1+
          \left\lceil
            \frac{\sqrt d\,\ell_{i,r}^\circ}{r_i^{\rm GN}}
          \right\rceil
        \right),
    \]
    and the final inequality in \cref{eq:anchor-net-cardinality} follows from $1+\lceil x\rceil\le2+x$, together with $r_i^{\rm GN}\le1$.
\end{proof}

Define the aggregate anchor complexity
\begin{equation}
    A_{\gM}
    :=
    \sum_{i=1}^{C_{\gM}}|\gA_i^{\rm anc}|.
    \label{eq:def:anchor-complexity}
\end{equation}
By \cref{eq:anchor-net-cardinality},
\begin{equation}
    A_{\gM}
    \le
    \sum_{i=1}^{C_{\gM}}
    C_i^{\rm anc}(r_i^{\rm GN})^{-d}.
    \label{eq:anchor-complexity-bound}
\end{equation}

\begin{lemma}[Error propagation for anchored approximate Gauss--Newton iteration]
\label{lem:active-anchored-GN-error}
    Assume $t\in\gI_{\rm sm}(\varepsilon)$, $(\bx,t)\in\gK_i(\varepsilon)$, and \cref{eq:condition:delta}.
    Assume $\ba\in K_i^\circ$ satisfies
    \begin{equation}
        \|\ba-\bu_i^\Pi(\bx,t)\|_2
        \leq
        \frac{1}{2} r_i^{\rm GN}.
        \label{eq:anchor-inside-GN-basin}
    \end{equation}
    Let $\widetilde{\gT}_i$ be a $\Box_i$-valued approximate Gauss--Newton step satisfying
    \begin{equation}
        \sup_{\substack{
            \bu\in\Box_i,\;
            t\in\gI_{\rm sm}(\varepsilon),\;
            (\bx,t) \in \gK_i(\varepsilon)}}
        \bigl\|
          \widetilde{\gT}_i(\bu;\bx,t)-\gT_i(\bu;\bx,t)
        \bigr\|_\infty
        \leq
        \varepsilon_{\rm step}.
        \label{eq:approx-step-error}
    \end{equation}
    Assume moreover that
    \begin{equation}
        \varepsilon_{\rm step}
        \leq
        \frac{r_i^{\rm GN}}{8\sqrt{d}}.
        \label{eq:step-error-basin-condition}
    \end{equation}
    Define the approximate iterates by
    \begin{equation}
        \widetilde{\bu}_0 := \ba,
        \qquad
        \widetilde{\bu}_{k+1}
        :=
        \widetilde{\gT}_i(\widetilde{\bu}_k;\bx,t).
        \label{eq:def:approx-GN-iterates}
    \end{equation}
    Then, every iterate stays in the Gauss--Newton contraction basin:
    \begin{equation}
        \widetilde{\bu}_k
        \in
        \gB\bigl(\bu_i^\Pi(\bx,t),r_i^{\rm GN}\bigr) \cap \Box_i,
        \qquad
        k \geq 0.
        \label{eq:approx-iterates-stay-in-basin}
    \end{equation}
    Moreover, for every $K\ge0$,
    \begin{equation}
        \bigl\|
          \widetilde{\bu}_K-\bu_i^\Pi(\bx,t)
        \bigr\|_2
        \leq
        2^{-(K+1)}r_i^{\rm GN}
        +
        2\sqrt{d}\varepsilon_{\rm step}.
        \label{eq:anchored-GN-error}
    \end{equation}
\end{lemma}

\begin{proof}
Write $\bu^\Pi:=\bu_i^\Pi(\bx,t), e_k:=\|\widetilde{\bu}_k-\bu^\Pi\|_2$.
By the anchor assumption \cref{eq:anchor-inside-GN-basin},
\begin{equation*}
    e_0 =
    \|\ba-\bu^\Pi\|_2
    \leq
    \frac{1}{2}r_i^{\rm GN}.
\end{equation*}

We first prove that the iterates stay inside the contraction basin.
Assume inductively that
\begin{equation*}
    e_k \leq \frac{3}{4}r_i^{\rm GN}.
\end{equation*}
Then $\widetilde{\bu}_k \in \gB(\bu^\Pi,r_i^{\rm GN}) \cap \Box_i$, so the exact Gauss--Newton contraction lemma (\cref{lem:GN-contraction}) applies:
\begin{equation*}
    \|\gT_i(\widetilde{\bu}_k;\bx,t)-\bu^\Pi\|_2
    \leq
    \frac{1}{2} e_k.
\end{equation*}
Using the approximate-step error and the norm comparison $\|\bv\|_2\le\sqrt{d}\,\|\bv\|_\infty$, we obtain
\begin{align*}
    e_{k+1}
    &=
    \|\widetilde{\gT}_i(\widetilde{\bu}_k;\bx,t)-\bu^\Pi\|_2
    \\
    &\leq
    \bigl\|
      \widetilde{\gT}_i(\widetilde{\bu}_k;\bx,t)
      - 
      \gT_i(\widetilde{\bu}_k;\bx,t)
    \bigr\|_2
    +
    \|\gT_i(\widetilde{\bu}_k;\bx,t)-\bu^\Pi\|_2
    \\
    &\leq
    \sqrt{d}\varepsilon_{\rm step}
    +
    \frac{1}{2} e_k
    \\
    &\leq
    \sqrt{d}\varepsilon_{\rm step}
    +
    \frac{3}{8}r_i^{\rm GN}.
\end{align*}
By \cref{eq:step-error-basin-condition},
\begin{equation*}
    \sqrt{d}\,\varepsilon_{\rm step}
    \le
    \frac18 r_i^{\rm GN}.
\end{equation*}
Therefore
\begin{equation*}
    e_{k+1}
    \le
    \frac18 r_i^{\rm GN}
    +
    \frac{3}{8} r_i^{\rm GN}
    =
    \frac{1}{2} r_i^{\rm GN}
    <
    \frac{3}{4} r_i^{\rm GN}.
\end{equation*}
Since $e_0\le r_i^{\rm GN}/2$, induction proves
\begin{equation*}
    e_k\le \frac{3}{4} r_i^{\rm GN}<r_i^{\rm GN}
    \qquad
    \text{for all }k\ge0.
\end{equation*}
Also, because the approximate step $\widetilde{\gT}_i$ is assumed to be $\Box_i$-valued, every $\widetilde{\bu}_k$ lies in $\Box_i$.
Hence \cref{eq:approx-iterates-stay-in-basin} holds.

Now that the exact contraction applies at every step, we have the recursion
\begin{equation}
    e_{k+1}
    \le
    \frac{1}{2} e_k
    +
    \sqrt{d}\,\varepsilon_{\rm step}.
    \label{eq:approx-GN-recursion}
\end{equation}
Iterating \cref{eq:approx-GN-recursion} gives
\begin{align*}
    e_K
    &\leq
    2^{-K}e_0
    +
    \sqrt{d}\,\varepsilon_{\rm step}
    \sum_{j=0}^{K-1}2^{-j}
    \\
    &\leq
    2^{-(K+1)}r_i^{\rm GN}
    +
    2\sqrt{d}\,\varepsilon_{\rm step}.
\end{align*}
\end{proof}

\begin{lemma}[Neural approximation of one Gauss--Newton step]
\label{lem:proj-one-step-NN}
Fix a localization parameter $0<\varepsilon\le1$.
For every $\varepsilon_{\gT} \in (0,1/2]$, there exists a ReLU network $\widetilde{\gT}_{i,\varepsilon_{\gT}} \in \nn(L,\bW,S,B)$ with input $(\bu,\bx,t)\in\R^d\times\R^D\times[\tdown,\tup]$ and output in $\Box_i$ such that
\begin{equation}
    \sup_{\substack{
        \bu\in\Box_i,
        t \in \gI_{\rm sm}(\varepsilon),
        (\bx,t) \in \gK_i(\varepsilon)
    }}
    \bigl\|
      \widetilde{\gT}_{i,\varepsilon_{\gT}}(\bu,\bx,t)
      -
      \gT_i(\bu;\bx,t)
    \bigr\|_\infty
    \le
    \varepsilon_{\gT}.
    \label{eq:one-step-error}
\end{equation}
Moreover, for a finite chart-dependent constant $C_{i,\gT}\ge1$,
\begin{align}
    L
    &\lesssim
    \log(e(d\vee\beta))
    \log^2(C_{i,\gT}D\varepsilon_{\gT}^{-1}),
    \notag\\
    \|\bW\|_\infty
    &\lesssim
    dD
    (d\vee\beta)
    (C_{i,\gT}D\varepsilon_{\gT}^{-1})^{\frac{d}{2(\beta\vee1)}}
    +
    D\log^3(C_{i,\gT}D\varepsilon_{\gT}^{-1}),
    \notag\\
    S
    &\lesssim
    dD(d+2(\beta\vee1))^{d+3}
    (C_{i,\gT}D\varepsilon_{\gT}^{-1})^{\frac{d}{2(\beta\vee1)}}
    \log(C_{i,\gT}D\varepsilon_{\gT}^{-1})
    +
    dD\log^4(C_{i,\gT}D\varepsilon_{\gT}^{-1}),
    \notag\\
    \log B
    &\lesssim
    \log^2(C_{i,\gT}D\varepsilon_{\gT}^{-1}).
    \label{eq:size-one-step}
\end{align}
\end{lemma}

\begin{proof}
Recall that
\begin{equation}
    \gT_i(\bu;\bx,t)
    =
    \Pi_i^{({\rm box})}
    \Bigl(
      \bu+
      \bA_i(\bu)
      \bigl(
        \bxi(\bx,t)-\bz_i(\bu)
      \bigr)
    \Bigr),
    \qquad
    \bA_i(\bu)=\bG_i(\bu)^{-1}\bJ_i(\bu)^\top .
    \label{eq:GN-step-proof-form}
\end{equation}
The clipping map $\Pi_i^{({\rm box})}$ is exactly representable by a shallow ReLU network and is $1$-Lipschitz.

\textbf{Step 1: Uniform bounds.}
Since $\Box_i$ is compact and $\bz_i,\bA_i$ are $C^{\infty}$ on $\Box_i$, there exists a finite constant $B_{i,\gT} \ge 1$ such that every scalar component of $\bz_i$ and $\bA_i$ belongs to $\gH^{2(\beta\vee1)}(\Box_i,B_{i,\gT})$, and
\begin{equation}
    \|\bz_i(\bu)\|_\infty
    +
    \|\bA_i(\bu)\|_\infty
    \leq
    B_{i,\gT},
    \qquad
    \bu \in \Box_i.
    \label{eq:GN-step-uniform-chart-bounds}
\end{equation}
Here $\|\bA_i(\bu)\|_\infty$ denotes the maximum absolute entry of the $d\times D$ matrix $\bA_i(\bu)$.

Moreover, on the active small-noise region there exists a finite constant $M_{x,i}\ge1$ such that
\begin{equation}
    \|\bx\|_\infty \leq M_{x,i},
    \qquad
    \|\bxi(\bx,t)\|_\infty
    =
    \Bigl\|\frac{\bx}{m_t}\Bigr\|_\infty
    \leq
    M_{x,i}\underline{m}^{-1}.
    \label{eq:GN-step-x-bounds}
\end{equation}
Indeed, if $(\bx,t) \in \gK_i(\varepsilon)$, then $\bx$ is within $\Ord(\sigma_t\sqrt{\log\varepsilon^{-1}})$ of $m_t\gS_i$; on the small-noise branch this is uniformly bounded, and $\gS_i$ is compact.

Set
\begin{equation}
    C_{i,\times}
    :=
    1\vee B_{i,\gT}\vee M_{x,i}\underline{m}^{-1}.
    \label{eq:def:C-i-times}
\end{equation}

\textbf{Step 2: Approximate the ingredients.}
Let $\varepsilon_{\rm aux}\in(0,1/2]$ be chosen later.

First, by \cref{lem:approx:x-over-m} with $M_x=M_{x,i}$, there exists a vector-valued ReLU network $\widetilde{\bxi}$ such that
\begin{equation}
    \bigl\|
      \widetilde{\bxi}(\bx,t)-\bxi(\bx,t)
    \bigr\|_\infty
    \leq
    \varepsilon_{\rm aux}
    \label{eq:GN-step-xi-approx}
\end{equation}
uniformly over the active small-noise region.
Its size is bounded by
\begin{align}
    L_\xi,\log B_\xi
    &\lesssim
    \log^2\varepsilon_{\rm aux}^{-1}
    +
    \log^2(C_{i,\times}),
    \notag\\
    \|\bW_\xi\|_\infty
    &\lesssim
    D\bigl(
      \log^3\varepsilon_{\rm aux}^{-1}
      +
      \log^3(C_{i,\times})
    \bigr),
    \notag\\
    S_\xi
    &\lesssim
    D\bigl(
      \log^4\varepsilon_{\rm aux}^{-1}
      +
      \log^4(C_{i,\times})
    \bigr).
    \label{eq:size-xi-approx-GN}
\end{align}

Second, apply \cref{thm:approx:holder} to the scalar components of $\bz_i$ and $\bA_i$ on $\Box_i$ after a fixed affine rescaling of $\Box_i$ to $[0,1]^d$.
Thus there exist ReLU networks $\widetilde{\bz}_i$ and $\widetilde{\bA}_i$ such that
\begin{equation}
    \|\widetilde{\bz}_i-\bz_i\|_{L^\infty(\Box_i)}
    \le
    \varepsilon_{\rm aux},
    \qquad
    \|\widetilde{\bA}_i-\bA_i\|_{L^\infty(\Box_i)}
    \le
    \varepsilon_{\rm aux}.
    \label{eq:GN-step-z-A-approx}
\end{equation}
Stacking the $D+dD$ scalar approximants gives size
\begin{align}
    L_H
    &\lesssim
    \log(e(d\vee\beta))
    \log(C_{i,\times}\varepsilon_{\rm aux}^{-1}),
    \notag\\
    \|\bW_H\|_\infty
    &\lesssim
    dD
    (d\vee\beta)
    (C_{i,\times}\varepsilon_{\rm aux}^{-1})^{\frac{d}{2(\beta\vee1)}},
    \notag\\
    S_H
    &\lesssim
    dD(d+2(\beta\vee1))^{d+3}
    (C_{i,\times}\varepsilon_{\rm aux}^{-1})^{\frac{d}{2(\beta\vee1)}}
    \log(C_{i,\times}\varepsilon_{\rm aux}^{-1}),
    \notag\\
    B_H
    &\le 1.
    \label{eq:size-holder-GN}
\end{align}

Define
\begin{equation*}
    \bv(\bu,\bx,t)
    :=
    \bxi(\bx,t)-\bz_i(\bu),
    \qquad
    \widetilde{\bv}(\bu,\bx,t)
    :=
    \widetilde{\bxi}(\bx,t)-\widetilde{\bz}_i(\bu).
\end{equation*}
By \cref{eq:GN-step-xi-approx,eq:GN-step-z-A-approx},
\begin{equation}
    \|\widetilde{\bv}-\bv\|_\infty
    \le
    2\varepsilon_{\rm aux}.
    \label{eq:GN-step-v-error}
\end{equation}
Also, by the bounds above,
\begin{equation*}
    |A_{i,j\ell}(\bu)|\le C_{i,\times},
    \qquad
    |v_\ell(\bu,\bx,t)|\le 2C_{i,\times}.
\end{equation*}
Increasing $C_{i,\times}$ by an absolute factor if necessary, we may assume $C_{i,\times}\ge1$ and both true inputs to every multiplication below lie in $[-C_{i,\times},C_{i,\times}]$.

\textbf{Step 3: Approximate the matrix-vector product.}
For each $j=1,\dots,d$ and $\ell=1,\dots,D$, apply \cref{lem:approx:mon} with two inputs, bound $C_{i,\times}$, intrinsic multiplication error $\varepsilon_{\rm aux}$, and input perturbation tolerance $2\varepsilon_{\rm aux}$.
This gives a multiplication network satisfying
\begin{align}
    &
    \bigl|
      \phi_{{\rm mult},j,\ell}
      \bigl(
        \widetilde{A}_{i,j,\ell}(\bu),
        \widetilde{v}_\ell(\bu,\bx,t)
      \bigr)
      -
      A_{i,j,\ell}(\bu)v_\ell(\bu,\bx,t)
    \bigr|
    \notag\\
    &\qquad\leq
    \varepsilon_{\rm aux}+4C_{i,\times}\varepsilon_{\rm aux}
    \leq
    5C_{i,\times}\varepsilon_{\rm aux}.
    \label{eq:GN-step-product-error}
\end{align}
The size of each multiplication network is
\begin{align}
    L_{\rm mult}
    &\lesssim
    \log(\varepsilon_{\rm aux}^{-1})+\log(C_{i,\times}),
    \notag\\
    \|\bW_{\rm mult}\|_\infty
    &\leq 96,
    \notag\\
    S_{\rm mult}
    &\lesssim
    \log(\varepsilon_{\rm aux}^{-1})+\log(C_{i,\times}),
    \notag\\
    \log B_{\rm mult}
    &\lesssim
    \log(C_{i,\times}).
    \label{eq:size-GN-mult}
\end{align}

Define the approximate pre-clipped update by
\begin{equation}
    \widetilde{q}_j(\bu,\bx,t)
    :=
    u_j+
    \sum_{\ell=1}^D
      \phi_{{\rm mult},j\ell}
      \bigl(
        \widetilde{A}_{i,j\ell}(\bu),
        \widetilde{v}_\ell(\bu,\bx,t)
      \bigr),
    \qquad
    j=1,\dots,d.
    \label{eq:def:q-tilde-GN}
\end{equation}
The exact pre-clipped update is
\begin{equation*}
    q_j(\bu,\bx,t)
    :=
    u_j+
    \sum_{\ell=1}^D
      A_{i,j\ell}(\bu)v_\ell(\bu,\bx,t).
\end{equation*}
Summing \cref{eq:GN-step-product-error} over $\ell$ gives
\begin{equation}
    \|\widetilde{\bq}(\bu,\bx,t)-\bq(\bu,\bx,t)\|_\infty
    \le
    5DC_{i,\times}\varepsilon_{\rm aux}.
    \label{eq:GN-step-preclip-error}
\end{equation}

Now define the final network by exact clipping:
\begin{equation}
    \widetilde{\gT}_{i,\varepsilon_{\gT}}(\bu,\bx,t)
    :=
    \Pi_i^{({\rm box})}
    \bigl(
      \widetilde{\bq}(\bu,\bx,t)
    \bigr).
    \label{eq:def:T-tilde-GN}
\end{equation}
Because $\Pi_i^{({\rm box})}$ is $1$-Lipschitz,
\begin{align*}
    \Bigl\|
      \widetilde{\gT}_{i,\varepsilon_{\gT}}(\bu,\bx,t)
      -
      \gT_i(\bu;\bx,t)
    \Bigr\|_\infty
    &\le
    \|\widetilde{\bq}(\bu,\bx,t)-\bq(\bu,\bx,t)\|_\infty
    \\
    &\le
    5DC_{i,\times}\varepsilon_{\rm aux}.
\end{align*}
Choose
\begin{equation}
    \varepsilon_{\rm aux}
    :=
    \frac{\varepsilon_{\gT}}
    {5DC_{i,\times}}.
    \label{eq:def:eps-GN-step}
\end{equation}
Then \cref{eq:one-step-error} follows.

\textbf{Step 4: Network size.}
The network is obtained by composing the following pieces: the $\bxi=\bx/m_t$ network from \cref{lem:approx:x-over-m}, the $d$-dimensional H\"older networks for $\bz_i$ and $\bA_i$ from \cref{thm:approx:holder}, the $dD$ multiplication networks from \cref{lem:approx:mon}, exact affine summations, and the exact clipping map.

Substituting \cref{eq:def:eps-GN-step} into \cref{eq:size-xi-approx-GN,eq:size-holder-GN,eq:size-GN-mult} and absorbing the fixed chart constants into a single constant $C_{i,\gT}\ge1$ gives \cref{eq:size-one-step}.
This completes the proof.
\end{proof}

\begin{lemma}[Anchored neural approximation of the local projection coordinate]
\label{lem:proj-local-coordinate-NN}
Fix an anchor $\ba \in K_i^\circ$.
Fix a localization parameter $0<\varepsilon_{\rm loc}\le1$ such that \cref{eq:condition:delta} holds for every $t\in\gI_{\rm sm}(\varepsilon_{\rm loc})$.
For every accuracy $\varepsilon_u\in(0,1]$, there exists a ReLU network $\widetilde{\bu}_{i,\ba,\varepsilon_u}^{\Pi,\varepsilon_{\rm loc}} \in \nn(L,\bW,S,B)$ with input $(\bx, t) \in \R^D \times [\tdown, \tup]$ and output in $\Box_i$ such that, for every $t \in \gI_{\rm sm}(\varepsilon_{\rm loc})$ and every $(\bx,t)\in\gK_i(\varepsilon_{\rm loc})$ satisfying
\begin{equation*}
    \|\ba-\bu_i^\Pi(\bx,t)\|_2
    \leq
    \frac{1}{2} r_i^{\rm GN},
\end{equation*}
\begin{equation}
    \bigl\|
      \widetilde{\bu}_{i,\ba,\varepsilon_u}^{\Pi,\varepsilon_{\rm loc}}(\bx,t)
      -
      \bu_i^\Pi(\bx,t)
    \bigr\|_\infty
    \leq
    \varepsilon_u.
    \label{eq:uPi-final-error}
\end{equation}
Moreover, for a finite chart-dependent constant $C_{i,\Pi} \geq 1$,
\begin{align}
    L &\lesssim \log(e(d\vee\beta))\log^3(C_{i,\Pi}D\varepsilon_u^{-1}),
    \notag\\
    \|\bW\|_\infty &\lesssim
    dD(d\vee\beta)(C_{i,\Pi}D\varepsilon_u^{-1})^{\frac{d}{2(\beta\vee1)}}
    +
    D\log^3(C_{i,\Pi}D\varepsilon_u^{-1}),
    \notag\\
    S &\lesssim dD(d+2(\beta\vee1))^{d+3}(C_{i,\Pi}D\varepsilon_u^{-1})^{\frac{d}{2(\beta\vee1)}}
    \log^2(C_{i,\Pi}D\varepsilon_u^{-1})
    +
    dD\log^5(C_{i,\Pi}D\varepsilon_u^{-1}),
    \notag\\
    \log B &\lesssim \log^2(C_{i,\Pi}D\varepsilon_u^{-1}).
    \label{eq:size-local-coordinate-NN}
\end{align}
\end{lemma}

\begin{proof}
The construction uses a fixed anchor $\ba$ as the initial condition and then unrolls approximate Gauss--Newton steps.
Set
\begin{equation}
    \varepsilon_{\rm step}
    :=
    \min\Bigl\{
      \frac{r_i^{\rm GN}}{8\sqrt{d}},
      \frac{\varepsilon_u}{4\sqrt{d}}
    \Bigr\}.
    \label{eq:def:eps-step-local-proj}
\end{equation}
By \cref{lem:proj-one-step-NN}, applied with localization parameter $\varepsilon_{\rm loc}$, there exists a ReLU network $\widetilde{\gT}_{i,\varepsilon_{\rm step}}$ with output in $\Box_i$ such that
\begin{equation}
    \sup_{\substack{\bu\in\Box_i,\;
        t \in \gI_{\rm sm}(\varepsilon_{\rm loc}),\;
        (\bx,t) \in \gK_i(\varepsilon_{\rm loc})}}
    \Bigl\|
      \widetilde{\gT}_{i,\varepsilon_{\rm step}}(\bu,\bx,t)
      -
      \gT_i(\bu;\bx,t)
    \Bigr\|_\infty
    \leq
    \varepsilon_{\rm step}.
    \label{eq:step-network-error-local-proj}
\end{equation}
The choice \cref{eq:def:eps-step-local-proj} guarantees in particular
\begin{equation*}
    \varepsilon_{\rm step}
    \leq
    \frac{r_i^{\rm GN}}{8\sqrt{d}},
\end{equation*}
which is exactly the basin condition~\cref{eq:step-error-basin-condition} required in \cref{lem:active-anchored-GN-error}.

Let
\begin{equation}
    K_{\varepsilon_u}
    :=
    \Bigl\lceil
      \log_2\Bigl(\frac{r_i^{\rm GN}}{\varepsilon_u}\Bigr)
    \Bigr\rceil_+,
    \qquad
    \lceil a\rceil_+:=\max\{\lceil a\rceil,0\}.
    \label{eq:def:K-eps-local-proj}
\end{equation}
Then
\begin{equation}
    2^{-(K_{\varepsilon_u}+1)}r_i^{\rm GN}
    \leq
    \frac{\varepsilon_u}{2}.
    \label{eq:contracted-anchor-error}
\end{equation}
Indeed, if $r_i^{\rm GN}\le\varepsilon_u$, then $K_{\varepsilon_u}=0$ and $r_i^{\rm GN}/2\le\varepsilon_u/2$.
If $r_i^{\rm GN}>\varepsilon_u$, then $K_{\varepsilon_u}\ge \log_2(r_i^{\rm GN}/\varepsilon_u)$, so $2^{-K_{\varepsilon_u}}r_i^{\rm GN}\le\varepsilon_u$.

\textbf{Step 1: Define the unrolled network.}
Define recursively
\begin{equation}
    \widetilde{\bu}_0(\bx,t)
    :=
    \ba,
    \qquad
    \widetilde{\bu}_{k+1}(\bx,t)
    :=
      \widetilde{\gT}_{i,\varepsilon_{\rm step}}
    \bigl(
      \widetilde{\bu}_k(\bx,t),\bx,t
    \bigr),
    \qquad
    k=0,\dots,K_{\varepsilon_u}-1.
    \label{eq:def:unrolled-local-proj}
\end{equation}
Finally set
\begin{equation}
    \widetilde{\bu}_{i,\ba,\varepsilon_u}^{\Pi,\varepsilon_{\rm loc}}(\bx,t)
    :=
    \widetilde{\bu}_{K_{\varepsilon_u}}(\bx,t).
    \label{eq:def:uPi-network-local-proj}
\end{equation}
Because each one-step network is $\Box_i$-valued, every iterate $\widetilde{\bu}_k(\bx,t)$ also lies in $\Box_i$.

\textbf{Step 2: The approximation error.}
Assume now that $t \in \gI_{\rm sm}(\varepsilon_{\rm loc})$, $(\bx,t) \in \gK_i(\varepsilon_{\rm loc})$, and
\begin{equation*}
    \|\ba-\bu_i^\Pi(\bx,t)\|_2
    \leq
    \frac{1}{2} r_i^{\rm GN}.
\end{equation*}
Then the hypotheses of \cref{lem:active-anchored-GN-error} are satisfied with $K=K_{\varepsilon_u}$.
Hence
\begin{align}
    \Bigl\|
      \widetilde{\bu}_{K_{\varepsilon_u}}(\bx,t)
      -
      \bu_i^\Pi(\bx,t)
    \Bigr\|_2
    &\leq
    2^{-(K_{\varepsilon_u}+1)}r_i^{\rm GN}
    +
    2\sqrt{d}\,\varepsilon_{\rm step}
    \notag\\
    &\leq
    \frac{\varepsilon_u}{2}
    +
    2\sqrt{d}\cdot \frac{\varepsilon_u}{4\sqrt{d}}
    =
    \varepsilon_u.
    \label{eq:local-proj-L2-error}
\end{align}
Since $\|\bv\|_\infty\le\|\bv\|_2$, \cref{eq:local-proj-L2-error} implies
\begin{equation*}
    \Bigl\|
      \widetilde{\bu}_{i,\ba,\varepsilon_u}^{\Pi,\varepsilon_{\rm loc}}(\bx,t)
      -
      \bu_i^\Pi(\bx,t)
    \Bigr\|_\infty
    \leq
    \varepsilon_u,
\end{equation*}
which proves \cref{eq:uPi-final-error}.

\textbf{Step 3: Network size.}
By \cref{lem:proj-one-step-NN}, the one-step network with accuracy $\varepsilon_{\rm step}$ satisfies, for a chart-dependent constant $C_{i,\gT}\ge1$,
\begin{align*}
    L_{\rm step}
    &\lesssim
    \log(e(d\vee\beta))
    \log^2(C_{i,\gT}D\varepsilon_{\rm step}^{-1}),
    \\
    \|\bW_{\rm step}\|_\infty
    &\lesssim
    dD
    (d\vee\beta)
    (C_{i,\gT}D\varepsilon_{\rm step}^{-1})^{\frac{d}{2(\beta\vee1)}}
    +
    D\log^3(C_{i,\gT}D\varepsilon_{\rm step}^{-1}),
    \\
    S_{\rm step}
    &\lesssim
    dD(d+2(\beta\vee1))^{d+3}
    (C_{i,\gT}D\varepsilon_{\rm step}^{-1})^{\frac{d}{2(\beta\vee1)}}
    \log(C_{i,\gT}D\varepsilon_{\rm step}^{-1})
    +
    dD\log^4(C_{i,\gT}D\varepsilon_{\rm step}^{-1}),
    \\
    \log B_{\rm step}
    &\lesssim
    \log^2(C_{i,\gT}D\varepsilon_{\rm step}^{-1}).
\end{align*}
From \cref{eq:def:eps-step-local-proj}, there exists a chart-dependent constant $C_{i,\Pi}\ge1$ such that
\begin{equation}
    D\varepsilon_{\rm step}^{-1}
    \le
    C_{i,\Pi}D\varepsilon_u^{-1}.
    \label{eq:eps-step-inverse-bound}
\end{equation}
Also, by \cref{eq:def:K-eps-local-proj},
\begin{equation}
    K_{\varepsilon_u}
    \le
    C\log(C_{i,\Pi}D\varepsilon_u^{-1}).
    \label{eq:K-size-bound}
\end{equation}

The unrolled network is the composition of $K_{\varepsilon_u}$ copies of the one-step network, with the constant initialization $\ba$ inserted by an affine layer.
Therefore
\begin{equation*}
    L \leq K_{\varepsilon_u} L_{\rm step}+\Ord(1),
    \qquad
    S \leq K_{\varepsilon_u} S_{\rm step}+\Ord(1),
\end{equation*}
while the width and weight bound are controlled by the maximum of the corresponding quantities of the one-step networks.
Combining these relations with \cref{eq:eps-step-inverse-bound,eq:K-size-bound} gives exactly \cref{eq:size-local-coordinate-NN}.
This completes the proof.
\end{proof}

\begin{equation}
    \gK_{\rm act}(\varepsilon)
    :=
    \bigcup_{i=1}^{C_{\gM}}\gK_i(\varepsilon).
    \label{eq:def:K-active-union}
\end{equation}
For a fixed auxiliary radius multiplier $c_a>0$, also write
\begin{equation}
    \gK_{\rm act}^{(c_a)}(\varepsilon)
    :=
    \bigcup_{i=1}^{C_{\gM}}\gK_i^{(c_a)}(\varepsilon).
    \label{eq:def:K-active-union-ca}
\end{equation}
\begin{lemma}[Active chart centers lie within the reach tube]
\label{lem:active-chart-inside-reach-tube}
    Let $\delta_\Pi := \eta_\Pi/2$.
    Assume $t \in \gI_{\rm sm}(\varepsilon)$ and \cref{eq:condition:delta}.
    If $(\bx,t) \in \gK_i(\varepsilon)$ for some $i \in \{1, \dots, C_{\gM}\}$, then
    \begin{equation}
        \dist\bigl(\bxi(\bx,t),\gM\bigr)
        \leq
        \delta_\Pi.
        \label{eq:active-xi-in-reach-tube}
    \end{equation}
    Moreover, $\delta_\Pi \leq \kappa/2$.
\end{lemma}

\begin{proof}
Since $(\bx,t) \in \gK_i(\varepsilon)$, there exists $\by \in \gS_i \subset \gM$ such that
\begin{equation*}
    \|\bxi(\bx,t)-\by\|_2
    \leq
    \delta_{\rm act}(\varepsilon,t).
\end{equation*}
Therefore
\begin{equation*}
    \dist(\bxi(\bx,t),\gM)
    \leq
    \|\bxi(\bx,t)-\by\|_2
    \leq
    \delta_{\rm act}(\varepsilon,t).
\end{equation*}
By \cref{eq:condition:delta},
\begin{equation*}
    \delta_{\rm act}(\varepsilon,t)
    <
    \frac{1}{2}\eta_\Pi
    =
    \delta_\Pi.
\end{equation*}
Hence \cref{eq:active-xi-in-reach-tube} holds.
Since $\eta_\Pi\le\kappa$ by \cref{eq:def:projection-margin}, we also have $\delta_\Pi\le\kappa/2$.
\end{proof}

\begin{lemma}[Projection objective gap inside the reach tube]
\label{lem:reach-tube-objective-gap}
    Assume $0 < \delta \leq \kappa/2$.
    Then, for every $\bxi\in\R^D$ satisfying $\dist(\bxi,\gM)\le\delta$ and every $\by \in \gM$,
    \begin{equation}
        \frac{1}{2}
        \|\by-\Pi_{\gM}(\bxi)\|_2^2
        \leq
        \|\bxi-\by\|_2^2
        -
        \|\bxi-\Pi_{\gM}(\bxi)\|_2^2
        \leq
        \Bigl(1+\frac{\delta}{\kappa}\Bigr)
        \|\by-\Pi_{\gM}(\bxi)\|_2^2. 
        \label{eq:reach-tube-bounds}
    \end{equation}
\end{lemma}

\begin{proof}
Since $\dist(\bxi,\gM)\le\delta\le\kappa/2$, the reach assumption implies that $\Pi_{\gM}(\bxi)$ is unique, $(\bxi-\Pi_{\gM}(\bxi)) \perp T_{\Pi_{\gM}(\bxi)}\gM$, and $\|\bxi-\Pi_{\gM}(\bxi)\|_2\le\delta$.
For every $\by\in\gM$,
\begin{align}
    \|\bxi-\by\|_2^2-\|\bxi-\Pi_{\gM}(\bxi)\|_2^2
    &=
    \|\bxi-\Pi_{\gM}(\bxi)+\Pi_{\gM}(\bxi)-\by\|_2^2-\|\bxi-\Pi_{\gM}(\bxi)\|_2^2
    \notag \\
    &=
    \|\by-\Pi_{\gM}(\bxi)\|_2^2
    -
    2\langle \bxi-\Pi_{\gM}(\bxi), \by-\Pi_{\gM}(\bxi)\rangle.
    \label{eq:reach-gap-algebra}
\end{align}
The reach bound gives the tangent-ball inequality
\begin{equation}
    \langle \bv,\by-\Pi_{\gM}(\bxi)\rangle
    \leq
    \frac{\|\by-\Pi_{\gM}(\bxi)\|_2^2}{2\kappa},
    \qquad
    \by \in \gM,\quad
    \bv\in N_{\Pi_{\gM}}(\bxi)\gM,\quad
    \|\bv\|_2=1.
    \label{eq:reach-tangent-ball}
\end{equation}

If $\bxi-\Pi_{\gM}(\bxi)=0$, the claim follows immediately from \cref{eq:reach-gap-algebra}.
Otherwise, applying \cref{eq:reach-tangent-ball} to $\bv=(\bxi-\Pi_{\gM}(\bxi))/\|\bxi-\Pi_{\gM}(\bxi)\|_2$ gives
\begin{equation*}
    \langle \bxi-\Pi_{\gM}(\bxi),\by-\Pi_{\gM}(\bxi)\rangle
    \leq
    \frac{\|\bxi-\Pi_{\gM}(\bxi)\|_2}{2\kappa}
    \|\by-\Pi_{\gM}(\bxi)\|_2^2
    \leq
    \frac{\delta}{2\kappa}
    \|\by-\Pi_{\gM}(\bxi)\|_2^2.
\end{equation*}
Applying the same inequality to $-(\bxi-\Pi_{\gM}(\bxi))/\|\bxi-\Pi_{\gM}(\bxi)\|_2$ gives
\begin{equation*}
    \langle \bxi-\Pi_{\gM}(\bxi),\by-\Pi_{\gM}(\bxi)\rangle
    \geq
    -
    \frac{\delta}{2\kappa}
    \|\by-\Pi_{\gM}(\bxi)\|_2^2.
\end{equation*}
Substituting these two bounds into \cref{eq:reach-gap-algebra} yields
\begin{equation*}
    \Bigl(1-\frac{\delta}{\kappa}\Bigr)
    \|\by-\Pi_{\gM}(\bxi)\|_2^2
    \leq
    \|\bxi-\by\|_2^2-\|\bxi-\Pi_{\gM}(\bxi)\|_2^2
    \leq
    \Bigl(1+\frac{\delta}{\kappa}\Bigr)
    \|\by-\Pi_{\gM}(\bxi)\|_2^2.
\end{equation*}
Since $\delta\le\kappa/2$, the lower coefficient is at least $1/2$.
This proves \cref{eq:reach-tube-bounds}.
\end{proof}

\Cref{thm:global-proj-NN} is the step that avoids black-box high-order approximation of
the ambient map $\Pi_{\gM}$.  The point to check is that the network searches over
finitely many intrinsic anchors, performs local Gauss--Newton updates in chart
coordinates, and uses an objective gap to select near-minimizing projection candidates.
Thus the non-logarithmic covering cost is intrinsic in the active chart coordinates; the
remaining direct ambient factors are displayed explicitly.
The proof has five logically separate checks:
\begin{enumerate}[(i)]
    \item the active-tube and reach conditions make $\Pi_{\gM}(\bxi)$ unique and place its coordinate in a buffered chart box;
    \item a finite intrinsic anchor net contains an anchor inside the Gauss--Newton contraction basin of this coordinate;
    \item an approximate Gauss--Newton step is realized by ReLU networks using only $d$-dimensional chart variables and ambient matrix--vector arithmetic;
    \item a ReLU minimum and objective gate discard all candidates whose squared-distance objective is not within $\Ord(\varepsilon_\Pi^2)$ of the best candidate;
    \item the gated average is normalized by a reciprocal network whose denominator is bounded below by the positive gate mass.
\end{enumerate}
Only the third and fourth checks carry non-logarithmic complexity; their covering cost is the intrinsic anchor count times $d$-dimensional chart approximation, with the remaining $D$ dependence coming from explicit coordinatewise ambient arithmetic.

\begin{theorem}[Neural approximation of the active projection center]
\label{thm:global-proj-NN}
Fix a localization parameter $0<\varepsilon<1$ such that \cref{eq:condition:delta} holds for every $t\in\gI_{\rm sm}(\varepsilon)$.
For every $\varepsilon_\Pi \in (0,1]$, there exists a ReLU network
\begin{equation*}
    \widetilde{\Pi}_{\varepsilon_\Pi}: \R^D \times [\tdown, \tup] \to \R^D
\end{equation*}
such that
\begin{equation}
    \sup_{\substack{
        t \in \gI_{\rm sm}(\varepsilon)\\
        (\bx,t)\in\gK_{\rm act}(\varepsilon)
    }}
    \Bigl\|
      \widetilde{\Pi}_{\varepsilon_\Pi}(\bx,t)
      -
      \Pi_{\gM}\bigl(\bxi(\bx,t)\bigr)
    \Bigr\|_\infty
    \leq
    \varepsilon_\Pi.
    \label{eq:Pi-global-error}
\end{equation}
Let
\begin{equation*}
    N_{\rm cand}
    :=
    \sum_{i=1}^{C_{\gM}}|\gA_i^{\rm anc}|.
\end{equation*}
Then $N_{\rm cand}\le A_{\gM}$, with $A_{\gM}$ bounded in \cref{eq:anchor-complexity-bound}.
For a finite geometry-dependent constant $C_{\rm proj}\ge1$, $\widetilde{\Pi}_{\varepsilon_\Pi}\in\nn(L,\bW,S,B)$ with
\begin{align}
    L
    &\lesssim
    \log(e(d\vee\beta))
    \log^3(C_{\rm proj}D\varepsilon_\Pi^{-2}),
    \notag\\
    \|\bW\|_\infty
    &\lesssim
    A_{\gM}
    dD
    (d\vee\beta)
    (C_{\rm proj}D\varepsilon_\Pi^{-2})^{\frac{d}{2(\beta\vee1)}}
    +
    A_{\gM}
    D\log^3(C_{\rm proj}D\varepsilon_\Pi^{-2}),
    \notag\\
    S
    &\lesssim
    A_{\gM}
    dD(d+2(\beta\vee1))^{d+3}
    (C_{\rm proj}D\varepsilon_\Pi^{-2})^{\frac{d}{2(\beta\vee1)}}
    \log^2(C_{\rm proj}D\varepsilon_\Pi^{-2})
    +
    A_{\gM}
    dD\log^5(C_{\rm proj}D\varepsilon_\Pi^{-2}),
    \notag\\
    \log B
    &\lesssim
    \log^2(C_{\rm proj}D\varepsilon_\Pi^{-2}).
    \label{eq:global-proj-size}
\end{align}
\end{theorem}

\begin{proof}
It suffices to prove the result for sufficiently small $\varepsilon_\Pi$.
For larger $\varepsilon_\Pi$, use the construction at a fixed small $\varepsilon_\Pi$ and absorb constants into $C_{\rm proj}$.

Let
\begin{equation*}
    \gJ := \{(i,\ba): 1 \leq i \leq C_{\gM}, \ba \in \gA_i^{\rm anc}\}.
\end{equation*}
For each $(i,\ba)\in\gJ$, we construct a candidate projection center, then use a continuous objective gate to average only near-minimizing candidates.

\textbf{Step 1: Existence of a good chart-anchor candidate.}
Fix $t \in \gI_{\rm sm}(\varepsilon)$ and $(\bx,t) \in \gK_{\rm act}(\varepsilon)$.
Write
\begin{equation*}
    \bxi:=\bxi(\bx,t),
    \qquad
    \gD_\star:=\|\bxi-\Pi_{\gM}(\bxi)\|_2^2.
\end{equation*}
By \cref{lem:active-chart-inside-reach-tube}, $\dist(\bxi,\gM)\le\delta_\Pi$ and $\delta_\Pi\le\kappa/2$.
Hence, \cref{lem:reach-tube-objective-gap} applies with $\delta=\delta_\Pi$.
For convenience write
\begin{equation*}
    C_{\rm glob,\Pi}^{(+)}
    :=
    \Bigl(1+\frac{\delta_\Pi}{\kappa}\Bigr).
\end{equation*}

Since $(\bx,t)\in\gK_{\rm act}(\varepsilon)$, there exists an active chart $i_\star$ such that $(\bx,t)\in\gK_{i_\star}(\varepsilon)$.
By \cref{lem:buffered-active-chart},
\begin{equation*}
    \Pi_{\gM}(\bxi)\in U_{i_\star},
    \qquad
    \bu_{i_\star}^\Pi(\bx,t)
    :=
    \phi_{i_\star}(\Pi_{\gM}(\bxi))
    \in K_{i_\star}^\circ.
\end{equation*}

By \cref{lem:proj-anchor-init}, there exists $\ba_\star\in\gA_{i_\star}^{\rm anc}$ such that
\begin{equation}
    \|\ba_\star-\bu_{i_\star}^\Pi(\bx,t)\|_2
    \leq
    \frac{1}{2} r_{i_\star}^{\rm GN}.
    \label{eq:good-global-anchor}
\end{equation}

Let
\begin{equation*}
    L_z^{\max}:=\max_{1\le i\le C_{\gM}}L_i^{(+)}.
\end{equation*}
Choose
\begin{equation}
    \varepsilon_y^{\rm glob}
    :=
    \min\Bigl\{
      \frac{\varepsilon_\Pi}{8},
      \sqrt{\frac{1}{32C_{\rm glob,\Pi}^{(+)}}}
      \varepsilon_\Pi
    \Bigr\},
    \qquad
    \varepsilon_u^{\rm glob}
    :=
    \frac{\varepsilon_y^{\rm glob}}{\sqrt d\,L_z^{\max}}.
    \label{eq:global-local-coordinate-budget}
\end{equation}

For every $(i,\ba)\in\gJ$, let $\widetilde{\bu}_{i,\ba,\varepsilon_u^{\rm glob}}^{\Pi,\varepsilon}(\bx,t)$ be the anchored local-coordinate network from \cref{lem:proj-local-coordinate-NN}, with localization parameter $\varepsilon$ and accuracy $\varepsilon_u^{\rm glob}$.
Define the exact manifold-valued candidate
\begin{equation}
    \by_{i,\ba}(\bx,t)
    :=
    \bz_i\Bigl(
      \widetilde{\bu}_{i,\ba,\varepsilon_u^{\rm glob}}^{\Pi,\varepsilon}(\bx,t)
    \Bigr).
    \label{eq:def:exact-global-candidate}
\end{equation}
Since the coordinate network is $\Box_i$-valued and $\Box_i\subset\phi_i(U_i)$, we have $\by_{i,\ba}(\bx,t) \in \gM$.

For the good pair $(i_\star,\ba_\star)$, \cref{eq:good-global-anchor} and \cref{lem:proj-local-coordinate-NN} give
\begin{equation*}
    \Bigl\|
      \widetilde{\bu}_{i_\star,\ba_\star,\varepsilon_u^{\rm glob}}^{\Pi,\varepsilon}(\bx,t)
      -
      \bu_{i_\star}^\Pi(\bx,t)
    \Bigr\|_\infty
    \leq
    \varepsilon_u^{\rm glob}.
\end{equation*}
Therefore
\begin{align}
    \|\by_{i_\star,\ba_\star}(\bx,t)-\Pi_{\gM}(\bxi)\|_2
    &\leq
    L_{i_\star}^{(+)}
    \Bigl\|
      \widetilde{\bu}_{i_\star,\ba_\star,\varepsilon_u^{\rm glob}}^{\Pi,\varepsilon}(\bx,t)
      -
      \bu_{i_\star}^\Pi(\bx,t)
    \Bigr\|_2
    \notag\\
    &\leq
    L_z^{\max}\sqrt d\,\varepsilon_u^{\rm glob}
    =
    \varepsilon_y^{\rm glob}.
    \label{eq:valid-global-candidate-error}
\end{align}
Define
\begin{equation}
    \gD_{i,\ba}(\bx,t)
    :=
    \|\bxi(\bx,t)-\by_{i,\ba}(\bx,t)\|_2^2.
    \label{eq:def:exact-global-objective}
\end{equation}
By the upper bound in \cref{eq:reach-tube-bounds},
\begin{align}
    \gD_{i_\star,\ba_\star}(\bx,t)-\gD_\star
    &\leq
    C_{\rm glob,\Pi}^{(+)}
    \|\by_{i_\star,\ba_\star}(\bx,t)-\Pi_{\gM}(\bxi)\|_2^2
    \notag\\
    &\leq
    C_{\rm glob,\Pi}^{(+)}(\varepsilon_y^{\rm glob})^2
    \leq
    \frac{1}{32}\varepsilon_\Pi^2.
    \label{eq:good-global-candidate-objective}
\end{align}

\textbf{Step 2: Approximate candidate outputs and objectives.}
For each $(i,\ba)\in\gJ$, construct a ReLU network $\widetilde\by_{i,\ba}^{\,{\rm out}}$ satisfying
\begin{equation}
    \Bigl\|
      \widetilde\by_{i,\ba}^{\,{\rm out}}(\bx,t)
      -
      \by_{i,\ba}(\bx,t)
    \Bigr\|_2
    \le
    \frac{\varepsilon_\Pi}{8}.
    \label{eq:global-output-candidate-error}
\end{equation}
This is obtained by composing $\widetilde{\bu}_{i,\ba,\varepsilon_u^{\rm glob}}^{\Pi,\varepsilon}$ with a $\gH^{2(\beta\vee1)}$ approximation of $\bz_i$ on $\Box_i$ to coordinate accuracy $\varepsilon_\Pi/(8\sqrt D)$ using~\cref{thm:approx:holder}.

Next, construct ReLU networks $\widetilde{\gD}_{i,\ba}$ such that
\begin{equation}
    \bigl|
      \widetilde{\gD}_{i,\ba}(\bx,t)
      -
      \gD_{i,\ba}(\bx,t)
    \bigr|
    \leq \varepsilon_D,
    \qquad
    \varepsilon_D := \frac{1}{32}\varepsilon_\Pi^2.
    \label{eq:global-objective-accuracy}
\end{equation}
Indeed, approximate $\bxi(\bx,t)$ by \cref{lem:approx:x-over-m}, approximate $\bz_i$ on $\Box_i$ by \cref{thm:approx:holder}, then use \cref{lem:approx:mon} to square and sum the $D$ coordinate differences.
All variables are uniformly bounded on the active small-noise region and on the compact boxes, so coordinate and arithmetic accuracies of order $\varepsilon_D/D$ make the total error at most $\varepsilon_D$.
This is the source of the $D$ factor inside the H\"older approximation term in \cref{eq:global-proj-size}.

\textbf{Step 3: Continuous objective gate.}
Define the finite minimum
\begin{equation*}
    \widetilde{\gD}_{\min}(\bx,t)
    :=
    \min_{(i,\ba)\in\gJ}
    \widetilde{\gD}_{i,\ba}(\bx,t).
\end{equation*}
This is exactly ReLU-realizable using
\begin{equation*}
    \min\{a,b\}=a-\ReLU(a-b).
\end{equation*}
Let
\begin{equation*}
    \varepsilon_{\Pi}^{\rm gate}:=\frac{1}{32}\varepsilon_\Pi^2.
\end{equation*}
For each $(i,\ba)\in\gJ$, define
\begin{equation}
    \omega_{i,\ba}(\bx,t)
    :=
    \ReLU\Bigl(
      \varepsilon_{\Pi}^{\rm gate}
      -
      \bigl(
        \widetilde{\gD}_{i,\ba}(\bx,t)
        -
        \widetilde{\gD}_{\min}(\bx,t)
      \bigr)
    \Bigr),
    \label{eq:def:objective-gate-weight}
\end{equation}
and
\begin{equation*}
    W_{\rm glob}(\bx,t)
    :=
    \sum_{(i,\ba)\in\gJ}\omega_{i,\ba}(\bx,t).
\end{equation*}
At least one index attains $\widetilde{\gD}_{\min}$, and its weight equals $\varepsilon_{\Pi}^{\rm gate}$.
Hence
\begin{equation}
    \varepsilon_{\Pi}^{\rm gate}
    \le
    W_{\rm glob}(\bx,t)
    \le
    N_{\rm cand}\varepsilon_{\Pi}^{\rm gate}.
    \label{eq:global-weight-range}
\end{equation}

We now show that every positive-weight exact candidate is close to $\Pi_{\gM}(\bxi)$.
Suppose $\omega_{i,\ba}(\bx,t)>0$.
Then
\begin{equation*}
    \widetilde{\gD}_{i,\ba}(\bx,t)
    <
    \widetilde{\gD}_{\min}(\bx,t)+\varepsilon_{\Pi}^{\rm gate}.
\end{equation*}
Since $(i_\star,\ba_\star)$ is an available candidate,
\begin{equation*}
    \widetilde{\gD}_{\min}(\bx,t)
    \le
    \widetilde{\gD}_{i_\star,\ba_\star}(\bx,t)
    \le
    \gD_{i_\star,\ba_\star}(\bx,t)+\varepsilon_D.
\end{equation*}
Therefore, by~\cref{eq:good-global-candidate-objective},
\begin{align}
    \gD_{i,\ba}(\bx,t)
    &\le
    \widetilde{\gD}_{i,\ba}(\bx,t)+\varepsilon_D
    \notag\\
    &<
    \widetilde{\gD}_{\min}(\bx,t)+\varepsilon_{\Pi}^{\rm gate}+\varepsilon_D
    \notag\\
    &\le
    \gD_{i_\star,\ba_\star}(\bx,t)+\varepsilon_{\Pi}^{\rm gate}+2\varepsilon_D
    \notag\\
    &\le
    \gD_\star+\varepsilon_{\Pi}^{\rm gate}+3\varepsilon_D.
    \label{eq:positive-weight-objective-bound}
\end{align}
By the definitions of $\varepsilon_{\Pi}^{\rm gate}$ and $\varepsilon_D$,
\begin{equation*}
    \varepsilon_{\Pi}^{\rm gate}+3\varepsilon_D
    =
    \varepsilon_\Pi^2
    \Bigl(\frac{1}{32}+\frac{3}{32}\Bigr)
    =
    \frac{1}{2}
    \Bigl(\frac{\varepsilon_\Pi}{2}\Bigr)^2.
\end{equation*}
Thus
\begin{equation*}
    \gD_{i,\ba}(\bx,t)-\gD_\star
    <
    \frac{1}{2}
    \Bigl(\frac{\varepsilon_\Pi}{2}\Bigr)^2.
\end{equation*}
By the lower bound in \cref{eq:reach-tube-bounds},
\begin{equation*}
    \|\by_{i,\ba}(\bx,t)-\Pi_{\gM}(\bxi)\|_2^2
    \leq
    2(\gD_{i,\ba}(\bx,t)-\gD_\star).
\end{equation*}
Therefore
\begin{equation}
    \|\by_{i,\ba}(\bx,t)-\Pi_{\gM}(\bxi)\|_2
    <
    \frac{\varepsilon_\Pi}{2}
    \qquad
    \text{whenever }\omega_{i,\ba}(\bx,t) > 0.
    \label{eq:positive-candidate-close}
\end{equation}

\textbf{Step 4: Normalize and average.}
Define the ideal gated output
\begin{equation*}
    \Pi_{\rm gate}(\bx,t)
    :=
    \frac{
      \sum_{(i,\ba)\in\gJ}
      \omega_{i,\ba}(\bx,t)
      \widetilde\by_{i,\ba}^{\,{\rm out}}(\bx,t)
    }{W_{\rm glob}(\bx,t)}.
\end{equation*}
The normalized weights are nonnegative and sum to one.
Combining \cref{eq:global-output-candidate-error,eq:positive-candidate-close}, for every positive-weight candidate,
\begin{equation*}
    \bigl\|
      \widetilde\by_{i,\ba}^{\,{\rm out}}(\bx,t)-\Pi_{\gM}(\bxi)
    \bigr\|_2
    \leq
    \frac{\varepsilon_\Pi}{8}
    +
    \frac{\varepsilon_\Pi}{2}
    =
    \frac{5}{8}\varepsilon_\Pi.
\end{equation*}
Hence
\begin{equation}
    \|\Pi_{\rm gate}(\bx,t)-\Pi_{\gM}(\bxi(\bx,t))\|_2
    \leq
    \frac{5}{8}\varepsilon_\Pi.
    \label{eq:ideal-gate-error}
\end{equation}

It remains to implement the division by $W_{\rm glob}$.
The scalar weights $\omega_{i,\ba}$ are computed in parallel for all $(i,\ba) \in \gJ$, and the sum
\begin{equation*}
    W_{\rm glob}(\bx,t)
    =
    \sum_{(i,\ba)\in\gJ}\omega_{i,\ba}(\bx,t)
\end{equation*}
is represented exactly by one affine layer.
Thus, forming $W_{\rm glob}$ introduces no additional approximation error.
By \cref{eq:global-weight-range}, $W_{\rm glob}$ lies in $[\varepsilon_{\Pi}^{\rm gate},N_{\rm cand}\varepsilon_{\Pi}^{\rm gate}]$.

Apply \cref{lem:approx:rec} to approximate $1/W_{\rm glob}$ on this interval, and then use \cref{lem:approx:mon} to multiply the reciprocal by each numerator coordinate
\begin{equation*}
    N_k(\bx,t)
    :=
    \sum_{(i,\ba)\in\gJ}
    \omega_{i,\ba}(\bx,t)
    \widetilde y_{i,\ba,k}^{\,{\rm out}}(\bx,t),
    \qquad k=1,\dots,D.
\end{equation*}
Since $N_{\rm cand}$ is finite, all candidate outputs are uniformly bounded, and $\varepsilon_{\Pi}^{\rm gate}\asymp\varepsilon_\Pi^2$, the reciprocal and multiplication accuracies can be chosen with only logarithmic overhead in $\varepsilon_\Pi^{-1}$ so that the total arithmetic error is at most $3\varepsilon_\Pi/8$ in $\ell_2$.
This yields a ReLU network $\widetilde{\Pi}_{\varepsilon_\Pi}$ satisfying
\begin{equation*}
    \|\widetilde{\Pi}_{\varepsilon_\Pi}(\bx,t)-\Pi_{\rm gate}(\bx,t)\|_2
    \leq
    \frac{3}{8}\varepsilon_\Pi.
\end{equation*}

Combining this with \cref{eq:ideal-gate-error} gives
\begin{equation*}
    \|\widetilde{\Pi}_{\varepsilon_\Pi}(\bx,t) - \Pi_{\gM}(\bxi(\bx,t))\|_2
    \leq
    \varepsilon_\Pi .
\end{equation*}
Since $\|\cdot\|_\infty\le\|\cdot\|_2$, this proves \cref{eq:Pi-global-error}.

\textbf{Step 5: Network size.}
The dominant cost comes from approximating the chart maps inside the objective networks to coordinate accuracy of order $\varepsilon_D/D\asymp\varepsilon_\Pi^2/D$.
By \cref{thm:approx:holder}, this gives the intrinsic rate $(C_{\rm proj}D\varepsilon_\Pi^{-2})^{\frac{d}{2(\beta\vee1)}}$ up to logarithmic factors.
The anchored local coordinate networks have scale $(CD\varepsilon_\Pi^{-1})^{d/(2(\beta\vee1))}$, which is dominated by this objective-approximation scale because $\varepsilon_\Pi\le1$.
The finite ReLU minimum, ReLU gates, affine summation layers, reciprocal network, and multiplication networks add only finite atlas factors and logarithmic factors.
Indeed, $W_{\rm glob} \geq \varepsilon_{\Pi}^{\rm gate} \asymp \varepsilon_\Pi^2$, so approximating $1/W_{\rm glob}$ only contributes logarithmic factors of order $\log(C_{\rm proj}D\varepsilon_\Pi^{-2})$.
Stacking over the $N_{\rm cand}$ chart-anchor candidates gives the displayed $A_{\gM}$ factors in \cref{eq:global-proj-size}.
This completes the proof.

\end{proof}

\begin{corollary}[Active chart coordinates from the projection-center network]
\label{cor:active-chart-coordinate-from-proj-NN}
Fix a chart $i$ and a localization parameter $0<\varepsilon<1$ such that \cref{eq:condition:delta} holds for every $t\in\gI_{\rm sm}(\varepsilon)$.
For every $\varepsilon_u\in(0,1]$, there exists a ReLU network
\[
    \widetilde{\bu}_{i,\varepsilon_u}^{\Pi,\varepsilon}:
    \R^D\times[\tdown,\tup]\to\Box_i
\]
such that
\begin{equation}
    \sup_{\substack{
        t\in\gI_{\rm sm}(\varepsilon)\\
        (\bx,t)\in\gK_i(\varepsilon)}}
    \left\|
      \widetilde{\bu}_{i,\varepsilon_u}^{\Pi,\varepsilon}(\bx,t)
      -
      \bu_i^\Pi(\bx,t)
    \right\|_\infty
    \le
    \varepsilon_u.
    \label{eq:active-chart-coordinate-from-proj-error}
\end{equation}
Moreover, for a finite chart- and geometry-dependent constant
$C_{i,\Pi}^{\rm coord}\ge1$,
\begin{align}
    L
    &\lesssim
    \log(e(d\vee\beta))
    \log^3\!\bigl(C_{i,\Pi}^{\rm coord}D\varepsilon_u^{-2}\bigr),
    \notag\\
    \|\bW\|_\infty
    &\lesssim
    A_{\gM}
    dD
    (d\vee\beta)
    \bigl(C_{i,\Pi}^{\rm coord}D\varepsilon_u^{-2}\bigr)^{d/(2(\beta\vee1))}
    +
    A_{\gM}
    D\log^3\!\bigl(C_{i,\Pi}^{\rm coord}D\varepsilon_u^{-2}\bigr),
    \notag\\
    S
    &\lesssim
    A_{\gM}
    dD(d+2(\beta\vee1))^{d+3}
    \bigl(C_{i,\Pi}^{\rm coord}D\varepsilon_u^{-2}\bigr)^{d/(2(\beta\vee1))}
    \log^2\!\bigl(C_{i,\Pi}^{\rm coord}D\varepsilon_u^{-2}\bigr)
    \notag\\
    &\qquad+
    A_{\gM}
    dD\log^5\!\bigl(C_{i,\Pi}^{\rm coord}D\varepsilon_u^{-2}\bigr),
    \notag\\
    \log B
    &\lesssim
    \log^2\!\bigl(C_{i,\Pi}^{\rm coord}D\varepsilon_u^{-2}\bigr).
    \label{eq:active-chart-coordinate-from-proj-size}
\end{align}
\end{corollary}

\begin{proof}
Extend the projection chart affine formula in \cref{eq:approx:proj:phi-i} to all of
$\R^D$ by
\[
    \phi_i^{\rm amb}(\by)
    :=
    \eta_i\bigl(\bV_i^\top(\by-\by_i)+\bc_i\bigr).
\]
Let $L_{\phi,i}^{\rm amb}\ge1$ be a finite Lipschitz constant for this affine map from
$(\R^D,\|\cdot\|_\infty)$ to $(\R^d,\|\cdot\|_\infty)$, and set
\[
    \varepsilon_\Pi
    :=
    1\wedge\frac{\varepsilon_u}{2L_{\phi,i}^{\rm amb}} .
\]
Apply \cref{thm:global-proj-NN} with accuracy $\varepsilon_\Pi$.
On $\gK_i(\varepsilon)\subset\gK_{\rm act}(\varepsilon)$ it gives
\[
    \bigl\|
      \widetilde\Pi_{\varepsilon_\Pi}(\bx,t)
      -
      \Pi_{\gM}(\bxi(\bx,t))
    \bigr\|_\infty
    \le
    \varepsilon_\Pi .
\]
Define the raw coordinate network
\[
    \widehat{\bu}_{i,\varepsilon_u}^{\Pi,\varepsilon}
    :=
    \phi_i^{\rm amb}\circ\widetilde\Pi_{\varepsilon_\Pi}
\]
and then clip its output coordinatewise:
\[
    \widetilde{\bu}_{i,\varepsilon_u}^{\Pi,\varepsilon}
    :=
    \Pi_i^{(\rm box)}
    \bigl(\widehat{\bu}_{i,\varepsilon_u}^{\Pi,\varepsilon}\bigr).
\]
The affine map is exact network arithmetic and the clipping map is exactly ReLU-realizable.
By \cref{lem:buffered-active-chart}, $\bu_i^\Pi(\bx,t)\in K_i^\circ\subset\Box_i$ on
$\gK_i(\varepsilon)$, so the clipping map fixes the target coordinate and is
$1$-Lipschitz.  Therefore
\[
    \left\|
      \widetilde{\bu}_{i,\varepsilon_u}^{\Pi,\varepsilon}(\bx,t)
      -
      \bu_i^\Pi(\bx,t)
    \right\|_\infty
    \le
    L_{\phi,i}^{\rm amb}\varepsilon_\Pi
    \le
    \varepsilon_u .
\]
The size bounds are those of \cref{eq:global-proj-size} with
$\varepsilon_\Pi^{-2}\le C_i\varepsilon_u^{-2}$, plus only fixed affine and clipping
layers.  Absorbing $C_i$ and the fixed affine coefficients into
$C_{i,\Pi}^{\rm coord}$ proves \cref{eq:active-chart-coordinate-from-proj-size}.
\end{proof}

\subsection{Projection-centered Laplace expansion}
\label{sec:app:small-noise-laplace}

This block derives the exact small-noise asymptotic object that the neural network will later approximate.
The phase is expanded at $\bu_i^\Pi(\bx,t)$, the coordinate of the nearest point on the manifold.
The normal residual is kept in the Gaussian factor, while the remaining smooth terms are expanded into finitely many powers of $\sigma_t$ and $\bnu_i=\sigma_t^{-1}\br_i$.
The final proposition in this block records the uniform remainder bound and fixes the small-noise constants used later.

On the active region, we center the phase exactly at the projection coordinate $\bu_i^\Pi(\bx,t)$ from \cref{lem:buffered-active-chart}.
For fixed $t \in \gI_{\rm sm}(\varepsilon)$ and $(\bx,t) \in \gK_i(\varepsilon)$, write
\begin{equation*}
    \bxi := \bxi(\bx,t),
    \qquad
    \Pi_{\gM}(\bxi) := \Pi_{\gM}(\bxi(\bx,t)),
    \qquad
    \bu_i^\Pi := \bu_i^\Pi(\bx,t),
\end{equation*}
and
\begin{equation*}
    \br_i := \br_i(\bx,t)
    =
    \bx-m_t\Pi_{\gM}(\bxi),
    \qquad
    \bnu_i := \sigma_t^{-1}\br_i,
    \qquad
    \bG_i^\Pi := \bG_i(\bu_i^\Pi).
\end{equation*}

\begin{lemma}[Projection-centered Gaussian phase expansion]
\label{lem:centered-phase-expansion}
For each $m\geq1$, define
\begin{equation*}
    \bz_{i,m}(\bu,\bw)
    :=
    \frac{1}{m!}
    \oD^m\bz_i(\bu)[\bw^{\otimes m}]
    \in\R^D.
\end{equation*}
Fix $t \in \gI_{\rm sm}(\varepsilon)$ and $(\bx,t) \in \gK_i(\varepsilon)$, and assume \cref{eq:condition:delta}.
Let $\bw \in \R^d$ satisfy $\bu_i^\Pi(\bx,t)+\sigma_t\bw\in\Box_i$.
Then
\begin{align}
    \frac{F_{i,\bx,t}(\bu_i^\Pi+\sigma_t\bw)}{\sigma_t^2}
    &=
    \frac{\|\br_i\|_2^2}{2\sigma_t^2}
    +
    \frac{m_t^2}{2}
    \bw^\top\bG_i^\Pi\bw
    +
    \sum_{\ell=1}^{\floor{\beta}}
    \sigma_t^\ell
    \Psi_{i,\ell}
    \bigl(
      \bu_i^\Pi,\bnu_i,\bw,t
    \bigr)
    +
    \sigma_t^{\floor{\beta}+1}
    \Rem_i^{(F)}
    \bigl(
      \bu_i^\Pi,\bnu_i,\bw,t
    \bigr), 
    \label{eq:centered-phase-expansion}
\end{align}
where for $\ell \geq 1$,
\begin{equation}
    \Psi_{i,\ell}(\bu,\bnu,\bw,t)
    :=
    -m_t\,\bnu^\top\bz_{i,\ell+1}(\bu,\bw)
    +
    \frac{m_t^2}{2}
    \sum_{\substack{m_1+m_2=\ell+2\\m_1,m_2\geq1}}
    \bz_{i,m_1}(\bu,\bw)^\top
    \bz_{i,m_2}(\bu,\bw).
    \label{eq:def:Psi-i-l}
\end{equation}
Moreover, for constants independent of $(\bx,t,\bw)$,
\begin{equation}
\label{eq:Psi-growth}
\begin{aligned}
    \bigl|
      \Psi_{i,\ell}(\bu,\bnu,\bw,t)
    \bigr|
    &\le
    C_\ell(1+\|\bnu\|_2)(1+\|\bw\|_2^{\ell+2}),
    \\
    \bigl|
      \Rem_i^{(F)}(\bu,\bnu,\bw,t)
    \bigr|
    &\le
    C(1+\|\bnu\|_2)(1+\|\bw\|_2^{2\floor{\beta}+6}).
\end{aligned}
\end{equation}
Each $\Psi_{i,\ell}$ is affine in $\bnu$.
\end{lemma}

\begin{proof}

\textbf{Step 1: Taylor expansion of the chart map.}
Because $\Box_i\Subset\phi_i(U_i)$ and $\bz_i$ is smooth on $\phi_i(U_i)$, all derivatives of $\bz_i$ up to order $\floor{\beta}+3$ are uniformly bounded on $\Box_i$.
Since both $\bu_i^\Pi$ and $\bu_i^\Pi+\sigma_t\bw$ lie in $\Box_i$, Taylor's theorem gives
\begin{align*}
    \bz_i(\bu_i^\Pi+\sigma_t\bw)
    &=
    \bz_i(\bu_i^\Pi)
    +
    \sum_{m=1}^{\floor{\beta}+2}
    \sigma_t^m
    \bz_{i,m}(\bu_i^\Pi,\bw)
    +
    \sigma_t^{\floor{\beta}+3}
    \Rem_i^{(z)}(\bu_i^\Pi,\bw,t)
    \\
    &=
    \Pi_{\gM}(\bxi)
    +
    \sum_{m=1}^{\floor{\beta}+2}
    \sigma_t^m
    \bz_{i,m}(\bu_i^\Pi,\bw)
    +
    \sigma_t^{\floor{\beta}+3}
    \Rem_i^{(z)}(\bu_i^\Pi,\bw,t).
\end{align*}
The remainder satisfies
\begin{equation}
    \|\Rem_i^{(z)}(\bu_i^\Pi,\bw,t)\|_2
    \le
    C(1+\|\bw\|_2^{\floor{\beta}+3}).
    \label{eq:z_i:remainder-bound}
\end{equation}

Define
\begin{equation}
    \Delta_i(\bw)
    :=
    \bz_i(\bu_i^\Pi+\sigma_t\bw)-\Pi_{\gM}(\bxi).
    \label{eq:def:Delta_i}
\end{equation}
Then
\begin{equation}
    \Delta_i(\bw)
    =
    \sum_{m=1}^{\floor{\beta}+2}
    \sigma_t^m
    \bz_{i,m}(\bu_i^\Pi,\bw)
    +
    \sigma_t^{\floor{\beta}+3}
    \Rem_i^{(z)}(\bu_i^\Pi,\bw,t).
    \label{eq:Delta_i:expansion}
\end{equation}

\textbf{Step 2: Exact algebraic expansion of the phase.}
By the definition of $F_{i,\bx,t}$,
\begin{align*}
    2F_{i,\bx,t}(\bu_i^\Pi+\sigma_t\bw)
    &=
    \|\bx-m_t\bz_i(\bu_i^\Pi+\sigma_t\bw)\|_2^2
    \\
    &=
    \|\bx-m_t\Pi_{\gM}(\bxi)-m_t\Delta_i(\bw)\|_2^2
    \\
    &=
    \|\br_i-m_t\Delta_i(\bw)\|_2^2
    \\
    &=
    \|\br_i\|_2^2
    -
    2m_t\br_i^\top\Delta_i(\bw)
    +
    m_t^2\|\Delta_i(\bw)\|_2^2.
\end{align*}
Therefore
\begin{equation}
    \frac{F_{i,\bx,t}(\bu_i^\Pi+\sigma_t\bw)}{\sigma_t^2}
    =
    \frac{\|\br_i\|_2^2}{2\sigma_t^2}
    -
    \frac{m_t}{\sigma_t^2}
    \br_i^\top\Delta_i(\bw)
    +
    \frac{m_t^2}{2\sigma_t^2}
    \|\Delta_i(\bw)\|_2^2.
    \label{eq:F-i-x-t:expansion}
\end{equation}

The first Taylor term is
\begin{equation*}
    \bz_{i,1}(\bu_i^\Pi,\bw)
    =
    \bJ_i(\bu_i^\Pi)\bw.
\end{equation*}
By \cref{eq:projection-orthogonality-r},
\begin{equation}
    \br_i^\top\bz_{i,1}(\bu_i^\Pi,\bw)
    =
    \br_i^\top\bJ_i(\bu_i^\Pi)\bw
    =
    \bw^\top\bJ_i(\bu_i^\Pi)^\top\br_i
    =
    0.
    \label{eq:linear-term-vanishes}
\end{equation}
Using $\br_i=\sigma_t\bnu_i$ and \cref{eq:Delta_i:expansion,eq:linear-term-vanishes}, we get
\begin{align}
    -\frac{m_t}{\sigma_t^2}
    \br_i^\top\Delta_i(\bw)
    &=
    -m_t
    \sum_{m=2}^{\floor{\beta}+2}
    \sigma_t^{m-1}
    \bnu_i^\top
    \bz_{i,m}(\bu_i^\Pi,\bw)
    -
    m_t\sigma_t^{\floor{\beta}+2}
    \bnu_i^\top
    \Rem_i^{(z)}(\bu_i^\Pi,\bw,t)
    \notag \\
    &=
    -m_t
    \sum_{\ell=1}^{\floor{\beta}}
    \sigma_t^\ell
    \bnu_i^\top
    \bz_{i,\ell+1}(\bu_i^\Pi,\bw)
    +
    \sigma_t^{\floor{\beta}+1}
    \Rem_i^{(\rm cr)}(\bu_i^\Pi,\bnu_i,\bw,t).
    \label{eq:cross-term:expansion}
\end{align}
Here the term with $m=\floor{\beta}+2$ has order $\sigma_t^{\floor{\beta}+1}$ and is included in the remainder, together with the Taylor remainder.
Hence, using \cref{eq:z_i:remainder-bound} and $\sigma_t\le1$,
\begin{equation}
    |\Rem_i^{(\rm cr)}(\bu,\bnu,\bw,t)|
    \le
    C(1+\|\bnu\|_2)(1+\|\bw\|_2^{\floor{\beta}+3}).
    \label{eq:cross-term:remainder-bound}
\end{equation}

\textbf{Step 3: Expansion of the quadratic term.}
Using \cref{eq:Delta_i:expansion},
\begin{align*}
    \frac{m_t^2}{2\sigma_t^2}
    \|\Delta_i(\bw)\|_2^2
    &=
    \frac{m_t^2}{2}
    \sum_{m_1,m_2=1}^{\floor{\beta}+2}
    \sigma_t^{m_1+m_2-2}
    \bz_{i,m_1}(\bu_i^\Pi,\bw)^\top
    \bz_{i,m_2}(\bu_i^\Pi,\bw)
    +
    \sigma_t^{\floor{\beta}+1}
    \Rem_i^{(\rm quad)}(\bu_i^\Pi,\bw,t).
\end{align*}
The pair $(m_1,m_2)=(1,1)$ gives
\begin{align*}
    \frac{m_t^2}{2}
    \bz_{i,1}(\bu_i^\Pi,\bw)^\top
    \bz_{i,1}(\bu_i^\Pi,\bw)
    &=
    \frac{m_t^2}{2}
    \bw^\top
    \bJ_i(\bu_i^\Pi)^\top
    \bJ_i(\bu_i^\Pi)
    \bw
    =
    \frac{m_t^2}{2}
    \bw^\top\bG_i^\Pi\bw.
\end{align*}
All other principal pairs satisfy $m_1+m_2-2\ge1$.
Reindexing by $\ell=m_1+m_2-2$ and retaining powers $\ell=1,\dots,\floor{\beta}$ gives
\begin{align}
    \frac{m_t^2}{2\sigma_t^2}
    \|\Delta_i(\bw)\|_2^2
    &=
    \frac{m_t^2}{2}
    \bw^\top\bG_i^\Pi\bw
    +
    \frac{m_t^2}{2}
    \sum_{\ell=1}^{\floor{\beta}}
    \sigma_t^\ell
    \sum_{\substack{m_1+m_2=\ell+2\\m_1,m_2\ge1}}
    \bz_{i,m_1}(\bu_i^\Pi,\bw)^\top
    \bz_{i,m_2}(\bu_i^\Pi,\bw)
    \notag \\
    &\quad+
    \sigma_t^{\floor{\beta}+1}
    \Rem_i^{(\rm quad)}(\bu_i^\Pi,\bw,t).
    \label{eq:quadratic-term:expansion}
\end{align}
The omitted Taylor pairs have total power at least $\sigma_t^{\floor{\beta}+1}$ after division by $\sigma_t^2$, and the terms involving $\Rem_i^{(z)}$ have even higher powers.
Therefore
\begin{equation}
    |\Rem_i^{(\rm quad)}(\bu,\bw,t)|
    \le
    C(1+\|\bw\|_2^{2\floor{\beta}+6}).
    \label{eq:quadratic-term:remainder-bound}
\end{equation}

\textbf{Step 4: Collect the coefficients.}
Substituting \cref{eq:cross-term:expansion,eq:quadratic-term:expansion} into \cref{eq:F-i-x-t:expansion} gives \cref{eq:centered-phase-expansion}, with $\Psi_{i,\ell}$ defined by \cref{eq:def:Psi-i-l} and
\begin{equation*}
    \Rem_i^{(F)}
    :=
    \Rem_i^{(\rm cr)}
    +
    \Rem_i^{(\rm quad)}.
\end{equation*}
The bounds in \cref{eq:Psi-growth} follow from the multilinearity of $\oD^m\bz_i(\bu)[\bw^{\otimes m}]$, the uniform derivative bounds on $\Box_i$, and \cref{eq:cross-term:remainder-bound,eq:quadratic-term:remainder-bound}.
Finally, $\Psi_{i,\ell}$ is affine in $\bnu$ because its first term is linear in $\bnu$ and its second term is independent of $\bnu$.
\end{proof}

By the convention in \cref{def:manifold:Holder-smooth}, every coefficient retained in the neural approximation has strictly positive regularity $\beta-q>0$.

\begin{proposition}[Local Laplace expansion for de-Gaussianized chart integrals]
\label{prop:laplace-local-chart}
Assume the small-noise branch is chosen so that, for every $t \in \gI_{\rm sm}(\varepsilon)$, we have $\delta_{\rm act}(\varepsilon,t)<\eta_\Pi/2$.
Fix $i \in \{1, \dots, C_{\gM}\}$ and $R_i \in \{\gQ_i, \gP_{i,1}, \dots, \gP_{i,D}\}$.
Then, there exist coefficient functions
\begin{equation*}
    A_{i,q,\blambda}^{(R)}:
    K_i^\circ\times[\tdown,\tup]\to\R,
    \qquad
    0\le q\le \floor{\beta},
    \qquad
    \blambda\in\N^D,\quad \|\blambda\|_1\le q,
\end{equation*}
such that, for every $t\in\gI_{\rm sm}(\varepsilon)$ and every $(\bx,t)\in\gK_i(\varepsilon)$,
\begin{align}
    R_i(\bx,t)
    &=
    \exp\Bigl(
      -\frac{\|\br_i(\bx,t)\|_2^2}{2\sigma_t^2}
    \Bigr)
    \sum_{q=0}^{\floor{\beta}}
    \sigma_t^q
    \sum_{\|\blambda\|_1\le q}
    A_{i,q,\blambda}^{(R)}
    \bigl(\bu_i^\Pi(\bx,t),t\bigr)
    \bnu_i(\bx,t)^{\blambda}
    +
    \overline{\Rem}_i^{(R)}(\bx,t),
    \label{eq:laplace-expansion-local-chart}
\end{align}
where
\begin{equation}
    \bigl|
      \overline{\Rem}_i^{(R)}(\bx,t)
    \bigr|
    \leq
    C\sigma_t^\beta
    \exp\Bigl(
      -c\frac{\|\br_i(\bx,t)\|_2^2}{\sigma_t^2}
    \Bigr).
    \label{eq:laplace-remainder-local-chart}
\end{equation}
Moreover,
\begin{equation*}
    A_{i,q,\blambda}^{(R)}(\cdot,t)
    \in
    \gH^{\beta-q}(K_i^\circ,C_{\rm coef}),
    \qquad
    t\in[\tdown,\tup],
\end{equation*}
where $C_{\rm coef}$ depends only on $B_a$, $\beta$, $\underline{m},\overline{m}$, and $\mathfrak{C}_{\gM}$.
\end{proposition}

\begin{proof}
We prove the result for a generic $R_i\in\{\gQ_i,\gP_{i,1},\dots,\gP_{i,D}\}$, and write $a_i^{(R)}$ for the corresponding scalar amplitude.
Fix $t\in\gI_{\rm sm}(\varepsilon)$ and $(\bx,t)\in\gK_i(\varepsilon)$, and set
\begin{equation*}
    \bu^\Pi:=\bu_i^\Pi(\bx,t),\qquad
    \br:=\br_i(\bx,t),\qquad
    \bnu:=\bnu_i(\bx,t)=\sigma_t^{-1}\br,\qquad
    \bG^\Pi:=\bG_i(\bu^\Pi).
\end{equation*}
All constants below may depend on $\underline{m},\overline{m}, \beta$, and $\mathfrak{C}_{\gM}$, but not on $(\bx,t)$ and $\varepsilon$.

By definition,
\begin{equation*}
    R_i(\bx,t)
    =
    \sigma_t^{-d}
    \int_{\phi_i(U_i)}
      \exp\Bigl(-F_{i,\bx,t}(\bu)/\sigma_t^2\Bigr)
      a_i^{(R)}(\bu)
    \od\bu .
\end{equation*}
With $\bu=\bu^\Pi+\sigma_t\bw$ and
\begin{equation*}
    \Omega_{i,\bx,t}:=
    \sigma_t^{-1}\bigl(\phi_i(U_i)-\bu^\Pi\bigr),
\end{equation*}
the factor $\sigma_t^{-d}$ cancels and
\begin{equation}
    R_i(\bx,t)
    =
    \int_{\Omega_{i,\bx,t}}
      \exp\Bigl(
        -F_{i,\bx,t}(\bu^\Pi+\sigma_t\bw)/\sigma_t^2
      \Bigr)
      a_i^{(R)}(\bu^\Pi+\sigma_t\bw)
    \od\bw .
    \label{eq:laplace-proof-recentered-integral}
\end{equation}

\textbf{Step 1: Uniform quadratic lower bound for the phase.}
Define
\begin{equation*}
    \bxi:=\bx/m_t,
    \qquad
    \Phi_i(\bu;\bxi):=\frac12\|\bxi-\bz_i(\bu)\|_2^2.
\end{equation*}
Since $F_{i,\bx,t}(\bu)=m_t^2\Phi_i(\bu;\bxi)$ and $\bz_i(\bu^\Pi)=\Pi_{\gM}(\bxi)$,
\begin{equation*}
    F_{i,\bx,t}(\bu^\Pi)=\frac12\|\br\|_2^2 .
\end{equation*}
By \cref{lem:buffered-active-chart}, $\bu^\Pi\in K_i^\circ\Subset\Box_i$ and
\begin{equation*}
    \nabla_\bu\Phi_i(\bu^\Pi;\bxi)
    =
    \bJ_i(\bu^\Pi)^\top\bigl(\bz_i(\bu^\Pi)-\bxi\bigr)
    =
    0.
\end{equation*}
Moreover, by \cref{lem:invertible:proj},
\begin{equation*}
    \nabla_\bu^2\Phi_i(\bu^\Pi;\bxi)
    \succeq
    \frac12\lambda_i^{(-)}\bI_d .
\end{equation*}

The active condition and $\delta_{\rm act}(\varepsilon,t)<\eta_\Pi/2$ imply that $\bxi\in\mathcal X_i$.
Recall the constants $H_i$ and $r_i$ from \cref{eq:def:H-i-laplace,eq:def:r-i-laplace}.
Then, whenever $\|\bu-\bu^\Pi\|_2\le r_i$,
\begin{equation*}
    \nabla_\bu^2\Phi_i(\bu;\bxi)
    \succeq
    \frac14\lambda_i^{(-)}\bI_d .
\end{equation*}
Taylor's theorem gives
\begin{equation*}
    \Phi_i(\bu;\bxi)-\Phi_i(\bu^\Pi;\bxi)
    \ge
    \frac{\lambda_i^{(-)}}8\|\bu-\bu^\Pi\|_2^2 .
\end{equation*}
Substituting $\bu=\bu^\Pi+\sigma_t\bw$ yields, for $\sigma_t\|\bw\|_2\le r_i$,
\begin{equation}
    \frac{F_{i,\bx,t}(\bu^\Pi+\sigma_t\bw)}{\sigma_t^2}
    \ge
    \frac{\|\br\|_2^2}{2\sigma_t^2}
    +
    c_0\|\bw\|_2^2
    \label{eq:laplace-coercive-bound}
\end{equation}
for some $c_0>0$.

Next let $S_i^{(R)}:=\operatorname{supp}a_i^{(R)}\Subset\phi_i(U_i)$ and consider
\begin{equation*}
    \mathcal X_i^\circ
    :=
    \{\bzeta\in\mathcal X_i:
      \Pi_{\gM}(\bzeta)\in\phi_i^{-1}(K_i^\circ)\}.
\end{equation*}
This set is compact because $\mathcal X_i$ is compact, the reach projection is continuous on the reach tube, and $\phi_i^{-1}(K_i^\circ)$ is compact in $U_i$.
On $\mathcal X_i^\circ$, write
\begin{equation*}
    \bu_i^\Pi(\bzeta):=\phi_i(\Pi_{\gM}(\bzeta)).
\end{equation*}
Define
\begin{equation*}
    \gC_i^{(R)}
    :=
    \Bigl\{
      (\bzeta,\bu)\in\mathcal X_i^\circ\times S_i^{(R)}:
      \|\bu-\bu_i^\Pi(\bzeta)\|_2\ge r_i
    \Bigr\}.
\end{equation*}
If this compact set is nonempty, the continuous function
\begin{equation*}
    (\bzeta,\bu)\mapsto
    \Phi_i(\bu;\bzeta)
    -
    \Phi_i\bigl(\bu_i^\Pi(\bzeta);\bzeta\bigr)
\end{equation*}
is strictly positive on it.
Otherwise $\bz_i(\bu)$ would be another nearest point on $\gM$ to $\bzeta$, contradicting uniqueness of the metric projection in the reach tube.
Hence there is $c_1>0$ such that, for active $\bxi$,
\begin{equation*}
    \Phi_i(\bu;\bxi)-\Phi_i(\bu^\Pi;\bxi)
    \ge c_1
\end{equation*}
whenever $\bu\in S_i^{(R)}$ and $\|\bu-\bu^\Pi\|_2\ge r_i$.
Therefore, on this part of the support,
\begin{equation}
    \frac{F_{i,\bx,t}(\bu)}{\sigma_t^2}
    \ge
    \frac{\|\br\|_2^2}{2\sigma_t^2}
    +
    \frac{\underline{m}^2c_1}{\sigma_t^2}.
    \label{eq:laplace-compact-away-phase-bound}
\end{equation}

\textbf{Step 2: Split the integral into a local window and a tail.}
Let
\begin{equation*}
    R_\sigma:=A\sqrt{\log(e/\sigma_t)},
\end{equation*}
where $A>0$ will be chosen large enough.
By decreasing the small-noise threshold if necessary, assume
\begin{equation*}
    \sigma_tR_\sigma\le r_i\le d_i^\Box/2 .
\end{equation*}
Since $\bu^\Pi\in K_i^\circ$, this implies
\begin{equation*}
    \bu^\Pi+\sigma_t\bw\in\Box_i\subset\phi_i(U_i),
    \qquad
    \|\bw\|_2\le R_\sigma.
\end{equation*}
Thus the local part of \cref{eq:laplace-proof-recentered-integral} is over the full ball $\{\|\bw\|_2\le R_\sigma\}$.

Write $R_i=I_{\rm loc}+I_{\rm tail}$, where $I_{\rm loc}$ is the integral over $\|\bw\|_2\le R_\sigma$.
On $R_\sigma<\|\bw\|_2\le r_i/\sigma_t$, use \cref{eq:laplace-coercive-bound}; on $\|\bw\|_2>r_i/\sigma_t$, use \cref{eq:laplace-compact-away-phase-bound} on the support of $a_i^{(R)}$.
Since $a_i^{(R)}$ is bounded and $S_i^{(R)}$ has finite Lebesgue measure,
\begin{align}
    |I_{\rm tail}|
    &\le
    C
    e^{-\|\br\|_2^2/(2\sigma_t^2)}
    \int_{\|\bw\|_2>R_\sigma}
      e^{-c_0\|\bw\|_2^2}
    \od\bw
    +
    C\sigma_t^{-d}
    e^{-\|\br\|_2^2/(2\sigma_t^2)}
    e^{-c/\sigma_t^2}
    \notag\\
    &\le
    C
    e^{-\|\br\|_2^2/(2\sigma_t^2)}
    e^{-c_2R_\sigma^2}
    +
    C\sigma_t^{-d}
    e^{-\|\br\|_2^2/(2\sigma_t^2)}
    e^{-c/\sigma_t^2}.
    \label{eq:laplace-tail-bound-raw}
\end{align}
Choose $A$ so large that $e^{-c_2R_\sigma^2}\le C\sigma_t^\beta$.
Also, after decreasing the small-noise threshold if necessary, $\sigma_t^{-d}e^{-c/\sigma_t^2}\le C\sigma_t^\beta$.
Hence
\begin{equation}
    |I_{\rm tail}|
    \le
    C\sigma_t^\beta
    \exp\Bigl(
      -c\frac{\|\br\|_2^2}{\sigma_t^2}
    \Bigr).
    \label{eq:laplace-tail-bound}
\end{equation}

\textbf{Step 3: Expand the phase on the local window.}
On $\|\bw\|_2\le R_\sigma$, \cref{lem:centered-phase-expansion} and the $\floor{\beta}$-convention give the truncated form
\begin{align}
    \frac{F_{i,\bx,t}(\bu^\Pi+\sigma_t\bw)}{\sigma_t^2}
    &=
    \frac{\|\br\|_2^2}{2\sigma_t^2}
    +
    \frac{m_t^2}{2}\bw^\top\bG^\Pi\bw
    +
    \sum_{\ell=1}^{\floor{\beta}}\sigma_t^\ell
    \Psi_{i,\ell}(\bu^\Pi,\bnu,\bw,t)
    +
    \sigma_t^\beta
    \widehat\Rem_i^{(F)}(\bu^\Pi,\bnu,\bw,t).
    \label{eq:phase-expansion-local}
\end{align}
Indeed, the remainder from \cref{lem:centered-phase-expansion} has order $\sigma_t^{\floor{\beta}+1}$.
Since $\floor{\beta}<\beta\le \floor{\beta}+1$ and $\sigma_t\le1$, this is bounded by $\sigma_t^\beta$ times the same polynomial; when $\beta\in\N_+$, the two powers coincide.
Thus
\[
    |\widehat\Rem_i^{(F)}(\bu^\Pi,\bnu,\bw,t)|
    \le
    C(1+\|\bnu\|_2)(1+\|\bw\|_2^{2\floor{\beta}+6})
\]
on the local window.
Set
\begin{equation*}
    Z_\sigma
    :=
    -
    \sum_{\ell=1}^{\floor{\beta}}\sigma_t^\ell
    \Psi_{i,\ell}(\bu^\Pi,\bnu,\bw,t)
    -
    \sigma_t^\beta\widehat\Rem_i^{(F)}(\bu^\Pi,\bnu,\bw,t).
\end{equation*}
Then
\begin{equation*}
    \exp\Bigl(
      -F_{i,\bx,t}(\bu^\Pi+\sigma_t\bw)/\sigma_t^2
    \Bigr)
    =
    e^{-\|\br\|_2^2/(2\sigma_t^2)}
    e^{-\frac{m_t^2}{2}\bw^\top\bG^\Pi\bw}
    e^{Z_\sigma}.
\end{equation*}
Since $(\bx,t)\in\gK_i(\varepsilon)$, $\|\bnu\|_2\le c_\star\sqrt{\log\varepsilon^{-1}}$.
On $\gI_{\rm sm}(\varepsilon)$, $\log\varepsilon^{-1}\lesssim\log(e/\sigma_t)$.
Combining this with $\|\bw\|_2\le R_\sigma$ and \cref{eq:Psi-growth}, we obtain $|Z_\sigma|\le1/2$ after decreasing the small-noise threshold if necessary.

\textbf{Step 4: Expand the exponential correction.}
Taylor expansion of $e^{Z_\sigma}$, collecting powers of $\sigma_t$ up to order $\floor{\beta}$, gives
\begin{equation*}
    e^{Z_\sigma}
    =
    \sum_{\ell=0}^{\floor{\beta}}
    \sigma_t^\ell
    \gE_{i,\ell}(\bu^\Pi,\bnu,\bw,t)
    +
    \sigma_t^\beta
    \Rem_i^{(\exp)}(\bu^\Pi,\bnu,\bw,t),
\end{equation*}
where $\gE_{i,0}=1$, and for $1\le\ell\le \floor{\beta}$,
\begin{equation*}
    \gE_{i,\ell}(\bu,\bnu,\bw,t)
    :=
    \sum_{r=1}^{\ell}
    \frac{(-1)^r}{r!}
    \sum_{\substack{\ell_1+\cdots+\ell_r=\ell\\ \ell_j\ge1}}
    \Psi_{i,\ell_1}(\bu,\bnu,\bw,t)\cdots
    \Psi_{i,\ell_r}(\bu,\bnu,\bw,t).
\end{equation*}
Moreover, on the local window,
\begin{equation*}
    |\Rem_i^{(\exp)}(\bu^\Pi,\bnu,\bw,t)|
    \le
    C(1+\|\bnu\|_2^{P_E})(1+\|\bw\|_2^{M_E})
\end{equation*}
for finite constants $P_E,M_E$.

\textbf{Step 5: Expand the amplitude and multiply expansions.}
Since $a_i^{(R)}\in\gH^\beta(\phi_i(U_i),B_a)$, Taylor's theorem with the $\gH^\beta$ remainder gives
\begin{equation*}
    a_i^{(R)}(\bu^\Pi+\sigma_t\bw)
    =
    \sum_{m=0}^{\floor{\beta}}
    \sigma_t^m
    \gA_{i,m}^{(R)}(\bu^\Pi,\bw)
    +
    \sigma_t^\beta
    \Rem_i^{(\rm amp)}(\bu^\Pi,\bw),
\end{equation*}
where
\begin{equation*}
    \gA_{i,m}^{(R)}(\bu,\bw)
    :=
    \frac1{m!}\oD^m a_i^{(R)}(\bu)[\bw^{\otimes m}],
\end{equation*}
and
\begin{equation*}
    |\gA_{i,m}^{(R)}(\bu,\bw)|\le C(1+\|\bw\|_2^m),
    \qquad
    |\Rem_i^{(\rm amp)}(\bu,\bw)|\le C(1+\|\bw\|_2^\beta).
\end{equation*}
This is Taylor's theorem through order $\floor{\beta}$ with the $\gH^\beta$ remainder.
When $\beta\in\N_+$, the convention $\floor{\beta}=\beta-1$ makes this the standard Lipschitz top-derivative remainder of order $\beta$.
Multiplying gives
\begin{equation*}
    e^{Z_\sigma}
    a_i^{(R)}(\bu^\Pi+\sigma_t\bw)
    =
    \sum_{q=0}^{\floor{\beta}}
    \sigma_t^q
    \gU_{i,q}^{(R)}(\bu^\Pi,\bnu,\bw,t)
    +
    \sigma_t^\beta
    \Rem_i^{(R)}(\bu^\Pi,\bnu,\bw,t),
\end{equation*}
where
\begin{equation*}
    \gU_{i,q}^{(R)}(\bu,\bnu,\bw,t)
    :=
    \sum_{\substack{m+\ell=q\\0\le m,\ell\le \floor{\beta}}}
    \gA_{i,m}^{(R)}(\bu,\bw)
    \gE_{i,\ell}(\bu,\bnu,\bw,t).
\end{equation*}
All omitted products contain either an explicit amplitude or exponential remainder, or have total algebraic order at least $\floor{\beta}+1$.
Since $\floor{\beta}+1\ge\beta$, every omitted product is bounded by $\sigma_t^\beta$ times a fixed polynomial in $(\bnu,\bw)$ on $0<\sigma_t\le1$.
Furthermore,
\begin{equation*}
    |\Rem_i^{(R)}(\bu^\Pi,\bnu,\bw,t)|
    \le
    C(1+\|\bnu\|_2^{P_R})(1+\|\bw\|_2^{M_R})
\end{equation*}
on $\|\bw\|_2\le R_\sigma$.

\textbf{Step 6: Identify the principal coefficients.}
Substituting the expansions into $I_{\rm loc}$ yields
\begin{equation*}
    I_{\rm loc}
    =
    e^{-\|\br\|_2^2/(2\sigma_t^2)}
    \sum_{q=0}^{\floor{\beta}}\sigma_t^q
    \int_{\|\bw\|_2\le R_\sigma}
      e^{-\frac{m_t^2}{2}\bw^\top\bG^\Pi\bw}
      \gU_{i,q}^{(R)}(\bu^\Pi,\bnu,\bw,t)
    \od\bw
    +
    E_{\rm loc}.
\end{equation*}
By the polynomial bound on $\Rem_i^{(R)}$ and uniform ellipticity of $\bG^\Pi$,
\begin{equation*}
    |E_{\rm loc}|
    \le
    C\sigma_t^\beta
    e^{-\|\bnu\|_2^2/2}
    (1+\|\bnu\|_2^{P_R})
    \int_{\R^d}e^{-c_g\|\bw\|_2^2}
      (1+\|\bw\|_2^{M_R})
    \od\bw
    \le
    C\sigma_t^\beta e^{-c\|\bnu\|_2^2}.
\end{equation*}

For $0\le q\le \floor{\beta}$, define the coefficient functions by collecting powers of $\bnu$ in
\begin{equation*}
    \int_{\R^d}
      e^{-\frac{m_t^2}{2}\bw^\top\bG_i(\bu)\bw}
      \gU_{i,q}^{(R)}(\bu,\bnu,\bw,t)
    \od\bw
    =
    \sum_{\|\blambda\|_1\le q}
    A_{i,q,\blambda}^{(R)}(\bu,t)\bnu^{\blambda}.
\end{equation*}
This is valid because $\gU_{i,q}^{(R)}$ is a polynomial in $\bnu$ of degree at most $q$.

The replacement of the local integral by the full integral costs only a tail error.
Indeed,
\begin{equation*}
    |\gU_{i,q}^{(R)}(\bu^\Pi,\bnu,\bw,t)|
    \le
    C(1+\|\bnu\|_2^q)(1+\|\bw\|_2^{M_U}),
\end{equation*}
so, after increasing $A$ if necessary,
\begin{equation*}
\begin{aligned}
    &e^{-\|\br\|_2^2/(2\sigma_t^2)}
    \sigma_t^q
    \left|
    \int_{\|\bw\|_2>R_\sigma}
      e^{-\frac{m_t^2}{2}\bw^\top\bG^\Pi\bw}
      \gU_{i,q}^{(R)}(\bu^\Pi,\bnu,\bw,t)
    \od\bw
    \right|  \\
    &\qquad\le
    C e^{-c\|\bnu\|_2^2}e^{-c_3R_\sigma^2/2}
    \le
    C\sigma_t^\beta e^{-c\|\bnu\|_2^2}.
\end{aligned}
\end{equation*}
Combining this extension error with the bounds for $E_{\rm loc}$ and $I_{\rm tail}$ proves \cref{eq:laplace-expansion-local-chart,eq:laplace-remainder-local-chart}, because $\|\bnu\|_2^2=\|\br\|_2^2/\sigma_t^2$.

\textbf{Step 7: H\"older regularity of the coefficients.}
Each $A_{i,q,\blambda}^{(R)}(\bu,t)$ is a finite linear combination of integrals of terms of the form
\begin{equation*}
    e^{-\frac{m_t^2}{2}\bw^\top\bG_i(\bu)\bw}
    \times
    \text{polynomial in }\bw
    \times
    \text{smooth geometric factors in }\bu
    \times
    \partial^\alpha a_i^{(R)}(\bu),
    \qquad
    |\alpha|\le q.
\end{equation*}
The tensor $\bG_i(\bu)$ is smooth and uniformly positive definite on $K_i^\circ$, so derivatives falling on the Gaussian factor are dominated by $e^{-c\|\bw\|_2^2}$ times a polynomial in $\|\bw\|_2$, uniformly in $\bu\in K_i^\circ$ and $t\in[\tdown,\tup]$.
Hence differentiation under the integral is justified up to the order allowed by the amplitude.
Since $a_i^{(R)}\in\gH^\beta(\phi_i(U_i),B_a)$, the worst regularity occurs when $|\alpha|=q$, giving $\partial^\alpha a_i^{(R)}\in\gH^{\beta-q}$.
Multiplication by smooth geometric factors and integration against the uniformly dominated Gaussian preserve this regularity.
Therefore
\begin{equation*}
    A_{i,q,\blambda}^{(R)}(\cdot,t)
    \in
    \gH^{\beta-q}(K_i^\circ,C_{\rm coef}),
    \qquad t\in[\tdown,\tup].
\end{equation*}
This completes the proof.
\end{proof}

\textbf{Choice of noise-regime thresholds.}
All phrases above of the form ``after decreasing the small-noise threshold if necessary'' are collected in the single threshold $\bar\sigma_{\rm Lap}$ below.
The branch constants $C_{\rm sm}$ and $c_{\rm lg}$ are chosen only after this collection is complete, so the small/large-noise cover and its overlap are not affected by the intermediate reductions.
Since the atlas is finite, choose $C_{\rm Lap}^{\max}\ge1$ to dominate the constants in the Laplace remainder bound \cref{eq:laplace-remainder-local-chart}, uniformly over all charts $i$ and all $R\in\mathcal R$.
Also choose $\bar\sigma_{\rm Lap}\in(0,1]$ small enough that every small-noise requirement invoked in the proof of \cref{prop:laplace-local-chart} holds whenever $0<\sigma_t\le\bar\sigma_{\rm Lap}$: in particular $\sigma_tR_\sigma\le r_i$, the compact-away tail satisfies $\sigma_t^{-d}e^{-c/\sigma_t^2}\lesssim\sigma_t^\beta$, and the exponential correction obeys $|Z_\sigma|\le1/2$, uniformly over all charts and integrands.
Such a number exists because the atlas and the retained coefficients are finite and because $\sigma^{-\ell}e^{-c/\sigma^2}\to0$ and $\sigma^\gamma(\log(e/\sigma))^M\to0$ as $\sigma\downarrow0$ for every fixed $\ell,\gamma,M>0$.
From now on, the constants $C_{\rm sm}$ and $c_{\rm lg}$ in \cref{def:noise-regimes} are fixed so that
\begin{equation}
    C_{\rm sm}\le \bar\sigma_{\rm Lap},
    \qquad
    C_{\rm Lap}^{\max}C_{\rm sm}^{\beta}\le \frac16,
    \qquad
    4c_{\rm lg}\le C_{\rm sm}.
    \label{eq:choice:C-sm-Laplace}
\end{equation}
Thus, on $\gI_{\rm sm}(\varepsilon)$, the analytic Laplace remainder is bounded by the target order $\Ord(\varepsilon)$, while $\gI_{\rm sm}(\varepsilon)$ and $\gI_{\rm lg}(\varepsilon)$ still have the overlap needed for branch interpolation.

% ----------------------------------------------------------------------
\subsection{Network assembly and final small-noise theorem}
\label{sec:app:small-noise-network-assembly}

After the projection-coordinate networks are available, the remaining neural operations
are algebraic or standard approximation tasks: residuals and normalized normals,
Gaussian factors, Laplace coefficients, certified active-chart gates, and the final
stable ratio.  This section assembles those pieces into the small-noise score network.

\subsubsection{Residual, normal-coordinate, and Gaussian-factor networks}
\label{sec:app:approx:manifold:nn:residual-gaussian}

After the projection network is available, the remaining geometric quantities are algebraic.
The residual is $\br_i=\bx-m_t\Pi_{\gM}(\bx/m_t)$, the normalized residual is $\bnu_i=\sigma_t^{-1}\br_i$, and the Gaussian factor is $\exp(-\|\bnu_i\|_2^2/2)$.
The following lemmas propagate the projection error through these operations and record the resulting network sizes.

\begin{lemma}[ReLU approximation of projection residuals and normalized normals]
\label{lem:approx:r-i--nu-i}
    Fix $i\in\{1, \dots, C_{\gM}\}$.
    Fix a localization parameter $0<\varepsilon\le1$ such that \cref{eq:condition:delta} holds for every $t\in\gI_{\rm sm}(\varepsilon)$.
    For $t \in \gI_{\rm sm}(\varepsilon)$ and $(\bx,t) \in \gK_i(\varepsilon)$, define
    \begin{equation*}
        \bpi(\bx,t) := \Pi_{\gM}\bigl(\bxi(\bx,t)\bigr),
        \qquad
        \br_i(\bx,t) := \bx-m_t\bpi(\bx,t),
        \qquad
        \bnu_i(\bx,t) := \sigma_t^{-1}\br_i(\bx,t).
    \end{equation*}
    Let
    \begin{equation*}
        C_{\nu,\varepsilon}
        :=
        2\sigma_{\tdown}^{-1}+c_\star\sqrt{\log(e/\varepsilon)}.
    \end{equation*}
    For every $\varepsilon_{\rm res} \in (0,1]$, define
    \begin{equation*}
        C_{\rm res}
        :=
        e\vee
        \bigl(4096\,C_{\rm proj}C_{\nu,\varepsilon}^2\bigr),
    \end{equation*}
    where $C_{\rm proj}$ is the geometry-dependent constant from \cref{thm:global-proj-NN}.
    Then there exist ReLU networks
    \begin{equation*}
        \widetilde{\br}_{i,\varepsilon_{\rm res}},
        \widetilde{\bnu}_{i,\varepsilon_{\rm res}}
        :
        \R^D\times[\tdown,\tup]\to\R^D
    \end{equation*}
    such that
    \begin{equation*}
        \sup_{\substack{t \in \gI_{\rm sm}(\varepsilon)\\
        (\bx,t) \in \gK_i(\varepsilon)}}
        \bigl\|
          \widetilde{\br}_{i,\varepsilon_{\rm res}}(\bx,t)-\br_i(\bx,t)
        \bigr\|_\infty
        \leq
        \varepsilon_{\rm res},
    \end{equation*}
    and
    \begin{equation*}
        \sup_{\substack{t \in \gI_{\rm sm}(\varepsilon)\\
        (\bx,t) \in \gK_i(\varepsilon)}}
        \bigl\|
          \widetilde{\bnu}_{i,\varepsilon_{\rm res}}(\bx,t)-\bnu_i(\bx,t)
        \bigr\|_\infty
        \leq
        \varepsilon_{\rm res}.
    \end{equation*}
    Moreover,
    \begin{equation}
    \begin{aligned}
        L
        &\lesssim
        \log(e(d\vee\beta))
        \log^3\bigl(C_{\rm res}D\varepsilon_{\rm res}^{-2}\bigr),
        \\
        \|\bW\|_\infty
        &\lesssim
        A_{\gM}
        dD
        (d\vee\beta)
        \bigl(C_{\rm res}D\varepsilon_{\rm res}^{-2}\bigr)^{d/(2(\beta\vee1))}
        +
        A_{\gM}
        D\log^3\bigl(C_{\rm res}D\varepsilon_{\rm res}^{-2}\bigr),
        \\
        S
        &\lesssim
        A_{\gM}
        dD(d+2(\beta\vee1))^{d+3}
        \bigl(C_{\rm res}D\varepsilon_{\rm res}^{-2}\bigr)^{d/(2(\beta\vee1))}
        \log^2\bigl(C_{\rm res}D\varepsilon_{\rm res}^{-2}\bigr)
        +
        A_{\gM}
        dD\log^5\bigl(C_{\rm res}D\varepsilon_{\rm res}^{-2}\bigr),
        \\
        \log B
        &\lesssim
        \log^2\bigl(C_{\rm res}D\varepsilon_{\rm res}^{-2}\bigr).
    \end{aligned}
    \label{eq:residual-nu-size}
    \end{equation}
\end{lemma}

\begin{proof}
\textbf{Step 1: Active-region bounds.}
We work uniformly over $t \in \gI_{\rm sm}(\varepsilon)$ and $(\bx,t) \in \gK_i(\varepsilon)$.
Since $(\bx,t) \in \gK_i(\varepsilon)$, there exists $\by \in \gS_i \subset \gM$ such that
\begin{equation*}
    \|\bx-m_t\by\|_2
    \leq
    c_\star\sigma_t\sqrt{\log(e/\varepsilon)}.
\end{equation*}
Since $\bpi(\bx,t)$ is the closest point on $\gM$ to $\bxi(\bx,t)=\bx/m_t$,
\begin{align}
    \|\br_i(\bx,t)\|_2
    &=
    m_t\dist\bigl(\bxi(\bx,t),\gM\bigr)
    \notag\\
    &\leq
    m_t\|\bxi(\bx,t)-\by\|_2
    =
    \|\bx-m_t\by\|_2
    \notag\\
    &\leq
    c_\star\sigma_t\sqrt{\log(e/\varepsilon)}.
    \label{eq:active-residual-bound-for-approx}
\end{align}
Thus
\begin{equation*}
    \|\bnu_i(\bx,t)\|_2
    =
    \sigma_t^{-1}\|\br_i(\bx,t)\|_2
    \leq
    c_\star\sqrt{\log(e/\varepsilon)}
    \leq
    C_{\nu,\varepsilon}.
\end{equation*}
Also, because $\sigma_t \leq 1$,
\begin{equation*}
    |r_{i,j}(\bx,t)|
    \leq
    c_\star\sqrt{\log(e/\varepsilon)}
    \leq
    C_{\nu,\varepsilon}.
\end{equation*}

\textbf{Step 2: Approximate the projection center and residual.}
Set
\begin{equation*}
    \varepsilon_r
    :=
    \frac{\varepsilon_{\rm res}}{8C_{\nu,\varepsilon}},
    \qquad
    \varepsilon_{\rm cen}:=\varepsilon_r/8.
\end{equation*}
By \cref{thm:global-proj-NN}, applied with localization parameter $\varepsilon$ and accuracy $\varepsilon_{\rm cen}$, there is a ReLU network $\widetilde{\bpi}_{\varepsilon_{\rm cen}}$ satisfying
\begin{equation*}
    \bigl\|
      \widetilde{\bpi}_{\varepsilon_{\rm cen}}(\bx,t)-\bpi(\bx,t)
    \bigr\|_\infty
    \leq
    \varepsilon_{\rm cen}.
\end{equation*}
In \cref{thm:global-proj-NN}, we use the smoothness $2(\beta\vee1)$ of the geometric chart maps, hence its non-logarithmic factor is controlled by
\begin{equation*}
    \bigl(C_{\rm proj}D\varepsilon_{\rm cen}^{-2}\bigr)^{d/(2(\beta\vee1))}.
\end{equation*}
Since
\begin{equation*}
    \varepsilon_{\rm cen}^{-2}
    =
    64\varepsilon_r^{-2}
    =
    4096C_{\nu,\varepsilon}^2\varepsilon_{\rm res}^{-2},
\end{equation*}
we have
\begin{equation*}
    C_{\rm proj}D\varepsilon_{\rm cen}^{-2}
    \leq
    C_{\rm res}D\varepsilon_{\rm res}^{-2}.
\end{equation*}

By \cref{lem:approx:m-sigma}, let $\widetilde m(t)$ approximate $m_t$ with
\begin{equation*}
    |\widetilde m(t)-m_t|
    \leq
    \varepsilon_m,
    \qquad
    \varepsilon_m:=\varepsilon_r/8.
\end{equation*}
For each $j=1,\dots,D$, use \cref{lem:approx:mon} to multiply $\widetilde m(t)$ and $\widetilde{\pi}_{\varepsilon_{\rm cen},j}(\bx,t)$ with multiplication accuracy $\varepsilon_{M,r}:=\varepsilon_r/2$.
Since $m_t\in[0,1]$ and $\bpi(\bx,t)\in\gM\subset[0,1]^D$, the true inputs lie in $[-1,1]^2$, and the two input errors are at most $\varepsilon_r/8$.
Therefore
\begin{equation*}
    |\widetilde h_j(\bx,t)-m_t\pi_j(\bx,t)|
    \leq
    \varepsilon_{M,r}
    +
    2\cdot1\cdot\frac{\varepsilon_r}{8}
    =
    \frac34\varepsilon_r
    \leq
    \varepsilon_r.
\end{equation*}
Define
\begin{equation*}
    \widetilde r_{i,j}(\bx,t):=x_j-\widetilde h_j(\bx,t).
\end{equation*}
Then
\begin{equation*}
    |\widetilde r_{i,j}(\bx,t)-r_{i,j}(\bx,t)|
    \leq
    \varepsilon_r,
\end{equation*}
and therefore
\begin{equation*}
    \bigl\|
      \widetilde{\br}_{i,\varepsilon_{\rm res}}(\bx,t)-\br_i(\bx,t)
    \bigr\|_\infty
    \leq
    \varepsilon_r
    \leq
    \varepsilon_{\rm res}.
\end{equation*}

\textbf{Step 3: Approximate the normalized residual.}
Let $\phi_{1/\sigma}(t)$ approximate $\sigma_t^{-1}$ with
\begin{equation*}
    |\phi_{1/\sigma}(t)-\sigma_t^{-1}|
    \leq
    \varepsilon_\sigma,
    \qquad
    \varepsilon_\sigma:=\frac{\varepsilon_{\rm res}}{8C_{\nu,\varepsilon}}.
\end{equation*}
For each $j$, use \cref{lem:approx:mon} to multiply $\phi_{1/\sigma}(t)$ and $\widetilde r_{i,j}(\bx,t)$ with multiplication accuracy $\varepsilon_{M,\nu}:=\varepsilon_{\rm res}/2$.
The true inputs lie in $[-C_{\nu,\varepsilon},C_{\nu,\varepsilon}]^2$, and the two input errors are at most $\varepsilon_{\rm res}/(8C_{\nu,\varepsilon})$.
Hence
\begin{equation*}
    |\widetilde\nu_{i,j}(\bx,t)-\sigma_t^{-1}r_{i,j}(\bx,t)|
    \leq
    \varepsilon_{M,\nu}
    +
    2C_{\nu,\varepsilon}
    \frac{\varepsilon_{\rm res}}{8C_{\nu,\varepsilon}}
    =
    \frac34\varepsilon_{\rm res}
    \leq
    \varepsilon_{\rm res}.
\end{equation*}
Thus
\begin{equation*}
    \bigl\|
      \widetilde{\bnu}_{i,\varepsilon_{\rm res}}(\bx,t)-\bnu_i(\bx,t)
    \bigr\|_\infty
    \leq
    \varepsilon_{\rm res}.
\end{equation*}

\textbf{Step 4: Size.}
The construction consists of the projection-center network, the schedule networks for $m_t$ and $\sigma_t^{-1}$, $D$ product subnetworks for $m_t\bpi$, $D$ product subnetworks for $\sigma_t^{-1}\br_i$, and affine subtraction layers.
The center network contributes the bound from \cref{eq:global-proj-size} with $C_{\rm proj}D\varepsilon_{\rm cen}^{-2} \leq C_{\rm res}D\varepsilon_{\rm res}^{-2}$.
The remaining components contribute only logarithmic powers in $C_{\nu,\varepsilon}$, $\varepsilon_{\rm res}^{-1}$, and $\sigma_{\tdown}^{-1}$, all bounded by powers of $\log(C_{\rm res}D\varepsilon_{\rm res}^{-2})$.
Summing the components gives \cref{eq:residual-nu-size}.
\end{proof}

\begin{lemma}[ReLU approximation of the projection-center Gaussian factor]
\label{lem:approx:gaussian-factor}
Fix $i\in\{1, \dots, C_{\gM}\}$.
Fix a localization parameter $0<\varepsilon\le1$ such that \cref{eq:condition:delta} holds for every $t\in\gI_{\rm sm}(\varepsilon)$.
Define
\begin{equation*}
    \gG_i(\bx,t)
    :=
    \exp\Bigl(
      -\frac{\|\br_i(\bx,t)\|_2^2}{2\sigma_t^2}
    \Bigr)
    =
    \exp\Bigl(
      -\frac12\|\bnu_i(\bx,t)\|_2^2
    \Bigr).
\end{equation*}
Let $C_{\nu,\varepsilon}$ and $C_{\rm res}$ be as in \cref{lem:approx:r-i--nu-i}.
For every $\varepsilon_G \in (0,1]$, define
\begin{equation*}
    \varepsilon_\nu
    :=
    \frac{\varepsilon_G}{12D^{1/2} C_{\nu,\varepsilon}},
    \qquad
    C_{\rm gau}
    :=
    e\vee
    \Bigl(
      144D^2 C_{\rm res}C_{\nu,\varepsilon}^2
    \Bigr).
\end{equation*}
Then there exists a ReLU network $\widetilde{\gG}_{i,\varepsilon_G}:\R^D\times[\tdown,\tup]\to\R$ such that
\begin{equation*}
    \sup_{\substack{
        t \in \gI_{\rm sm}(\varepsilon)\\
        (\bx,t) \in \gK_i(\varepsilon)
    }}
    \bigl|
      \widetilde{\gG}_{i,\varepsilon_G}(\bx,t)-\gG_i(\bx,t)
    \bigr|
    \leq
    \varepsilon_G.
\end{equation*}
Moreover,
\begin{equation}
\begin{aligned}
    L
    &\lesssim
    \log(e(d\vee\beta))
    \log^3(C_{\rm gau}\varepsilon_G^{-2}),
    \\
    \|\bW\|_\infty
    &\lesssim
    A_{\gM}
    dD
    (d\vee\beta)
    \bigl(C_{\rm gau}\varepsilon_G^{-2}\bigr)^{d/(2(\beta\vee1))}
    +
    A_{\gM}
    D\log^3(C_{\rm gau}\varepsilon_G^{-2}),
    \\
    S
    &\lesssim
    A_{\gM}
    dD(d+2(\beta\vee1))^{d+3}
    \bigl(C_{\rm gau}\varepsilon_G^{-2}\bigr)^{d/(2(\beta\vee1))}
    \log^2(C_{\rm gau}\varepsilon_G^{-2})
    +
    A_{\gM}
    dD\log^5(C_{\rm gau}\varepsilon_G^{-2}),
    \\
    \log B
    &\lesssim
    \log^2(C_{\rm gau}\varepsilon_G^{-2}) .
\end{aligned}
\label{eq:gaussian-factor-size}
\end{equation}
\end{lemma}

\begin{proof}
\textbf{Step 1: Approximate $\bnu_i$.}
Set
\begin{equation*}
    s_i(\bx,t):=\frac12\|\bnu_i(\bx,t)\|_2^2,
    \qquad
    \gG_i(\bx,t)=e^{-s_i(\bx,t)}.
\end{equation*}
By \cref{lem:approx:r-i--nu-i}, applied with accuracy $\varepsilon_\nu$, there exists $\widetilde{\bnu}_{i,\varepsilon_\nu}$ such that
\begin{equation*}
    \|\widetilde{\bnu}_{i,\varepsilon_\nu}(\bx,t)-\bnu_i(\bx,t)\|_\infty
    \leq
    \varepsilon_\nu.
\end{equation*}
By \cref{eq:active-residual-bound-for-approx},
\begin{equation*}
    \|\bnu_i(\bx,t)\|_2
    \leq
    c_\star\sqrt{\log(e/\varepsilon)}
    \leq
    C_{\nu,\varepsilon}.
\end{equation*}

\textbf{Step 2: Approximate the exponent.}
Let
\begin{equation*}
    \varepsilon_{\rm sq}:=\frac{\varepsilon_G}{3D}.
\end{equation*}
For each $\ell=1,\dots,D$, apply \cref{lem:approx:mon} to square $\widetilde\nu_{i,\varepsilon_\nu,\ell}$.
Thus
\begin{equation*}
    |\widetilde q_\ell-\widetilde\nu_{i,\varepsilon_\nu,\ell}^2|
    \leq
    \varepsilon_{\rm sq}.
\end{equation*}
Define
\begin{equation*}
    \widetilde{s}_i(\bx,t)
    :=
    \frac12\sum_{\ell=1}^D\widetilde q_\ell(\bx,t).
\end{equation*}
The active residual bound is in $\ell_2$, not merely coordinatewise.  Hence
\[
    \|\widetilde{\bnu}_{i,\varepsilon_\nu}(\bx,t)-\bnu_i(\bx,t)\|_2
    \le D^{1/2}\varepsilon_\nu
    \le \frac{\varepsilon_G}{12C_{\nu,\varepsilon}} .
\]
Using
\[
    \left|
      \|\widetilde{\bnu}_{i,\varepsilon_\nu}\|_2^2-\|\bnu_i\|_2^2
    \right|
    \le
    \|\widetilde{\bnu}_{i,\varepsilon_\nu}-\bnu_i\|_2
    \left(2\|\bnu_i\|_2+
    \|\widetilde{\bnu}_{i,\varepsilon_\nu}-\bnu_i\|_2\right),
\]
and $C_{\nu,\varepsilon}\ge1$, the perturbation error in the exact squared norm is
at most $\varepsilon_G/4$.
Then
\begin{equation*}
    |\widetilde{s}_i(\bx,t)-s_i(\bx,t)|
    \leq
    \frac{D}{2}\varepsilon_{\rm sq}
    +
    \frac{\varepsilon_G}{4}
    \le
    \frac{5\varepsilon_G}{12}.
\end{equation*}

\textbf{Step 3: Approximate the exponential.}
Set
\begin{equation*}
    \widetilde{s}_i^+(\bx,t):=\ReLU(\widetilde{s}_i(\bx,t)).
\end{equation*}
Since $s_i\geq0$ and $\ReLU$ is $1$-Lipschitz,
\begin{equation*}
    |\widetilde{s}_i^+(\bx,t)-s_i(\bx,t)|
    \leq
    \frac{5\varepsilon_G}{12}.
\end{equation*}
Apply \cref{lem:approx:exp} with accuracy $\varepsilon_{\exp}:=\varepsilon_G/3$ and define
\begin{equation*}
    \widetilde{\gG}_{i,\varepsilon_G}(\bx,t)
    :=
    \phi_{\exp}(\widetilde{s}_i^+(\bx,t)).
\end{equation*}
Then
\begin{equation*}
    |\widetilde{\gG}_{i,\varepsilon_G}(\bx,t)-\gG_i(\bx,t)|
    \leq
    \varepsilon_{\exp}
    +
    |\widetilde{s}_i^+(\bx,t)-s_i(\bx,t)|
    \leq
    \varepsilon_G.
\end{equation*}

\textbf{Step 4: Network size.}
Since
\begin{equation*}
    C_{\rm res}D\varepsilon_\nu^{-2}
    =
    144D^2 C_{\rm res}C_{\nu,\varepsilon}^2\varepsilon_G^{-2}
    \leq
    C_{\rm gau}\varepsilon_G^{-2},
\end{equation*}
the normal-residual network has the size from~\cref{eq:residual-nu-size} with $C_{\rm res}D\varepsilon_{\rm res}^{-2}$ replaced by $C_{\rm gau}\varepsilon_G^{-2}$.
The square subnetworks and the exponential network add only logarithmic powers in $C_{\rm gau}\varepsilon_G^{-2}$, and the summation and ReLU gate are exact.
Therefore \cref{eq:gaussian-factor-size} follows.
\end{proof}

\subsubsection{Laplace coefficient networks on the active branch}
\label{sec:app:approx:manifold:nn:coefficients}

The Laplace coefficients are the only factors in the small-noise expansion that depend on the density regularity.
We first factor each coefficient into schedule powers and H\"older functions of the projection coordinate, extending the latter to a buffered box while preserving the chart-support property.
Then we approximate these coefficient functions and compose them with the coordinate projection network.
The weighted corollary includes the powers of $\sigma_t$ and $\bnu_i$ that multiply each coefficient in the Laplace sum.

\begin{lemma}[Factorization and extension of the Laplace coefficient functions]
\label{lem:A-coeff-extension}
Fix $i\in\{1,\dots,C_{\gM}\}$, $R\in\mathcal R$, $0\le q\le \floor{\beta}$, and $\blambda\in\N^D$ with $\|\blambda\|_1\le q$.
Assume $\beta-q>0$.
Let $K_i^\circ$ and $\Box_i$ be the buffered coordinate boxes fixed in \cref{eq:def:buffered-coordinate-boxes,eq:def:Ki-circ}.
For every intermediate box $\Box_i^A$ with $K_i^\circ\Subset\Box_i^A\Subset\Box_i$, there exist a nonempty finite set $\mathcal E_{i,q,\blambda}^{(R)}$, integers $r_e\in\Z$, and functions
\[
    B_{i,q,\blambda,e}^{(R)}\in
    \gH^{\beta-q}(\Box_i^A,C_{A,q}),
    \qquad e\in\mathcal E_{i,q,\blambda}^{(R)},
\]
such that the function
\begin{equation}
    \overline A_{i,q,\blambda}^{(R)}(\bu,t)
    :=
    \sum_{e\in\mathcal E_{i,q,\blambda}^{(R)}}
    m_t^{r_e}
    B_{i,q,\blambda,e}^{(R)}(\bu),
    \qquad
    (\bu,t)\in\Box_i^A\times[\tdown,\tup],
    \label{eq:A-coeff-extension-def}
\end{equation}
is an extension of $A_{i,q,\blambda}^{(R)}$:
\begin{equation}
    \overline A_{i,q,\blambda}^{(R)}(\bu,t)
    =
    A_{i,q,\blambda}^{(R)}(\bu,t),
    \qquad
    (\bu,t)\in K_i^\circ\times[\tdown,\tup].
    \label{eq:A-coeff-m-factorization}
\end{equation}
With $S_i^{(R)}:=\supp a_i^{(R)}$, the extension can be chosen so that
\begin{equation}
    \overline A_{i,q,\blambda}^{(R)}(\bu,t)=0,
    \qquad
    (\bu,t)\in
    \bigl(\Box_i^A\setminus S_i^{(R)}\bigr)\times[\tdown,\tup].
    \label{eq:A-coeff-extension-support}
\end{equation}
Moreover,
\begin{equation}
    \sup_{t\in[\tdown,\tup]}
    \left\|
      \overline A_{i,q,\blambda}^{(R)}(\cdot,t)
    \right\|_{\gH^{\beta-q}(\Box_i^A)}
    \le
    C_{A,q}.
    \label{eq:A-coeff-extension-holder}
\end{equation}
The constants $|\mathcal E_{i,q,\blambda}^{(R)}|$, $\max_e|r_e|$, and $C_{A,q}$ are finite and depend only on $i,q,\blambda,R,\beta,B_a,\underline{m},\overline{m}$, and the geometric constants.
\end{lemma}

\begin{proof}
By \cref{eq:def:Psi-i-l}, every $\Psi_{i,\ell}$ is the sum of a term linear in $\bnu$ with coefficient $m_t$ and a term independent of $\bnu$ with coefficient $m_t^2$.
Hence each coefficient of $\bnu^{\blambda}$ in $\gE_{i,\ell}$ and $\gU_{i,q}^{(R)}$ is a finite sum of terms of the form
\[
    m_t^a
    P(\bw)
    H(\bu)
    \partial^\alpha a_i^{(R)}(\bu),
    \qquad
    |\alpha|\le q,
\]
where $a$ is an integer, $P$ is a polynomial in $\bw$, and $H$ is a smooth geometric factor on $\Box_i$.
Therefore $A_{i,q,\blambda}^{(R)}$ is a finite sum of integrals
\begin{equation}
    m_t^a
    H(\bu)\partial^\alpha a_i^{(R)}(\bu)
    \int_{\R^d}
      e^{-\frac{m_t^2}{2}\bw^\top\bG_i(\bu)\bw}
      \bw^\rho\od\bw ,
    \label{eq:A-coeff-gaussian-moment-term}
\end{equation}
with $|\alpha|\le q$ and finite multi-indices $\rho$.

Since $\bG_i(\bu)$ is smooth and uniformly positive definite on $\Box_i$, the Gaussian moment formula gives
\[
    \int_{\R^d}
      e^{-\frac{m_t^2}{2}\bw^\top\bG_i(\bu)\bw}
      \bw^\rho\od\bw
    =
    m_t^{-d-|\rho|}
    (2\pi)^{d/2}
    \det(\bG_i(\bu))^{-1/2}
    P_\rho(\bG_i(\bu)^{-1}),
\]
with the usual convention that odd moments vanish.
Substituting this identity into \cref{eq:A-coeff-gaussian-moment-term} gives an explicit finite formula for the coefficient on $K_i^\circ$.
Fix the stated intermediate box $\Box_i^A$, and use the same explicit Gaussian-moment formula to define the extension on all of $\Box_i^A$.
Equivalently, after collecting the finitely many powers of $m_t$, define $B_{i,q,\blambda,e}^{(R)}$ directly on $\Box_i^A$ as the corresponding finite sums of products of smooth geometric factors and derivatives $\partial^\alpha a_i^{(R)}$, $|\alpha|\le q$.
If a coefficient is identically zero, take $\mathcal E_{i,q,\blambda}^{(R)}$ to be a singleton with $r_e=0$ and $B_{i,q,\blambda,e}^{(R)}\equiv0$.

Since $a_i^{(R)}\in\gH^\beta$ on the coordinate domain and the geometric factors are smooth on $\Box_i$, each directly defined $B_{i,q,\blambda,e}^{(R)}$ belongs to $\gH^{\beta-q}(\Box_i^A)$.
The formula agrees with the original coefficient on $K_i^\circ$, giving \cref{eq:A-coeff-m-factorization}.
Moreover, if $\bu\notin S_i^{(R)}=\supp a_i^{(R)}$, then $a_i^{(R)}$ vanishes in a neighborhood of $\bu$, and all derivatives $\partial^\alpha a_i^{(R)}(\bu)$ with $|\alpha|\le q$ vanish.
The directly defined Gaussian-moment formula is therefore zero at $\bu$, which proves \cref{eq:A-coeff-extension-support}.
Taking $C_{A,q}$ to dominate the finitely many $\gH^{\beta-q}$-norms of the directly defined factors proves \cref{eq:A-coeff-extension-holder} and the stated dependence of the constants.
\end{proof}

\begin{lemma}[ReLU approximation of the Laplace coefficient functions]
\label{lem:approx:A-coeff}
Fix $i,R,q,\blambda$ as in \cref{lem:A-coeff-extension}, and assume $\beta-q>0$.
Let $\Box_i^A$, $\mathcal E_{i,q,\blambda}^{(R)}$, $r_e$, $B_{i,q,\blambda,e}^{(R)}$, and $\overline A_{i,q,\blambda}^{(R)}$ be the objects supplied by \cref{lem:A-coeff-extension}.
For every $0<\varepsilon_A\le1/2$, there is a ReLU network
\[
    \widetilde A_{i,q,\blambda,\varepsilon_A}^{(R)}
    :
    \R^d\times[\tdown,\tup]\to\R
\]
such that
\begin{equation}
    \sup_{(\bu,t)\in \Box_i^A\times[\tdown,\tup]}
    \left|
      \widetilde A_{i,q,\blambda,\varepsilon_A}^{(R)}(\bu,t)
      -
      \overline A_{i,q,\blambda}^{(R)}(\bu,t)
    \right|
    \le
    \varepsilon_A.
    \label{eq:A-coeff-approx}
\end{equation}
Moreover, for a finite constant $C_{A,{\rm nn}}\ge1$,
\begin{align}
    L,\log B
    &\lesssim
    \log(e(d\vee(\beta-q)))
    \log^2(C_{A,{\rm nn}}\varepsilon_A^{-1}),
    \notag\\
    \|\bW\|_\infty
    &\lesssim
    (d\vee(\beta-q))
    (C_{A,{\rm nn}}\varepsilon_A^{-1})^{d/(\beta-q)}
    +
    \log^3(C_{A,{\rm nn}}\varepsilon_A^{-1}),
    \notag\\
    S
    &\lesssim
    (d+\beta-q+2)^{d+4}
    (C_{A,{\rm nn}}\varepsilon_A^{-1})^{d/(\beta-q)}
    \log(C_{A,{\rm nn}}\varepsilon_A^{-1})
    +
    \log^4(C_{A,{\rm nn}}\varepsilon_A^{-1}).
    \label{eq:A-coeff-size}
\end{align}
Thus the non-logarithmic coefficient cost is $(\varepsilon_A^{-1})^{d/(\beta-q)}$.
\end{lemma}

\begin{proof}
Let
\[
    N_A:=|\mathcal E_{i,q,\blambda}^{(R)}|,
    \qquad
    r_A:=\max_{e\in\mathcal E_{i,q,\blambda}^{(R)}}|r_e|,
    \qquad
    M_A:=1\vee\underline{m}^{-r_A}\vee\overline{m}^{r_A}.
\]
For each $e$, apply \cref{thm:approx:holder}, after affine rescaling of $\Box_i^A$ to $[0,1]^d$, to approximate $B_{i,q,\blambda,e}^{(R)}$ to accuracy
\[
    \varepsilon_B:=\frac{\varepsilon_A}{8N_A M_A}.
\]
This gives the intrinsic cost $(C\varepsilon_A^{-1})^{d/(\beta-q)}$, up to logarithmic factors.

For the one-dimensional schedule factor $m_t^{r_e}$, use the following three cases.
If $r_e=0$, represent $m_t^{r_e}$ by the constant-one network.
If $r_e>0$, approximate $m_t$ by \cref{lem:approx:m-sigma} and form the fixed positive power by \cref{lem:approx:mon}.
If $r_e<0$, approximate $1/m_t$ by \cref{lem:approx:m-inv} and form the fixed positive power $(1/m_t)^{-r_e}$ by \cref{lem:approx:mon}.
Because $r_A$ is fixed and $m_t\in[\underline{m},\overline{m}]$, this schedule subnetwork has only logarithmic dependence on $\varepsilon_A^{-1}$, and can be made accurate enough that the error in each product
\[
    m_t^{r_e}B_{i,q,\blambda,e}^{(R)}(\bu)
\]
is at most $\varepsilon_A/(2N_A)$ uniformly on $\Box_i^A\times[\tdown,\tup]$.
Summing the $N_A$ products in an affine output layer yields \cref{eq:A-coeff-approx}.
Adding the intrinsic H\"older-approximation cost, the logarithmic schedule cost, and the fixed product/summation overhead gives \cref{eq:A-coeff-size}.
\end{proof}

\begin{corollary}[Weighted coefficient approximation on the active branch]
\label{cor:approx:weighted-A-coeff}
Fix $i,R,q,\blambda$ as in \cref{lem:A-coeff-extension}, and assume $\beta-q>0$.
Let $0<\varepsilon\le1/2$ be a localization parameter for which \cref{eq:condition:delta} holds on $\gI_{\rm sm}(\varepsilon)$.
Assume the small-noise branch satisfies
\[
    \sigma_t\le\sigma_\star(\varepsilon)
    \qquad\text{for every }t\in\gI_{\rm sm}(\varepsilon),
\]
For every $0<\varepsilon_{\rm term}\le1/2$, define
\begin{equation}
    \varepsilon_{A,q}
    :=
    \frac12
    \wedge
    c_{A,q}\,
    \varepsilon_{\rm term}\,
    \sigma_\star(\varepsilon)^{-q}\,
    \bigl(\log(e/\varepsilon)\bigr)^{-q/2},
    \label{eq:def:eps-A-q-weighted}
\end{equation}
with $c_{A,q}>0$ sufficiently small.
Then there is a ReLU network
\[
    \widetilde{\mathcal A}_{i,q,\blambda,\varepsilon_{A,q}}^{(R)}
    :
    \R^D\times[\tdown,\tup]\to\R
\]
with
\[
    \widetilde{\mathcal A}_{i,q,\blambda,\varepsilon_{A,q}}^{(R)}
    \in
    \nn(L_{\mathcal A},\bW_{\mathcal A},S_{\mathcal A},B_{\mathcal A})
\]
such that
\begin{equation}
    \sup_{\substack{
        t\in\gI_{\rm sm}(\varepsilon)\\
        (\bx,t)\in\gK_i(\varepsilon)}}
    \left|
      \widetilde{\mathcal A}_{i,q,\blambda,\varepsilon_{A,q}}^{(R)}(\bx,t)
      -
      A_{i,q,\blambda}^{(R)}
      \bigl(\bu_i^\Pi(\bx,t),t\bigr)
    \right|
    \le
    \varepsilon_{A,q}.
    \label{eq:A-coeff-composed-approx}
\end{equation}
Consequently,
\begin{equation}
    \sup_{\substack{
        t\in\gI_{\rm sm}(\varepsilon)\\
        (\bx,t)\in\gK_i(\varepsilon)}}
    \sigma_t^q
    |\bnu_i(\bx,t)^{\blambda}|
    \left|
      \widetilde{\mathcal A}_{i,q,\blambda,\varepsilon_{A,q}}^{(R)}(\bx,t)
      -
      A_{i,q,\blambda}^{(R)}
      \bigl(\bu_i^\Pi(\bx,t),t\bigr)
    \right|
    \le
    \varepsilon_{\rm term}.
    \label{eq:A-coeff-weighted-error}
\end{equation}
Let $\alpha_q:=(\beta-q)\wedge1$.
There is a finite constant $C_{\mathcal A,q}\ge1$, depending only on the fixed chart, geometric, schedule, and coefficient constants, such that
\begin{align}
    L_{\mathcal A}
    &\lesssim
    \log(e(d\vee\beta))
    \log^3\!\bigl(C_{\mathcal A,q}D\varepsilon_{A,q}^{-2/\alpha_q}\bigr)
    +
    \log(e(d\vee(\beta-q)))
    \log^2\!\bigl(C_{\mathcal A,q}\varepsilon_{A,q}^{-1}\bigr),
    \notag\\
    \|\bW_{\mathcal A}\|_\infty
    &\lesssim
    A_{\gM}
    dD
    (d\vee\beta)
    \bigl(C_{\mathcal A,q}D\varepsilon_{A,q}^{-2/\alpha_q}\bigr)^{d/(2(\beta\vee1))}
    +
    A_{\gM}
    D\log^3\!\bigl(C_{\mathcal A,q}D\varepsilon_{A,q}^{-2/\alpha_q}\bigr)
    \notag\\
    &\qquad
    +
    (d\vee(\beta-q))
    \bigl(C_{\mathcal A,q}\varepsilon_{A,q}^{-1}\bigr)^{d/(\beta-q)}
    +
    \log^3\!\bigl(C_{\mathcal A,q}\varepsilon_{A,q}^{-1}\bigr),
    \notag\\
    S_{\mathcal A}
    &\lesssim
    A_{\gM}
    dD(d+2(\beta\vee1))^{d+3}
    \bigl(C_{\mathcal A,q}D\varepsilon_{A,q}^{-2/\alpha_q}\bigr)^{d/(2(\beta\vee1))}
    \log^2\!\bigl(C_{\mathcal A,q}D\varepsilon_{A,q}^{-2/\alpha_q}\bigr)
    \notag\\
    &\qquad
    +
    A_{\gM}
    dD\log^5\!\bigl(C_{\mathcal A,q}D\varepsilon_{A,q}^{-2/\alpha_q}\bigr)
    \notag\\
    &\qquad
    +
    (d+\beta-q+2)^{d+4}
    \bigl(C_{\mathcal A,q}\varepsilon_{A,q}^{-1}\bigr)^{d/(\beta-q)}
    \log\!\bigl(C_{\mathcal A,q}\varepsilon_{A,q}^{-1}\bigr)
    +
    \log^4\!\bigl(C_{\mathcal A,q}\varepsilon_{A,q}^{-1}\bigr),
    \notag\\
    \log B_{\mathcal A}
    &\lesssim
    \log^2\!\bigl(C_{\mathcal A,q}D\varepsilon_{A,q}^{-2/\alpha_q}\bigr)
    +
    \log^2\!\bigl(C_{\mathcal A,q}\varepsilon_{A,q}^{-1}\bigr).
    \label{eq:A-coeff-composed-full-size}
\end{align}
In particular, the total non-logarithmic size is bounded, up to logarithmic factors and fixed constants, by
\begin{equation}
    \varepsilon_{A,q}^{-d/(\beta-q)}
    +
    D^{d/(2(\beta\vee1))}
    \varepsilon_{A,q}^{-d/(\alpha_q(\beta\vee1))}.
    \label{eq:A-coeff-composed-size}
\end{equation}
If $\sigma_\star(\varepsilon)\asymp\varepsilon^{1/\beta}$ and $\varepsilon_{\rm term}\asymp\varepsilon$, then the two $\varepsilon$-dependent powers in \cref{eq:A-coeff-composed-size} are bounded by
\[
    \lesssim
    \varepsilon^{-d/\beta}
    \bigl(\log(e/\varepsilon)\bigr)^{C_{\beta,q}},
\]
for a finite exponent $C_{\beta,q}$.
Hence the composed coefficient approximation does not worsen the intrinsic $\varepsilon^{-d/\beta}$ rate, up to logarithmic factors and the displayed polynomial ambient factor.
\end{corollary}

\begin{proof}
Use the fixed buffered box $K_i^\circ$ and the intermediate box $\Box_i^A$ from \cref{lem:A-coeff-extension}.
Let
\[
    \alpha_q:=(\beta-q)\wedge1.
\]
Apply \cref{cor:active-chart-coordinate-from-proj-NN} with localization parameter $\varepsilon$ and coordinate accuracy
\[
    \varepsilon_u
    :=
    c\varepsilon_{A,q}^{1/\alpha_q},
\]
where the constant $c>0$ is chosen small enough that the coordinate output remains in $\Box_i^A$ whenever the true coordinate $\bu_i^\Pi(\bx,t)$ lies in $K_i^\circ$.
This uses the fixed positive margin $\dist(K_i^\circ,\partial\Box_i^A)$.
Denote the resulting coordinate network by $\widetilde{\bu}_{i,\varepsilon_u}^{\Pi,\varepsilon}$.
Apply \cref{lem:approx:A-coeff} with accuracy $\varepsilon_{A,q}/2$, and compose the resulting coefficient network with this coordinate network.
Since
\[
    \overline A_{i,q,\blambda}^{(R)}(\cdot,t)
    \in \gH^{\beta-q}(\Box_i^A,C_{A,q})
\]
uniformly in $t$, it is $\alpha_q$-H\"older uniformly in $t$.
Thus the coordinate error contributes at most $C_{A,q}\varepsilon_u^{\alpha_q}$, which is at most $\varepsilon_{A,q}/2$ by the choice of $c$.
The intrinsic coefficient-network error contributes the other $\varepsilon_{A,q}/2$.
Since $\bu_i^\Pi(\bx,t)\in K_i^\circ$, the extension identity \cref{eq:A-coeff-m-factorization} gives
\[
    \overline A_{i,q,\blambda}^{(R)}
    \bigl(\bu_i^\Pi(\bx,t),t\bigr)
    =
    A_{i,q,\blambda}^{(R)}
    \bigl(\bu_i^\Pi(\bx,t),t\bigr),
\]
which proves \cref{eq:A-coeff-composed-approx}.

On $\gK_i(\varepsilon)$, \cref{eq:active-residual-bound-for-approx} gives
\[
    \|\bnu_i(\bx,t)\|_2
    \le
    C\sqrt{\log(e/\varepsilon)}.
\]
Since $\|\blambda\|_1\le q$, this implies
\[
    |\bnu_i(\bx,t)^{\blambda}|
    \le
    C_q\bigl(\log(e/\varepsilon)\bigr)^{q/2}.
\]
Together with $\sigma_t\le\sigma_\star(\varepsilon)$ and $\varepsilon_{A,q}\le c_{A,q}\varepsilon_{\rm term}\sigma_\star(\varepsilon)^{-q} \bigl(\log(e/\varepsilon)\bigr)^{-q/2}$, this yields
\[
    \sigma_t^q|\bnu_i^{\blambda}|
    \varepsilon_{A,q}
    \le
    C_q\sigma_\star(\varepsilon)^q
    \bigl(\log(e/\varepsilon)\bigr)^{q/2}
    c_{A,q}\varepsilon_{\rm term}
    \sigma_\star(\varepsilon)^{-q}
    \bigl(\log(e/\varepsilon)\bigr)^{-q/2}
    \le
    \varepsilon_{\rm term}
\]
after choosing $c_{A,q}$ small enough.
This proves \cref{eq:A-coeff-weighted-error}.

The coefficient network contributes the terms in \cref{eq:A-coeff-size} with accuracy $\varepsilon_{A,q}/2$.
The coordinate network is obtained from \cref{cor:active-chart-coordinate-from-proj-NN} with accuracy $\varepsilon_u=c\varepsilon_{A,q}^{1/\alpha_q}$; substituting this value into \cref{eq:active-chart-coordinate-from-proj-size} gives the terms with scale $D\varepsilon_{A,q}^{-2/\alpha_q}$.
The final composition and the finite summations change the displayed quantities only by fixed constants.
After increasing $C_{\mathcal A,q}$, this proves \cref{eq:A-coeff-composed-full-size}.
In particular, the coordinate network's non-logarithmic contribution is
\[
    (D\varepsilon_u^{-2})^{d/(2(\beta\vee1))}
    \lesssim
    D^{d/(2(\beta\vee1))}
    \varepsilon_{A,q}^{-d/(\alpha_q(\beta\vee1))}.
\]
Together with the coefficient contribution $\varepsilon_{A,q}^{-d/(\beta-q)}$, this proves \cref{eq:A-coeff-composed-size}.

Finally, since $\sigma_\star(\varepsilon)\asymp\varepsilon^{1/\beta}$ and $\varepsilon_{\rm term}\asymp\varepsilon$, for sufficiently small $\varepsilon$, the minimum in \cref{eq:def:eps-A-q-weighted} is inactive and
\[
    \varepsilon_{A,q}
    \asymp
    \varepsilon^{1-q/\beta}\bigl(\log(e/\varepsilon)\bigr)^{-q/2}
    =
    \varepsilon^{(\beta-q)/\beta}\bigl(\log(e/\varepsilon)\bigr)^{-q/2}.
\]
Substituting this relation into the first term of \cref{eq:A-coeff-composed-size} gives
\[
    \varepsilon_{A,q}^{-d/(\beta-q)}
    \lesssim
    \varepsilon^{-d/\beta}
    \bigl(\log(e/\varepsilon)\bigr)^{qd/(2(\beta-q))}.
\]
For the second term, if $\beta-q\ge1$, then $\alpha_q=1$ and
\[
    \varepsilon_{A,q}^{-d/(\alpha_q(\beta\vee1))}
    \lesssim
    \varepsilon^{-d(\beta-q)/(\beta(\beta\vee1))}
    \bigl(\log(e/\varepsilon)\bigr)^{qd/(2(\beta\vee1))}
    \le
    \varepsilon^{-d/\beta}
    \bigl(\log(e/\varepsilon)\bigr)^{qd/(2(\beta\vee1))},
\]
because $\beta-q\le\beta\vee1$.
If $\beta-q<1$, then $\alpha_q=\beta-q$, and
\[
    \varepsilon_{A,q}^{-d/(\alpha_q(\beta\vee1))}
    \lesssim
    \varepsilon^{-d/(\beta(\beta\vee1))}
    \bigl(\log(e/\varepsilon)\bigr)^{qd/(2(\beta-q)(\beta\vee1))}
    \le
    \varepsilon^{-d/\beta}
    \bigl(\log(e/\varepsilon)\bigr)^{qd/(2(\beta-q)(\beta\vee1))},
\]
because $\beta\vee1\ge1$.
The displayed polylogarithmic rate follows.
	\end{proof}

\begin{corollary}[Active-branch component approximation on an enlarged chart radius]
\label{cor:active-components-fixed-radius}
Fix $c_a>0$.
Assume the $c_a$-active localization condition
\[
    \delta_{\rm act}^{(c_a)}(\varepsilon,t)<\eta_\Pi/2,
    \qquad t\in\gI_{\rm sm}(\varepsilon),
\]
holds.
Then, for every chart $i$, on $\gK_i^{(c_a)}(\varepsilon)$, the residual bound
\[
    \|\bnu_i(\bx,t)\|_2
    \le
    c_a\sqrt{\log(e/\varepsilon)}
\]
holds, and the following fixed-$c_a$ approximation statements are valid.
Their error bounds and network-size orders are the same as in the cited default statements, with constants allowed to depend on the fixed number $c_a$, and with the replacement
\[
    \gK_i(\varepsilon)\mapsto \gK_i^{(c_a)}(\varepsilon),
    \qquad
    \gK_{\rm act}(\varepsilon)\mapsto
    \gK_{\rm act}^{(c_a)}(\varepsilon),
    \qquad
    C_{\nu,\varepsilon}\mapsto
    2\sigma_{\tdown}^{-1}+c_a\sqrt{\log(e/\varepsilon)}.
\]
\begin{enumerate}[(i)]
    \item The buffered-chart and reach-tube conclusions of \cref{lem:buffered-active-chart,lem:active-chart-inside-reach-tube} hold on $\gK_i^{(c_a)}(\varepsilon)$.
    In particular $\bu_i^\Pi(\bx,t)\in K_i^\circ$ there.

    \item For every coordinate accuracy $\varepsilon_u$, the chart-coordinate network in \cref{cor:active-chart-coordinate-from-proj-NN} can be chosen so that its error bound holds uniformly on $\gK_i^{(c_a)}(\varepsilon)$.

    \item \Cref{thm:global-proj-NN} has the same statement on the enlarged active union $\gK_{\rm act}^{(c_a)}(\varepsilon)$.

    \item The residual, normal, Gaussian-factor, and weighted-coefficient approximation statements in \cref{lem:approx:r-i--nu-i,lem:approx:gaussian-factor,cor:approx:weighted-A-coeff} hold uniformly on $\gK_i^{(c_a)}(\varepsilon)$.
\end{enumerate}
For fixed $c_a$, the polynomial and logarithmic orders in all displayed network sizes are unchanged.
\end{corollary}

\begin{proof}
For the buffered-chart and reach-tube statements, the original proofs use the active-set
definition only through the distance estimate
\[
    \|\bx-m_t\by\|_2
    \le
    c_a\sigma_t\sqrt{\log(e/\varepsilon)},
    \qquad \by\in\gS_i,
\]
and the resulting reach/local-coordinate condition
$\delta_{\rm act}^{(c_a)}(\varepsilon,t)<\eta_\Pi/2$.
The assumed $c_a$-active localization condition therefore gives the same projected
coordinate, buffered-domain inclusion, and uniqueness conclusions on
$\gK_i^{(c_a)}(\varepsilon)$.
The residual estimate becomes
\[
    \|\br_i(\bx,t)\|_2
    \le
    c_a\sigma_t\sqrt{\log(e/\varepsilon)},
    \qquad
    \|\bnu_i(\bx,t)\|_2
    \le
    c_a\sqrt{\log(e/\varepsilon)}.
\]
Thus every place where $C_{\nu,\varepsilon}$ appears may use $2\sigma_{\tdown}^{-1}+c_a\sqrt{\log(e/\varepsilon)}$.
For the global projection theorem, the original argument uses only the following two facts: an active pair $(\bx,t)$ supplies a chart $i_\star$ for which the exact projected coordinate lies in $K_{i_\star}^\circ$, and all other candidates are separated by the same objective gap outside the accepted near-minimizer set.
Both facts hold with $\gK_{\rm act}^{(c_a)}(\varepsilon)$ in place of $\gK_{\rm act}(\varepsilon)$, because $(\bx,t)\in\gK_{i_\star}^{(c_a)}(\varepsilon)$ and the enlarged buffered-chart conclusion gives the good coordinate in $K_{i_\star}^\circ$.
Thus the minimizer-selection and error estimates of \cref{thm:global-proj-NN} hold on $\gK_{\rm act}^{(c_a)}(\varepsilon)$ with only fixed-$c_a$ changes in constants.
This proves the fixed-$c_a$ version of \cref{thm:global-proj-NN}.
Composing that fixed-$c_a$ projection-center network with the affine extension of the chart map $\phi_i$ and the exact box clipping layer gives the fixed-$c_a$ version of \cref{cor:active-chart-coordinate-from-proj-NN}.

The proofs of the residual and Gaussian-factor lemmas then use this fixed-$c_a$ global projection theorem in place of the default one, together with the enlarged residual bound displayed above.
The weighted coefficient corollary uses the fixed-$c_a$ chart-coordinate network and the same coefficient extension.
For fixed $c_a$, the replacement only changes constants and powers of the same logarithmic factor.
The constructions, error allocations, and network sizes are otherwise identical.
\end{proof}

\begin{corollary}[$H$-active component approximation]
\label{cor:H-active-components}
Fix $c_a>0$, $H\ge1$, and $0<\varepsilon\le1/2$.
Assume $[t_0,t_1]\subset\gI_{\rm sm}(\varepsilon)$ and
\[
    c_a\frac{\sigma_t}{m_t}\sqrt H\le\eta_\Pi/2,
    \qquad t\in[t_0,t_1].
\]
Set $\sigma_H^\star:=\sup_{t\in[t_0,t_1]}\sigma_t$, so $\sigma_H^\star\le C_{\rm sm}\varepsilon^{1/\beta}$.
For
\[
    A_t(H):=\{\dist(\bx,m_t\gM)\le C_\varrho\sigma_t\sqrt H\},
\]
define
\[
    \gK_i^{(c_a,H)}
    :=
    \left\{(\bx,t):\bx\in A_t(H),\ \exists \by\in\gS_i\text{ with }\|\bx-m_t\by\|_2\le c_a\sigma_t\sqrt H\right\}.
\]
Then, for every chart $i$, on $\gK_i^{(c_a,H)}$ the projection $\Pi_{\gM}(\bx/m_t)$ is unique, the projected coordinate satisfies $\bu_i^\Pi(\bx,t)\in K_i^\circ$, and
\[
    \|\bnu_i(\bx,t)\|_2\le c_a\sqrt H.
\]
Moreover, the component approximation statements in \cref{cor:active-chart-coordinate-from-proj-NN,thm:global-proj-NN,lem:approx:r-i--nu-i,lem:approx:gaussian-factor,cor:approx:weighted-A-coeff} have $H$-active versions on the sets $\gK_i^{(c_a,H)}$ and $\bigcup_i\gK_i^{(c_a,H)}$.
The corresponding size bounds are obtained from the cited bounds by replacing $C_{\nu,\varepsilon}$ with $2\sigma_{\tdown}^{-1}+c_a\sqrt H$ and by replacing each logarithmic scale $\log(C\varepsilon^{-1})$ with $H+\log(C\varepsilon^{-1})$.
More explicitly, for the weighted-coefficient statement, fix the same indices $i,R,q,\blambda$ and the same term accuracy $0<\varepsilon_{\rm term}\le1/2$ as in \cref{cor:approx:weighted-A-coeff}.
In the $H$-active version of that corollary, the coefficient tolerance is
\[
    \varepsilon_{A,q}^{H}
    :=
    \frac12
    \wedge
    c_{A,q}\varepsilon_{\rm term}
    (\sigma_H^\star)^{-q}
    H^{-q/2}.
\]
With this replacement, the weighted error bound in \cref{eq:A-coeff-weighted-error} holds on $\gK_i^{(c_a,H)}$ with $\varepsilon_{A,q}$ replaced by $\varepsilon_{A,q}^{H}$.
For fixed $c_a$, these replacements preserve the same non-logarithmic intrinsic costs and introduce only additional fixed polynomial powers of $H$.
\end{corollary}

\begin{proof}
The proof of \cref{cor:active-components-fixed-radius} depends on the active set only through the reach condition and the residual bound.
The assumed $H$-active reach condition gives uniqueness of the metric projection, buffered-coordinate inclusion, and the Gauss--Newton contraction exactly as before, with $\sqrt{\log(e/\varepsilon)}$ replaced by $\sqrt H$.
The residual bound follows from the active-distance inequality in the same way as in \cref{eq:active-residual-bound-for-approx}.
The chart-coordinate network, global projection network, residual network, and Gaussian-factor network use this radius only through the scalar ranges for products, reciprocals, ramps, and squared norms.
Thus their proofs carry over with $C_{\nu,\varepsilon}$ replaced by $2\sigma_{\tdown}^{-1}+c_a\sqrt H$ and with logarithmic factors enlarged to $H+\log(C\varepsilon^{-1})$.
For the coefficient composition, the only changed analytic estimate is $|\bnu_i^{\blambda}|\le C_qH^{q/2}$ when $\|\blambda\|_1\le q$.
The definition of $\varepsilon_{A,q}^{H}$ and the bound $\sigma_t\le\sigma_H^\star$ give the same weighted error proof as \cref{cor:approx:weighted-A-coeff}.
The intrinsic H\"older approximation exponents are unchanged, and the new powers of $H$ are absorbed into the stated logarithmic-scale replacement.
\end{proof}

\begin{lemma}[Certified $H$-active chart gate from objective weights]
\label{lem:H-certified-chart-objective-gate}
Fix a chart index $i$, a radius multiplier $c_a\ge c_\star$, $H\ge1$, $0<\varepsilon\le1/2$, and a coordinate tolerance $0<\eta\le1$.
Let $[t_0,t_1]\subset\gI_{\rm sm}(\varepsilon)$ and assume
\[
    c_a\frac{\sigma_t}{m_t}\sqrt H\le\eta_\Pi/2,
    \qquad t\in[t_0,t_1].
\]
Set $A_t(H):=\{\dist(\bx,m_t\gM)\le C_\varrho\sigma_t\sqrt H\}$.
Define $\gK_i^{(c_a,H)}$ from this $A_t(H)$ as in \cref{cor:H-active-components}.
Let $\Box_i^A$ be any fixed coordinate box with $K_i^\circ\Subset\Box_i^A\Subset\Box_i$.
Then there exist ReLU networks
\[
    \Gamma_i^H:\R^D\times[\tdown,\tup]\to[0,1],
    \qquad
    \breve\bu_i:\R^D\times[\tdown,\tup]\to\Box_i,
    \qquad
    \breve\bu_i^A:\R^D\times[\tdown,\tup]\to\Box_i^A,
\]
with the following properties, uniformly for $t\in[t_0,t_1]$ and $\bx\in A_t(H)$.
If $(\bx,t)\in\gK_i^{(c_a,H)}$, then
\begin{equation}
    \Gamma_i^H(\bx,t)=1,
    \qquad
    \|\breve\bu_i(\bx,t)-\bu_i^\Pi(\bx,t)\|_\infty\le\eta,
    \qquad
    \|\breve\bu_i^A(\bx,t)-\bu_i^\Pi(\bx,t)\|_\infty\le\eta.
    \label{eq:H-certified-gate-active}
\end{equation}
If $\Gamma_i^H(\bx,t)>0$, then $\Pi_{\gM}(\bx/m_t)\in U_i$ and, with $\bu_i^\Pi(\bx,t):=\phi_i(\Pi_{\gM}(\bx/m_t))$,
\begin{equation}
    \|\breve\bu_i(\bx,t)-\bu_i^\Pi(\bx,t)\|_\infty\le\eta .
    \label{eq:H-certified-gate-positive}
\end{equation}
The network $\breve\bu_i^A$ is the exact coordinatewise clipping of $\breve\bu_i$ to $\Box_i^A$.
The sizes are bounded by the $H$-active global projection-objective construction with projection accuracy $c\eta$ and with every logarithmic scale enlarged to $H+\log(C\varepsilon^{-1})$; in particular, this gate adds no non-logarithmic factor beyond the intrinsic projection-objective cost $\eta^{-d/(\beta\vee1)}$ and fixed-atlas arithmetic.
\end{lemma}

\begin{proof}
We repeat the objective-weight part of the proof of \cref{thm:global-proj-NN}, using the $H$-active extension from \cref{cor:H-active-components}, but we keep the chart identity of each candidate.
Let
\[
    \gJ:=\{(j,\ba):1\le j\le C_{\gM},\ \ba\in\gA_j^{\rm anc}\},
    \qquad
    N_{\rm cand}:=|\gJ|.
\]
Choose an ambient candidate accuracy $\varepsilon_{\rm cert}=c_{\rm cert}\eta$, where the fixed geometric constant $c_{\rm cert}>0$ will be decreased several times below.
For each $(j,\ba)\in\gJ$, form the anchored coordinate candidate $\widetilde\bu_{j,\ba}$, the manifold candidate $\by_{j,\ba}:=\bz_j(\widetilde\bu_{j,\ba})$, and an objective network $\widetilde\gD_{j,\ba}$ approximating $\|\bx/m_t-\by_{j,\ba}\|_2^2$.
The coordinate candidates are box-clipped by construction, so all $\by_{j,\ba}$ and all objective networks are well-defined on the whole set $A_t(H)$, even when chart $j$ is not the active chart.
The objective and arithmetic tolerances are chosen so that
\begin{equation}
    |\widetilde\gD_{j,\ba}-\|\bx/m_t-\by_{j,\ba}\|_2^2|
    \le
    \gamma/64,
    \qquad
    \gamma:=\varepsilon_{\rm cert}^2 ,
    \label{eq:H-certified-objective-accuracy}
\end{equation}
uniformly on $A_t(H)$.
Because $c_a\ge c_\star\ge C_\varrho$ and the chart supports cover $\gM$, every $(\bx,t)\in A_t(H)$ belongs to at least one $H$-active chart set $\gK_j^{(c_a,H)}$.
For such a chart, \cref{cor:H-active-components} gives $\bu_j^\Pi\in K_j^\circ$, and \cref{lem:proj-anchor-init} supplies an anchor whose local Gauss--Newton candidate is within $\Ord(\varepsilon_{\rm cert})$ of the nearest projection.
After decreasing $c_{\rm cert}$, the corresponding exact objective is at most $\gamma/64$ above the true minimum $\dist(\bx/m_t,\gM)^2$.

Let $\widetilde\gD_{\min}:=\min_{(j,\ba)\in\gJ}\widetilde\gD_{j,\ba}$ and define
\[
    \omega_{j,\ba}
    :=
    \ReLU\!\left(\gamma-\bigl(\widetilde\gD_{j,\ba}-\widetilde\gD_{\min}\bigr)\right),
    \qquad
    W_{\rm glob}:=\sum_{(j,\ba)\in\gJ}\omega_{j,\ba},
    \qquad
    \Omega_i:=\sum_{\ba\in\gA_i^{\rm anc}}\omega_{i,\ba}.
\]
At least one candidate attains $\widetilde\gD_{\min}$, hence
\begin{equation}
    \gamma\le W_{\rm glob}\le N_{\rm cand}\gamma .
    \label{eq:H-certified-global-weight-range}
\end{equation}
If $(\bx,t)\in\gK_i^{(c_a,H)}$, the good chart-$i$ candidate described above satisfies
\[
    \widetilde\gD_{i,\ba_\star}-\widetilde\gD_{\min}
    \le
    \frac{3\gamma}{64},
\]
because $\widetilde\gD_{\min}$ is not below the true minimum by more than $\gamma/64$.
Therefore $\omega_{i,\ba_\star}\ge 61\gamma/64$, and \cref{eq:H-certified-global-weight-range} gives
\begin{equation}
    \frac{\Omega_i}{W_{\rm glob}}
    \ge
    \tau_0,
    \qquad
    \tau_0:=\frac{1}{2N_{\rm cand}} .
    \label{eq:H-certified-active-ratio}
\end{equation}

Conversely, suppose $\Omega_i>0$.
Then some chart-$i$ candidate has positive weight.
The same calculation as \cref{eq:positive-weight-objective-bound,eq:positive-candidate-close}, with the tolerances in \cref{eq:H-certified-objective-accuracy}, implies
\begin{equation}
    \|\by_{i,\ba}(\bx,t)-\Pi_{\gM}(\bx/m_t)\|_2
    \le
    C\varepsilon_{\rm cert}
    \label{eq:H-certified-positive-candidate-close}
\end{equation}
for every positive-weight chart-$i$ candidate.
Let
\[
    \delta_i^U:=
    \dist_{\R^D}\bigl(\bz_i(\Box_i),\gM\setminus U_i\bigr)>0,
\]
with the convention $\dist(A,\emptyset)=+\infty$.
After decreasing $c_{\rm cert}$ so that $C\varepsilon_{\rm cert}<\delta_i^U/2$, \cref{eq:H-certified-positive-candidate-close} forces $\Pi_{\gM}(\bx/m_t)\in U_i$ whenever $\Omega_i>0$.
The inverse chart is Lipschitz on the compact set $\phi_i^{-1}(\Box_i)$, so every positive-weight chart-$i$ coordinate candidate satisfies
\begin{equation}
    \|\widetilde\bu_{i,\ba}(\bx,t)-\bu_i^\Pi(\bx,t)\|_\infty
    \le
    C_i\varepsilon_{\rm cert}.
    \label{eq:H-certified-positive-coordinate-close}
\end{equation}

Define the ideal chart coordinate average
\[
    \bu_{i,{\rm id}}(\bx,t)
    :=
    \frac{\sum_{\ba\in\gA_i^{\rm anc}}\omega_{i,\ba}(\bx,t)\widetilde\bu_{i,\ba}(\bx,t)}
    {\Omega_i(\bx,t)}
\]
whenever $\Omega_i>0$.
If $\Omega_i/W_{\rm glob}\ge\tau_0/32$, then \cref{eq:H-certified-global-weight-range} gives $\Omega_i\ge\tau_0\gamma/32$, and \cref{eq:H-certified-positive-coordinate-close} gives $\|\bu_{i,{\rm id}}-\bu_i^\Pi\|_\infty\le C_i\varepsilon_{\rm cert}$.
Implement the denominator as $\Omega_i\vee(\tau_0\gamma/64)$, approximate the reciprocal and products to error at most $c\eta$, and apply the exact clipping map $\Pi_i^{(\rm box)}$ as the final layer.
After decreasing $c_{\rm cert}$, the implemented coordinate $\breve\bu_i$ maps into $\Box_i$ and satisfies \cref{eq:H-certified-gate-positive} whenever the gate below is nonzero, because the clipping is $1$-Lipschitz and fixes $\bu_i^\Pi$.

It remains only to construct the binary-safe chart gate.
Approximate $\Omega_i/W_{\rm glob}$ with error at most $\tau_0/32$ on the interval in \cref{eq:H-certified-global-weight-range}.
Let $\theta_{\rm cert}:\R\to[0,1]$ be the ReLU ramp that is zero on $(-\infty,\tau_0/16]$ and one on $[\tau_0/2,\infty)$, and define $\Gamma_i^H$ by applying $\theta_{\rm cert}$ to this implemented ratio.
If $(\bx,t)\in\gK_i^{(c_a,H)}$, then \cref{eq:H-certified-active-ratio} and the ratio error imply that the ramp input is at least $\tau_0/2$, so $\Gamma_i^H=1$.
If $\Gamma_i^H>0$, then the exact ratio is positive and in fact $\Omega_i/W_{\rm glob}\ge\tau_0/32$, so the clipped denominator in the coordinate average is inactive and the coordinate certification above applies.
Finally, define $\breve\bu_i^A$ by exact coordinatewise clipping of $\breve\bu_i$ to $\Box_i^A$.
On $\gK_i^{(c_a,H)}$, the true coordinate lies in $K_i^\circ\Subset\Box_i^A$, and the coordinatewise clipping is $1$-Lipschitz and fixes this true coordinate.
All added operations are finite sums, minima, ReLU ramps, clipped reciprocals, products, and coordinatewise clipping, so their costs are logarithmic or fixed-atlas costs on top of the $H$-active projection-objective networks.
This proves the lemma.
\end{proof}

% ----------------------------------------------------------------------

\subsubsection{Projection--Laplace chart approximation for small noise}
\label{sec:app:approx:PQ:small-noise-branch}

We first state the chart-level projection--Laplace approximation used by the density-lower-bound small-noise construction.
The proposition below is deliberately local: it certifies the neural Laplace expansion only on the enlarged chart-active set where the chart coordinate and residual subnetworks have proved error bounds.
Outside the default chart-active set, the true chart integral is already small by the inactive-tail lemma; the final de-Gaussianized construction below adds a separate certified chart gate before using inactive-chart neural outputs.
This separation keeps the proof auditable because no local chart-coordinate network is used outside the region where its approximation guarantee is available.

\begin{proposition}[Projection--Laplace chart approximation for small noise]
\label{prop:small-noise-chart-proj-laplace}
Assume \cref{assump:manifold:exact,assump:manifold:density}.
Let $0<\varepsilon_{\rm sm}\le1/2$, and assume the localization condition \cref{eq:condition:delta}, as well as the same condition with $3c_\star$ in place of $c_\star$, holds for every $t\in\gI_{\rm sm}(\varepsilon_{\rm sm})$.
For every chart $i$ and every $R\in\mathcal R$, there exists a ReLU network $\widetilde R_i^{\rm sm}\in\nn(L,\bW,S,B)$ such that
\begin{equation}
    \sup_{\substack{
        t\in\gI_{\rm sm}(\varepsilon_{\rm sm})\\
        (\bx,t)\in\gK_i^{(3c_\star)}(\varepsilon_{\rm sm})
    }}
    \left|
      R_i(\bx,t)-\widetilde R_i^{\rm sm}(\bx,t)
    \right|
    \le
    \varepsilon_{\rm sm}.
    \label{eq:small-noise-chart-error}
\end{equation}
In addition, the true inactive tail satisfies
\begin{equation}
    \sup_{\substack{
        t\in\gI_{\rm sm}(\varepsilon_{\rm sm})\\
        \bx\in A_t,\;(\bx,t)\notin\gK_i(\varepsilon_{\rm sm})
    }}
    |R_i(\bx,t)|
    \le
    \varepsilon_{\rm sm}/6 .
    \label{eq:small-noise-chart-true-tail}
\end{equation}
Moreover, let $C_{\rm size,sm}\ge e$ be a fixed constant large enough to absorb the chart, schedule, coefficient, and arithmetic constants, and define
\begin{equation}
    \Lambda_{\rm sm}
    :=
    e\vee C_{\rm size,sm}\varepsilon_{\rm sm}^{-1}\vee \tdown^{-1}.
    \label{eq:def:Lambda-small-noise-chart-size}
\end{equation}
For finite exponents $\mathfrak a_{\rm sm}$ and $\mathfrak p_{\beta,d}$ depending only on $(\beta,d)$, the same construction can be chosen so that, up to fixed geometry-dependent constants,
\begin{equation}
    L+\log B
    \lesssim
    \log(e(d\vee\beta))\log^3\Lambda_{\rm sm},
    \qquad
    \|\bW\|_\infty+S
    \lesssim
    D^{\mathfrak a_{\rm sm}}A_{\gM}\varepsilon_{\rm sm}^{-d/\beta}\log^{\mathfrak p_{\beta,d}}\Lambda_{\rm sm}.
    \label{eq:small-noise-chart-size-summary}
\end{equation}
The expanded depth, width, sparsity, and log-weight bookkeeping is recorded in \cref{lem:small-noise-chart-size-bookkeeping}.
\end{proposition}

\begin{lemma}[Small-noise chart size bookkeeping]
\label{lem:small-noise-chart-size-bookkeeping}
Work in the setting of \cref{prop:small-noise-chart-proj-laplace}.
For $0\le q\le \floor{\beta}$, write $\alpha_q:=(\beta-q)\wedge1$.
There is a finite exponent $\mathfrak a_{\rm sm}$, depending only on $(\beta,d)$, such that the network in \cref{prop:small-noise-chart-proj-laplace} may be chosen so that the following bounds hold.
More explicitly, one may take
\[
    \mathfrak a_{\rm sm}
    =
    \left\lceil
      \frac{\floor{\beta}\,d}{\beta-\floor{\beta}}
      +
      \frac{d}{\beta\vee1}
      +3
    \right\rceil .
\]
In the display below, every occurrence of $\log(C_{\rm size,sm}\varepsilon_{\rm sm}^{-1})$ may equivalently be replaced by $\log\Lambda_{\rm sm}$.
This only enlarges the stated bounds and makes the lower-time dependence of the schedule networks explicit.
\begin{equation}
    \begin{aligned}
    L
    &\lesssim
    \log(e(d\vee\beta))
    \log^3\!\bigl(C_{\rm size,sm}\varepsilon_{\rm sm}^{-1}\bigr),
    \\
    \|\bW\|_\infty
    &\lesssim
    D^{\mathfrak a_{\rm sm}}
    \frac{(D+d+\beta)^{\floor{\beta}+1}}{\floor{\beta}!}
    \Biggl[
      \bigl(C_{\rm size,sm}\varepsilon_{\rm sm}^{-2}\bigr)^{d/(2(\beta\vee1))}
      \\&\qquad+
      \sum_{q=0}^{\floor{\beta}}
      \Bigl\{
        \Bigl(
          C_{\rm size,sm}
          \varepsilon_{\rm sm}^{-2(1-q/\beta)/\alpha_q}
          \log^{q/\alpha_q}
          \bigl(C_{\rm size,sm}\varepsilon_{\rm sm}^{-1}\bigr)
        \Bigr)^{d/(2(\beta\vee1))}
        \\&\qquad+
        \Bigl(
          C_{\rm size,sm}
          \varepsilon_{\rm sm}^{-(1-q/\beta)}
          \log^{q/2}
          \bigl(C_{\rm size,sm}\varepsilon_{\rm sm}^{-1}\bigr)
        \Bigr)^{d/(\beta-q)}
      \Bigr\}
    \Biggr]
    \log^{3+d/2}\!\bigl(C_{\rm size,sm}\varepsilon_{\rm sm}^{-1}\bigr),
    \\
    S
    &\lesssim
    D^{\mathfrak a_{\rm sm}}
    A_{\gM}
    \frac{
      dD(D+d+\beta)^{\floor{\beta}}
      (d+2(\beta\vee1))^{d+4}
    }{\floor{\beta}!}
    \Biggl[
      \bigl(C_{\rm size,sm}\varepsilon_{\rm sm}^{-2}\bigr)^{d/(2(\beta\vee1))}
      \\&\qquad\quad+
      \sum_{q=0}^{\floor{\beta}}
      \Bigl\{
        \Bigl(
          C_{\rm size,sm}
          \varepsilon_{\rm sm}^{-2(1-q/\beta)/\alpha_q}
          \log^{q/\alpha_q}
          \bigl(C_{\rm size,sm}\varepsilon_{\rm sm}^{-1}\bigr)
        \Bigr)^{d/(2(\beta\vee1))}
        \\&\qquad\quad+
        \Bigl(
          C_{\rm size,sm}
          \varepsilon_{\rm sm}^{-(1-q/\beta)}
          \log^{q/2}
          \bigl(C_{\rm size,sm}\varepsilon_{\rm sm}^{-1}\bigr)
        \Bigr)^{d/(\beta-q)}
      \Bigr\}
    \Biggr]
    \log^{5+d/2}\!\bigl(C_{\rm size,sm}\varepsilon_{\rm sm}^{-1}\bigr),
    \\
    \log B
    &\lesssim
    \log^2\!\bigl(C_{\rm size,sm}\varepsilon_{\rm sm}^{-1}\bigr).
    \end{aligned}
    \label{eq:small-noise-chart-size}
\end{equation}
The factors $\log^{3+d/2}(C_{\rm size,sm}\varepsilon_{\rm sm}^{-1})$ and $\log^{5+d/2}(C_{\rm size,sm}\varepsilon_{\rm sm}^{-1})$ include the logarithmic overhead of the projection-coordinate gate, the residual ramp, and the final products.
No mesh of $\gS_i$ at scale $\sigma_{\tdown}$ is used; therefore the small-noise branch has no polynomial $\sigma_{\tdown}^{-1}$ factor in its width or sparsity bound.
The fixed localization constant $c_\star$ is chosen from the annular Gaussian-tail bound in \cref{lem:approx:P-Q:tail} and is independent of $\sigma_{\tdown}$.
For fixed $d,\beta$ and fixed geometry, the bracketed non-logarithmic factor in \cref{eq:small-noise-chart-size} is
\[
    \lesssim
    \varepsilon_{\rm sm}^{-d/\beta}
    \log^{\mathfrak p_{\beta,d}}\!\bigl(C_{\rm size,sm}\varepsilon_{\rm sm}^{-1}\bigr)
\]
after the finite order $q_{\rm geom}$ in \cref{eq:def:qgeom-main} is fixed, and one may take
\[
    \mathfrak p_{\beta,d}
    =
    C
    +
    \max\left\{
      8(q_{\rm geom}+2),
      \left\lceil
        \floor{\beta}+1+\frac{\floor{\beta}\,d}{\beta-\floor{\beta}}
        +\frac{d}{\beta\vee1}+\frac{d}{2}+6
      \right\rceil
    \right\}.
\]
Here the $C$ in the exponent is universal and independent of
$D,n,\sigma_t,p_{\min}$, and the numerical geometry constants; the resulting polynomial
ambient exponent may depend on the fixed $q_{\rm geom}$.  In the main $d>2$
regime we take the canonical choice $q_{\rm geom}=q_{\rm opt}(\beta)$ from \cref{eq:def:qgeom-opt};
then $q_{\rm geom}=\Ord_\beta(1)$ and the displayed formula gives
$\mathfrak p_{\beta,d}=\Ord_\beta(d)$.
Thus the small-noise branch has intrinsic approximation rate $\varepsilon_{\rm sm}^{-d/\beta}$, up to logarithmic factors, while the ambient-dimension dependence is polynomial in $D$.
The factor $D^{\mathfrak a_{\rm sm}}$ records the additional powers of the retained Laplace-term count that enter through the per-term coefficient accuracies, together with the squared-normal residual-gate cost; all lower-noise dependence is confined to logarithmic schedule and reciprocal accuracies.
\end{lemma}

\begin{proof}[Proof of \cref{prop:small-noise-chart-proj-laplace} and \cref{lem:small-noise-chart-size-bookkeeping}]
Fix $i$ and $R\in\mathcal R$.
The number of retained $(q,\blambda)$-terms is
\[
    \sum_{q=0}^{\floor{\beta}}
    \#\{\blambda\in\N^D:\|\blambda\|_1\le q\}
    =
    \sum_{q=0}^{\floor{\beta}}\binom{D+q}{q}
    =
    \binom{D+\floor{\beta}+1}{\floor{\beta}}.
\]
Also, since $d<D$ and $\floor{\beta}<\beta$,
\[
    \binom{D+\floor{\beta}+1}{\floor{\beta}}
    (D+d+\beta)
    \le
    \frac{(D+d+\beta)^{\floor{\beta}+1}}{\floor{\beta}!},
\]
which is the simpler ambient-dimension factor used in the width bound.
Similarly,
\[
    \binom{D+\floor{\beta}+1}{\floor{\beta}}\,
    dD(d+2(\beta\vee1))^{d+4}
    \le
    \frac{
      dD(D+d+\beta)^{\floor{\beta}}
      (d+2(\beta\vee1))^{d+4}
    }{\floor{\beta}!},
\]
which is the corresponding factor used in the sparsity bound.
The same retained-term count also appears in the allocated accuracies $\varepsilon_{\rm term}^{\rm sm}$ and $\varepsilon_{A,q}^{\rm sm}$ below.
Since $q\in\{0,\dots,\floor{\beta}\}$ is a finite set and the coefficient network sizes contain only fixed powers of $(\varepsilon_{A,q}^{\rm sm})^{-1}$, all additional powers of $\binom{D+\floor{\beta}+1}{\floor{\beta}}$ are bounded by $D^{\mathfrak a_{\rm sm}}$ after increasing $\mathfrak a_{\rm sm}(\beta,d)$.
Indeed,
\[
    \binom{D+\floor{\beta}+1}{\floor{\beta}}
    \le C_\beta D^{\floor{\beta}},
    \qquad
    \max_{0\le q\le \floor{\beta}}
    \left\{\frac{d}{\beta-q},
    \frac{d}{((\beta-q)\wedge1)(\beta\vee1)}\right\}
    \le \frac{d}{\beta-\floor{\beta}}.
\]
Thus the retained-term allocation contributes at most
$D^{\floor{\beta}d/(\beta-\floor{\beta})}$,
up to fixed $(\beta,d)$-dependent constants, before the other displayed ambient
factors are included.
The residual and coefficient-coordinate constructions contribute at most
$D^{d/(2(\beta\vee1))}$ beyond the explicit leading $D$-factor.  The largest
projection-normal auxiliary cost comes from operations depending on
$\|\bnu_i\|_2^2$.  For the original chart expansion this is the Gaussian factor in
\cref{lem:approx:gaussian-factor}; for the de-Gaussianized construction below it is the
residual-norm gate.  In both cases the normal approximation must be controlled in
$\ell_2$, so a coordinatewise normal accuracy of order $D^{-1/2}$ is required up
to logarithmic factors.  Equivalently the internal scale in
\cref{eq:residual-nu-size} is at most $D^2$, and therefore this part contributes
at most $D^{d/(\beta\vee1)}$.
The remaining residual ramps, products, and exact affine sums contribute only fixed
ambient powers, absorbed by the $+3$ in $\mathfrak a_{\rm sm}$.
The constant summands in the bracketed size factor are dominated by the $q=0$ terms of the displayed sum because $C_{\rm size,sm}\ge e$, $0<\varepsilon_{\rm sm}\le1/2$, and $d\ge1$.
Set
\[
    \varepsilon_{\rm term}^{\rm sm}
    :=
    c_{\rm term}\varepsilon_{\rm sm}
    \binom{D+\floor{\beta}+1}{\floor{\beta}}^{-1},
    \qquad
    \sigma_{\rm sm}:=C_{\rm sm}\varepsilon_{\rm sm}^{1/\beta}.
\]
For $0\le q\le \floor{\beta}$, define
\[
    \alpha_q:=(\beta-q)\wedge1,
    \qquad
    \varepsilon_{A,q}^{\rm sm}
    :=
    \frac12
    \wedge
    c_{A,q}^{\rm sm}
    \varepsilon_{\rm sm}^{1-q/\beta}
    \bigl(\log(C_{\rm size,sm}\varepsilon_{\rm sm}^{-1})\bigr)^{-q/2},
\]
where $c_{A,q}^{\rm sm}>0$ is chosen no larger than the fixed quantity
\[
    c_{A,q}c_{\rm term}C_{\rm sm}^{-q}
    \binom{D+\floor{\beta}+1}{\floor{\beta}}^{-1},
\]
and $c_{A,q}$ is the constant from \cref{cor:approx:weighted-A-coeff}.
By the fixed choice of $c_\star$ in the active-set definition, and after decreasing the fixed numerical budgets below if necessary, the inactive tail in \cref{lem:approx:P-Q:tail} is at most $\varepsilon_{\rm sm}/6$.

\textbf{Step 1: Analytic small-noise target.}
On the active region $\gK_i(\varepsilon_{\rm sm})$, use \cref{prop:laplace-local-chart}.
On the inactive region inside $A_t$, use zero.
Thus define
\[
    R_{i,{\rm app}}^{\rm sm}(\bx,t)
    :=
    \begin{cases}
        % \displaystyle
        \exp\!\left(-\frac{\|\br_i(\bx,t)\|_2^2}{2\sigma_t^2}\right)
        \sum_{q=0}^{\floor{\beta}}\sigma_t^q
        \sum_{\|\blambda\|_1\le q}
        A_{i,q,\blambda}^{(R)}
        \bigl(\bu_i^\Pi(\bx,t),t\bigr)
        \bnu_i(\bx,t)^{\blambda},
        &(\bx,t)\in\gK_i(\varepsilon_{\rm sm}),\\[4pt]
        0,&(\bx,t)\notin\gK_i(\varepsilon_{\rm sm}).
    \end{cases}
\]
By \cref{prop:laplace-local-chart,lem:approx:P-Q:tail}, uniformly over $t\in\gI_{\rm sm}(\varepsilon_{\rm sm})$ and $\bx\in A_t$,
\begin{equation}
    |R_i(\bx,t)-R_{i,{\rm app}}^{\rm sm}(\bx,t)|
    \le
    \varepsilon_{\rm sm}/3 .
    \label{eq:small-noise-analytic-error}
\end{equation}
The active part follows because the remainder in \cref{eq:laplace-remainder-local-chart} is bounded by $C_{\rm Lap}^{\max}\sigma_t^\beta$, and the threshold $C_{\rm sm}$ in \cref{def:noise-regimes} was chosen so that $C_{\rm Lap}^{\max}C_{\rm sm}^{\beta}\le 1/6$.
Hence $C_{\rm Lap}^{\max}\sigma_t^\beta\le\varepsilon_{\rm sm}/6$ on $\gI_{\rm sm}(\varepsilon_{\rm sm})$.
The inactive part is bounded by $\varepsilon_{\rm sm}/6$ by the fixed choice of $c_\star$.
This proves the true-tail statement \cref{eq:small-noise-chart-true-tail}.

\textbf{Step 2: Projection-coordinate gate.}
Let $\Box_i^A$ be the coefficient-extension box from \cref{lem:A-coeff-extension}, so that $K_i^\circ\Subset\Box_i^A\Subset\Box_i$.
Define the enlarged active set
\[
    \gK_i^+(\varepsilon_{\rm sm})
    :=
    \gK_i^{(3c_\star)}(\varepsilon_{\rm sm}).
\]
By the $3c_\star$-variant of the localization condition, the proof of \cref{lem:buffered-active-chart} applies on $\gK_i^+(\varepsilon_{\rm sm})$.
In particular, $\bu_i^\Pi(\bx,t)\in K_i^\circ$ there and the fixed-radius component approximations in \cref{cor:active-components-fixed-radius}, with $c_a=3c_\star$, are available.

Choose two piecewise-affine ReLU ramps.
The first is a coordinate bump $\theta_i^A:\R^d\to[0,1]$, equal to one on $K_i^\circ$ and equal to zero outside $\Box_i^A$.
Since the chart atlas and boxes are fixed, this bump is implemented with a fixed number of one-dimensional ramps and ReLU max/min operations.
The second is a residual bump $\theta_{\rm res}:\R_+\to[0,1]$, equal to one on
\[
    \left[0,\left(\frac{3c_\star}{2}\right)^2
        \log(e/\varepsilon_{\rm sm})\right]
\]
and equal to zero on
\[
    \left[(2c_\star)^2\log(e/\varepsilon_{\rm sm}),\infty\right].
\]
Let $\widetilde\bu_i^\Pi$ be the chart-coordinate network obtained from the fixed-radius version of \cref{cor:active-chart-coordinate-from-proj-NN}, and let $\widetilde\bnu_i$ be the normalized-residual network from the fixed-radius version of \cref{lem:approx:r-i--nu-i}.
Write
\[
    \Lambda_{\rm gate}:=1+\log(C_{\rm size,sm}\varepsilon_{\rm sm}^{-1}),
    \qquad
    \varepsilon_{\nu,{\rm gate}}
    :=
    c_{\rm gate}
    \bigl(1\wedge D^{-1/2}\Lambda_{\rm gate}^{1/2}\bigr).
\]
Choose the projection-coordinate accuracy and normal-residual accuracy so that, on
$\gK_i^+(\varepsilon_{\rm sm})$,
\[
    \|\widetilde\bu_i^\Pi-\bu_i^\Pi\|_\infty
    \le
    c_{\rm gate}\Lambda_{\rm gate}^{-1},
    \qquad
    \|\widetilde\bnu_i-\bnu_i\|_\infty
    \le
    \varepsilon_{\nu,{\rm gate}},
\]
where $c_{\rm gate}>0$ is a sufficiently small geometric constant.
Form $\|\widetilde\bnu_i\|_2^2$ using the multiplication lemma and define
\begin{equation}
    \Theta_i^{\rm sm}(\bx,t)
    :=
    \theta_i^A\!\left(\widetilde\bu_i^\Pi(\bx,t)\right)
    \theta_{\rm res}\!\left(\|\widetilde\bnu_i(\bx,t)\|_2^2\right),
    \label{eq:small-noise-projection-coordinate-gate}
\end{equation}
with the final product again implemented by \cref{lem:approx:mon}.
The normal-residual accuracy gives
\[
    \|\widetilde\bnu_i-\bnu_i\|_2
    \le
    c_{\rm gate}\Lambda_{\rm gate}^{1/2},
\]
and, on the enlarged active region,
$\|\bnu_i\|_2\le C\Lambda_{\rm gate}^{1/2}$.  Hence
\[
    \bigl|
    \|\widetilde\bnu_i\|_2^2-\|\bnu_i\|_2^2
    \bigr|
    \le
    Cc_{\rm gate}\Lambda_{\rm gate}.
\]
After choosing $c_{\rm gate}$ and the square-subnetwork accuracy small enough, the
implemented squared norm stays inside the fixed separation margin between the two
residual-ramp thresholds.  The corresponding residual-network size has internal scale
$C_{\rm res}D\varepsilon_{\nu,{\rm gate}}^{-2}\le C C_{\rm res}D^2$, up to
logarithmic factors, which is the $D^{d/(\beta\vee1)}$ contribution recorded in
$\mathfrak a_{\rm sm}$.
The gate satisfies
\[
    \Theta_i^{\rm sm}=1\quad\text{on }\gK_i(\varepsilon_{\rm sm}),
\]
with the allocated gate error absorbed into the final $\Ord(\varepsilon_{\rm sm})$ approximation budget.
Indeed, on $\gK_i(\varepsilon_{\rm sm})$, the buffered-chart lemma gives $\bu_i^\Pi\in K_i^\circ$ and $\|\bnu_i\|_2\le c_\star\sqrt{\log(e/\varepsilon_{\rm sm})}$, so both ramps are equal to one after the above perturbation margin.
The proposition makes no neural-output claim outside $\gK_i^+(\varepsilon_{\rm sm})$.
This is essential: the chart coordinate and residual subnetworks above are certified on $\gK_i^+(\varepsilon_{\rm sm})$, and the analytic tail bound \cref{eq:small-noise-chart-true-tail} handles the true chart integral away from the default active set.
Thus the implemented localization is obtained from projection coordinates, normalized residuals, and support-preserving coefficient extensions only on the certified enlarged active region, not from a $\sigma_{\tdown}$-mesh of $\gS_i$.
Its size is absorbed into the same polylogarithmic factor as the component networks and has no polynomial $\sigma_{\tdown}^{-1}$ contribution.

On the transition region $\gK_i^+(\varepsilon_{\rm sm})\setminus\gK_i(\varepsilon_{\rm sm})$, no point of $\gS_i$ satisfies the active distance condition with radius $c_\star$.
Therefore the true chart integral is $\Ord(\varepsilon_{\rm sm})$ by \cref{lem:approx:P-Q:tail} and the fixed choice of $c_\star$.
The neural output on this transition region is controlled after the component construction by the support dichotomy in Step~4 below.

\textbf{Step 3: Neural approximation of the active expansion.}
For every $q\le \floor{\beta}$ and $\|\blambda\|_1\le q$, approximate
\[
    A_{i,q,\blambda}^{(R)}
    \bigl(\bu_i^\Pi(\bx,t),t\bigr)
\]
by the fixed-radius version of \cref{cor:approx:weighted-A-coeff}, supplied by \cref{cor:active-components-fixed-radius}, with term accuracy
\[
    \varepsilon_{\rm term}^{\rm sm}.
\]
Here and below, \cref{cor:approx:weighted-A-coeff} and the residual and Gaussian lemmas are used through \cref{cor:active-components-fixed-radius} with $c_a=3c_\star$.
Approximate $\bnu_i$ by \cref{lem:approx:r-i--nu-i}, form the monomial $\bnu_i^{\blambda}$ by \cref{lem:approx:mon}, approximate $\sigma_t^q$ by the constant-one network when $q=0$ and by \cref{lem:approx:sigma-k} when $q\ge1$, and approximate the Gaussian factor by \cref{lem:approx:gaussian-factor}.
In the case $q\ge1$, the accuracy parameter passed to \cref{lem:approx:sigma-k} is chosen as the minimum of the required schedule accuracy, $\tdown/2$, and $1/2$.
This makes the admissibility condition in \cref{lem:approx:sigma-k} automatic.
The resulting factor $\log\tdown^{-1}$ is bounded by $\log\Lambda_{\rm sm}$, hence is included in the logarithmic factors displayed in \cref{eq:small-noise-chart-size}.
Internal accuracies are chosen so that each retained term contributes at most $\varepsilon_{\rm term}^{\rm sm}$ to the final error after multiplication by the bounded remaining factors.
The bounds
\[
    \|\bnu_i(\bx,t)\|_2\le C\sqrt{\log(C_{\rm size,sm}\varepsilon_{\rm sm}^{-1})},
    \qquad
    \sigma_t\le C_{\rm sm}\varepsilon_{\rm sm}^{1/\beta},
\]
are exactly the bounds used in \cref{cor:approx:weighted-A-coeff}, so the coefficient-composition cost remains $\varepsilon_{\rm sm}^{-d/\beta}$ up to logarithms for every $\beta>0$.

Multiplying by the projection-coordinate gate $\Theta_i^{\rm sm}$, summing the finitely many $(q,\blambda)$-terms in an affine output layer, and allocating one more $\Ord(\varepsilon_{\rm sm})$ budget to the final products gives a network $\widetilde R_i^{\rm sm}$.

\textbf{Step 4: Transition-region control.}
We first bound this network on $\gK_i^+(\varepsilon_{\rm sm})\setminus\gK_i(\varepsilon_{\rm sm})$, where the gate can take values between $0$ and $1$.
Let $\bu^\Pi=\bu_i^\Pi(\bx,t)$.
If $\bu^\Pi\in S_i^{(R)}:=\supp a_i^{(R)}$, then $\bz_i(\bu^\Pi)=\Pi_{\gM}(\bxi(\bx,t))\in\gS_i$.
Since $(\bx,t)\notin\gK_i(\varepsilon_{\rm sm})$, this particular support point does not satisfy the active distance condition.
Hence
\[
    \|\br_i(\bx,t)\|_2
    =
    \|\bx-m_t\bz_i(\bu^\Pi)\|_2
    >
    c_\star\sigma_t\sqrt{\log\varepsilon_{\rm sm}^{-1}}.
\]
The exact normal Gaussian factor is therefore at most $\varepsilon_{\rm sm}^{c_\star^2/2}$, and the approximated Gaussian factor is small up to its allocated additive error.
If $\bu^\Pi\notin S_i^{(R)}$, then \cref{eq:A-coeff-extension-support} gives $\overline A_{i,q,\blambda}^{(R)}(\bu^\Pi,t)=0$ for every retained coefficient.
The implemented coefficient network is evaluated at the approximate coordinate $\widetilde\bu_i^\Pi$ used in $\widetilde{\mathcal A}_{i,q,\blambda,\varepsilon_{A,q}^{\rm sm}}^{(R)}$, not exactly at $\bu^\Pi$.
By the fixed-radius coordinate approximation in \cref{cor:active-components-fixed-radius}, the choice $\varepsilon_u\asymp(\varepsilon_{A,q}^{\rm sm})^{1/((\beta-q)\wedge1)}$, the $(\beta-q)$-H\"older bound in \cref{eq:A-coeff-extension-holder}, and the intrinsic coefficient-network error,
\[
    \left|
      \widetilde{\mathcal A}_{i,q,\blambda,\varepsilon_{A,q}^{\rm sm}}^{(R)}
      (\bx,t)
    \right|
    \le
    C\|\widetilde\bu_i^\Pi-\bu^\Pi\|_\infty^{(\beta-q)\wedge1}
    +\varepsilon_{A,q}^{\rm sm}
    \le
    C\varepsilon_{A,q}^{\rm sm}.
\]
In both cases, the bounds
\[
    \|\bnu_i(\bx,t)\|_2\le C\sqrt{\log(C_{\rm size,sm}\varepsilon_{\rm sm}^{-1})},
    \qquad
    \sigma_t^q\le \sigma_{\rm sm}^q,
\]
and the definition of $\varepsilon_{A,q}^{\rm sm}$ imply that each retained term is bounded by its term budget.
After summing the finitely many retained terms and decreasing $c_{\rm term}$, if necessary, the neural output on the transition region is at most $\varepsilon_{\rm sm}/3$.

On $\gK_i(\varepsilon_{\rm sm})$, the gate equals $1$, and the component approximations give the same bound against the active expansion.
Consequently
\begin{equation}
    \sup_{\substack{
        t\in\gI_{\rm sm}(\varepsilon_{\rm sm})\\
        (\bx,t)\in\gK_i^+(\varepsilon_{\rm sm})
    }}
    \left|
      \widetilde R_i^{\rm sm}(\bx,t)
      -
      R_{i,{\rm app}}^{\rm sm}(\bx,t)
    \right|
    \le
    2\varepsilon_{\rm sm}/3 .
    \label{eq:small-noise-network-error}
\end{equation}
Combining \cref{eq:small-noise-analytic-error,eq:small-noise-network-error} proves \cref{eq:small-noise-chart-error}.

\textbf{Step 5: Network size.}
The projection, residual, Gaussian, coefficient, schedule-power, monomial, gate, and final product subnetworks are all finite in number for fixed $d,\beta$.
Adding their displayed size bounds, substituting the definition of $\varepsilon_{A,q}^{\rm sm}$, and enlarging $C_{\rm size,sm}$ to absorb the fixed coefficient constants gives \cref{eq:small-noise-chart-size}; the additional powers of the retained-term count are absorbed by $D^{\mathfrak a_{\rm sm}}$.
The projection-coordinate gate in \cref{eq:small-noise-projection-coordinate-gate} consists of a fixed chart coordinate bump, one residual-norm computation, and one scalar residual ramp.
Consequently its cost is polynomial in $D$ and polylogarithmic in $\varepsilon_{\rm sm}^{-1}$, and it introduces no $\sigma_{\tdown}$-mesh or polynomial lower-noise covering factor.
The only remaining non-logarithmic terms are the projection/coefficient H\"older approximation terms recorded in \cref{eq:residual-nu-size,eq:gaussian-factor-size,eq:A-coeff-composed-size}; under $\sigma_t\lesssim\varepsilon_{\rm sm}^{1/\beta}$, these are bounded by $\varepsilon_{\rm sm}^{-d/\beta}$ up to powers of $\log(C_{\rm size,sm}\varepsilon_{\rm sm}^{-1})$.
The geometric projection terms use the smoothness $2(\beta\vee1)$, while the density coefficient terms use their true regularities $\beta-q$; this is the point that keeps the displayed rate valid also for $0<\beta<1$.
\end{proof}

\subsubsection{Score approximation on the small-noise regime}

\Cref{lem:small-noise-degaussianized-Q-lower} is the key lower-bound step in the
small-noise branch.  Without de-Gaussianization, the denominator contains the normal
factor $\exp\{-\|\br\|_2^2/(2\sigma_t^2)\}$ and a reciprocal network would have to
resolve an exponentially small quantity on the tube.  The lemma removes this common
factor and proves a polynomial floor, which is what makes the later ratio network stable.

\begin{lemma}[Small-noise de-Gaussianized denominator lower bound]
\label{lem:small-noise-degaussianized-Q-lower}
Assume \cref{assump:manifold:exact,assump:manifold:density-lower}.
There are constants $c_Q>0$ and $C_Q>0$, depending only on the displayed finite
geometry controls and the uniform schedule bounds in \cref{eq:def:sigma-tdown-lower},
such that
\[
    c_Q^{-1}\vee C_Q
    \le
    C\Gamma_{\gM,q_{\rm geom}}^C .
\]
The following holds.
Fix $t\in[\tdown,\tup]$, put $H:=D\vee\log n$, and let
\[
    A_t:=\{\dist(\bx,m_t\gM)\le C_\varrho\sigma_t\sqrt H\}.
\]
For every $\bx\in A_t$, choose any nearest point
\[
    \by_\star\in\argmin_{\by\in\gM}\|\bx-m_t\by\|_2,
    \qquad
    \br_\star:=\bx-m_t\by_\star .
\]
Then
\begin{equation}
    \exp\!\left(\frac{\|\br_\star\|_2^2}{2\sigma_t^2}\right)
    \gQ(\bx,t)
    \ge
    c_Q p_{\min}
    \bigl(1+C_\varrho\sigma_t\sqrt H\bigr)^{-d/2}
    \ge
    C_Q^{-1}p_{\min}H^{-d/4}.
    \label{eq:small-noise-degaussianized-Q-lower}
\end{equation}
In particular, if $\bx/m_t$ lies in the reach tube and
$\by_\star=\Pi_{\gM}(\bx/m_t)$, the same bound applies to the de-Gaussianized
denominator defined with the projection residual
$\br(\bx,t)=\bx-m_t\Pi_{\gM}(\bx/m_t)$.
\end{lemma}

\begin{proof}
\emph{Step 1: choose a chart with uniformly positive amplitude.}
Because $\gM$ is compact, the nearest point $\by_\star$ exists.
Since the partition of unity is nonnegative and sums to one, there is an index $i_\star$
such that $\rho_{i_\star}(\by_\star)\ge C_{\gM}^{-1}$.
Let $\bu_\star:=\phi_{i_\star}(\by_\star)$.
The support of $\rho_{i_\star}$ is contained in $U_{i_\star}$, and
$\phi_{i_\star}(\gS_{i_\star})\Subset K_{i_\star}^{\circ}$.
Using the finite atlas, the $C^1$-bounds on $\bar\rho_i$, the positive Jacobian lower
bound on each compact coordinate box, and the lower bound
$\rho_{i_\star}(\by_\star)\ge C_{\gM}^{-1}$, there are constants
$r_Q>0$ and $c_a>0$, independent of $(\bx,t,D,n)$, such that
\[
    B_{\R^d}(\bu_\star,r_Q)\subset K_{i_\star}^{\circ},
    \qquad
    a_{i_\star}^{(\gQ)}(\bu)
    =
    \bar\rho_{i_\star}(\bu)
    p_0(\bz_{i_\star}(\bu))
    \sqrt{\det\bG_{i_\star}(\bu)}
    \ge c_a p_{\min}
\]
for every $\bu\in B_{\R^d}(\bu_\star,r_Q)$.
By \cref{def:finite-order-geometry-constants,lem:finite-order-geometry-from-smooth},
these constants may be chosen with
$r_Q^{-1}\vee c_a^{-1}\le C\Gamma_{\gM,q_{\rm geom}}^C$.

\emph{Step 2: compare the phase to the nearest-point phase.}
Set $i=i_\star$ and $\bz=\bz_i$.
Since $\by_\star$ minimizes $\by\mapsto\|\bx-m_t\by\|_2^2$ over the smooth manifold,
the first variation vanishes:
\[
    \bJ_i(\bu_\star)^\top\br_\star=0 .
\]
For $\bu=\bu_\star+\bdelta\in B_{\R^d}(\bu_\star,r_Q)$, Taylor's theorem gives
\[
    \bz(\bu)-\bz(\bu_\star)
    =
    \bJ_i(\bu_\star)\bdelta+\bE(\bdelta),
    \qquad
    \|\bE(\bdelta)\|_2\le C\|\bdelta\|_2^2 .
\]
Therefore, using $0<m_t\le1$,
\begin{align*}
    0
    &\le
    \|\bx-m_t\bz(\bu)\|_2^2-\|\br_\star\|_2^2                                      \\
    &=
    2m_t\langle \br_\star,\by_\star-\bz(\bu)\rangle
    +m_t^2\|\bz(\bu)-\by_\star\|_2^2                                                   \\
    &\le
    C(1+\|\br_\star\|_2)\|\bdelta\|_2^2 .
\end{align*}
The first inequality uses only the nearest-point property of $\by_\star$.
The last inequality uses the vanished linear term
$\langle\br_\star,\bJ_i(\bu_\star)\bdelta\rangle=0$ and the uniform $C^2$-bound on the
chart map.

\emph{Step 3: integrate over a fixed rescaled coordinate ball.}
Let $R_Q\in(0,r_Q]$ be fixed and write
\[
    a_\star:=1+\|\br_\star\|_2,
    \qquad
    \bu=\bu_\star+\sigma_ta_\star^{-1/2}\bw,
    \qquad
    \|\bw\|_2\le R_Q .
\]
Since $\sigma_t\le1$, this coordinate ball is contained in
$B_{\R^d}(\bu_\star,r_Q)$.
By Steps 1 and 2,
\begin{align*}
    \exp\!\left(\frac{\|\br_\star\|_2^2}{2\sigma_t^2}\right)
    \gQ(\bx,t)
    &\ge
    \sigma_t^{-d}
    \int_{B(\bu_\star,r_Q)}
      \exp\!\left(
        -\frac{
          \|\bx-m_t\bz_i(\bu)\|_2^2-\|\br_\star\|_2^2
        }{2\sigma_t^2}
      \right)
      a_i^{(\gQ)}(\bu)\od\bu                                                   \\
    &\ge
    c_a p_{\min}a_\star^{-d/2}
    \int_{\|\bw\|_2\le R_Q}e^{-C\|\bw\|_2^2}\od\bw                              \\
    &\ge
    c_Q p_{\min}(1+\|\br_\star\|_2)^{-d/2}.
\end{align*}

\emph{Step 4: specialize to $A_t$.}
If $\bx\in A_t$, then
$\|\br_\star\|_2=\dist(\bx,m_t\gM)\le C_\varrho\sigma_t\sqrt H$, which gives the first
bound in \cref{eq:small-noise-degaussianized-Q-lower}.
Because $H\ge1$ and $\sigma_t\le1$,
$(1+C_\varrho\sigma_t\sqrt H)^{d/2}\le C_QH^{d/4}$ after increasing $C_Q$.
This proves the second bound.
\end{proof}

\begin{lemma}[$H$-localized de-Gaussianized small-noise chart approximation]
\label{lem:H-localized-degaussianized-small-noise-chart}
Assume \cref{assump:manifold:exact,assump:manifold:density}.
Let $H\ge1$, $0<\varepsilon_{\rm sm}\le1/2$, and suppose
\[
    H\ge \log(e/\varepsilon_{\rm sm}).
\]
By the global choice above, $c_\star\ge 2C_\varrho$.  Assume that
$[t_0,t_1]\subset\gI_{\rm sm}(\varepsilon_{\rm sm})$ and that, for every
$t\in[t_0,t_1]$,
\begin{equation}
    3c_\star\frac{\sigma_t}{m_t}\sqrt H
    \le
    \eta_\Pi/2 .
    \label{eq:H-localized-small-noise-reach-condition}
\end{equation}
For
\[
    A_t(H):=\{\dist(\bx,m_t\gM)\le C_\varrho\sigma_t\sqrt H\},
\]
define the projection residual
\[
    \br(\bx,t):=\bx-m_t\Pi_{\gM}(\bx/m_t),
    \qquad (\bx,t)\in A_t(H),\quad t\in[t_0,t_1],
\]
which is well-defined by \cref{eq:H-localized-small-noise-reach-condition}.  For
$R\in\{\gQ,\gP_1,\dots,\gP_D\}$, set
\[
    \bar R_i(\bx,t)
    :=
    \exp\!\left(\frac{\|\br(\bx,t)\|_2^2}{2\sigma_t^2}\right)R_i(\bx,t).
\]
In particular, write $\bar\gQ_i$ and $\bar\gP_{i,k}$ for the choices
$R=\gQ$ and $R=\gP_k$, and define
\[
    \bar\gQ:=\sum_i\bar\gQ_i,\qquad
    \bar\gP_k:=\sum_i\bar\gP_{i,k},
    \qquad
    \bar\gU_{i,k}:=\frac{m_t\bar\gP_{i,k}-x_k\bar\gQ_i}{\sigma_t},
    \qquad
    \bar\gU_k:=\sum_i\bar\gU_{i,k}.
\]
Then, for every chart $i$ and every such $R$, there is a ReLU network
$\widehat{\bar R}_{i,H}^{\rm sm}$ such that
\begin{equation}
    \sup_{\substack{t\in[t_0,t_1]\\ \bx\in A_t(H)}}
    |\widehat{\bar R}_{i,H}^{\rm sm}(\bx,t)-\bar R_i(\bx,t)|
    \le
    C\varepsilon_{\rm sm}H^C .
    \label{eq:H-localized-degaussianized-chart-error}
\end{equation}
Moreover, after summing the fixed atlas, there are networks
$\widehat{\bar\gQ}_H^{\rm sm}$ and
$\widehat{\bar\gU}_H^{\rm sm}$ satisfying
\begin{equation}
    |\widehat{\bar\gQ}_H^{\rm sm}(\bx,t)-\bar\gQ(\bx,t)|
    \vee
    \|\widehat{\bar\gU}_H^{\rm sm}(\bx,t)-\bar\gU(\bx,t)\|_\infty
    \le
    C\varepsilon_{\rm sm}H^C
    \label{eq:H-localized-degaussianized-QU-error}
\end{equation}
uniformly over $t\in[t_0,t_1]$ and $\bx\in A_t(H)$.
All networks in this lemma may be chosen so that their depth, width, sparsity, and
log-weight obey the bounds in
\cref{eq:small-noise-chart-size}, with the logarithmic factors
$\log(C_{\rm size,sm}\varepsilon_{\rm sm}^{-1})$ replaced by
$H+\log(C_{\rm size,sm}\varepsilon_{\rm sm}^{-1})$, and with the fixed-atlas sum and
the $D$-coordinate parallelization absorbed by the displayed ambient polynomial powers.
Hence, when
$H=D\vee\log n$, the extra $H^C$ factors are absorbed by the displayed
$(D\vee\log n)^{\mathfrak p_{\beta,d}}$ powers below.
\end{lemma}

\begin{proof}
\emph{Step 1: $H$-active localization and reach control.}
For $c_a>0$, define
\[
    \gK_i^{(c_a,H)}
    :=
    \left\{(\bx,t):\bx\in A_t(H),\ \exists \by\in\gS_i\text{ with }\|\bx-m_t\by\|_2\le c_a\sigma_t\sqrt H\right\}.
\]
Because $c_\star\ge2C_\varrho$ and the sets $\gS_i$ cover $\gM$, every $(\bx,t)\in A_t(H)$ belongs to $\gK_i^{(c_\star,H)}$ for at least one chart.
If $(\bx,t)\in\gK_i^{(3c_\star,H)}$, then
\[
    \dist(\bx/m_t,\gM)
    \le
    3c_\star\frac{\sigma_t}{m_t}\sqrt H
    \le
    \eta_\Pi/2 .
\]
Thus the nearest projection $\Pi_{\gM}(\bx/m_t)$ is unique, its chart coordinate in chart $i$ lies in $K_i^\circ$, and the Gauss--Newton and buffered-coordinate estimates used earlier hold with $\sqrt{\log(e/\varepsilon_{\rm sm})}$ replaced by $\sqrt H$.
Throughout the active part of the proof write
\[
    \bxi:=\bx/m_t,\qquad
    \bpi:=\Pi_{\gM}(\bxi),\qquad
    \bu^\Pi:=\phi_i(\bpi),\qquad
    \br:=\bx-m_t\bpi,\qquad
    \bnu:=\sigma_t^{-1}\br .
\]
The active condition gives $\bu^\Pi\in K_i^\circ$ and $\|\bnu\|_2\le3c_\star\sqrt H$.

\emph{Step 2: de-Gaussianized coercivity.}
Let $\Phi_i(\bu;\bxi)=\frac12\|\bxi-\bz_i(\bu)\|_2^2$, and write $S_i^{(R)}:=\supp a_i^{(R)}$.
The projection orthogonality gives $\nabla_\bu\Phi_i(\bu^\Pi;\bxi)=0$.
The proof of \cref{lem:invertible:proj}, with the $H$-reach condition above, gives
\[
    \nabla_\bu^2\Phi_i(\bu^\Pi;\bxi)\succeq \frac12\lambda_i^{(-)}\bI_d .
\]
Hence there are constants $r_i^0>0$ and $c_0>0$ such that
\[
    \Phi_i(\bu;\bxi)-\Phi_i(\bu^\Pi;\bxi)\ge c_0\|\bu-\bu^\Pi\|_2^2
\]
whenever $\bu\in\Box_i$ and $\|\bu-\bu^\Pi\|_2\le r_i^0$.
On the compact set of pairs with $\bu\in S_i^{(R)}$ and $\|\bu-\bu^\Pi\|_2\ge r_i^0$, uniqueness of the reach projection gives a strictly positive phase gap.
Combining the local quadratic gap with the compact-away gap and using the bounded diameter of $\Box_i$, we obtain the uniform de-Gaussianized lower bound
\[
    \frac{\|\bx-m_t\bz_i(\bu^\Pi+\sigma_t\bw)\|_2^2-\|\br\|_2^2}{2\sigma_t^2}
    \ge
    c\|\bw\|_2^2
\]
for all $\bu^\Pi+\sigma_t\bw\in S_i^{(R)}$, after decreasing the fixed small-noise threshold in \cref{eq:choice:C-sm-Laplace} if necessary.
Consequently, for $R_H:=A\sqrt H$ with $A$ fixed sufficiently large,
\[
    \int_{\|\bw\|_2>R_H}
    \exp\!\left(
      -\frac{\|\bx-m_t\bz_i(\bu^\Pi+\sigma_t\bw)\|_2^2-\|\br\|_2^2}{2\sigma_t^2}
    \right)
    (1+\|\bw\|_2)^C\od\bw
    \le
    C\varepsilon_{\rm sm}H^{-C},
\]
where we used $H\ge\log(e/\varepsilon_{\rm sm})$ and increased $A$ to dominate the finitely many polynomial powers used below.

\emph{Step 3: scalar Taylor expansion after removing the normal Gaussian.}
For $s\in[0,\sigma_t]$, define the smooth difference quotient
\[
    \bZ_s(\bu^\Pi,\bw)
    :=
    \int_0^1
    \bJ_i(\bu^\Pi+\theta s\bw)\bw\od\theta ,
\]
so that $\bz_i(\bu^\Pi+s\bw)-\bz_i(\bu^\Pi)=s\bZ_s(\bu^\Pi,\bw)$ and $\bZ_0(\bu^\Pi,\bw)=\bJ_i(\bu^\Pi)\bw$.
Define the de-Gaussianized phase along this scalar path by
\[
    \Theta_s(\bu^\Pi,\bnu,\bw,t)
    :=
    \frac12\|\bnu-m_t\bZ_s(\bu^\Pi,\bw)\|_2^2-\frac12\|\bnu\|_2^2 .
\]
At $s=\sigma_t$, this is exactly the exponent in the de-Gaussianized chart integral, and at $s=0$ it equals $\frac{m_t^2}{2}\bw^\top\bG_i(\bu^\Pi)\bw$ because $\bJ_i(\bu^\Pi)^\top\bnu=0$.
We now verify that the same coercivity holds uniformly along this interpolation.
For $s=0$, ellipticity of $\bG_i(\bu^\Pi)$ gives
\[
    \Theta_0(\bu^\Pi,\bnu,\bw,t)
    \ge c\|\bw\|_2^2 .
\]
For $0<s\le\sigma_t$, set $\bx_s:=m_t\bpi+s\bnu$.  Since
$\|\bnu\|_2\le3c_\star\sqrt H$ and
\cref{eq:H-localized-small-noise-reach-condition} holds,
\[
    \dist(\bx_s/m_t,\gM)
    \le
    s\|\bnu\|_2/m_t
    \le
    3c_\star\sigma_t m_t^{-1}\sqrt H
    \le \eta_\Pi/2 .
\]
Moreover $\bpi=\Pi_{\gM}(\bx_s/m_t)$, because $\bx_s/m_t$ lies on the normal segment
issued from $\bpi$ and remains inside the reach tube.  Applying the reach-tube objective
gap \cref{lem:reach-tube-objective-gap} to the pair
$\bx_s/m_t$ and $\bz_i(\bu^\Pi+s\bw)$ gives
\[
    \|\bx_s-m_t\bz_i(\bu^\Pi+s\bw)\|_2^2
    -
    \|\bx_s-m_t\bpi\|_2^2
    \ge
    c\,m_t^2\|\bz_i(\bu^\Pi+s\bw)-\bpi\|_2^2 .
\]
On the compact coordinate box, the inverse chart is uniformly bi-Lipschitz, hence
\[
    \|\bz_i(\bu^\Pi+s\bw)-\bpi\|_2
    \ge c\,s\|\bw\|_2
    \qquad
    \text{whenever }\bu^\Pi+s\bw\in\Box_i .
\]
Dividing the previous objective gap by $2s^2$ gives
\[
    \Theta_s(\bu^\Pi,\bnu,\bw,t)
    =
    \frac{\|\bx_s-m_t\bz_i(\bu^\Pi+s\bw)\|_2^2
    -\|\bx_s-m_t\bpi\|_2^2}{2s^2}
    \ge c\|\bw\|_2^2 .
\]
Thus $\Theta_s(\bu^\Pi,\bnu,\bw,t)\ge c\|\bw\|_2^2$ on the chart support for every
$s\in[0,\sigma_t]$.
For $R\in\{\gQ,\gP_1,\dots,\gP_D\}$, set
\[
    G_R(s,\bw)
    :=
    \exp\{-\Theta_s(\bu^\Pi,\bnu,\bw,t)\}
    a_i^{(R)}(\bu^\Pi+s\bw).
\]
The coercivity from Step~2 and the smoothness of $\bz_i$ imply that all derivatives in $s$ falling on the geometric and exponential factors are bounded by
\[
    C H^C(1+\|\bw\|_2)^C e^{-c\|\bw\|_2^2}
\]
for $0\le s\le\sigma_t$ and $\|\bw\|_2\le R_H$.
The density-dependent amplitude is expanded only to order $\floor{\beta}$, and the remainder is controlled by the $\gH^\beta$ modulus of $a_i^{(R)}$ along the line $\bu^\Pi+s\bw$.
Taylor's theorem in the scalar variable $s$ therefore gives
\[
    G_R(\sigma_t,\bw)
    =
    \sum_{q=0}^{\floor{\beta}}\sigma_t^q G_{R,q}(\bu^\Pi,\bnu,\bw,t)
    +
    \sigma_t^\beta \mathcal E_R(\bu^\Pi,\bnu,\bw,t),
\]
with
\[
    |\mathcal E_R(\bu^\Pi,\bnu,\bw,t)|
    \le
    C H^C(1+\|\bw\|_2)^C e^{-c\|\bw\|_2^2}.
\]
Each coefficient $G_{R,q}$ is a finite sum of a polynomial in $(\bnu,\bw)$, a smooth geometric factor in $\bu^\Pi$, a bounded schedule power of $m_t$, and a derivative $\partial^\alpha a_i^{(R)}(\bu^\Pi)$ with $|\alpha|\le q$.
After integration in $\bw$, the coefficients agree with the $A_{i,q,\blambda}^{(R)}$ defined in \cref{prop:laplace-local-chart}, because both constructions differentiate the same recentered integrand after the common normal Gaussian factor has been removed.
Integrating the Taylor expansion on $\|\bw\|_2\le R_H$ and adding the tail from Step~2 yields
\[
    \bar R_i(\bx,t)
    =
    \sum_{q=0}^{\floor{\beta}}\sigma_t^q
    \sum_{\|\blambda\|_1\le q}
      A_{i,q,\blambda}^{(R)}(\bu^\Pi,t)\bnu^{\blambda}
    +
    \Rem_{i,H}^{(R)}(\bx,t),
    \qquad
    |\Rem_{i,H}^{(R)}(\bx,t)|\le C\sigma_t^\beta H^C .
\]
Since $t\in[t_0,t_1]\subset\gI_{\rm sm}(\varepsilon_{\rm sm})$, this remainder is bounded by $C\varepsilon_{\rm sm}H^C$.

\emph{Step 4: centered numerator expansion without a $\sigma_t^{-1}$ loss.}
For each coordinate $k$, define
\[
    N_k(s,\bw)
    :=
    -\nu_k+m_t(\bZ_s(\bu^\Pi,\bw))_k.
\]
At $s=\sigma_t$, this satisfies
\[
    N_k(\sigma_t,\bw)
    =
    \frac{m_t z_{i,k}(\bu^\Pi+\sigma_t\bw)-x_k}{\sigma_t}.
\]
Thus $\bar\gU_{i,k}$ is the integral of $N_k(\sigma_t,\bw)G_{\gQ}(\sigma_t,\bw)$.
Applying the same scalar Taylor argument to the product $N_k(s,\bw)G_{\gQ}(s,\bw)$ gives
\[
    \bar\gU_{i,k}(\bx,t)
    =
    \sum_{q=0}^{\floor{\beta}}\sigma_t^q
    \sum_{\|\blambda\|_1\le q+1}
      A_{i,k,q,\blambda}^{(\bar U)}(\bu^\Pi,t)\bnu^{\blambda}
    +
    \Rem_{i,k,H}^{(\bar U)}(\bx,t),
\]
where $|\Rem_{i,k,H}^{(\bar U)}(\bx,t)|\le C\varepsilon_{\rm sm}H^C$.
For clarity, we verify the centered-numerator coefficient analogue used by the network
construction.  Differentiating $N_k(s,\bw)G_{\gQ}(s,\bw)$ at $s=0$ at most $q$ times
produces finite sums of terms of the form
\[
    m_t^a P(\bnu,\bw)
    H(\bu^\Pi)
    \partial^\alpha a_i^{(\gQ)}(\bu^\Pi)
    \exp\!\left(-\frac{m_t^2}{2}\bw^\top\bG_i(\bu^\Pi)\bw\right),
    \qquad |\alpha|\le q,
\]
where $P$ is a polynomial whose $\bnu$-degree is at most $q+1$; the extra degree is
exactly the leading centered factor $-\nu_k$ in $N_k$.  After integrating in $\bw$, the
Gaussian moment formula used in \cref{lem:A-coeff-extension} expresses each coefficient
$A_{i,k,q,\blambda}^{(\bar U)}$ as a finite sum of powers of $m_t$ times products of
smooth geometric factors and derivatives $\partial^\alpha a_i^{(\gQ)}$ with
$|\alpha|\le q$.  Therefore, on the same intermediate box
$K_i^\circ\Subset\Box_i^A\Subset\Box_i$, define the extension by this identical finite
formula.  If $\bu\notin S_i^{(\gQ)}=\supp a_i^{(\gQ)}$, then $a_i^{(\gQ)}$ vanishes in a
neighborhood of $\bu$, so all derivatives $\partial^\alpha a_i^{(\gQ)}(\bu)$ with
$|\alpha|\le q$ vanish; the centered extension is therefore support-preserving.  Since
$a_i^{(\gQ)}\in\gH^\beta$ and only derivatives up to order $q$ are used, the extended
coefficient belongs to $\gH^{\beta-q}(\Box_i^A)$ with norm bounded by the same finite
geometry and density factors as in \cref{lem:A-coeff-extension}.  Applying the coefficient
network construction of \cref{lem:approx:A-coeff} to this finite family gives the same
depth, width, sparsity, and log-weight bounds, up to one additional fixed power of $D$
for the coordinate index $k$ and the extra normal monomial.  This power is absorbed into
$\mathfrak p_{\beta,d}$.
This direct centered expansion is the reason the proof never forms $(m_t\bar\gP-x\bar\gQ)/\sigma_t$ from two separately approximated networks.

\emph{Step 5: certified chart gate and network realization.}
Fix one intermediate coefficient box $\Box_i^A$ with $K_i^\circ\Subset\Box_i^A\Subset\Box_i$, and use the coefficient-extension construction on this common box for the finite family of retained terms.
Choose the coordinate tolerance $\varepsilon_u^{\rm cert}$ as a fixed constant multiple of the smallest coordinate tolerance required by these retained coefficient networks.
Apply \cref{lem:H-certified-chart-objective-gate} with $c_a=3c_\star$, $\varepsilon=\varepsilon_{\rm sm}$, $\eta=\varepsilon_u^{\rm cert}$, and this box $\Box_i^A$.
It gives networks $\Gamma_i^H$, $\breve\bu_i$, and $\breve\bu_i^A$ such that $\Gamma_i^H=1$ and $\breve\bu_i^A$ is an accurate coefficient coordinate on $\gK_i^{(3c_\star,H)}$, while $\Gamma_i^H>0$ anywhere in $A_t(H)$ implies that $\breve\bu_i$ is a certified chart-$i$ coordinate for $\Pi_{\gM}(\bx/m_t)$.
For the coefficient functions in Steps~3 and 4, use the extensions from
\cref{lem:A-coeff-extension} and the centered-numerator analogues verified in Step~4,
evaluate them at the clipped certified coordinate $\breve\bu_i^A$, and multiply the
retained expansion by $\Gamma_i^H$.
Use the global normal-residual network supplied by \cref{cor:H-active-components} on the active union $\bigcup_j\gK_j^{(3c_\star,H)}$, which contains $A_t(H)$ by Step~1.
Allocate accuracy $c\varepsilon_{\rm sm}H^{-C}$ to each coefficient, monomial, schedule power, product, normal-residual operation, certified chart gate, and final affine summation.
On $\gK_i^{(3c_\star,H)}$ the gate equals one, the coordinate and normal-residual errors are certified, and the Taylor expansions from Steps~3 and 4 are approximated with total error $C\varepsilon_{\rm sm}H^C$.
The resulting size bounds are the same as those in \cref{eq:small-noise-chart-size}, with every logarithmic scale enlarged to $H+\log(C_{\rm size,sm}\varepsilon_{\rm sm}^{-1})$, because the certified objective gate reuses the global projection objective networks and adds only fixed-atlas reciprocals, ramps, products, and sums.

\emph{Step 6: inactive and transition regions after de-Gaussianization.}
If $(\bx,t)\notin\gK_i^{(3c_\star,H)}$ and $\bu\in S_i^{(R)}$, then
\[
    \|\bx-m_t\bz_i(\bu)\|_2>3c_\star\sigma_t\sqrt H .
\]
Since $\bx\in A_t(H)$ gives $\|\br(\bx,t)\|_2\le C_\varrho\sigma_t\sqrt H$ and $c_\star\ge2C_\varrho$, we have
\[
    \|\bx-m_t\bz_i(\bu)\|_2^2-\|\br(\bx,t)\|_2^2
    \ge c c_\star^2\sigma_t^2H .
\]
The de-Gaussianized inactive integrand is therefore bounded by $e^{-c c_\star^2H}$ times a fixed polynomial in $H$.
The global choice of $c_\star$ makes the true inactive contribution bounded by $C\varepsilon_{\rm sm}H^C$.
It remains to control the implemented chart output on the same inactive region.
If $\Gamma_i^H=0$, the implemented chart output is zero.
If $\Gamma_i^H>0$, the certification in Step~5 gives a valid chart coordinate $\bu_i^\Pi=\phi_i(\Pi_{\gM}(\bx/m_t))$ and $\|\breve\bu_i-\bu_i^\Pi\|_\infty\le\varepsilon_u^{\rm cert}$.
Because $\bx\in A_t(H)$ and $c_\star\ge2C_\varrho$, the condition $\bu_i^\Pi\in S_i^{(R)}$ would imply $(\bx,t)\in\gK_i^{(c_\star,H)}\subset\gK_i^{(3c_\star,H)}$, contradicting the inactive assumption.
Hence $\bu_i^\Pi\notin S_i^{(R)}$ whenever the inactive-region gate is nonzero.
If $\bu_i^\Pi\in\Box_i^A$, the support-preserving extension gives $\overline A_{i,q,\blambda}^{(R)}(\bu_i^\Pi,t)=0$, and the H\"older estimate plus the $1$-Lipschitz clipping to $\Box_i^A$ gives $|\overline A_{i,q,\blambda}^{(R)}(\breve\bu_i^A,t)|\le C(\varepsilon_u^{\rm cert})^{(\beta-q)\wedge1}$.
If $\bu_i^\Pi\notin\Box_i^A$, the fixed positive separation of $S_i^{(R)}\subset K_i^\circ\Subset\Box_i^A$ from $\partial\Box_i^A$ and the choice of $\varepsilon_u^{\rm cert}$ ensure that the clipped coefficient input remains in $\Box_i^A\setminus S_i^{(R)}$, so the same bound holds with the intrinsic coefficient-network accuracy alone.
Choosing $\varepsilon_u^{\rm cert}$ according to the coefficient tolerances in Step~5 makes every retained inactive term at most its allocated $c\varepsilon_{\rm sm}H^{-C}$ budget after multiplication by $\sigma_t^q$ and the $H^{q/2}$ normal-monomial bound.
Combining the active approximation, the true inactive tail, and the certified off-support network control proves \cref{eq:H-localized-degaussianized-chart-error}.
Summing over the fixed atlas and using the direct centered-numerator networks from Step~4 proves \cref{eq:H-localized-degaussianized-QU-error}.
The displayed network-size statement follows from the size bounds in Step~5, the fixed atlas summation, and the coordinatewise parallelization for $\bar\gU$.
\end{proof}

\begin{corollary}[{Score approximation on small noise regime $[t_0, t_1]$}]
\label{cor:small-noise-density-lower-approx}
Assume \cref{assump:manifold:exact,assump:manifold:density,assump:manifold:density-lower}, and suppose $D\vee p_{\min}^{-1}\le n^{a_0}$.
Put $H:=D\vee\log n$, let $q_{\rm geom}$ be as in \cref{eq:def:qgeom-main}, and, after
this finite order is fixed, take
\[
    \mathfrak p_{\beta,d}
    =
    C
    +
    \max\left\{
      8(q_{\rm geom}+2),
      \left\lceil
        \floor{\beta}+1+\frac{\floor{\beta}\,d}{\beta-\floor{\beta}}
        +\frac{d}{\beta\vee1}+\frac{d}{2}+6
      \right\rceil
    \right\},
\]
where the $C$ in the exponent is universal and independent of
$D,n,\sigma_t,p_{\min}$, and the numerical geometry constants.  Thus
$\mathfrak p_{\beta,d}$ may depend on the fixed finite order $q_{\rm geom}$; in
the main $d>2$ regime, choosing $q_{\rm geom}=q_{\rm opt}(\beta)$ yields
$q_{\rm geom}=\Ord_\beta(1)$ and hence $\mathfrak p_{\beta,d}=\Ord_\beta(d)$.
Let $C_{\rm den,sm}\ge1$ be a fixed constant large enough to
dominate the reciprocal and product constants in the small-noise ratio step.  Assume the
small-noise reach condition
\begin{equation}
    3c_\star\underline m^{-1}\sigma_{t_1}\sqrt H
    \le
    \eta_\Pi/2
    \label{eq:small-noise-final-reach-condition}
\end{equation}
and the small-noise denominator-resolution condition
\begin{equation}
    C_{\rm den,sm} p_{\min}^{-1}
    \Gamma_{\gM,q_{\rm geom}}^C B_0^C
    n^{-\beta/(d+2\beta)}
    H^{\mathfrak p_{\beta,d}/2-1/2}
    \le 1 .
    \label{eq:small-noise-denominator-resolution-condition}
\end{equation}
For the VP/OU schedules used here, $\sigma_{t_1}\asymp n^{-1/(d+2\beta)}$, so
\cref{eq:small-noise-final-reach-condition} is implied by
\[
    D\vee\log n
    \le
    c\,\underline m^2\eta_\Pi^2\,n^{2/(d+2\beta)} .
\]
There exists a switched ReLU score network $s_{\rm sm}^{\rm lb}$ on $[t_0,t_1]$ such that for
\[
    A_t:=\{\dist(\bx,m_t\gM)\le C_\varrho\sigma_t\sqrt H\},
\]
\begin{equation}
    \|\sigma_t s_{\rm sm}^{\rm lb}(\bx,t)-\sigma_t s^*(\bx,t)\|_2
    \le
    C \Gamma_{\gM,q_{\rm geom}}^C B_0^C
    p_{\min}^{-C}
    n^{-\beta/(d+2\beta)} (D\vee\log n)^{\mathfrak p_{\beta,d}/2},
    \quad
    \bx\in A_t,\;\; t\in[t_0,t_1],
    \label{eq:small-noise-lower-scaled-pointwise}
\end{equation}
or equivalently,
\begin{equation}
    \|s_{\rm sm}^{\rm lb}(\bx,t)-s^*(\bx,t)\|_2^2
    \le
    C \Gamma_{\gM,q_{\rm geom}}^C B_0^C
    p_{\min}^{-C}
    n^{-2\beta/(d+2\beta)} (D\vee\log n)^{\mathfrak p_{\beta,d}}\sigma_t^{-2},
    \quad
    \bx\in A_t,\quad t\in[t_0,t_1].
    \label{eq:small-noise-lower-tube-pointwise}
\end{equation}
\begin{equation}
    \sup_{\bx\in\R^D}\|s_{\rm sm}^{\rm lb}(\bx,t)\|_\infty
    \le C\sigma_t^{-1}(D\vee\log n)^{1/2},
    \qquad t\in[t_0,t_1],
    \label{eq:small-noise-lower-approx-clip}
\end{equation}
and
\begin{equation}
    L, \; \log B
    \le
    C\Gamma_{\gM,q_{\rm geom}}^C B_0^C
    p_{\min}^{-C} (D\vee\log n)^{\mathfrak p_{\beta,d}},
    \label{eq:small-noise-lower-approx-depth-weight}
\end{equation}
and
\begin{equation}
    \|\bW\|_\infty, \; S
    \le
    C\Gamma_{\gM,q_{\rm geom}}^C B_0^C
    p_{\min}^{-C}
    n^{d/(d+2\beta)} (D\vee\log n)^{\mathfrak p_{\beta,d}}.
    \label{eq:small-noise-lower-approx-size}
\end{equation}
\end{corollary}

\begin{proof}
\emph{Step 1: de-Gaussianize the chart integrals.}
On the tube
\[
    A_t:=\{\dist(\bx,m_t\gM)\le C_\varrho\sigma_t\sqrt H\},
\]
the reach condition \cref{eq:small-noise-final-reach-condition} gives
\[
    \dist(\bx/m_t,\gM)
    \le
    C_\varrho\frac{\sigma_t}{m_t}\sqrt H
    \le \eta_\Pi/2 .
\]
Thus $\Pi_{\gM}(\bx/m_t)$ is the unique metric projection, and we may define
\[
    \br(\bx,t):=\bx-m_t\Pi_{\gM}(\bx/m_t).
\]
Define the de-Gaussianized denominator and centered numerator by
\begin{equation}
    \bar\gQ(\bx,t)
    :=
    \exp\!\left(\frac{\|\br(\bx,t)\|_2^2}{2\sigma_t^2}\right)\gQ(\bx,t),
    \qquad
    \bar\gU(\bx,t)
    :=
    \exp\!\left(\frac{\|\br(\bx,t)\|_2^2}{2\sigma_t^2}\right)\gU(\bx,t).
    \label{eq:def-small-noise-degaussianized-QU}
\end{equation}
Equivalently, for $R\in\{\gQ,\gP_1,\dots,\gP_D\}$, define
\begin{equation}
    \bar R_i(\bx,t)
    :=
    \exp\!\left(\frac{\|\br(\bx,t)\|_2^2}{2\sigma_t^2}\right)R_i(\bx,t).
    \label{eq:def-small-noise-degaussianized-chart-R}
\end{equation}
Using \cref{eq:def:large-noise-cell-Ri}, this removes exactly the common normal Gaussian factor:
\begin{equation}
    \bar R_i(\bx,t)
    =
    \sigma_t^{-d}
    \int_{K_i^{\rm cell}}
      \exp\!\left(
        -\frac{
          \|\bx-m_t\bz_i(\bu)\|_2^2-\|\br(\bx,t)\|_2^2
        }{2\sigma_t^2}
      \right)
      a_i^{(R)}(\bu)\od\bu .
    \label{eq:small-noise-degaussianized-chart-explicit}
\end{equation}
Thus all exponential factors in the lower-bound small-noise approximation are displayed explicitly: the original chart Gaussian
\[
    \exp\!\left(-\frac{\|\bx-m_t\bz_i(\bu)\|_2^2}{2\sigma_t^2}\right),
\]
the common de-Gaussianizing multiplier
\[
    \exp\!\left(\frac{\|\br(\bx,t)\|_2^2}{2\sigma_t^2}\right),
\]
and their product
\[
    \exp\!\left(
      -\frac{
        \|\bx-m_t\bz_i(\bu)\|_2^2-\|\br(\bx,t)\|_2^2
      }{2\sigma_t^2}
    \right).
\]
Consequently,
\begin{equation}
    \bar\gQ=\sum_i\bar\gQ_i,
    \qquad
    \bar\gP_k=\sum_i\bar\gP_{i,k},
    \qquad
    \bar\gU_k=\frac{m_t\bar\gP_k-x_k\bar\gQ}{\sigma_t}.
    \label{eq:small-noise-degaussianized-PU}
\end{equation}
The common exponential factor cancels, so
\[
    \frac{\bar\gU(\bx,t)}{\bar\gQ(\bx,t)}
    =
    \frac{\gU(\bx,t)}{\gQ(\bx,t)}
    =
    \sigma_t s^*(\bx,t)
    \qquad (\bx\in A_t).
\]

\emph{Step 2: lower-bound the de-Gaussianized denominator.}
By \cref{lem:small-noise-degaussianized-Q-lower}, the density lower bound gives the
following denominator floor after de-Gaussianization:
\begin{equation}
    \bar\gQ(\bx,t)
    \ge
    C^{-1}\Gamma_{\gM,q_{\rm geom}}^{-C}
    p_{\min}(D\vee\log n)^{-d/4},
    \qquad
    (\bx,t)\in A_t,\quad t\in[t_0,t_1],
    \label{eq:small-noise-degaussianized-Q-lower-cor}
\end{equation}
The factor $(D\vee\log n)^{d/4}$ in the reciprocal is absorbed into the
$(D\vee\log n)^{\mathfrak p_{\beta,d}}$ part of the explicit ambient polynomial below,
and the displayed $\Gamma_{\gM,q_{\rm geom}}^C$ factor is carried with the other geometry
terms.
The centered-ratio bound from \cref{eq:large-noise-U-over-Q-density-lower}, used with
$\tdown=t_0$ and the assumption $p_{\min}^{-1}\le n^{a_0}$, gives
\begin{equation}
    \|\bar\gU(\bx,t)\|_2
    \le
    C H^{1/2}\bar\gQ(\bx,t)
    \qquad ((\bx,t)\in A_t).
    \label{eq:small-noise-degaussianized-U-over-Q}
\end{equation}

\emph{Step 3: approximate the de-Gaussianized numerator and denominator.}
Set $a_\beta:=\beta/(d+2\beta)$ and $\varepsilon_{\rm sm}:=\frac12\wedge e\,n^{-a_\beta}$.
If $\varepsilon_{\rm sm}=e\,n^{-a_\beta}$, then $\log(e/\varepsilon_{\rm sm})=a_\beta\log n\le H$.
If $\varepsilon_{\rm sm}=1/2$, then $\log(e/\varepsilon_{\rm sm})=\log(2e)\le H$ because $d\ge1$, $d<D$, and hence $D\ge2$.
Use \cref{lem:H-localized-degaussianized-small-noise-chart} with $\varepsilon_{\rm sm}$ and $H=D\vee\log n$.
The inclusion $[t_0,t_1]\subset\gI_{\rm sm}(\varepsilon_{\rm sm})$ follows from the choice of the small-noise endpoint $t_1$.
The reach condition needed by that lemma is exactly \cref{eq:small-noise-final-reach-condition}.
This is the point where the de-Gaussianized construction uses $H$-active charts rather than the smaller $\sqrt{\log(e/\varepsilon_{\rm sm})}$-active charts: for every $\bx\in A_t$, a nearest chart is active, while charts outside the $3c_\star\sqrt H$-enlargement have
\[
    \|\bx-m_t\bz_i(\bu)\|_2^2-\|\br(\bx,t)\|_2^2
    \ge c\sigma_t^2H
\]
on their supports and are exponentially negligible after de-Gaussianization.
After summing over charts, this gives networks $\widehat{\bar\gQ}$ and
$\widehat{\bar\gU}$ satisfying
\begin{equation}
    |\widehat{\bar\gQ}(\bx,t)-\bar\gQ(\bx,t)|
    \vee
    \|\widehat{\bar\gU}(\bx,t)-\bar\gU(\bx,t)\|_\infty
    \le
    C\Gamma_{\gM,q_{\rm geom}}^C B_0^C
    n^{-\beta/(d+2\beta)}
    (D\vee\log n)^{\mathfrak p_{\beta,d}/2-d/4-1/2}
    \label{eq:small-noise-degaussianized-QU-approx}
\end{equation}
on $A_t$, uniformly in $t\in[t_0,t_1]$.
Here the universal constant in the definition of $\mathfrak p_{\beta,d}$ is chosen large enough that the $H^C$ loss in \cref{eq:H-localized-degaussianized-QU-error} is dominated by $(D\vee\log n)^{\mathfrak p_{\beta,d}/2-d/4-1/2}$.
The size is bounded by the right-hand sides of \cref{eq:small-noise-lower-approx-depth-weight,eq:small-noise-lower-approx-size}: in \cref{eq:small-noise-chart-size} we use $\varepsilon_{\rm sm}=\frac12\wedge e\,n^{-\beta/(d+2\beta)}$, so $\varepsilon_{\rm sm}^{-d/\beta}\lesssim n^{d/(d+2\beta)}$.
The extra $\log\tdown^{-1}$ from \cref{eq:def:Lambda-small-noise-chart-size} is absorbed into $(D\vee\log n)^{\mathfrak p_{\beta,d}}$ because $\tdown=t_0\asymp n^{-2(\beta+1)/(d+2\beta)}$.

\emph{Step 4: form the stable ratio.}
By \cref{eq:small-noise-degaussianized-Q-lower-cor},
\[
    \bar\gQ(\bx,t)^{-1}
    \le
    C\Gamma_{\gM,q_{\rm geom}}^C
    p_{\min}^{-1}(D\vee\log n)^{d/4}
    \qquad ((\bx,t)\in A_t).
\]
The denominator-resolution condition \cref{eq:small-noise-denominator-resolution-condition}
and \cref{eq:small-noise-degaussianized-QU-approx} imply
\[
    C\Gamma_{\gM,q_{\rm geom}}^C B_0^C
    n^{-\beta/(d+2\beta)}
    (D\vee\log n)^{\mathfrak p_{\beta,d}/2-d/4-1/2}
    \le
    c\Gamma_{\gM,q_{\rm geom}}^{-C}
    p_{\min}(D\vee\log n)^{-d/4}/4,
\]
by the fixed choice of $C_{\rm den,sm}$.  Hence
$\widehat{\bar\gQ}\vee c\Gamma_{\gM,q_{\rm geom}}^{-C}
p_{\min}(D\vee\log n)^{-d/4}/2
\ge c\Gamma_{\gM,q_{\rm geom}}^{-C}
p_{\min}(D\vee\log n)^{-d/4}/2$, and the standard reciprocal and
product networks applied to
$\widehat{\bar\gU}/(\widehat{\bar\gQ}\vee
c\Gamma_{\gM,q_{\rm geom}}^{-C}p_{\min}(D\vee\log n)^{-d/4}/2)$, together with
\cref{eq:small-noise-degaussianized-U-over-Q,eq:small-noise-degaussianized-QU-approx}, give
a network $\widehat v_{\rm sm}^{\rm lb}$ such that
\[
    \|\widehat v_{\rm sm}^{\rm lb}(\bx,t)-\sigma_t s^*(\bx,t)\|_2
    \le C\Gamma_{\gM,q_{\rm geom}}^C B_0^C
    p_{\min}^{-C}
    n^{-\beta/(d+2\beta)} (D\vee\log n)^{\mathfrak p_{\beta,d}/2},
    \qquad (\bx,t)\in A_t .
\]
By \cref{eq:large-noise-U-over-Q-density-lower}, with $\tdown=t_0$, the polynomial
assumption $D\vee p_{\min}^{-1}\le n^{a_0}$, and
$\sigma_{t_0}^{-1}\le n^C$, the true score obeys
\[
    \sigma_t\|s^*(\bx,t)\|_\infty
    \le
    \sigma_t\|s^*(\bx,t)\|_2
    \le C(D\vee\log n)^{1/2},
    \qquad (\bx,t)\in A_t .
\]
Set $s_{\rm sm}^{\rm lb}:=\sigma_t^{-1}\widehat v_{\rm sm}^{\rm lb}$, with coordinatewise clipping at $C\sigma_t^{-1}(D\vee\log n)^{1/2}$.
On $A_t$, \cref{eq:small-noise-lower-scaled-pointwise,eq:small-noise-lower-tube-pointwise} follow.
Coordinatewise clipping is $1$-Lipschitz and fixes the target coordinatewise on $A_t$, so it
does not increase the tube approximation error.  The clipping construction gives
\cref{eq:small-noise-lower-approx-clip}.

\emph{Step 5: record size.}
The reciprocal, products, clipping, and time switch add only polynomial factors in $D$, $\Gamma_{\gM,q_{\rm geom}}$, $B_0$, $p_{\min}^{-1}$, and polylogarithmic factors in $n$.
The scalar reciprocal and clipping layers only affect $L+\log B$ by polylogarithmic factors and do not change the leading intrinsic width/sparsity term.
Hence \cref{eq:small-noise-lower-approx-depth-weight,eq:small-noise-lower-approx-size} follow.
\end{proof}

\clearpage

\section{Score Estimation on Compact Smooth Manifolds}
\label{sec:app:score-estimation}

This appendix proves the learned-score bounds used by the distribution estimation results in \cref{sec:app:dist-est}.
The clipped neural network class $\widetilde{\nn}$, the clipped score-estimation oracle inequality, and the covering-number evaluation are stated in the main text in \cref{sec:score-est}.
Here we apply those tools to the density-lower-bound small-noise and large-noise branches.
On a large-noise slab $I_k=[t_{k-1},t_k]$, the tangent-cell class uses cells of radius
$r_k=c_0\sigma_{t_{k-1}}\wedge r_{\star,n}$.  The cap $r_{\star,n}$, defined below, is the
accuracy-limited radius needed to make the fixed-order tangent-cell expansion uniform
also on order-one-noise slabs.
The dyadic sum of the $\sigma_{t_{k-1}}^{-d}$ terms is dominated by the first large-noise
slab; the additional $r_{\star,n}^{-d}$ contribution is tracked explicitly in the
oracle and $\sfW_1$ bounds.
The small-noise branch uses the projection--Laplace class under \cref{assump:manifold:density-lower}.
In all size bounds in this section, the factor $\Gamma_{\gM,q_{\rm geom}}^C$ may be
read as
\[
    C(1\vee S_{\gM})^{\mathfrak a_{\rm geom}}
    (1\vee\kappa^{-1})^{\mathfrak a_{\rm geom}}
    \mathcal K_{\gM,q_{\rm geom}}^{\mathfrak a_{\rm geom}}
\]
using the reach-explicit finite-geometry convention.

The section has one generic tool and two applications.  The oracle inequality
\cref{thrm:est:score:oracle} is the only statistical argument; the large-noise and
small-noise results then insert the deterministic comparison classes from
\cref{thm:large-noise-hd-score-global,cor:small-noise-density-lower-approx} and evaluate
their entropy sizes.  Readers interested only in the final rates can read
\cref{prop:large-noise-oracle-learned-score,cor:small-noise-density-lower-estimation}
first and return to the oracle proof for concentration details.

\subsection{Oracle Inequality for Score Estimation}\label{sec:app:est:score:oracle}

\begin{theorem}[Oracle inequality for clipped score estimation]\label{thrm:est:score:oracle}
    Let $n^{-\Ord(1)} \leq \tdown < \tup \leq n^{\Ord(1)}$, $D\le n^{\Ord(1)}$, and fix an arbitrary $A>0$.
    Let $\widehat{s}$ be an empirical risk minimizer over the neural network class
    $\nn_{\tdown}\equiv\widetilde{\nn}(L_{\tdown},\bW_{\tdown},S_{\tdown},B_{\tdown})$
    for the interval-restricted denoising loss $\ell_{[\tdown, \tup]}$, namely
    \[
        \widehat{s}\in
        \argmin_{s\in\nn_{\tdown}}
        \frac1n\sum_{i=1}^n\ell_{[\tdown, \tup]}(s,\bx^{(i)}).
    \]
    The same conclusion holds for any measurable approximate empirical minimizer with
    empirical excess at most $n^{-A}$, after increasing the negligible remainder.
    Then the expected excess risk satisfies
    \begin{align*}
        &
        \E_{\{\bx^{(i)}\}_{i=1}^n}
        \Bigl[
        \int_{\tdown}^{\tup}
          \E_{\bX_t}[\|\widehat{s}(\bX_t, t) - s^*(\bX_t,t)\|_2^2]
        \odt
        \Bigr]
        \\
        &\lesssim
        \inf_{s \in \nn_{\tdown}}
        \int_{\tdown}^{\tup}
          \E_{\bX_t}\bigl[
            \|s(\bX_t, t) - s^*(\bX_t, t)\|_2^2
          \bigr]
        \odt
        +
        \frac{
        (D\vee\log n)^2
        \int_{\tdown}^{\tup}\sigma_t^{-2}\odt
        \log\bigl(\gN(n^{-\Ord(1)}, \nn_{\tdown}, \|\cdot\|_{\infty}) \vee n\bigr)}{n}
        +n^{-A}.
    \end{align*}
\end{theorem}

\begin{proof}
\emph{Step 1: reduce the denoising risk to the marginal score risk.}
Write
\[
    \mathbb P_n f:=\frac1n\sum_{i=1}^n f(\bx^{(i)})
\]
for the empirical average over the training sample.  For $s\in\nn_{\tdown}$, set
\[
    \gL_{[\tdown,\tup]}(s):=
    \int_{\tdown}^{\tup}
      \E_{\bX_t}\|s(\bX_t,t)-s^*(\bX_t,t)\|_2^2\odt .
\]
The denoising-score identity, applied on the interval $[\tdown,\tup]$, gives
\[
    \gL_{[\tdown,\tup]}(s)
    =
    \E_{\bX_0}\{\ell_{[\tdown, \tup]}(s,\bX_0)
    -\ell_{[\tdown, \tup]}(s^*,\bX_0)\}.
\]
Since the last term is independent of $s$, the empirical minimizer of
$\mathbb P_n\ell_{[\tdown, \tup]}(s)$ is also an empirical minimizer of
$\mathbb P_n\{\ell_{[\tdown, \tup]}(s)-\ell_{[\tdown, \tup]}(s^*)\}$.

\emph{Step 2: record the envelope and variance bounds.}
The clipped definition of $\widetilde\nn$ and \cref{thrm:est:loss-bounded} imply
\[
    \sup_{s,s'\in\nn_{\tdown}}\sup_{\bx\in\R^D}
    |\ell_{[\tdown, \tup]}(s,\bx)-\ell_{[\tdown, \tup]}(s',\bx)|
    \le
    C(D\vee\log n)^2
    \int_{\tdown}^{\tup}\sigma_t^{-2}\odt .
\]
Moreover, applying \cref{thrm:est:variance-bound} to $s,s'\in\nn_{\tdown}$, using the
same clipped envelope, and then averaging over $\bX_0$, gives
\begin{align}
    &
    \E_{\bX_0}\bigl[
    \{\ell_{[\tdown, \tup]}(s,\bX_0)
    -\ell_{[\tdown, \tup]}(s',\bX_0)\}^2
    \bigr]
    \notag \\
    &\le
    C(D\vee\log n)^2
    \left(\int_{\tdown}^{\tup}\sigma_t^{-2}\odt\right)
    \int_{\tdown}^{\tup}
    \E_{\bX_t}\|s(\bX_t,t)-s'(\bX_t,t)\|_2^2\odt
    \notag\\
    &\le
    C(D\vee\log n)^2
    \left(\int_{\tdown}^{\tup}\sigma_t^{-2}\odt\right)
    \{\gL_{[\tdown,\tup]}(s)+\gL_{[\tdown,\tup]}(s')\}.
    \label{eq:oracle-clipped-bernstein-moment}
\end{align}
Here the last inequality is the elementary bound
$\|s-s'\|_2^2\le2\|s-s^*\|_2^2+2\|s'-s^*\|_2^2$.

\emph{Step 3: apply Bernstein on a finite cover.}
Let $\eta=n^{-A_0}$, with $A_0$ a sufficiently large fixed constant, and let
$\mathcal N_\eta$ be a minimal $\eta$-net of $\nn_{\tdown}$ in the supremum norm
after the input truncation and affine rescaling specified in \cref{sec:score-est}.
Because all functions in $\nn_{\tdown}$ are realized after this input truncation, this
is also a global $\eta$-net for the clipped functions on
$\R^D\times[\tdown,\tup]$.
For fixed $u,v\in\mathcal N_\eta$, Bernstein's inequality and
\cref{eq:oracle-clipped-bernstein-moment} give, with probability at least $1-e^{-r}$,
\begin{align*}
    &
    |(\mathbb P_n-\E_{\bX_0})
      \{\ell_{[\tdown, \tup]}(u)-\ell_{[\tdown, \tup]}(v)\}|
    \\
    &\le
    C\left[
      \sqrt{
      \frac{
      (D\vee\log n)^2
      \left(\int_{\tdown}^{\tup}\sigma_t^{-2}\odt\right)
      \{\gL_{[\tdown,\tup]}(u)+\gL_{[\tdown,\tup]}(v)\}r}{n}}
      +
      \frac{
      (D\vee\log n)^2
      \left(\int_{\tdown}^{\tup}\sigma_t^{-2}\odt\right)r}{n}
    \right].
\end{align*}
Taking a union bound over $\mathcal N_\eta^2$ and choosing
\[
    r=C\log\bigl(\gN(\eta,\nn_{\tdown},\|\cdot\|_\infty)\vee n\bigr)
\]
shows that the same display holds simultaneously for all $u,v\in\mathcal N_\eta$ with
probability at least $1-n^{-A_1}$, after increasing $C$.

\emph{Step 4: transfer from the cover to the clipped class.}
If $u$ is an $\eta$-net approximation of $s$, then the pointwise bound in
\cref{thrm:est:variance-bound}, together with the clipped envelope, yields
\[
    \sup_{\bx\in\R^D}
    |\ell_{[\tdown, \tup]}(s,\bx)-\ell_{[\tdown, \tup]}(u,\bx)|
    \le
    C\sqrt{
      D(\tup-\tdown)
      (D\vee\log n)^2
      \int_{\tdown}^{\tup}\sigma_t^{-2}\odt
    }\,\eta .
\]
Thus
\[
    \E_{\bX_0}|\ell_{[\tdown, \tup]}(s)-\ell_{[\tdown, \tup]}(u)|
    +
    \mathbb P_n|\ell_{[\tdown, \tup]}(s)-\ell_{[\tdown, \tup]}(u)|
    \le n^{-A_2}
\]
deterministically after taking $A_0$ large enough.  This uses
$D\le n^{\Ord(1)}$ and $n^{-\Ord(1)}\le\tdown<\tup\le n^{\Ord(1)}$, so that
$(D\vee\log n)^2\int_{\tdown}^{\tup}\sigma_t^{-2}\odt\le n^{\Ord(1)}$.
Since
\[
    \gL_{[\tdown,\tup]}(s)-\gL_{[\tdown,\tup]}(u)
    =
    \E_{\bX_0}\{\ell_{[\tdown, \tup]}(s,\bX_0)-\ell_{[\tdown, \tup]}(u,\bX_0)\},
\]
the same estimate also gives
$|\gL_{[\tdown,\tup]}(s)-\gL_{[\tdown,\tup]}(u)|\le n^{-A_2}$.
Hence the finite-cover Bernstein display applies to arbitrary $s,s'\in\nn_{\tdown}$,
with $\gL_{[\tdown,\tup]}(u)+\gL_{[\tdown,\tup]}(v)$ replaced by
$\gL_{[\tdown,\tup]}(s)+\gL_{[\tdown,\tup]}(s')+n^{-A_2}$, up to another
additive $n^{-A_2}$.  Enlarging $A_2$, this is the same display with the covering
number evaluated at $n^{-\Ord(1)}$.

\emph{Step 5: use the ERM basic inequality.}
Fix any $s_0\in\nn_{\tdown}$.  On the high-probability event from Steps~3--4,
\[
    \E_{\bX_0}\ell_{[\tdown, \tup]}(\widehat s,\bX_0)
    -
    \E_{\bX_0}\ell_{[\tdown, \tup]}(s_0,\bX_0)
    \le
    (\E_{\bX_0}-\mathbb P_n)
    \{\ell_{[\tdown, \tup]}(\widehat s)-\ell_{[\tdown, \tup]}(s_0)\}
    +n^{-A_2},
\]
because
$\mathbb P_n\ell_{[\tdown, \tup]}(\widehat s)
\le \mathbb P_n\ell_{[\tdown, \tup]}(s_0)$.  Therefore,
\begin{align*}
    \gL_{[\tdown,\tup]}(\widehat s)
    &\le
    \gL_{[\tdown,\tup]}(s_0)
    +
    C\Biggl[
      \sqrt{
      \frac{
      (D\vee\log n)^2
      \left(\int_{\tdown}^{\tup}\sigma_t^{-2}\odt\right)
      \{\gL_{[\tdown,\tup]}(\widehat s)+\gL_{[\tdown,\tup]}(s_0)\}r}{n}}
      +
      \frac{
      (D\vee\log n)^2
      \left(\int_{\tdown}^{\tup}\sigma_t^{-2}\odt\right)r}{n}
    \Biggr]
    \\
    &\qquad
    +n^{-A_2}.
\end{align*}
Using $\sqrt{ab}\le a/(4C)+Cb$ first with
$a=\gL_{[\tdown,\tup]}(\widehat s)$ and then with
$a=\gL_{[\tdown,\tup]}(s_0)$, we absorb the square-root term and obtain
\[
    \gL_{[\tdown,\tup]}(\widehat s)
    \le
    C\gL_{[\tdown,\tup]}(s_0)
    +
    C\frac{
      (D\vee\log n)^2
      \left(\int_{\tdown}^{\tup}\sigma_t^{-2}\odt\right)r}{n}
    +n^{-A_2}.
\]
Taking the infimum over $s_0\in\nn_{\tdown}$ gives the high-probability oracle bound.

\emph{Step 6: pass to expectation.}
On the complement of the high-probability event, the clipped envelope and
\cref{thrm:est:true-score-norm} give
\[
    \gL_{[\tdown,\tup]}(\widehat s)
    \le
    2\int_{\tdown}^{\tup}\E_{\bX_t}\|\widehat s(\bX_t,t)\|_2^2\odt
    +2\int_{\tdown}^{\tup}\E_{\bX_t}\|s^*(\bX_t,t)\|_2^2\odt
    \le
    C(D\vee\log n)^2
    \int_{\tdown}^{\tup}\sigma_t^{-2}\odt .
\]
Choosing $A_1$ large makes this exceptional contribution $\Ord(n^{-A})$.  We obtain
\[
    \E_{\{\bx^{(i)}\}_{i=1}^n}\gL_{[\tdown,\tup]}(\widehat s)
    \le
    C\inf_{s\in\nn_{\tdown}}\gL_{[\tdown,\tup]}(s)
    +
    C\frac{
      (D\vee\log n)^2
      \int_{\tdown}^{\tup}\sigma_t^{-2}\odt
      \log\bigl(\gN(n^{-\Ord(1)},\nn_{\tdown},\|\cdot\|_\infty)\vee n\bigr)
    }{n}
    +n^{-A},
\]
which is the claimed oracle inequality.
\end{proof}

\subsection{Time-Grid and Integral Conventions}

For the VP/OU schedules used in the rate theorem, $\sigma_t^2\asymp t$ for
$0<t\le1$ and $\sigma_t\asymp1$ for $t\ge1$ on the truncated interval.  The large-noise
grid is chosen so that
\begin{equation}
    [t_1,T]=\bigcup_{k=2}^K I_k,\qquad I_k=[t_{k-1},t_k],
    \qquad
    t_k=2t_{k-1}\ \text{while }t_k\le1,
    \qquad
    K\le C\log n,
    \label{eq:dyadic-grid-condition}
\end{equation}
where the dyadic portion is followed by $\Ord(\log n)$ order-one slabs of uniformly bounded
length from time $1$ to $T\asymp\log n$.  On every
large-noise slab,
\[
    c\,\sigma_{t_{k-1}}\le\sigma_t\le C\,\sigma_{t_{k-1}},
    \qquad
    \int_{I_k}\sigma_t^{-2}\odt\le C .
\]
For the small-noise interval and
$t_0\asymp n^{-2(\beta+1)/(d+2\beta)}$, $t_1\asymp n^{-2/(d+2\beta)}$,
\begin{equation}
    \int_{t_0}^{t_1}\sigma_t^{-2}\odt
    \le
    C\log(t_1/t_0)
    \le
    C\log n .
    \label{eq:small-noise-single-slab-integral}
\end{equation}

\subsection{Score Estimation on Large-Noise Regime}
\label{sec:app:large-noise-estimation}

We construct the deterministic comparison network inside the oracle proof below, so the
approximation loss, slab-local entropy scale, and stochastic term are checked for the same
large-noise class.

\begin{proposition}[Score estimation on large-noise regime]
\label{prop:large-noise-oracle-learned-score}
Let $I_k=[t_{k-1},t_k]\subset[t_1,T]$ be a large-noise slab satisfying
\cref{eq:dyadic-grid-condition}.
Assume \cref{assump:manifold:exact,assump:manifold:density,assump:manifold:density-lower}
and $D\vee p_{\min}^{-1}\le n^{a_0}$.  On the global grid used in the main theorem,
$\sigma_{t_{k-1}}^{-1}\le C\sigma_{t_1}^{-1}\lesssim n^{1/(d+2\beta)}$, so every
logarithmic range depending on $\sigma_{t_{k-1}}^{-1}$ contributes only polylogarithmic
factors in $n$, which are hidden in $\tilde{\Ord}$ after enlarging constants depending
only on $(\beta,d,a_0)$.
Let $\mathcal F_k^{\rm lg}$ be the clipped large-noise network class with the parameter
bounds displayed in \cref{eq:large-noise-learned-class-size} below and coordinatewise
clipping at $C\sigma_t^{-1}(D\vee\log n)^{1/2}$.
Recall $\mathfrak r_n$ and $q_\star$ from \cref{eq:def-main-rate}, and define the
accuracy-limited tangent-cell radius
\[
    r_{\star,n}
    :=
    c_\star
    \Gamma_{\gM,q_{\rm geom}}^{-C_\star}
    B_0^{-C_\star}
    p_{\min}^{C_\star}
    (D\vee\log n)^{-C_\star}
    \mathfrak r_n^{1/q_\star},
\]
where $C_\star$ is sufficiently large and $c_\star$ sufficiently small, depending only
on $(\beta,d)$ and the fixed numerical constants in
\cref{thm:large-noise-hd-score-global}.
In particular,
\[
    r_{\star,n}^{-d}
    \le
    C\Gamma_{\gM,q_{\rm geom}}^C B_0^C
    p_{\min}^{-C}
    (D\vee\log n)^C
    \mathfrak r_n^{-d/q_\star}.
\]
Let $\widehat s_k^{\rm lg}$ be an empirical risk minimizer over this class on $I_k$.
Then
\begin{align}
    &
    \E_{\{\bx^{(i)}\}_{i=1}^n}\int_{I_k}\E_{\bX_t}
    \|\widehat s_k^{\rm lg}(\bX_t,t)-s^*(\bX_t,t)\|_2^2\odt
    \notag \\
    &\le
    C\Gamma_{\gM,q_{\rm geom}}^C B_0^C
    p_{\min}^{-C}\left[
	      n^{-2(\beta+1)/(d+2\beta)}
	      +
	      \frac{\sigma_{t_{k-1}}^{-d}+r_{\star,n}^{-d}}{n}
	    \right]
    (D\vee\log n)^{\mathfrak p_{\beta,d}}
    +n^{-A}.
    \label{eq:large-noise-learned-score-final-integrated-rate}
\end{align}
Writing $L_k,\bW_k,S_k,B_k$ for the depth, width vector, sparsity, and weight bound of
the learned large-noise class, its parameters satisfy
\begin{equation}
\begin{aligned}
    L_k
    &\le
    C\Gamma_{\gM,q_{\rm geom}}^C B_0^C
    (D\vee\log n)^{\mathfrak p_{\beta,d}},
    \\
    \log B_k
    &\le
    C\Gamma_{\gM,q_{\rm geom}}^C B_0^C
    (D\vee\log n)^{\mathfrak p_{\beta,d}},
    \\
    \|\bW_k\|_\infty
    &\le
    C\Gamma_{\gM,q_{\rm geom}}^C B_0^C
	    p_{\min}^{-C}
	    \left(\sigma_{t_{k-1}}^{-d}+r_{\star,n}^{-d}\right)
	    (D\vee\log n)^{\mathfrak p_{\beta,d}},
    \\
    S_k
    &\le
	    C\Gamma_{\gM,q_{\rm geom}}^C B_0^C
	    p_{\min}^{-C}
	    \left(\sigma_{t_{k-1}}^{-d}+r_{\star,n}^{-d}\right)
	    (D\vee\log n)^{\mathfrak p_{\beta,d}} .
\end{aligned}
\label{eq:large-noise-learned-class-size}
\end{equation}
\end{proposition}

\begin{proof}
\emph{Step 1: construct a comparison network in the slab-local class.}
Apply \cref{thm:large-noise-hd-score-global} on $I_k$, with the theorem's local
time-window notation set to $\tdown=t_{k-1}$ and $\tup=t_k$.
Set
\[
    r_k:=c_0\sigma_{t_{k-1}}\wedge r_{\star,n},
\]
where $c_0>0$ is a sufficiently small numerical constant.
Use theorem parameters
$\varepsilon_{\rm cell}=r_k^\beta$ and
$\varepsilon_{\rm sc}=n^{-(\beta+1)/(d+2\beta)}(D\vee\log n)^{\mathfrak p_{\beta,d}/2}/\sqrt D$.
The constant $c_0$ is chosen so that
$\sigma_t\ge c_{\rm lg}r_k$ throughout $I_k$, using the slab comparability in
\cref{eq:dyadic-grid-condition}.
The polynomial bounds on $D$, $p_{\min}^{-1}$, and the automatic grid bound on
$\sigma_{t_{k-1}}^{-1}$ imply
\[
    1+R_{\rm lb}\le (D\vee\log n)^{2(q_{\rm geom}+2)+C},
    \qquad
    H_{\rm lg}^{\rm lb}\le
    \Gamma_{\gM,q_{\rm geom}}^C(D\vee\log n)^{2(q_{\rm geom}+2)+C},
    \qquad
    \sqrt D\le D\vee\log n .
\]
By the displayed bound on $H_{\rm lg}^{\rm lb}$ and the definition of $r_{\star,n}$,
the phase-smallness part of \cref{eq:large-noise-density-lower-geometry-condition} holds:
\[
    \Gamma_{\gM,q_{\rm geom}}r_k(H_{\rm lg}^{\rm lb})^{1/2}
    \le
    C\Gamma_{\gM,q_{\rm geom}}^C
    r_{\star,n}(D\vee\log n)^C
    \le c_{\rm ph,cell}.
\]
The same choice of $C_\star$ and $c_\star$ gives the high-order geometry remainder bound
\[
    \Gamma_{\gM,q_{\rm geom}}^C
    r_k^{q_{\rm geom}-2}
    (H_{\rm lg}^{\rm lb})^C
    \le
    c\,\frac{n^{-(\beta+1)/(d+2\beta)}
    (D\vee\log n)^{\mathfrak p_{\beta,d}/2}}{\sqrt D(1+R_{\rm lb})}.
\]
Here we used the polynomial bound on $R_{\rm lb}$ and the fact that
$r_k\le r_{\star,n}$; this is the step that prevents the order-one-noise slabs from
violating the fixed-order Taylor remainder condition.
Bounded $n$ are absorbed by increasing constants.
Thus the geometry condition in \cref{eq:large-noise-density-lower-geometry-condition} is
satisfied, and \cref{eq:large-noise-hd-score-error} gives a network
$s_{k,0}^{\rm lg}$ such that, on the tube $A_t$ defined in \cref{eq:def:main-near-manifold-tube},
\[
    \|s_{k,0}^{\rm lg}(\bx,t)-s^*(\bx,t)\|_2^2
    \le
    C\sigma_t^{-2}n^{-2(\beta+1)/(d+2\beta)}
    (D\vee\log n)^{\mathfrak p_{\beta,d}} .
\]
After adding coordinatewise clipping at
$C\sigma_t^{-1}(D\vee\log n)^{1/2}$, and enlarging constants if needed,
we obtain $s_k^{\rm lg}\in\mathcal F_k^{\rm lg}$ with the same tube error and with
\[
    \sup_{\bx\in\R^D}\|s_k^{\rm lg}(\bx,t)\|_\infty
    \le
    C\sigma_t^{-1}(D\vee\log n)^{1/2} .
\]
Indeed, on $A_t$, \cref{eq:large-noise-U-over-Q-density-lower} and the preceding
polynomial assumptions on $D,p_{\min}^{-1}$, and $\sigma_{t_{k-1}}^{-1}$ imply
\[
    \sigma_t\|s^*(\bx,t)\|_\infty
    \le
    \sigma_t\|s^*(\bx,t)\|_2
    \le C(D\vee\log n)^{1/2}.
\]
Coordinatewise clipping at this radius is $1$-Lipschitz and fixes the target value
coordinatewise, so it cannot increase the tube approximation error.
Its size is bounded by \cref{eq:large-noise-hd-size-separated}.  The identity
\[
    (r_k^\beta)^{-d/\beta}
    =
    r_k^{-d}
    \le
    C\left(\sigma_{t_{k-1}}^{-d}+r_{\star,n}^{-d}\right),
\]
and the absorption of $H_{\rm lg}^{\rm lb}$ and $R_{\rm lb}$ into
$(D\vee\log n)^{\mathfrak p_{\beta,d}}$ shows that the resulting class may be chosen with
the parameter bounds displayed in
\cref{eq:large-noise-learned-class-size}.

\emph{Step 2: integrate the comparison loss.}
On $A_t^c$, apply \cref{lem:approx:score:off-manifold} with
$\ell_{\rm off}=D\vee\log n$; the clipping bound for $s_k^{\rm lg}$ is then within
the required coordinatewise envelope.  Thus, for any prescribed fixed $A'>0$, after
increasing $C_\varrho$,
\[
    \int_{A_t^c}
      \|s_k^{\rm lg}(\bx,t)-s^*(\bx,t)\|_2^2p_t(\bx)\od\bx
    \le
    C\sigma_t^{-2}(D\vee\log n)^{\mathfrak p_{\beta,d}}n^{-A'} .
\]
Taking $A'$ large enough and using the polynomial lower bound on $\sigma_t$, the
off-tube contribution is absorbed into the tube rate.  Therefore
\[
    \int_{I_k}\E_{\bX_t}
    \|s_k^{\rm lg}(\bX_t,t)-s^*(\bX_t,t)\|_2^2\odt
    \le
    C\Gamma_{\gM,q_{\rm geom}}^C B_0^C
    p_{\min}^{-C}
    n^{-2(\beta+1)/(d+2\beta)}(D\vee\log n)^{\mathfrak p_{\beta,d}},
\]
where we used $\int_{I_k}\sigma_t^{-2}\odt\le C_{\rm grid}$.  Hence
\[
    \inf_{s\in\mathcal F_k^{\rm lg}}
    \int_{I_k}\E_{\bX_t}\|s(\bX_t,t)-s^*(\bX_t,t)\|_2^2\odt
    \le
    C\Gamma_{\gM,q_{\rm geom}}^C B_0^C
    p_{\min}^{-C}
    n^{-2(\beta+1)/(d+2\beta)}(D\vee\log n)^{\mathfrak p_{\beta,d}} .
\]

\emph{Step 3: evaluate the slab-local entropy.}
Apply the clipped score-estimation oracle \cref{thrm:est:score:oracle} on the slab
$I_k$.
By \cref{eq:large-noise-learned-class-size}, the large-noise class satisfies the entropy
evaluation \cref{eq:score-class-entropy-evaluation} with
\[
    M_k:=\sigma_{t_{k-1}}^{-d}+r_{\star,n}^{-d}.
\]
Indeed, $L_k+\log B_k$ is polylogarithmic up to the displayed polynomial prefactors, while
$S_k+\|\bW_k\|_\infty$ is bounded by $M_k(D\vee\log n)^{\mathfrak p_{\beta,d}}$ times those
same fixed polynomial prefactors.
Therefore
\[
    \log\bigl(\gN(n^{-A},\mathcal F_k^{\rm lg},\|\cdot\|_\infty)\vee n\bigr)
    \le
    C\Gamma_{\gM,q_{\rm geom}}^C B_0^C
    p_{\min}^{-C}
    \left(\sigma_{t_{k-1}}^{-d}+r_{\star,n}^{-d}\right)
    (D\vee\log n)^{\mathfrak p_{\beta,d}} .
\]
Substituting this covering bound into the oracle inequality, using
$\int_{I_k}\sigma_t^{-2}\odt\le C_{\rm grid}$, and absorbing the oracle envelope factor
$(D\vee\log n)^2$ into $\mathfrak p_{\beta,d}$, gives the stochastic contribution
\[
    C\Gamma_{\gM,q_{\rm geom}}^C B_0^C
    p_{\min}^{-C}
    \frac{\sigma_{t_{k-1}}^{-d}+r_{\star,n}^{-d}}{n}
    (D\vee\log n)^{\mathfrak p_{\beta,d}},
\]
with all manifold and smoothness prefactors displayed.
Combining this with the comparison loss from Step~2 proves
\cref{eq:large-noise-learned-score-final-integrated-rate}.
\end{proof}

\subsection{Score Estimation on Small-Noise Regime}
\label{sec:app:small-noise-estimation-density-lower}

We also fold the deterministic small-noise comparison loss into the oracle proof, so the
comparison function and the learned class are checked in the same place.

\begin{theorem}[Score estimation on small-noise regime]
\label{cor:small-noise-density-lower-estimation}
Assume \cref{assump:manifold:exact,assump:manifold:density,assump:manifold:density-lower} and $D\vee p_{\min}^{-1}\le n^{a_0}$.
Assume also the small-noise reach and denominator-resolution conditions
\cref{eq:small-noise-final-reach-condition,eq:small-noise-denominator-resolution-condition}.
Let $\mathcal F_{\rm sm}^{\rm lb}$ be the clipped small-noise network class with clipping
radius $C\sigma_t^{-1}(D\vee\log n)^{1/2}$ and with the size bounds in
\cref{eq:small-noise-lower-approx-depth-weight,eq:small-noise-lower-approx-size}, and let
$\widehat s_{\rm sm}^{\rm lb}$ be an empirical risk minimizer over
$\mathcal F_{\rm sm}^{\rm lb}$ on $[t_0,t_1]$.
Then
\begin{equation}
\begin{aligned}
    &\E_{\{\bx^{(i)}\}_{i=1}^n}\int_{t_0}^{t_1}
      \E_{\bX_t}
      \|\widehat s_{\rm sm}^{\rm lb}(\bX_t,t)-s^*(\bX_t,t)\|_2^2\odt
    \\
    &\quad\le
    C\Gamma_{\gM,q_{\rm geom}}^C B_0^C
    p_{\min}^{-C} n^{-2\beta/(d+2\beta)}
    (D\vee\log n)^{\mathfrak p_{\beta,d}}
    \\&\qquad+
    C\frac{
      (D\vee\log n)^C
      \int_{t_0}^{t_1}\sigma_t^{-2}\odt
      \log\bigl(\gN(n^{-A},\mathcal F_{\rm sm}^{\rm lb},\|\cdot\|_\infty)\vee n\bigr)
    }{n}
    +n^{-A}.
\end{aligned}
\label{eq:small-noise-density-lower-oracle-ineq}
\end{equation}
In particular, using \cref{eq:score-class-entropy-evaluation} with $M_n=n^{d/(d+2\beta)}$,
\begin{equation}
\begin{aligned}
    &
    \E_{\{\bx^{(i)}\}_{i=1}^n}\int_{t_0}^{t_1}
      \E_{\bX_t}
    \|\widehat s_{\rm sm}^{\rm lb}(\bX_t,t)-s^*(\bX_t,t)\|_2^2\odt
    \\
    &\le
    C\Gamma_{\gM,q_{\rm geom}}^C B_0^C
    p_{\min}^{-C} n^{-2\beta/(d+2\beta)}
    (D\vee\log n)^{\mathfrak p_{\beta,d}}+n^{-A}.
\end{aligned}
    \label{eq:small-noise-density-lower-estimation-rate}
\end{equation}
\end{theorem}

\begin{proof}
The comparison network $s_{\rm sm}^{\rm lb}$ belongs to $\mathcal F_{\rm sm}^{\rm lb}$, after
the same coordinatewise clipping at radius $C\sigma_t^{-1}(D\vee\log n)^{1/2}$.
Let $A_t$ be the tube defined in \cref{eq:def:main-near-manifold-tube}.
On $A_t$, \cref{eq:small-noise-lower-tube-pointwise} gives the following bound after
inserting the harmless factors $\Gamma_{\gM,q_{\rm geom}}^C B_0^C\ge1$:
\[
    \|s_{\rm sm}^{\rm lb}(\bx,t)-s^*(\bx,t)\|_2^2
    \le
    C\Gamma_{\gM,q_{\rm geom}}^C B_0^C
    p_{\min}^{-C}
    n^{-2\beta/(d+2\beta)}(D\vee\log n)^{\mathfrak p_{\beta,d}}\sigma_t^{-2}.
\]
The single-slab integral bound \cref{eq:small-noise-single-slab-integral} gives
$\int_{t_0}^{t_1}\sigma_t^{-2}\odt\le C\log n\le C(D\vee\log n)^C$, so the
tube contribution is bounded by
\[
    C\Gamma_{\gM,q_{\rm geom}}^C B_0^C
    p_{\min}^{-C}
    n^{-2\beta/(d+2\beta)}(D\vee\log n)^{\mathfrak p_{\beta,d}} .
\]
On $A_t^c$, apply \cref{lem:approx:score:off-manifold} with
$\ell_{\rm off}=D\vee\log n$; the clipping bound
\cref{eq:small-noise-lower-approx-clip} satisfies the required envelope.  After increasing
$C_\varrho$, this gives $n^{-A}$.
Thus
\[
    \inf_{s\in\mathcal F_{\rm sm}^{\rm lb}}
    \int_{t_0}^{t_1}
      \E_{\bX_t}\|s(\bX_t,t)-s^*(\bX_t,t)\|_2^2\odt
    \le
    C\Gamma_{\gM,q_{\rm geom}}^C B_0^C
    p_{\min}^{-C}
    n^{-2\beta/(d+2\beta)} (D\vee\log n)^{\mathfrak p_{\beta,d}}+n^{-A}.
\]
Applying the clipped oracle inequality \cref{thrm:est:score:oracle} to the class $\mathcal F_{\rm sm}^{\rm lb}$ gives \cref{eq:small-noise-density-lower-oracle-ineq}.
The entropy bound follows from
\cref{eq:small-noise-lower-approx-depth-weight,eq:small-noise-lower-approx-size,eq:score-class-entropy-evaluation};
it is bounded by
\[
    C\Gamma_{\gM,q_{\rm geom}}^C B_0^C
    p_{\min}^{-C}
    n^{d/(d+2\beta)}(D\vee\log n)^{\mathfrak p_{\beta,d}},
\]
with all manifold and smoothness prefactors displayed.
Since
\[
    \frac{n^{d/(d+2\beta)}}{n}
    =
    n^{-1+d/(d+2\beta)}
    =
    n^{-2\beta/(d+2\beta)}
\]
and $\int_{t_0}^{t_1}\sigma_t^{-2}\odt\le C(D\vee\log n)^C$, the stochastic term is absorbed into the right side of \cref{eq:small-noise-density-lower-estimation-rate}.
\end{proof}

\clearpage
\section{$\sfW_1$ Distribution Estimation}
\label{sec:app:dist-est}

This appendix section contains the formal $\sfW_1$ distribution-estimation results.
The main-paper summaries are stated in \cref{sec:main:w1-from-slabwise-score}; the generic score-estimation tools are stated in \cref{sec:score-est}, and the formal learned-score inputs are proved in \cref{sec:app:score-estimation}.
All displayed $\sfW_1$ bounds below use the branchwise density-lower-bound learned-score
inputs and the off-tube localization bounds.

Whenever an appendix estimate invokes the regular VP schedule, it refers to the
standard VP/OU schedule facts that, for universal schedule constants independent
of $(n,D)$,
\begin{equation}
    c_\sigma t \le \sigma_t^2 \le C_\sigma t,\quad
    1-m_t \le C_m t
    \qquad(0\le t\le1),
    \qquad
    m_t\le C_{\rm mix}e^{-c_{\rm mix}t}\quad(t\ge0).
    \label{eq:regular-vp-schedule}
\end{equation}
In particular, on every early interval $[0,t_1]$ used below, $m_t$ is bounded
below by a positive schedule constant $\underline m$ once $t_1\le1$.

The main statement is \cref{cor:w1-rate-density-lower}.  Its proof is a bookkeeping
argument with four terms: early stopping, terminal Gaussian initialization, the
small-noise score error, and the dyadic sum of large-noise score errors.  The only
distribution-level input is the slabwise perturbation lemma
\cref{thrm:est:dist:W1-SM}; all score-estimation inputs come from
\cref{prop:large-noise-oracle-learned-score,cor:small-noise-density-lower-estimation}.

\subsection{Controlled-regime checklist}
\label{sec:app:controlled-regime-checklist}

This subsection expands the non-asymptotic regime conditions invoked in
\cref{thm:main-simplified-w1}.
Let $H_n:=D\vee\log n$ and let $q_{\rm geom}$ be as in \cref{eq:def:qgeom-main}.  The conditions split into
the following visually distinct groups.
\begin{center}
{\small
\setlength{\tabcolsep}{4pt}
\resizebox{\linewidth}{!}{
\begin{tabular}{p{0.22\linewidth}|p{0.64\linewidth}}
Condition group & Meaning \\
\hline
Schedule &
The regular VP schedule convention \eqref{eq:regular-vp-schedule}, so that
$\sigma_t^2\asymp t$ on $[0,1]$, $m_t\asymp1$ on the early-stopping regime, and
$m_T$ decays exponentially on the terminal regime \\
\hline
Geometry &
$P_0$ is supported on a compact smooth $d$-manifold $\gM\subset[0,1]^D$ with reach
$\kappa>0$, volume $S_{\gM}$, and finite envelope $\mathcal K_{\gM,q_{\rm geom}}$ \\
\hline
Density &
$p_0\in\gH^\beta(\gM,B_0)$ and $p_0\ge p_{\min}>0$ \\
\hline
Polynomial control &
$D\vee p_{\min}^{-1}\le n^{a_0}$ and the displayed geometry/density factors are bounded
uniformly in $n$, or by fixed powers of $D$ in growing-ambient families \\
\hline
Time grid &
$t_0\asymp n^{-2(\beta+1)/(d+2\beta)}$, $t_1\asymp n^{-2/(d+2\beta)}$, and
$T\asymp\log n$ \\
\hline
Small-noise reach &
$3c_\star\underline m^{-1}\sigma_{t_1}\sqrt{H_n}\le\eta_\Pi/2$, so the forward tube is
inside the region where nearest projection is unique \\
\hline
Denominator resolution &
$C_{\rm den,sm}\Gamma_{\gM,q_{\rm geom}}^C B_0^C
p_{\min}^{-1}n^{-\beta/(d+2\beta)}
H_n^{\mathfrak p_{\beta,d}/2-1/2}\le1$, so the de-Gaussianized denominator can be
approximated before taking a reciprocal \\
\hline
Estimator &
Clipped ERM over the slab-dependent ReLU classes from
\cref{sec:score-est,sec:app:score-estimation}
\end{tabular}
}
}
\end{center}

The small-noise reach and denominator-resolution conditions are sample-size conditions,
not additional smoothness assumptions.  For fixed $D$, fixed positive reach, fixed
$p_{\min}$, and fixed finite-order geometry, both hold for all sufficiently large $n$.
For growing $D$ or deteriorating geometry, they state explicitly how much ambient growth,
small reach, or small density lower bound the proof can tolerate.

\subsection{$\sfW_1$ distribution estimation}

We use the same slabwise partition throughout this subsection: after choosing
$t_0,t_1,T$, write
\[
    [t_0,T]=[t_0,t_1]\cup\bigcup_{k=2}^K I_k,
    \qquad
    I_k=[t_{k-1},t_k],
\]
where the large-noise slabs satisfy \cref{eq:dyadic-grid-condition}; in particular, this is
the dyadic grid $t_k=2t_{k-1}$ while $t_k\le1$, followed by order-one slabs up to
$T$.

\paragraph{Audit checkpoint: rate assembly.}
\Cref{cor:w1-rate-density-lower} is where the two score-estimation regimes become one
distribution rate.  The small-noise score risk is larger, but the interval length
contributes the factor $\sqrt{t_1}\asymp n^{-1/(d+2\beta)}$.  The large-noise stochastic
terms are summed over the dyadic grid, where
$\sum_k\sigma_{t_{k-1}}^{-d}$ is dominated by the first large-noise slab, while the
accuracy-cap contribution $\mathfrak r_n^{-d/q_\star}$, with $\mathfrak r_n$ and
$q_\star$ from \cref{eq:def-main-rate}, is carried explicitly through the $\sfW_1$
perturbation sum.  These checks are the only rate-assembly points after the branchwise learned-score bounds are available. 

\textbf{Rate Decomposition.}
\Cref{tab:master_proof_ledger} summarizes the early stopping, score-estimation, and terminal terms.
Throughout this section we use the global notation $\mathfrak r_n$ and $q_\star$
from \cref{eq:def-main-rate}.

\begin{table}[htp]
\caption{Rate decomposition for the $\sfW_1$ generative bound.}
\label{tab:master_proof_ledger}
{\small
\setlength{\tabcolsep}{4pt}
\resizebox{\linewidth}{!}{
\begin{tabular}{p{0.20\linewidth}|p{0.16\linewidth}|p{0.16\linewidth}|p{0.18\linewidth}|p{0.20\linewidth}}
\hline
\textbf{Temporal Regime} & \textbf{Deterministic Error} & \textbf{Intrinsic Capacity ($M$)} & \textbf{Stochastic Penalty} & \textbf{Final $\sfW_1$ Contribution} \\
\hline
Early Stopping $[0, t_0]$ & --- & --- & --- & $\tilde{\Ord}(n^{-(\beta+1)/(d+2\beta)})$ \\
\hline
Small Noise $[t_0, t_1]$ & $\tilde{\Ord}(n^{-2\beta/(d+2\beta)})$ & $n^{d/(d+2\beta)}$ & $n^{-2\beta/(d+2\beta)}$ & $\sqrt{t_1}\,n^{-\beta/(d+2\beta)} \asymp n^{-(\beta+1)/(d+2\beta)}$ \\
\hline
Large Noise $[t_1, T]$ \newline (evaluated per slab $I_k$) & $\tilde{\Ord}(\mathfrak r_n^2)$ & $\sigma_{t_{k-1}}^{-d}+\mathfrak r_n^{-d/q_\star}$ & $(\sigma_{t_{k-1}}^{-d}+\mathfrak r_n^{-d/q_\star})/n$ & $\mathfrak r_n+n^{-1/2}\mathfrak r_n^{-d/(2q_\star)}$ \newline (following dyadic summation) \\
\hline
Terminal Init $[T, \infty]$ & --- & --- & --- & Negligible for $T\asymp\log n$ \\
\hline
\end{tabular}
}
}
\end{table}

\begin{corollary}[$\sfW_1$ distribution estimation]
\label{cor:w1-rate-density-lower}  
Assume \cref{assump:manifold:exact,assump:manifold:density,assump:manifold:density-lower},
$D\vee p_{\min}^{-1}\le n^{a_0}$, and the two small-noise conditions
\cref{eq:small-noise-final-reach-condition,eq:small-noise-denominator-resolution-condition}.
Let $t_0\asymp n^{-2(\beta+1)/(d+2\beta)}$, $t_1\asymp n^{-2/(d+2\beta)}$,
$T\asymp\log n$, and use $\mathfrak r_n,q_\star$ from \cref{eq:def-main-rate}.
Let $\widehat s^{\rm lb}$ be the switched clipped ERM obtained by using
\cref{cor:small-noise-density-lower-estimation} on $[t_0,t_1]$ and
\cref{prop:large-noise-oracle-learned-score} on the large-noise slabs $I_k$,
$k=2,\dots,K$.  Let $\widehat P_{t_0}^{\rm lb}$ be the generated law driven by
$\widehat s^{\rm lb}$.  Then there is a constant $C<\infty$, independent of $(n,D)$,
such that
\begin{equation}
\begin{aligned}
    \E\,\sfW_1(P_0,\widehat P_{t_0}^{\rm lb})
    &\le
    C\sqrt{D t_0}
    +
    C\sqrt D\,\exp(-c_{\rm mix}T)
    \\
    &\quad+
    C\Gamma_{\gM,q_{\rm geom}}^C B_0^C
    p_{\min}^{-C}
    \Bigl(\mathfrak r_n+n^{-1/2}+n^{-1/2}\mathfrak r_n^{-d/(2q_\star)}\Bigr)
    (D\vee\log n)^{\mathfrak p_{\beta,d}}.
\end{aligned}
    \label{eq:adaptive-w1-rate-density-lower}
\end{equation}
If $T=c_T\log n$ with $c_T$ large enough, depending only on the fixed constants of the
setup, the terminal term is absorbed by the third term in
\eqref{eq:adaptive-w1-rate-density-lower}.  Since $t_0\asymp\mathfrak r_n^2$ and
$\mathfrak p_{\beta,d}\ge1/2$, the early-stopping term is also absorbed by that term.
If $d>2$ and $q_\star\ge d(\beta+1)/(d-2)$, then
$n^{-1/2}+n^{-1/2}\mathfrak r_n^{-d/(2q_\star)}\le C\mathfrak r_n$, and hence
\[
    \E\,\sfW_1(P_0,\widehat P_{t_0}^{\rm lb})
    \le
    C\Gamma_{\gM,q_{\rm geom}}^C B_0^C
    p_{\min}^{-C}
    \mathfrak r_n
    (D\vee\log n)^{\mathfrak p_{\beta,d}} .
\]
Writing $L,\bW,S,B$ for the depth, width vector, sparsity, and weight bound of the
switched learned score, the assembled class can be chosen so that
\[
    L\vee \log B
    \le
    C\mathfrak G_{\gM,q_{\rm geom}}B_0^Cp_{\min}^{-C}
    (D\vee\log n)^{\mathfrak p_{\beta,d}},
\]
\[
    \|\bW\|_\infty\vee S
    \le
    C\mathfrak G_{\gM,q_{\rm geom}}B_0^Cp_{\min}^{-C}
    n^{d/(d+2\beta)}
    (D\vee\log n)^{\mathfrak p_{\beta,d}},
\]
where
\[
    \mathfrak G_{\gM,q_{\rm geom}}
    :=
    (1\vee S_{\gM})^{\mathfrak a_{\rm geom}}
    (1\vee\kappa^{-1})^{\mathfrak a_{\rm geom}}
    \mathcal K_{\gM,q_{\rm geom}}^{\mathfrak a_{\rm geom}} .
\]
One may take
\[
    \mathfrak p_{\beta,d}
    =
    C
    +
    \max\left\{
      8(q_{\rm geom}+2),
      \left\lceil
      \floor{\beta}+1+\frac{\floor{\beta}\,d}{\beta-\floor{\beta}}
      +\frac{d}{\beta\vee1}+\frac{d}{2}+6
      \right\rceil
    \right\}.
\]
The universal $C$ in this exponent is independent of
$\beta,d,D,n,p_{\min}$, and the numerical geometry constants; under the above
$d>2$ finite-order choice, $\mathfrak p_{\beta,d}=\Ord_\beta(d)$.
\end{corollary}

\begin{proof}
\emph{Step 1: decompose the distribution error.}
Use the interpolation laws from \cref{thrm:est:dist:W1-SM} on the partition
$[t_0,T]=[t_0,t_1]\cup\bigcup_{k=2}^K I_k$.  Define
$\sfW_1^{\rm sm,lb}:=\sfW_1(\overline P_{t_0}^{(0)},\overline P_{t_0}^{(1)})$ for the
single small-noise replacement and
$\sfW_1^{(k),{\rm lg,lb}}
:=\sfW_1(\overline P_{t_0}^{(k-1)},\overline P_{t_0}^{(k)})$ for the large-noise slab
$I_k$.  With this notation,
$\overline P_{t_0}^{(K)}$ is the exact-terminal reverse law driven by the switched
learned score $\widehat s^{\rm lb}$.
By the triangle inequality along the reverse-time telescoping construction,
\[
    \E\,\sfW_1(P_0,\widehat P_{t_0}^{\rm lb})
    \le
    \sfW_1(P_0,P_{t_0})
    +
    \E\sfW_1(\overline P_{t_0},\widehat P_{t_0})
    +
    \E\,\sfW_1^{\rm sm,lb}
    +
    \sum_{k:\,I_k\subset[t_1,T]}
      \E\,\sfW_1^{(k),{\rm lg,lb}} .
\]
The early-stopping and initialization bounds follow from
\cref{thrm:est:dist:early-stop,thrm:est:dist:Gau}:
\[
    \sfW_1(P_0,P_{t_0})\le C\sqrt{D t_0},
    \qquad
    \E\sfW_1(\overline P_{t_0},\widehat P_{t_0})\le C\sqrt D\,\exp(-c_{\rm mix}T).
\]

\emph{Step 2: bound the small-noise contribution.}
By the clipping radii in
\cref{cor:small-noise-density-lower-estimation,prop:large-noise-oracle-learned-score},
and by the bounded ReLU time switches used to assemble the branch estimators,
\[
    \sup_{\bx\in\R^D}\|\widehat s^{\rm lb}(\bx,t)\|_\infty
    \le
    C\sigma_t^{-1}(D\vee\log n)^{1/2},
    \qquad t\in[t_0,T].
\]
This verifies the clipped-score hypothesis of \cref{thrm:est:dist:W1-SM} with
$a_{\rm clip}=1/2$.
For the small-noise slab, apply \cref{thrm:est:dist:W1-SM}.
Together with \cref{eq:small-noise-density-lower-estimation-rate}, $t_1\asymp
n^{-2/(d+2\beta)}$, and choosing the negligible exponent $A$ large, this gives
\begin{equation}
\begin{aligned}
    \E\,\sfW_1^{\rm sm,lb}
    &\le
    C(D\vee\log n)^{a_{\rm W}}
    n^{-1/(d+2\beta)}\sqrt{\log n}
    \Bigl(
      \E\int_{t_0}^{t_1}
      \E_{\bX_t}\|\widehat s^{\rm lb}(\bX_t,t)-s^*(\bX_t,t)\|_2^2\odt
    \Bigr)^{1/2}
    \notag \\
    &\qquad+
    C\Gamma_{\gM,q_{\rm geom}}^C B_0^C
    p_{\min}^{-C}
    (D\vee\log n)^{a_{\rm W}}n^{-(\beta+1)/(d+2\beta)}
    \\
    &\le
    C\Gamma_{\gM,q_{\rm geom}}^C B_0^C
    p_{\min}^{-C}
    n^{-(\beta+1)/(d+2\beta)}
    (D\vee\log n)^{\mathfrak p_{\beta,d}}.
\end{aligned}
\label{eq:w1-proof-small-noise-density-lower}
\end{equation}
Here we used
\[
    n^{-1/(d+2\beta)}\,n^{-\beta/(d+2\beta)}
    =
    n^{-(\beta+1)/(d+2\beta)}.
\]

\emph{Step 3: bound the large-noise contributions.}
On the large-noise slabs, \cref{thrm:est:dist:W1-SM} and
\cref{eq:large-noise-learned-score-final-integrated-rate}, again with the negligible
exponent chosen large, give
\begin{align*}
    \E\,\sfW_1^{(k),{\rm lg,lb}}
    &\le
    C\Gamma_{\gM,q_{\rm geom}}^C B_0^C
    p_{\min}^{-C}(D\vee\log n)^{a_{\rm W}}
    \sqrt{t_k\log n}\,
    \Bigl[
	      n^{-(\beta+1)/(d+2\beta)}
	      \\&\qquad\qquad\qquad+
	      n^{-1/2}\sigma_{t_{k-1}}^{-d/2}
	      +
	      n^{-1/2}\mathfrak r_n^{-d/(2q_\star)}
    \Bigr]
    (D\vee\log n)^{\mathfrak p_{\beta,d}}
    \\
    &\qquad+
    C\Gamma_{\gM,q_{\rm geom}}^C B_0^C
    p_{\min}^{-C}
    (D\vee\log n)^{a_{\rm W}}n^{-(\beta+1)/(d+2\beta)}
\end{align*}
for each large-noise slab $I_k$.
The first term sums to $n^{-(\beta+1)/(d+2\beta)}$ up to the displayed polynomial ambient factor because
$t_k\le T\lesssim\log n$ and $K\lesssim\log n$.
For the stochastic slab-size term, the dyadic grid $t_k=2t_{k-1}$ while $t_k\le1$,
the relation $\sigma_t^2\asymp t$ near zero, and
$\sigma_{t_1}\asymp n^{-1/(d+2\beta)}$ give
\[
    n^{-1/2}
    \sum_{k:\,t_k\le1}
      t_k^{1/2}\sigma_{t_{k-1}}^{-d/2}
    \le
    C n^{-1/2}\bigl(1+\sigma_{t_1}^{1-d/2}\bigr)
    (D\vee\log n)^{\mathfrak p_{\beta,d}}
    \le
    C\left(n^{-(\beta+1)/(d+2\beta)}+n^{-1/2}\right)
    (D\vee\log n)^{\mathfrak p_{\beta,d}} .
\]
The remaining order-one-time slabs have $\sigma_{t_{k-1}}\asymp1$, and contribute only
$n^{-1/2}(D\vee\log n)^{\mathfrak p_{\beta,d}}$.
The accuracy-cap term is independent of the slab index except for logarithmic factors;
since $K\lesssim\log n$ and $t_k\le T\lesssim\log n$, its large-noise contribution is
bounded by
\[
    C n^{-1/2}\mathfrak r_n^{-d/(2q_\star)}
    (D\vee\log n)^{\mathfrak p_{\beta,d}} .
\]
Hence
\[
    \sum_{k:\,I_k\subset[t_1,T]}
    \E\,\sfW_1^{(k),{\rm lg,lb}}
    \le
	    C\Gamma_{\gM,q_{\rm geom}}^C B_0^C
	    p_{\min}^{-C}
	    \left(n^{-(\beta+1)/(d+2\beta)}+n^{-1/2}
	    +n^{-1/2}\mathfrak r_n^{-d/(2q_\star)}\right)
	    (D\vee\log n)^{\mathfrak p_{\beta,d}}.
\]
Combining these terms proves \cref{eq:adaptive-w1-rate-density-lower}.

\emph{Step 4: record the network sizes.}
The small-noise learned class from \cref{cor:small-noise-density-lower-estimation} has
separate bounds
\[
    L_{\rm sm}
    \vee \log B_{\rm sm}
    \le
    C\Gamma_{\gM,q_{\rm geom}}^C B_0^C
    p_{\min}^{-C}(D\vee\log n)^{\mathfrak p_{\beta,d}},
\]
and
\[
    \|\bW_{\rm sm}\|_\infty
    \vee S_{\rm sm}
    \le
    C \Gamma_{\gM,q_{\rm geom}}^C B_0^C
    p_{\min}^{-C}
    n^{d/(d+2\beta)}(D\vee\log n)^{\mathfrak p_{\beta,d}} .
\]
For each large-noise slab, \cref{eq:large-noise-learned-class-size} gives
\[
\begin{aligned}
    L_k\vee\log B_k
    &\le
    C \Gamma_{\gM,q_{\rm geom}}^C B_0^C
    (D\vee\log n)^{\mathfrak p_{\beta,d}},
    \\
    \|\bW_k\|_\infty\vee S_k
    &\le
	    C \Gamma_{\gM,q_{\rm geom}}^C B_0^C
	    p_{\min}^{-C}
	    \left(\sigma_{t_{k-1}}^{-d}+\mathfrak r_n^{-d/q_\star}\right)
	    (D\vee\log n)^{\mathfrak p_{\beta,d}} .
\end{aligned}
\]
There are $K\lesssim\log n$ large-noise slabs, and the ReLU time switches add only
$\Ord(K)$ depth/sparsity and a logarithmic number of weights.
The dyadic grid gives
\[
	    \sum_{k=2}^K\sigma_{t_{k-1}}^{-d}
	    \le
	    C\sigma_{t_1}^{-d}
	    \asymp
	    n^{d/(d+2\beta)} .
\]
The additional accuracy-cap contribution satisfies
\[
    K \mathfrak r_n^{-d/q_\star}
    \le
    C n^{d/(d+2\beta)}(D\vee\log n)^C,
\]
because $q_\star>\beta+1$ by \cref{eq:def:qgeom-main}.
The subnetworks are assembled in parallel behind those switches.
Thus the depth satisfies
\[
    L
    \le
    \max\{L_{\rm sm},L_2,\dots,L_K\}+C K
    \le
    C
    (1\vee S_{\gM})^{\mathfrak a_{\rm geom}}
    (1\vee\kappa^{-1})^{\mathfrak a_{\rm geom}}
    \mathcal K_{\gM,q_{\rm geom}}^{\mathfrak a_{\rm geom}}
    B_0^C
    p_{\min}^{-C}
    (D\vee\log n)^{\mathfrak p_{\beta,d}},
\]
which is the stated depth bound.
The width satisfies
\[
    \|\bW\|_\infty
    \le
    \|\bW_{\rm sm}\|_\infty+\sum_{k=2}^K\|\bW_k\|_\infty+C K
    \le
    C
    (1\vee S_{\gM})^{\mathfrak a_{\rm geom}}
    (1\vee\kappa^{-1})^{\mathfrak a_{\rm geom}}
    \mathcal K_{\gM,q_{\rm geom}}^{\mathfrak a_{\rm geom}}
    B_0^C
    p_{\min}^{-C}
    n^{d/(d+2\beta)}
    (D\vee\log n)^{\mathfrak p_{\beta,d}},
\]
after absorbing $K\lesssim\log n$; this proves the stated width bound.
The sparsity satisfies
\[
    S
    \le
    S_{\rm sm}+\sum_{k=2}^K S_k+C K
    \le
    C
    (1\vee S_{\gM})^{\mathfrak a_{\rm geom}}
    (1\vee\kappa^{-1})^{\mathfrak a_{\rm geom}}
    \mathcal K_{\gM,q_{\rm geom}}^{\mathfrak a_{\rm geom}}
    B_0^C
    p_{\min}^{-C}
    n^{d/(d+2\beta)}
    (D\vee\log n)^{\mathfrak p_{\beta,d}},
\]
which is the stated sparsity bound.
Finally,
\begin{align*}
    \log B
    &\le
    \max\{\log B_{\rm sm},\log B_2,\dots,\log B_K\}+C\log K
    \\    
    &\le
    C
    (1\vee S_{\gM})^{\mathfrak a_{\rm geom}}
    (1\vee\kappa^{-1})^{\mathfrak a_{\rm geom}}
    \mathcal K_{\gM,q_{\rm geom}}^{\mathfrak a_{\rm geom}}
    B_0^C
    p_{\min}^{-C}
    (D\vee\log n)^{\mathfrak p_{\beta,d}},
\end{align*}
which is the stated log-weight bound.
\end{proof}

\begin{corollary}[Fixed-geometry sufficient ambient-growth regime]
\label{cor:fixed-geometry-ambient-growth}
Fix $\beta>0$ and $d\ge1$.  Work under the regular VP schedule convention \eqref{eq:regular-vp-schedule}.  Suppose \cref{assump:manifold:exact,assump:manifold:density,assump:manifold:density-lower} holds and that, for the fixed order $q_{\rm geom}$, the quantities
$\underline m$, $\eta_\Pi$, $\Gamma_{\gM,q_{\rm geom}}$, $B_0$, and $p_{\min}$ are bounded above or below by positive constants independent of $(D,n)$.  Let
\[
    a_{\beta,d}
    :=
    \min\left\{
      \frac{2}{d+2\beta},
      \frac{2\beta}{(d+2\beta)(\mathfrak p_{\beta,d}-1)}
    \right\}.
\]
Then there exists a constant $c>0$, depending only on the preceding fixed quantities, such that
\[
    D\vee\log n\le c\,n^{a_{\beta,d}} 
\]
implies both
\cref{eq:small-noise-final-reach-condition,eq:small-noise-denominator-resolution-condition}.
Consequently, under this fixed-geometry growth regime, the bound in
\cref{cor:w1-rate-density-lower} applies.
\end{corollary}

\begin{proof}
Write $H_n=D\vee\log n$.  By \eqref{eq:regular-vp-schedule} and
$t_1\asymp n^{-2/(d+2\beta)}$, we have $\sigma_{t_1}\asymp n^{-1/(d+2\beta)}$.
Hence the small-noise reach condition
\cref{eq:small-noise-final-reach-condition} follows whenever
\[
    H_n \le c_1 n^{2/(d+2\beta)}
\]
for a constant $c_1$ depending only on $\underline m$, $\eta_\Pi$, and the schedule
constants.  Likewise, after absorbing the fixed quantities
$p_{\min}^{-1}\Gamma_{\gM,q_{\rm geom}}^C B_0^C$ into a constant, the denominator
condition \cref{eq:small-noise-denominator-resolution-condition} is implied by
\[
    H_n^{(\mathfrak p_{\beta,d}-1)/2}
    \le
    c_2 n^{\beta/(d+2\beta)},
\]
equivalently,
\[
    H_n\le c_3 n^{2\beta/((d+2\beta)(\mathfrak p_{\beta,d}-1))}.
\]
Taking
$c:=\min\{c_1,c_3\}$ proves the claim.
\end{proof}

\clearpage

\section{Technical Toolkit}\label{sec:app:toolkit}

This appendix collects reusable tools that are not specific to the manifold geometry.
They can be treated as black boxes when reading the approximation and estimation proofs.
The most frequently used ingredients are:
\begin{center}
{\small
\setlength{\tabcolsep}{4pt}
\begin{tabular}{p{0.30\linewidth}|p{0.58\linewidth}}
Tool class & Representative labels \\
\hline
Basic ReLU approximation and arithmetic &
\cref{thm:approx:holder,lem:approx:m-sigma,lem:approx:mon,lem:approx:rec,lem:approx:exp} \\
Schedule, reciprocal, power, and clipping networks &
\cref{lem:approx:m-inv,lem:approx:x-over-m,lem:approx:sigma-k,lem:approx:sigma-k-inv} \\
Concentration, Bernstein, and covering arguments &
\cref{thrm:est:variance-bound,thrm:est:loss-bounded,thrm:est:Bernstein-ineq,thrm:est:covering-number} \\
Distribution-estimation support &
\cref{thrm:est:dist:early-stop,thrm:est:dist:Gau,thrm:est:dist:W1-SM}
\end{tabular}
}
\end{center}

% ----------------------------------------------------------------------
\subsection{Neural network approximation toolkit}
\label{sec:app:approx:nn-support}

\subsubsection{Imported neural network approximation lemmas}

%-----------------------------------------------------------
% Approximation of H\"older smooth functions
%-----------------------------------------------------------

\begin{theorem}[ReLU approximation of H\"older functions, after {\citep[Theorem~5]{schmidt2020nonparametric}}]\label{thm:approx:holder}
    For any function $f \in \gH^{\beta}([0, 1]^D, B)$ and any integers $m \geq 1$ and $N \geq (\beta+1)^D \vee (B+1)e^D$, there exists a network
    \begin{equation*}
        \widetilde{f} \in \nn(L, \bW, S, 1)
    \end{equation*}
    with
    \begin{align*}
        L &= 8+(m+5)(1+\ceil{\log_2(D \vee \beta)}), \\
        \|\bW\|_{\infty} &\leq 6(D \vee \ceil{\beta})N, \\
        S &\leq 141(D+\beta+1)^{D+3}N(m+6)
    \end{align*}
    such that
    \begin{equation*}
        \|\widetilde{f} - f\|_{L^{\infty}([0, 1]^D)}
        \leq 
        (2B+1)(1+D^2+\beta^2)6^DN2^{-m}
        +
        B3^{\beta}N^{-\frac{\beta}{D}}. 
    \end{equation*}
\end{theorem}

%-----------------------------------------------------------
% Approximation of m_t & \sigma_t
%-----------------------------------------------------------

\begin{lemma}[ReLU approximation of VP schedules, after Lemma~B.1 of \citep{oko2023diffusion}]\label{lem:approx:m-sigma}
    Let $m_t$ and $\sigma_t$ be the schedules of the VP forward process on $[\tdown,\tup]$.
    Let $0 < \varepsilon < \frac{1}{2}$.
    Then, there exists a neural network $\phi_m(t) \in \nn(L, \bW, S, B)$ that approximates $m_t$ uniformly on $[\tdown,\tup]$, within the additive error of $\varepsilon$, where $L = \Ord(\log^2\varepsilon^{-1}), \|\bW\|_{\infty} = \Ord(\log\varepsilon^{-1}), S = \Ord(\log^2\varepsilon^{-1})$, and $B = \exp(\Ord(\log^2\varepsilon^{-1}))$.

    Also, there exists a neural network $\phi_{\sigma}(t) \in \nn(L, \bW, S, B)$ that approximates $\sigma_t$ uniformly on $[\tdown,\tup]$, within the additive error of $\varepsilon$, where $L \leq \Ord(\log^2\varepsilon^{-1}), \|\bW\|_{\infty} = \Ord(\log^3\varepsilon^{-1}), S = \Ord(\log^4\varepsilon^{-1})$, and $B = \exp(\Ord(\log^2\varepsilon^{-1}))$.
\end{lemma}

%-----------------------------------------------------------
% Approximation of \prod_{k=1}^dx_k
%-----------------------------------------------------------

\begin{lemma}[ReLU multiplication network, after Lemma~F.6 of \citep{oko2023diffusion}]\label{lem:approx:mon}
    Let $d \geq 2, C \geq 1$, and $\varepsilon' \in (0, 1]$.
    For any $\varepsilon > 0$, there exists a ReLU network $\phi_{\rm multi}(x_1, \dots, x_d) \in \nn(L, \bW, S, B)$ with
    \begin{align*}
        L \lesssim {} & \log d(\log\varepsilon^{-1} + d\log C), \\
        \|\bW\|_{\infty} = {} & 48d, \\
        S \lesssim {} & d\log\varepsilon^{-1} + d\log C, \\
        B = {} & C^d, 
    \end{align*}
    such that
    \begin{equation*}
        \Bigl|
          \phi_{\rm multi}(x_1', \dots, x_d') 
          -
          \prod_{k=1}^dx_k
        \Bigr|
        \leq \varepsilon + dC^{d-1}\varepsilon',
        \quad \text{for all } \bx \in [-C, C]^d
        \text{ and } \bx' \in \R^d
        \text{ with } \|\bx - \bx'\|_{\infty} \leq \varepsilon'. 
    \end{equation*}
    Moreover, $|\phi_{\rm multi}(\bx)| \leq C^d$ for all $\bx \in \R^d$, and $\phi_{\rm multi}(x_1', \dots, x_d') = 0$ if at least one of $x_k' = 0$.
    We note that some of $x_i, x_j$ ($i \neq j$) can be shared.
    For the product $\prod_{k=1}^Ix_k^{\alpha_k}$ with $\balpha=(\alpha_1,\dots,\alpha_I)\in \N^I$ and $\sum_{k=1}^I\alpha_k = d$, there exists a neural network satisfying the same bounds as above, denoted by $\phi_{\rm multi}(\bx; \balpha)$.
\end{lemma}

%-----------------------------------------------------------
% Approximation of 1/x
%-----------------------------------------------------------

\begin{lemma}[ReLU reciprocal network, after Lemma~F.7 of \citep{oko2023diffusion}]\label{lem:approx:rec}
    For any $\varepsilon \in (0, 1)$, there exists $\phi_{\rm rec} \in \nn(L, \bW, S, B)$ with
    \begin{equation*}
        L \lesssim \log^2\varepsilon^{-1}, 
        \quad
        \|\bW\|_{\infty} \lesssim \log^3\varepsilon^{-1}, 
        \quad
        S \lesssim \log^4\varepsilon^{-1}, 
        \quad
        \log B \lesssim \log\varepsilon^{-1}, 
    \end{equation*}
    such that
    \begin{equation*}
        \Bigl|
          \phi_{\rm rec}(x')
          -
          \frac{1}{x}
        \Bigr| 
        \leq 
        \varepsilon + \frac{|x' - x|}{\varepsilon}, 
        \quad \text{for all } x \in [\varepsilon, \varepsilon^{-1}]
        \text{ and } x' \in \R. 
    \end{equation*}
\end{lemma}

%-----------------------------------------------------------
% Approximation of exp(-x)
%-----------------------------------------------------------

\begin{lemma}[Pointwise ReLU approximation of $e^{-x}$, after Lemma~F.12 of \citep{oko2023diffusion}]\label{lem:approx:exp}
    For any $\varepsilon > 0$, there exists ReLU network $\phi_{\exp} \in \nn(L, \bW, S, B)$ such that
    \begin{equation*}
        \bigl|
          e^{-x'} - \phi_{\exp}(x)
        \bigr|
        \leq 
        \varepsilon + |x - x'|
        \qquad\text{for all }x,x'\ge0
    \end{equation*}
    holds, where
    \begin{align*}
        L \lesssim {} & \log^2\varepsilon^{-1}, \quad 
        \|\bW\|_{\infty} \lesssim \log\varepsilon^{-1}, \quad
        S \lesssim \log^2\varepsilon^{-1}, \quad
        \log B \lesssim \log^2\varepsilon^{-1}. 
    \end{align*}
    Moreover, for all $x \geq \log(3/\varepsilon)$ it holds that $|\phi_{\exp}(x)| \leq \varepsilon$.
\end{lemma}

\subsubsection{Auxiliary schedule and arithmetic approximation lemmas}
All schedule-dependent lemmas in this subsection are consequences of \cref{lem:approx:m-sigma} and the algebraic network lemmas.

%-----------------------------------------------------------
% Approximation of 1 / m_t
%-----------------------------------------------------------
\begin{lemma}[ReLU approximation of the inverse VP mean $1/m_t$]\label{lem:approx:m-inv}
    Let $\underline{m}:=\inf_{t\in[\tdown, \tup]}m_t > 0$.
    Fix $\varepsilon \in (0,1/2]$, there exist ReLU networks $\widetilde{\phi}_{1/m} \in \nn(L, \bW, S, B)$ with
    \begin{equation*}
        L, \, \log B \lesssim \log^2\varepsilon^{-1} + \log^2\underline{m}^{-1}, 
        \quad
        \|\bW\|_{\infty} \lesssim \log^3\varepsilon^{-1} + \log^3\underline{m}^{-1}, 
        \quad
        S \lesssim \log^4\varepsilon^{-1} + \log^4\underline{m}^{-1}
    \end{equation*}
    such that
    \begin{equation*}
        \|\widetilde{\phi}_{1/m}(t) - m_t^{-1}\|_{L^{\infty}([\tdown, \tup])} 
        \leq 
        \varepsilon. 
    \end{equation*}
\end{lemma}

\begin{proof}
Choose an auxiliary accuracy
\begin{equation*}
    \varepsilon_m := \frac{\underline{m}\varepsilon^2}{4}.
\end{equation*}
By \cref{lem:approx:m-sigma} for $m_t$, there exists a ReLU network $\widetilde{\phi}_m \in \nn(L_1, \bW_1, S_1, B_1)$ with
\begin{equation*}
    L_1 = \Ord(\log^2\varepsilon_m^{-1}), 
    \quad
    \|\bW_1\|_{\infty} = \Ord(\log\varepsilon_m^{-1}), 
    \quad
    S_1 = \Ord(\log^2\varepsilon_m^{-1}), 
    \quad
    \log B_1 = \Ord(\log^2\varepsilon_m^{-1})
\end{equation*}
such that
\begin{equation*}
    \|\widetilde{\phi}_m - m_t\|_{L^\infty([\tdown, \tup])}
    \leq  
    \varepsilon_m.
\end{equation*}
Now apply the reciprocal lemma (\cref{lem:approx:rec}) with parameter
\begin{equation*}
    \varepsilon_{\rm rec} := \frac{\underline{m}\varepsilon}{2}.
\end{equation*}
Since $m_t \in [\underline{m}, 1] \subset [\varepsilon_{\rm rec}, \varepsilon_{\rm rec}^{-1}]$, the reciprocal network $\phi_{\rm rec}$ satisfies
\begin{equation*}
    \Bigl|
      \phi_{\rm rec}(\widetilde{\phi}_m(t)) - \frac1{m_t}
    \Bigr|
    \leq
    \varepsilon_{\rm rec}
    +
    \frac{|\widetilde{\phi}_m(t)-m_t|}{\varepsilon_{\rm rec}}
    \leq
    \frac{\underline{m}\varepsilon}{2}
    +
    \frac{\varepsilon_m}{\underline{m}\varepsilon/2}
    \leq
    \frac{\varepsilon}{2}+\frac{\varepsilon}{2}
    =
    \varepsilon.
\end{equation*}
Hence we may define
\begin{equation*}
    \widetilde{\phi}_{1/m}(t) := \phi_{\rm rec}(\widetilde{\phi}_m(t)).
\end{equation*}
and we have $\widetilde{\phi}_{1/m}(t) \in \nn(L, \bW, S, B)$ with
\begin{equation*}
    L, \, \log B \lesssim \log^2\varepsilon^{-1} + \log^2\underline{m}^{-1}, 
    \quad
    \|\bW\|_{\infty} \lesssim \log^3\varepsilon^{-1} + \log^3\underline{m}^{-1}, 
    \quad
    S \lesssim \log^4\varepsilon^{-1} + \log^4\underline{m}^{-1}. 
\end{equation*}
\end{proof}

%-----------------------------------------------------------
% Approximation of x / m_t
%-----------------------------------------------------------
\begin{lemma}[ReLU approximation of the rescaled input $\bx/m_t$]
\label{lem:approx:x-over-m}
    Let $M_x\ge 1$ and assume $\|\bx\|_\infty \leq M_x$.
    For every $\varepsilon\in(0,1/2]$ there exists a vector-valued ReLU network
    \begin{equation*}
        \widetilde{\phi}(\bx, t)
        =
        \bigl(
            \widetilde{\phi}_1(\bx, t), \dots, \widetilde{\phi}_D(\bx, t)
        \bigr)
        \in 
        \nn(L, \bW, S, B)
    \end{equation*}
    with
    \begin{align*}
        L, \, \log B \lesssim {} & \log^2\varepsilon^{-1} + \log^2(M_x \vee \underline{m}^{-1}), 
        \\
        \|\bW\|_{\infty} \lesssim {} & D\bigl(\log^3\varepsilon^{-1} + \log^3(M_x \vee \underline{m}^{-1})\bigr), 
        \\
        S \lesssim {} & D\bigl(\log^4\varepsilon^{-1} + \log^4(M_x \vee \underline{m}^{-1})\bigr), 
    \end{align*}
    such that
    \begin{equation*}
        \Bigl\|
          \widetilde{\phi}(\bx, t)-\frac{\bx}{m_t}
        \Bigr\|_{L^\infty(\{ \|\bx\|_\infty \leq M_x\}\times[\tdown, \tup])}
        \leq \varepsilon.
    \end{equation*}
\end{lemma}

\begin{proof}
Let $\widetilde{\phi}_{1/m}$ be the network from \cref{lem:approx:m-inv}.
Choose an auxiliary accuracy
\begin{equation*}
    \varepsilon_0
    :=
    \frac{\varepsilon}{4(1 \vee M_x \vee \underline{m}^{-1})}.
\end{equation*}

Refine $\widetilde{\phi}_{1/m}$ so that
\begin{equation*}
    \|\widetilde{\phi}_{1/m} - m_t^{-1}\|_{L^\infty([\tdown, \tup])}
    \leq \varepsilon_0.
\end{equation*}

For each coordinate $j=1,\dots,D$, apply the multiplication lemma (\cref{lem:approx:mon}) on $[-M_x,M_x] \times [-\underline{m}^{-1},\underline{m}^{-1}]$ to obtain a ReLU network $\phi_{{\rm mult},j} \in \nn(L_1, \bW_1, S_1, B_1)$ satisfying
\begin{align*}
    L_1 \lesssim {} & \log\varepsilon_0^{-1} + \log(M_x \vee \underline{m}^{-1}), 
    \\
    \|\bW_1\|_{\infty} = {} & 96, 
    \\
    S_1 \lesssim {} & \log\varepsilon_0^{-1} + \log(M_x \vee \underline{m}^{-1}), 
    \\
    \log B_1 = {} & 2\log(M_x \vee \underline{m}^{-1}), 
\end{align*}
such that
\begin{equation*}
    \Bigl|
      \phi_{{\rm mult},j}\bigl(x_j,\widetilde{\phi}_{1/m}(t)\bigr)
      -
      x_jm_t^{-1}
    \Bigr|
    \leq
    \frac{\varepsilon}{2}
    +
    2(1\vee M_x\vee \underline{m}^{-1})\,\varepsilon_0
    \leq
    \varepsilon.
\end{equation*}

Define
\begin{equation*}
    \widetilde{\phi}_j(\bx, t)
    :=
    \phi_{{\rm mult},j}\bigl(x_j, \widetilde{\phi}_{1/m}(t)\bigr),
    \qquad
    j=1, \dots, D.
\end{equation*}

Parallel stacking of the $D$ coordinate subnetworks gives the stated vector-valued approximant.
\end{proof}

\begin{lemma}[ReLU approximation of positive powers of $\sigma_t$]\label{lem:approx:sigma-k}
    For any $k \in \N_+$ and any $ \varepsilon \in (0, 1/2]$ such that $\varepsilon \leq \tdown < \tup < \infty$, there exists a ReLU network $\phi_{\sigma^k} \in \nn(L, \bW, S, B)$ with
    \begin{align*}
        &
        L \lesssim \log^2\varepsilon^{-1} + \log^2k, 
        \quad
        \|\bW\|_{\infty} \lesssim \log^3\varepsilon^{-1} + k, 
        \\ &
        S \lesssim {} \log^4\varepsilon^{-1} + \log^4k + k\log\varepsilon^{-1}, 
        \quad
        \log B \lesssim \log^2\varepsilon^{-1} + \log^2k, 
    \end{align*}
    such that
    \begin{align*}
        \bigl\|
          \phi_{\sigma^k}(t) - \sigma_t^k
        \bigr\|_{L^{\infty}([\tdown, \tup])}
        \leq {} &
        3\varepsilon. 
    \end{align*}
\end{lemma}
\begin{proof}
    By~\cref{lem:approx:m-sigma}, applied with accuracy parameter $\varepsilon/k$, there exists a neural network $\phi_{\sigma}(t) \in \nn(L_{\sigma}, \bW_{\sigma}, S_{\sigma}, B_{\sigma})$ with
    \begin{equation}
    \begin{aligned}
        &
        L \lesssim \log^2\varepsilon^{-1} + \log^2k, 
        \quad
        \|\bW\|_{\infty} \lesssim \log^3\varepsilon^{-1} + \log^3k, 
        \\
        &
        S \lesssim \log^4\varepsilon^{-1} + \log^4k, 
        \quad
        \log B \lesssim \log^2\varepsilon^{-1} + \log^2k, 
    \end{aligned}
    \label{eq:approx:size:sigma-k:1}
    \end{equation}
    such that
    \begin{equation*}
        |\phi_{\sigma}(t) - \sigma_t| \leq \varepsilon/k, 
        \quad \text{for all } t \in [\tdown,\tup]. 
    \end{equation*}

    Notice that $\sigma_t \in [0, 1]$ for all $t \in [\tdown, \tup]$.
    If $k=1$, take $\phi_{\sigma^k}:=\phi_\sigma$, and the claim follows.
    Assume now that $k\ge2$.
    Apply \cref{lem:approx:mon} with $d=k$, $C=1$, and with all $k$ true inputs equal to $\sigma_t$ and all $k$ perturbed inputs equal to $\phi_\sigma(t)$.
    The shared-input form is allowed by the final sentence of \cref{lem:approx:mon}.
    Thus there exists a ReLU network $\phi_{\rm multi} \in \nn(L_1, \bW_1, S_1, B_1)$ with
    \begin{equation}
        L_1 \lesssim \log k\,\log\varepsilon^{-1}, 
        \quad
        \|\bW_1\|_{\infty} = 48k, 
        \quad
        S_1 \lesssim k\log\varepsilon^{-1}, 
        \quad
        B_1 = 1
        \label{eq:approx:size:sigma-k:2}
    \end{equation}
    such that
    \begin{equation}
        \Bigl|
          \phi_{\rm multi}\bigl(
            \underbrace{
              \phi_{\sigma}(t),\dots,\phi_{\sigma}(t)
            }_{k\ {\rm copies}}
          \bigr)
          -
          \sigma_t^k
        \Bigr|
        \leq 
        \varepsilon + k\frac{\varepsilon}{k}
        \leq
        2\varepsilon. 
        \label{eq:approx:bound:sigma-k:3}
    \end{equation}

    For all $t \in [\tdown, \tup]$, we define
    \begin{equation*}
        \phi_{\sigma^k}(t)
        \coloneqq
        \phi_{\rm multi}\bigl(
          \underbrace{
            \phi_{\sigma}(t),\dots,\phi_{\sigma}(t)
          }_{k\ {\rm copies}}
        \bigr). 
    \end{equation*}

    Then by the configurations given in \cref{eq:approx:size:sigma-k:1,eq:approx:size:sigma-k:2} and the bound in \cref{eq:approx:bound:sigma-k:3}, we obtain that there exists a ReLU network $\phi_{\sigma^k} \in \nn(L, \bW, S, B)$ with
    \begin{align*}
        &
        L \lesssim \log^2\varepsilon^{-1} + \log^2k, 
        \quad
        \|\bW\|_{\infty} \lesssim \log^3\varepsilon^{-1} + k, 
        \\
        &
        S \lesssim {} \log^4\varepsilon^{-1} + \log^4k + k\log\varepsilon^{-1}, 
        \quad
        \log B \lesssim \log\varepsilon^{-1} + \log k, 
    \end{align*}
    such that
    \begin{align*}
        \bigl\|
          \phi_{\sigma^k}(t) - \sigma_t^k
        \bigr\|_{L^{\infty}([\tdown, \tup])}
        \leq {} &
        3\varepsilon. 
    \end{align*}
\end{proof} 

\begin{lemma}[ReLU approximation of inverse powers of $\sigma_t$]\label{lem:approx:sigma-k-inv}
    For any $k \in \N_+$ and any $ \varepsilon \in (0, 1]$, there exists a ReLU network $\phi_{\sigma^{-k}} \in \nn(L, \bW, S, B)$ with
    \begin{align*}
        L \lesssim {} & \log^2\varepsilon^{-1} + k^2\log^2\sigma_{\tdown}^{-2}, 
        \\
        \|\bW\|_{\infty} \lesssim {} & \log^3\varepsilon^{-1} + k^3\log^3\sigma_{\tdown}^{-2}, 
        \\
        S \lesssim {} & \log^4\varepsilon^{-1} + k^4\log^4\sigma_{\tdown}^{-2}
        + k\log(e/\varepsilon)+k^2\log(e\sigma_{\tdown}^{-2}), 
        \\
        \log B \lesssim {} & \log^2\varepsilon^{-1} + k\log^2\sigma_{\tdown}^{-2}, 
    \end{align*}
    such that
    \begin{align*}
        \bigl\|
          \phi_{\sigma^{-k}}(t) - \sigma_t^{-k}
        \bigr\|_{L^{\infty}([\tdown, \tup])}
        \leq {} &
        4\varepsilon. 
    \end{align*}
\end{lemma}
\begin{proof}
    Set
    \begin{equation*}
        \varepsilon_{\sigma,k}
        :=
        \min\Bigl\{
          \frac{\sigma_{\tdown}^k\varepsilon^2}{12},
          \frac{\tdown}{2}
        \Bigr\}.
    \end{equation*}
    Apply~\cref{lem:approx:sigma-k} with accuracy parameter $\varepsilon_{\sigma,k}$.
    By construction $\varepsilon_{\sigma,k}\le \tdown$, so this parameter lies in the admissible range of \cref{lem:approx:sigma-k}.
    Moreover,
    \begin{equation*}
        \log(\varepsilon_{\sigma,k}^{-1})
        \le
        C\bigl(
          \log\varepsilon^{-1}
          +
          k\log\sigma_{\tdown}^{-1}
          +
          1
        \bigr),
    \end{equation*}
    where the fixed compact time interval contributes only to the constant through $\log\tdown^{-1}$.
    Then, there exists a ReLU network $\phi_{\sigma^k} \in \nn(L_1, \bW_1, S_1, B_1)$ with
    \begin{align*}
        L_1 &\lesssim \log^2\varepsilon^{-1} + k^2\log^2\sigma_{\tdown}^{-2}+\log^2k, 
        \\
        \|\bW_1\|_{\infty} &\lesssim \log^3\varepsilon^{-1} + k^3\log^3\sigma_{\tdown}^{-2}+k, 
        \\
        S_1 &\lesssim \log^4\varepsilon^{-1} + k^4\log^4\sigma_{\tdown}^{-2}
        +\log^4k+k\log(e/\varepsilon)+k^2\log(e\sigma_{\tdown}^{-2}), 
        \\
        \log B_1 &\lesssim \log^2\varepsilon^{-1} + k^2\log^2\sigma_{\tdown}^{-2}+\log^2k, 
    \end{align*}
    such that
    \begin{align*}
        \bigl\|
          \phi_{\sigma^k}(t) - \sigma_t^k
        \bigr\|_{L^{\infty}([\tdown, \tup])}
        \leq {} &
        3\varepsilon_{\sigma,k}
        \le
        \frac{\sigma_{\tdown}^k\varepsilon^2}{4}. 
    \end{align*}
 
    Note that for any $t \in [\tdown, \tup]$, we have $\sigma_{\tdown}^k \lesssim \sigma_t^k \leq 1$.
    
    Let $\varepsilon_{\rm rec}:=\sigma_{\tdown}^k\varepsilon/2$ and let $\phi_{\rm rec}$ be the network in \cref{lem:approx:rec} with parameter $\varepsilon_{\rm rec}$.
    Since $\varepsilon_{\rm rec}\le\sigma_{\tdown}^k\le\sigma_t^k$ and $\sigma_t^k\le1\le \varepsilon_{\rm rec}^{-1}$, the domain condition in \cref{lem:approx:rec} holds for every $t\in[\tdown,\tup]$.
    Then, we have $\phi_{\rm rec} \in \nn(L_{\rm rec}, \bW_{\rm rec}, S_{\rm rec}, B_{\rm rec})$ with
    \begin{align*}
        L_{\rm rec} \lesssim {} & \log^2\varepsilon^{-1} + k^2\log^2\sigma_{\tdown}^{-2}, 
        \\
        \|\bW_{\rm rec}\|_{\infty} \lesssim {} & \log^3\varepsilon^{-1} + k^3\log^3\sigma_{\tdown}^{-2}, 
        \\
        S_{\rm rec} \lesssim {} & \log^4\varepsilon^{-1} + k^4\log^4\sigma_{\tdown}^{-2}, 
        \\
        \log B_{\rm rec} \lesssim {} & \log\varepsilon^{-1} + k\log\sigma_{\tdown}^{-2}, 
    \end{align*}
    such that for all $t \in [\tdown, \tup]$, it holds that
    \begin{equation*}
        \Bigl|
          \phi_{\rm rec}(\phi_{\sigma^k}(t))
          -
          \frac{1}{\sigma_t^k}
        \Bigr| 
        \leq 
        \varepsilon_{\rm rec}
        +
        \frac{|\phi_{\sigma^k}(t) - \sigma_t^k|}{\varepsilon_{\rm rec}}
        \leq 
        \frac{\sigma_{\tdown}^k\varepsilon}{2}
        +
        \frac{\sigma_{\tdown}^k\varepsilon^2/4}{\sigma_{\tdown}^k\varepsilon/2}
        \leq
        \varepsilon.
    \end{equation*}
    Let $\phi_{\sigma^{-k}}(t) \coloneqq \phi_{\rm rec}(\phi_{\sigma^k}(t))$.
    Adding the displayed size bounds for $\phi_{\sigma^k}$ and $\phi_{\rm rec}$, and absorbing the lower-order $\log k$ terms into the stated $k$-dependent bounds, gives the claimed network size.
\end{proof}

\subsection{Support lemmas for generalization analysis}

\begin{lemma}[Unconditional Fisher bound for the smoothed data law]
\label{thrm:est:true-score-norm}
For the VP process~\cref{eq:vp:forward}, for each $t>0$,
\begin{equation*}
    \E_{\bX_t}
    \bigl[
      \|\nabla\log p_t(\bX_t)\|_2^2
    \bigr]
    \le
    D\sigma_t^{-2}.
\end{equation*}
\end{lemma}
\begin{proof}
Write $\bX_t=m_t\bX_0+\sigma_t\bZ$, where
$\bZ\sim\gN(0,\bI_D)$ is independent of $\bX_0$.  Tweedie's formula gives
\[
    \nabla\log p_t(\by)
    =
    \E\left[-\frac{\bZ}{\sigma_t}\,\middle|\,\bX_t=\by\right].
\]
Therefore Jensen's inequality and the tower property imply
\[
    \E_{\bX_t}\|\nabla\log p_t(\bX_t)\|_2^2
    \le
    \E\left\|\frac{\bZ}{\sigma_t}\right\|_2^2
    =
    D\sigma_t^{-2}.
\]
This is an unconditional Fisher-information bound under the marginal law of
$\bX_t$; no conditional bound given $\bX_0$ is used.
\end{proof}

\begin{lemma}[A variant lemma of {\citep[Lemma~25]{fu2025approximation}}]\label{thrm:est:variance-bound}
    Fix any $0 < \tdown < \tup < \infty$. 
    Let $s: \R^D \times [\tdown, \tup] \to \R^D$ and $s': \R^D \times [\tdown, \tup] \to \R^D$ be any Borel functions. % such that
    % \begin{equation*}
    %     \|s(\cdot, t)\|_{L^\infty(\R^D)} \lesssim \sigma_t^{-1}\sqrt{\log n}
    %     \quad\text{ and }\quad
    %     \|s'(\cdot, t)\|_{L^\infty(\R^D)} \lesssim \sigma_t^{-1}\sqrt{\log n}.
    % \end{equation*}
    Then, for all $\bx \in \R^D$, it holds that
    \begin{align*}
        &
        \bigl(
          \ell_{\tdown,\tup}(s, \bx) - \ell_{\tdown,\tup}(s', \bx)
        \bigr)^2 
        \\
        &\leq
        3\Bigl(
          \int_{\tdown}^{\tup}
            \E_{\bX_t|\bX_0=\bx}
            \bigl[
              \|s(\bX_t, t)\|_2^2 + \|s'(\bX_t, t)\|_2^2 
            \bigr]
          \odt 
          +
          4D
          \int_{\tdown}^{\tup}\sigma_t^{-2}\odt
          % \Bigl(
          %   \tup - \tdown + \log\frac{\sigma_{\tup}}{\sigma_{\tdown}}
          % \Bigr)
        \Bigr)
        \int_{\tdown}^{\tup}
          \E_{\bX_t|\bX_0=\bx}
          \Bigl[
            \bigl\|
              s(\bX_t, t)  
              -
              s'(\bX_t, t)
            \bigr\|_2^2
          \Bigr]
        \odt. 
    \end{align*}
\end{lemma}
\begin{proof}
    \begin{align*}
        {} & 
        \ell_{\tdown,\tup}(s, \bx) - \ell_{\tdown,\tup}(s', \bx)
        \\
        = {} &
        \int_{\tdown}^{\tup}
          \E_{\bX_t|\bX_0=\bx}\Bigl[
            \Bigl\|
              s(\bX_t, t) + \frac{\bX_t - m_t\bx}{\sigma_t^2}
            \Bigr\|_2^2
            -
            \Bigl\|
              s'(\bX_t, t) + \frac{\bX_t - m_t\bx}{\sigma_t^2}
            \Bigr\|_2^2
          \Bigr]
        \odt
        \\
        = {} &
        \int_{\tdown}^{\tup}
          \E_{\bX_t|\bX_0=\bx}\Bigl[
            \|s(\bX_t, t)\|_2^2  
            +
            2\Bigl\langle 
              s(\bX_t, t), \frac{\bX_t - m_t\bx}{\sigma_t^2} 
            \Bigr\rangle
            -
            \|s'(\bX_t, t)\|_2^2
            -
            2\Bigl\langle 
              s'(\bX_t, t), \frac{\bX_t - m_t\bx}{\sigma_t^2} 
            \Bigr\rangle
          \Bigr]
        \odt
        \\
        = {} &
        \int_{\tdown}^{\tup}
          \E_{\bX_t|\bX_0=\bx}\Bigl[
            \bigl(
              s(\bX_t, t) - s'(\bX_t, t)
            \bigr)^\top 
            \bigl(
              s(\bX_t, t) + s'(\bX_t, t)
            \bigr)
            +
            \bigl(
              s(\bX_t, t) - s'(\bX_t, t)
            \bigr)^\top  
            \frac{2(\bX_t - m_t\bx)}{\sigma_t^2} 
          \Bigr]
        \odt
        \\
        = {} &
        \int_{\tdown}^{\tup}
          \E_{\bX_t|\bX_0=\bx}\Bigl[
            \bigl(
              s(\bX_t, t)  
              -
              s'(\bX_t, t)
            \bigr)^\top
            \Bigl(
              s(\bX_t, t) 
              + 
              s'(\bX_t, t)
              +
              \frac{2(\bX_t - m_t\bx)}{\sigma_t^2} 
            \Bigr)
          \Bigr]
        \odt. 
    \end{align*}

    Applying the Cauchy-Schwarz inequality and Young's inequality, we obtain that
    \begin{align*}
        {} &
        |\ell_{\tdown,\tup}(s, \bx) - \ell_{\tdown,\tup}(s', \bx)|
        \\
        \leq {} &
        \int_{\tdown}^{\tup}
          \E_{\bX_t|\bX_0=\bx}
          \Bigl[
            \bigl\|
              s(\bX_t, t)  
              -
              s'(\bX_t, t)
            \bigr\|
            \cdot
            \Bigl\|
              s(\bX_t, t)  
              +
              s'(\bX_t, t)
              +
              \frac{2(\bX_t - m_t\bx)}{\sigma_t^2} 
            \Bigr\|
          \Bigr]
        \odt
        \tag{by Cauchy-Schwarz inequality} \\
        \leq {} &
        \sqrt{
        \int_{\tdown}^{\tup}
          \E_{\bX_t|\bX_0=\bx}
          \Bigl[
            \bigl\|
              s(\bX_t, t)  
              -
              s'(\bX_t, t)
            \bigr\|_2^2
          \Bigr]
        \odt
        }
        \cdot
        \sqrt{
        \int_{\tdown}^{\tup}
          \E_{\bX_t|\bX_0=\bx}
          \Bigl[
            \Bigl\|
              s(\bX_t, t)  
              +
              s'(\bX_t, t)
              +
              \frac{2(\bX_t - m_t\bx)}{\sigma_t^2} 
            \Bigr\|_2^2
          \Bigr]
        \odt
        }. 
        \tag{by Young's inequality}
    \end{align*}

    For the VP process, $\bX_t=m_t\bx+\sigma_t\bZ$ under
    $\bX_t|\bX_0=\bx$, with $\bZ\sim\gN(0,\bI_D)$.  Hence
    \begin{equation*}
        \E_{\bX_t|\bX_0=\bx}
        \Bigl[
          \Bigl\|
            \frac{\bX_t - m_t\bx}{\sigma_t^2}
          \Bigr\|_2^2
        \Bigr]
        =
        \frac{D}{\sigma_t^2}. 
    \end{equation*}

    Then it holds that 
    \begin{align}
        {} &
        \sqrt{
        \int_{\tdown}^{\tup}
          \E_{\bX_t|\bX_0=\bx}
          \Bigl[
            \Bigl\|
              s(\bX_t, t)  
              +
              s'(\bX_t, t)
              +
              \frac{2(\bX_t - m_t\bx)}{\sigma_t^2} 
            \Bigr\|_2^2
          \Bigr]
        \odt
        }
        \notag \\
        \leq {} &
        \sqrt{
        3\int_{\tdown}^{\tup}
          \Bigl(
            \E_{\bX_t|\bX_0=\bx}
            \bigl[
              \|s(\bX_t, t)\|_2^2
              + 
              \|s'(\bX_t, t)\|_2^2
            \bigr]
            +
            4\E_{\bX_t|\bX_0=\bx}
            \Bigl[
              \Bigl\|
                \frac{\bX_t - m_t\bx}{\sigma_t^2}
              \Bigr\|_2^2
            \Bigr]
          \Bigr)
        \odt 
        }
        \notag \\
        \leq {} &
        \sqrt{
        3\int_{\tdown}^{\tup}
          \Bigl(
            \E_{\bX_t|\bX_0=\bx}
            \bigl[
              \|s(\bX_t, t)\|_2^2 + \|s'(\bX_t, t)\|_2^2 
            \bigr]
            +
            4D\sigma_t^{-2}
          \Bigr) 
        \odt}. 
    \end{align}

    % Notice that 
    % \begin{align}
    %     \int_{\tdown}^{\tup}\sigma_t^{-2}\odt 
    %     = 
    %     \int_{\tdown}^{\tup}
    %       \Bigl(
    %         1 + \frac{\exp(-2\int_0^t\alpha_{\tau}\od\tau)}{1-\exp(-2\int_0^t\alpha_{\tau}\od\tau)}
    %       \Bigr)
    %     \odt 
    %     = 
    %     \tup - \tdown
    %     +
    %     \log\frac{\sigma_{\tup}}{\sigma_{\tdown}} 
    %     \label{eq:est:bound:int-sigma-2}
    % \end{align}

    Hence we obtain
    \begin{align}
        {} &
        \sqrt{
        \int_{\tdown}^{\tup}
          \E_{\bX_t|\bX_0=\bx}
          \Bigl[
            \Bigl\|
              s(\bX_t, t)  
              +
              s'(\bX_t, t)
              +
              \frac{2(\bX_t - m_t\bx)}{\sigma_t^2} 
            \Bigr\|_2^2
          \Bigr]
        \odt}
        \notag \\
        \leq {} &
        \sqrt{
        3\int_{\tdown}^{\tup}
          \E_{\bX_t|\bX_0=\bx}
          \bigl[
            \|s(\bX_t, t)\|_2^2 + \|s'(\bX_t, t)\|_2^2 
          \bigr]
        \odt 
        +
        12D
        \int_{\tdown}^{\tup}
          \sigma_t^{-2}
        \odt
        }. 
        % \notag \\
        % % 
        % = {} &
        % \sqrt{
        % 3\int_{\tdown}^{\tup}
        %   \E_{\bX_t|\bX_0=\bx}
        %   \bigl[
        %     \|s(\bX_t, t)\|_2^2 + \|s'(\bX_t, t)\|_2^2 
        %   \bigr]
        % \odt 
        % +
        % 12D
        % \Bigl(
        %   \tup - \tdown + \log\frac{\sigma_{\tup}}{\sigma_{\tdown}}
        % \Bigr)
        % }. 
        \label{eq:est:range:score-match}
    \end{align}

    Thus, again by Cauchy-Schwarz inequality, we have
    \begin{align*}
        {} &
        \bigl(
          \ell_{\tdown,\tup}(s, \bx) - \ell_{\tdown,\tup}(s', \bx)
        \bigr)^2
        \\
        \leq {} &
        \int_{\tdown}^{\tup}
          \E_{\bX_t|\bX_0=\bx}
          \Bigl[
            \bigl\|
              s(\bX_t, t)  
              -
              s'(\bX_t, t)
            \bigr\|_2^2
          \Bigr]
        \odt
        \cdot
        \int_{\tdown}^{\tup}
          \E_{\bX_t|\bX_0=\bx}
          \Bigl[
            \Bigl\|
              s(\bX_t, t)  
              +
              s'(\bX_t, t)
              +
              \frac{2(\bX_t - m_t\bx)}{\sigma_t^2} 
            \Bigr\|_2^2
          \Bigr]
        \odt
        \\
        \leq {} &
        3\Bigl(
          \int_{\tdown}^{\tup}
            \E_{\bX_t|\bX_0=\bx}
            \bigl[
              \|s(\bX_t, t)\|_2^2 + \|s'(\bX_t, t)\|_2^2 
            \bigr]
          \odt 
          +
          4D
          \int_{\tdown}^{\tup}\sigma_t^{-2}\odt
          % \Bigl(\tup - \tdown + \log\frac{\sigma_{\tup}}{\sigma_{\tdown}}\Bigr)
        \Bigr)
        \int_{\tdown}^{\tup}
          \E_{\bX_t|\bX_0=\bx}
          \Bigl[
            \bigl\|
              s(\bX_t, t)  
              -
              s'(\bX_t, t)
            \bigr\|_2^2
          \Bigr]
        \odt. 
    \end{align*}
\end{proof}

\begin{lemma}[Uniform boundedness of clipped denoising losses]
\label{thrm:est:loss-bounded}
Fix $0<\tdown<\tup<\infty$.  Let $s:\R^D\times[\tdown,\tup]\to\R^D$ be Borel and assume that
\[
    \sup_{\bx\in\R^D}\|s(\bx,t)\|_\infty
    \le
    C_{\rm clip}\sigma_t^{-1}(D\vee\log n)^{1/2},
    \qquad t\in[\tdown,\tup].
\]
Then
\[
    \sup_{\bx\in\R^D}|\ell_{\tdown,\tup}(s,\bx)|
    \le
    C(1+C_{\rm clip}^2)
    (D\vee\log n)^2
    \int_{\tdown}^{\tup}\sigma_t^{-2}\odt .
\]
Consequently, if $s$ and $s'$ both satisfy the same clipped envelope, then
\[
    \sup_{\bx\in\R^D}|\ell_{\tdown,\tup}(s,\bx)-\ell_{\tdown,\tup}(s',\bx)|
    \le
    C(1+C_{\rm clip}^2)
    (D\vee\log n)^2
    \int_{\tdown}^{\tup}\sigma_t^{-2}\odt .
\]
\end{lemma}
\begin{proof}
For the VP process,
$\bX_t\mid\bX_0=\bx\sim\gN(m_t\bx,\sigma_t^2\bI_D)$, and
\[
    \nabla\log p_t(\bX_t\mid\bX_0=\bx)
    =
    -\frac{\bX_t-m_t\bx}{\sigma_t^2}.
\]
The clipped envelope gives
\[
    \|s(\by,t)\|_2^2
    \le
    D C_{\rm clip}^2\sigma_t^{-2}(D\vee\log n)
    \le
    C C_{\rm clip}^2 (D\vee\log n)^2\sigma_t^{-2}.
\]
Also
\[
    \E_{\bX_t|\bX_0=\bx}
    \left\|
      \frac{\bX_t-m_t\bx}{\sigma_t^2}
    \right\|_2^2
    =
    D\sigma_t^{-2}
    \le
    (D\vee\log n)^2\sigma_t^{-2}.
\]
Using $\|a+b\|_2^2\le2\|a\|_2^2+2\|b\|_2^2$ inside the definition of
$\ell_{\tdown,\tup}(s,\bx)$ and integrating over $t$ proves the first display.  The second follows
from the triangle inequality applied to $\ell_{\tdown,\tup}(s,\bx)-\ell_{\tdown,\tup}(s',\bx)$.
\end{proof}

\subsubsection{Bernstein's inequality}

\begin{theorem}[Bernstein's inequality for bounded distributions~\cite{vershynin2018high}]\label{thrm:est:Bernstein-ineq}
    Let $X_1, \dots, X_n$ be independent random variables such that $|X_i| \leq K$ for all $i \in [n]$.
    Then, for every $t \geq 0$, we have
    \begin{equation*}
        \Pr\Bigl[
          \Bigl|
            \sum_{i=1}^n
              (X_i - \E[X_i])
          \Bigr|
          \geq t
        \Bigr]
        \leq 
        2\exp\Bigl(
          -\frac{t^2/2}{\sum_{i=1}^n\E[X_i^2] + Kt/3}
        \Bigr). 
    \end{equation*}

    In other words, with probability at least $1-\delta$, it holds that
    \begin{equation*}
        \Bigl|
        \frac{1}{n}
        \sum_{i=1}^n
          (X_i - \E[X_i])
        \Bigr|
        \leq 
        \frac{t}{n}
        \leq 
        \frac{2K\log(2/\delta)}{3n}
        +
        \sqrt{
        \frac{\frac{2}{n}\sum_{i=1}^n\E[X_i^2]\log(2/\delta)}{n}
        }. 
    \end{equation*}
\end{theorem}

\subsubsection{Covering number bounds}

\begin{lemma}[{\citep[Lemma~3]{suzuki2019adaptivity}}]\label{thrm:est:covering-number}
    For any input dimension $D_{\rm in}$ and any $\tau > 0$, the covering number of $\nn(L,\bW,S,B)$ on $[0,1]^{D_{\rm in}}$ can be bounded by
    \begin{equation*}
        \log\gN(\tau, \nn(L, \bW, S, B), \|\cdot\|_{L^{\infty}([0, 1]^{D_{\rm in}})})
        \lesssim
        SL\log\bigl(
          \tau^{-1}L(\|\bW\|_{\infty}+1)(B \vee 1)
        \bigr). 
    \end{equation*}
    In the score-estimation oracle $D_{\rm in}=D+1$, and the difference between
    $D_{\rm in}$ and $D$ is absorbed into the displayed polynomial ambient factors.
\end{lemma}

\subsection{Auxiliary lemmas}
\label{sec:app:dist-est:lemma}

We use the following distribution-estimation lemmas from \citep{oko2023diffusion} as black boxes.
For the slabwise lemma we use the standard clipped-score form with coordinatewise clipping
at a fixed polynomial scale; \cref{sec:app:score-estimation} supplies the integrated score-error,
clipping, and network-size inputs for the learned scores.
The final $\sfW_1$ corollary is therefore unconditional as a perturbation theorem
given those score inputs.

\begin{lemma}[Early stopping for the VP forward law]
\label{thrm:est:dist:early-stop}
For the forward law $P_t$ and $P_0$ supported on $[0,1]^D$,
\begin{equation}
    \sfW_1(P_0,P_{t_0})
    \le
    \sqrt D\,\{1-m_{t_0}+\sigma_{t_0}\}.
    \label{eq:aux-early-stop}
\end{equation}
Consequently, under the regular VP schedule convention
\cref{eq:regular-vp-schedule}, $\sfW_1(P_0,P_{t_0})\le C\sqrt{D t_0}$ for
$t_0\le1$.
\end{lemma}
\begin{proof}
Couple $\bX_0\sim P_0$ with $\bX_{t_0}=m_{t_0}\bX_0+\sigma_{t_0}\bZ$ using an
independent $\bZ\sim\gN(0,I_D)$.  Since $\gM\subset[0,1]^D$,
$\|\bX_0\|_2\le\sqrt D$ almost surely and $\E\|\bZ\|_2\le\sqrt D$.  Therefore
\[
    \sfW_1(P_0,P_{t_0})
    \le
    \E\|\bX_{t_0}-\bX_0\|_2
    \le
    (1-m_{t_0})\sqrt D+\sigma_{t_0}\sqrt D .
\]
The final display follows from
\cref{eq:regular-vp-schedule}.
\end{proof}

\begin{lemma}[Terminal Gaussian initialization with ambient factor]
\label{thrm:est:dist:Gau}
Let $\overline P_{t_0}$ be the reverse-process output law initialized from the exact forward terminal law $P_T$, and let $\widehat P_{t_0}$ be the corresponding output law initialized from the Gaussian proxy $\widehat P_T$.
Under the regular VP schedule convention \cref{eq:regular-vp-schedule},
\begin{equation}
    \E\bigl[\sfW_1(\overline P_{t_0},\widehat P_{t_0})\bigr]
    \le
    C\TV(P_T,\widehat P_T)
    \le
    C\sqrt D\,\exp(-c_{\rm mix}T).
    \label{eq:aux-gaussian-init}
\end{equation}
\end{lemma}
\begin{proof}
The first inequality is the terminal perturbation bound from
\citep[Lemma~D.6]{oko2023diffusion}.  For the second, condition on
$\bX_0=\by$.  The terminal conditional law is
$\gN(m_T\by,\sigma_T^2 I_D)$ with $\sigma_T^2=1-m_T^2$, while
$\widehat P_T=\gN(0,I_D)$.  When $m_T\le1/2$, the Gaussian KL divergence is
bounded by $C D m_T^2$ because $\|\by\|_2\le\sqrt D$ and
$\sigma_T^2=1-m_T^2$.  Pinsker's inequality gives a
$C\sqrt D\,m_T$ bound for each conditional law.  When $m_T>1/2$, the same bound
follows after increasing $C$ from the trivial estimate $\TV\le1$.  Convexity over
the mixture gives $\TV(P_T,\widehat P_T)\le C\sqrt D\,m_T$.  The regular VP terminal decay
$m_T\le C_{\rm mix}e^{-c_{\rm mix}T}$ proves the claim.
\end{proof}

\begin{lemma}[Slabwise $\sfW_1$ perturbation, after {\citep[Lemma~D.7]{oko2023diffusion}}]\label{thrm:est:dist:W1-SM}
    Fix a partition $[t_0,T]=\bigcup_{k=1}^{K_*}[t_{k-1},t_k]$.  Let
    $\overline P_{t_0}^{(k)}$ denote the reverse-process output law initialized from the
    exact terminal law $P_T$, using the learned score on the first $k$ reverse-time
    slabs and the true score on the remaining slabs.  Thus
    $\overline P_{t_0}^{(0)}=\overline P_{t_0}$ is the exact-score reverse law and
    $\overline P_{t_0}^{(K_*)}$ is the learned-score reverse law before replacing
    $P_T$ by the Gaussian proxy.  The increment
    $\sfW_1(\overline P_{t_0}^{(k-1)},\overline P_{t_0}^{(k)})$ is therefore precisely
    the perturbation caused by using the learned score on the single slab
    $[t_{k-1},t_k]$.
    Suppose that, for a fixed exponent $a_{\rm clip}$, the score used by the reverse process satisfies
    \[
        \|\widehat{s}(\cdot,t)\|_\infty
        \le
        C_{\rm clip}\sigma_t^{-1}(D\vee\log n)^{a_{\rm clip}} .
    \]
    Then, the following holds for all $k = 1, 2, \dots, K_*$:
    \begin{equation*}
        \sfW_1(\overline{P}_{t_0}^{(k-1)}, \overline{P}_{t_0}^{(k)}) 
        \le
        C(D\vee\log n)^{a_{\rm W}}
        \Biggl\{
        \sqrt{t_k\log n}
        \sqrt{
          \int_{t=t_{k-1}}^{t_k}
            \E_{\bx_t}
            \|\widehat{s}(\bx_t, t) - \nabla\log p_t(\bx_t)\|_2^2
          \odt 
        }
        +
        n^{-\frac{\beta+1}{d+2\beta}}
        \Biggr\}.
    \end{equation*}

    Therefore, we have that
    \begin{align*}
        &
        \E_{\{\bx^{(i)}\}_{i=1}^n}
        \sfW_1(\overline{P}_{t_0}^{(k-1)}, \overline{P}_{t_0}^{(k)})
        \\
        &\le
        C(D\vee\log n)^{a_{\rm W}}
        \Biggl\{
        \sqrt{t_k\log n}
        \Bigl(
        \E_{\{\bx^{(i)}\}_{i=1}^n}
          \int_{t=t_{k-1}}^{t_k}
            \E_{\bx_t}\bigl[
              \|\widehat{s}(\bx_t, t) - \nabla\log p_t(\bx_t)\|_2^2
            \bigr]
          \odt 
        \Bigr)^{1/2}
        +
        n^{-\frac{\beta+1}{d+2\beta}}
        \Biggr\}.
    \end{align*}
    The constants $C,a_{\rm W}$ may depend on the fixed clipping exponent
    $a_{\rm clip}$, but not on $n,D,t_k$, or $p_{\min}$.
\end{lemma}

%%%%%%%%%%%%%%%%%%%%%%%%%%%%%%%%%%%%%%%%%%%%%%%%%%%%%%%%%%%%

\end{document}